\documentclass{article}

\usepackage[nonatbib,final]{neurips_2024}

\usepackage{wrapfig}

\usepackage[utf8]{inputenc} 
\usepackage[T1]{fontenc}    
\usepackage{hyperref}       
\usepackage{url}            
\usepackage{booktabs, multicol, multirow}       
\usepackage{amsfonts}       
\usepackage{amsmath}
\usepackage{amsthm}

\usepackage{caption}
\usepackage{subcaption,multirow}
\usepackage{boldline}
\usepackage{hhline}
\usepackage{enumitem,makecell,yfonts}

\usepackage{nicefrac}       
\usepackage{microtype}      
\usepackage[toc,page]{appendix}

\usepackage{upgreek,morenotations2,rotating}
\newcolumntype{?}{!{\vrule width 1pt}}

\usepackage{makecell,colortbl,skak}
\usepackage{yfonts}
\usepackage{tgadventor}

\title{\papertitle}

\author{%
Richard Nock \\
    Google Research \\
    \texttt{richardnock@google.com}
    \And
    Mathieu Guillame-Bert \\
    Google \\
    \texttt{gbm@google.com}
}

\begin{document}

\maketitle

\begin{abstract}
  We focus on generative AI for a type of data that still represent one of the most prevalent form of data: tabular data. Our paper introduces two key contributions: a new powerful class of forest-based models fit for such tasks and a simple training algorithm with strong convergence guarantees in a boosting model that parallels that of the original weak / strong supervised learning setting. This algorithm can be implemented by a few tweaks to the most popular induction scheme for decision tree induction (\textit{i.e. supervised learning}) with two classes. Experiments on the quality of generated data display substantial improvements compared to the state of the art. The losses our algorithm minimize and the structure of our models make them practical for related tasks that require fast estimation of a density given a generative model and an observation (even partially specified): such tasks include missing data imputation and density estimation. Additional experiments on these tasks reveal that our models can be notably good contenders to diverse state of the art methods, relying on models as diverse as (or mixing elements of) trees, neural nets, kernels or graphical models. 

\end{abstract}

\section{Introduction}\label{sec-intro}

  \setlength\tabcolsep{1pt}
  \newcommand{\onethird}{0.15}

  There is a substantial resurgence of interest in the ML community around tabular data, not just because it is still one of the most prominent kind of data available \cite{cmm+NF}: it is recurrently a place of sometimes heated debates on what are the best model architectures to solve related problems. For example, even for well-posed problems with decades of theoretical formalization like supervised learning \cite{sEO}, after more than a decade of deep learning disruption \cite{kshIC}, there is still much ink spilled in the debate decision forests vs neural nets \cite{mkvprgwWD}. Generative AI\footnote{We use this now common parlance expression on purpose, to avoid confusion with the other "generative" problem that consists in modelling densities \cite{cpdJI,rwID}.} makes no exception. Where the consensus has long been established on the best categories of architectures for data like image and text (neural nets), tabular data still flourishes with a variety of model architectures building up -- or mixing -- elements from knowledge representation, logics, kernel methods, graph theory, and of course neural nets \cite{cATOK,dALA,gpmxwocbGA,ngGT,pOE,tavCA,vpcPC,wbkwAR} (see Section \ref{sec-rel}). Because of the remarkable nature of tabular data where a single variable can bring considerable information about a target to model, each of these classes can be a relevant choice at least in \textit{some} cases (such is is the conclusion of \cite{mkvprgwWD} in the context of supervised learning). A key differentiator between model classes is then training and the formal guarantees it can provide \cite{ngGT,wbkwAR}.\\
\noindent \textbf{In this paper}, we introduce new generative models based on sets of trees that we denote as \textit{generative forests} (\geot), along with a training algorithm which has two remarkable features: it is extremely simple and brings strong convergence guarantees in a weak / strong learning model that parallels that of the original boosting model of Valiant's PAC learning \cite{kTO}. These guarantees improve upon the best state of the art guarantees \cite{wbkwAR,ngGT}. Our training algorithm, \topdownGT, supports training from data with missing values and is simple enough to be implementable by a few tweaks on the popular induction scheme for \textit{decision tree induction} with two classes, which is supported by a huge number of repositories / ML software implementing algorithms like CART or C4.5 \cite{bfosCA,scikit-learn,cgXA,qC4,wfDM}. From the model standpoint, generative forests bring a sizeable combinatorial advantage over its closest competitors, generative trees \cite{ngGT} (see Table \ref{circgauss-intro} for an example) and adversarial random forests \cite{wbkwAR}.\\
Experiments on a variety of simulated or readily available domains display that our models can substantially improve upon state of the art, with models of ours as simple as a set of stumps potentially competing with other approaches building much more complex models.\\
The models we build have an additional benefit: it is computationally easy to compute the full density given an observation, \textit{even partially specified}; hence, our generative models can also be used for side tasks like missing data imputation or density estimation. Additional experiments clearly display that our approach can be a good contender to the state of the art. To save space and preserve readability, all proofs and additional experiments and results are given in an Appendix. 
\begin{table}
  \centering
  \resizebox{0.5\columnwidth}{!}{\begin{tabular}{ccc} \Xhline{2pt}
    domain & 1 tree, 50 splits & \geot, 50 stumps\\
    \includegraphics[trim=0bp 0bp 0bp 0bp,clip,width=0.2\textwidth]{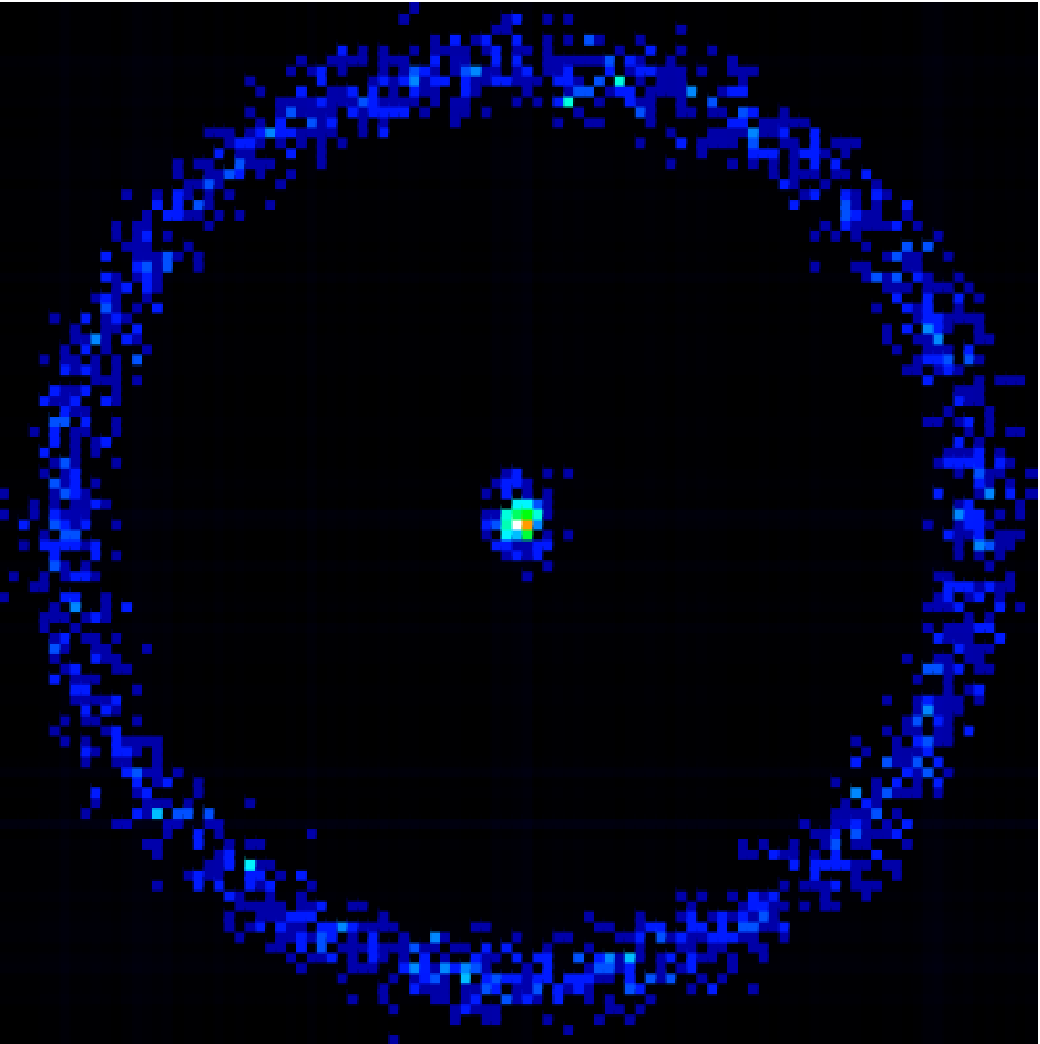}
    & \includegraphics[trim=0bp 0bp 0bp 0bp,clip,width=0.2\textwidth]{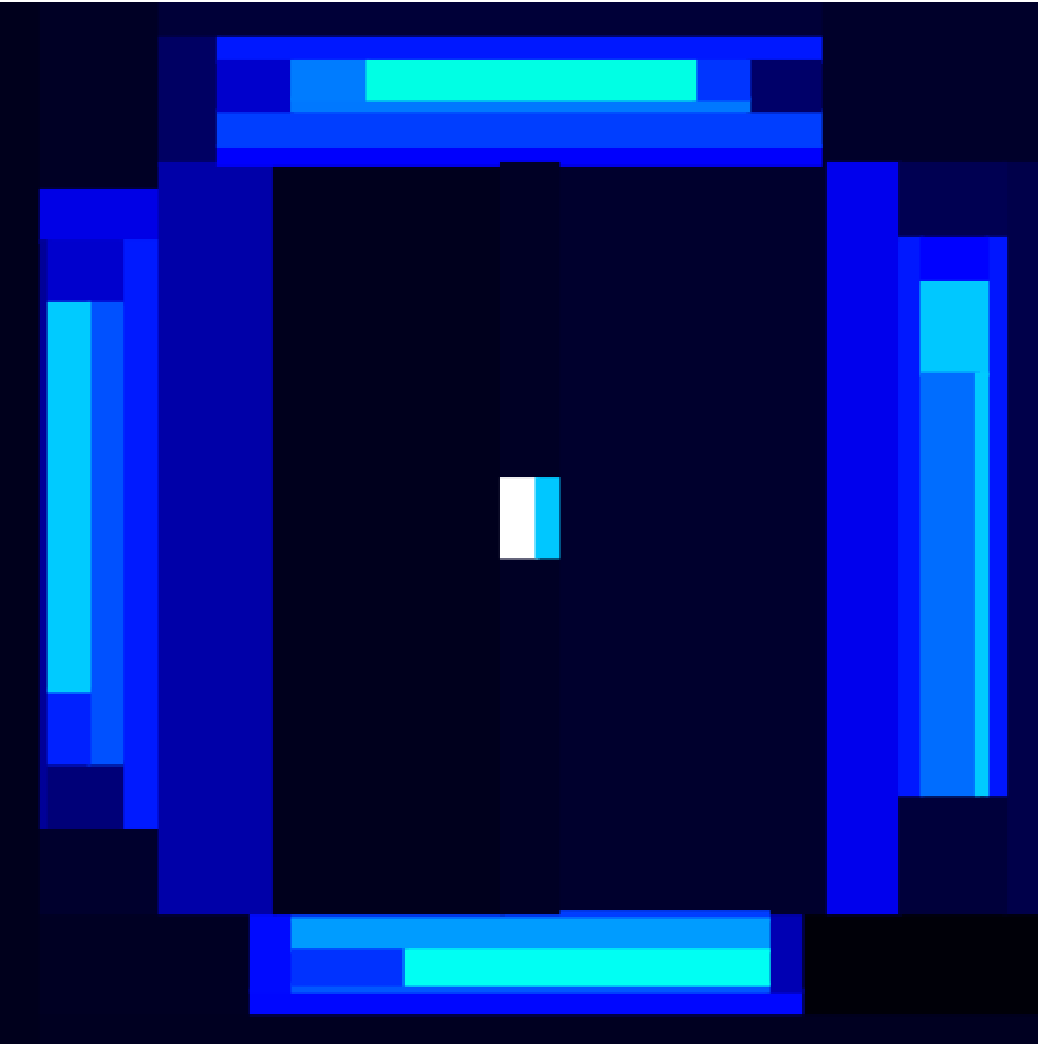}& \includegraphics[trim=0bp 0bp 0bp 0bp,clip,width=0.2\textwidth]{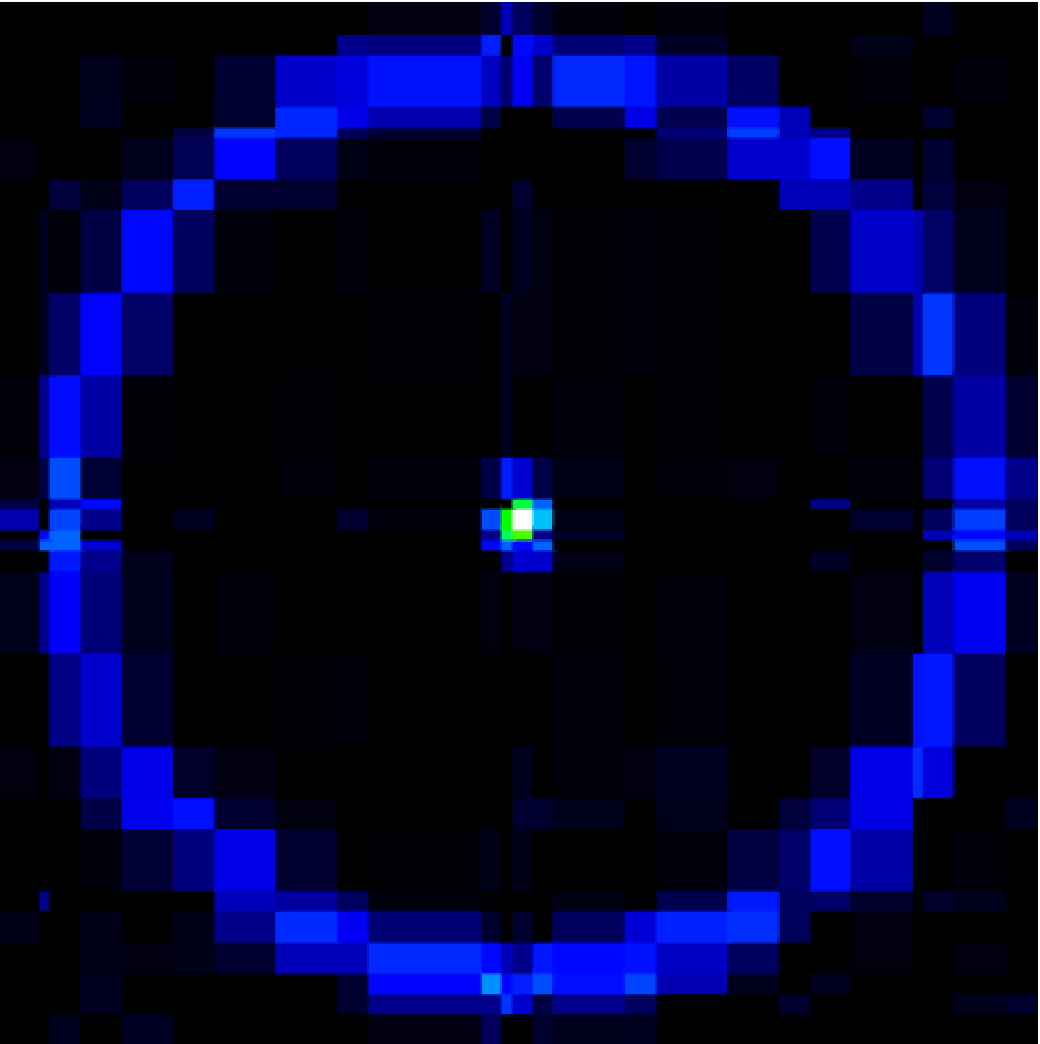}\\  \Xhline{2pt}
    \end{tabular}}
    \caption{\textit{Left}: domain \texttt{circgauss}. \textit{Center}: density learned by a generative forest (\geot) consisting of a single tree, boosted for a small number (50) of iterations. \textit{Right}: density learned by a \geot~consisting of 50 boosted tree stumps (\textit{Center} and \textit{Right} learned using \topdownGT). In a domain $\mathcal{X}$ with dimension $d$, a single tree with $n$ splits can only partition the domain in $n+1$ parts. On the other hand, a set of $n$ stumps in a \geot~can boost this number to $n^{\tilde{\Omega}(d)}$ (the tilda omits $\log$ dependencies), which explains why the right density appears so much better than the central one, even when each tree is just a stump.}
    \label{circgauss-intro}
\bignegspace
\bignegspace
\negspace
  \end{table}

\bignegspace
\section{Related work}\label{sec-rel}
\bignegspace

It would not do justice to the large amount of work in the field of "Generative AI" for tabular data to just sample a few of them, so we devote a part of the Appendix to an extensive review of the state of the art. Let us just mention that, unlike for unstructured data like images, there is a huge variety of model types, based on trees \cite{cpdJI,ngGT,wbkwAR}, neural networks \cite{gpmxwocbGA,geBG,kwAE,xscvMT}, probabilistic circuits \cite{cvvPC,vpcPC,spdSP}, kernel methods \cite{cATOK,pOE}, graphical models \cite{tavCA} (among others: note that some are in fact hybrid models). The closest approaches to ours are \cite{wbkwAR} and \cite{ngGT}, because the models include trees with a stochastic activation of edges to pick leaves, and a leaf-dependent data generation process. While \cite{ngGT} learn a single tree, \cite{wbkwAR} use a way to generate data from a set of trees -- called an adversarial random forest -- which is simple: sample a tree, and then sample an observation from the tree. The model is thus simple but not economical: each observation is generated by a single tree only, each of them thus having to represent a good sampler in its own. In our case, generating one observation makes use of \textit{all} trees (Figure \ref{fig:generation-sketch}, see also Section \ref{sec-mod}). We also note that the theory part of \cite{wbkwAR} is limited because it resorts to statistical consistency (infinite sample) and Lipschitz continuity of the target density, with second derivative continuous, square integrable and monotonic. Similarly, \cite{ngGT} resort to a wide set of proper losses, but with a symmetry constraint impeding in a generative context. As we shall see, our theoretical framework does not suffer from such impediments.

\setlength\tabcolsep{3pt}
\begin{figure}
  \centering
  \resizebox{0.6\columnwidth}{!}{\begin{tabular}{?c?c?}\Xhline{2pt}
    ARFs \cite{wbkwAR} & \geot~(this paper)\\
    \includegraphics[trim=80bp 38bp 30bp 80bp,clip,width=0.45\columnwidth]{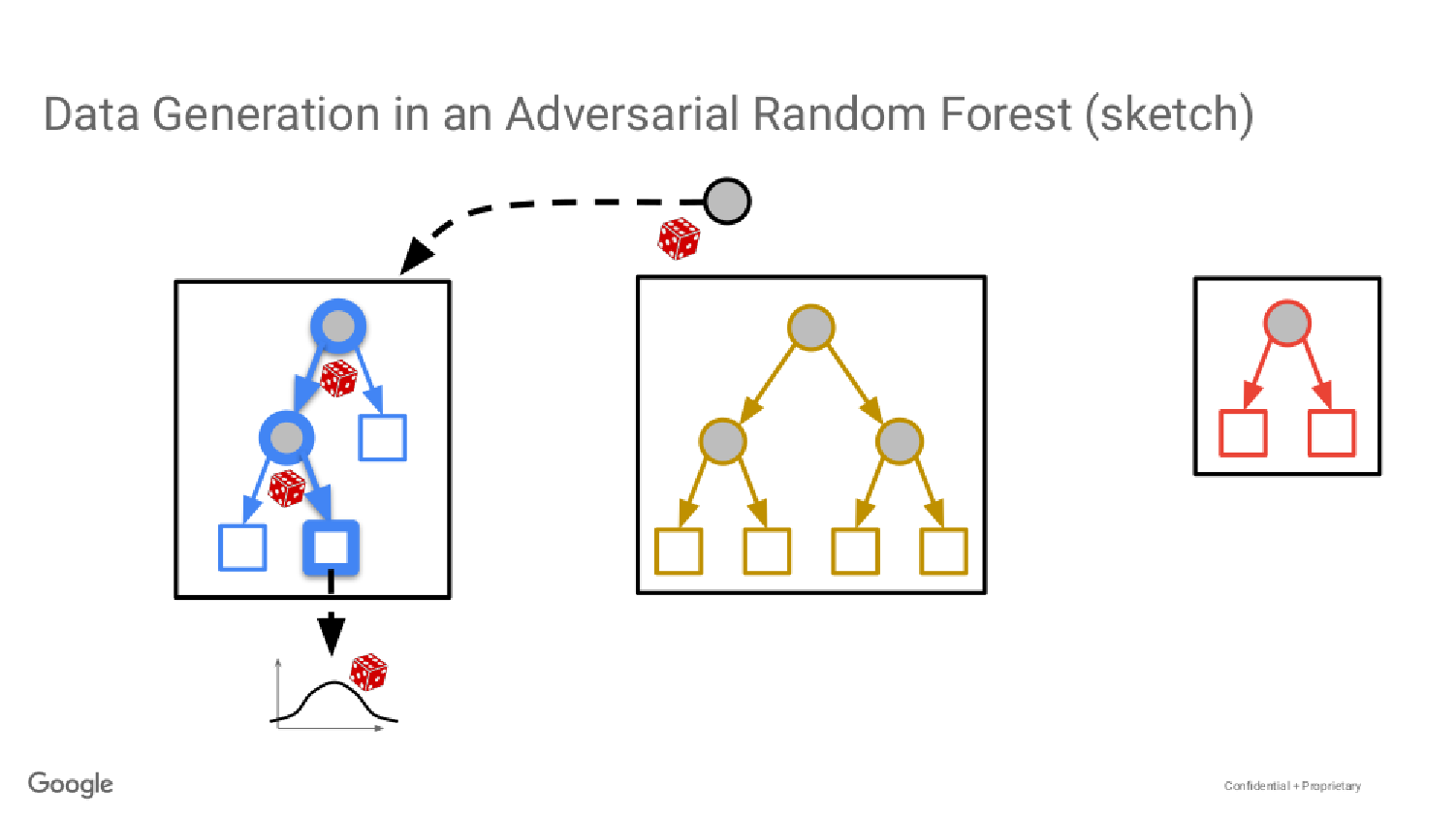} & \includegraphics[trim=80bp 38bp 30bp 80bp,clip,width=0.45\columnwidth]{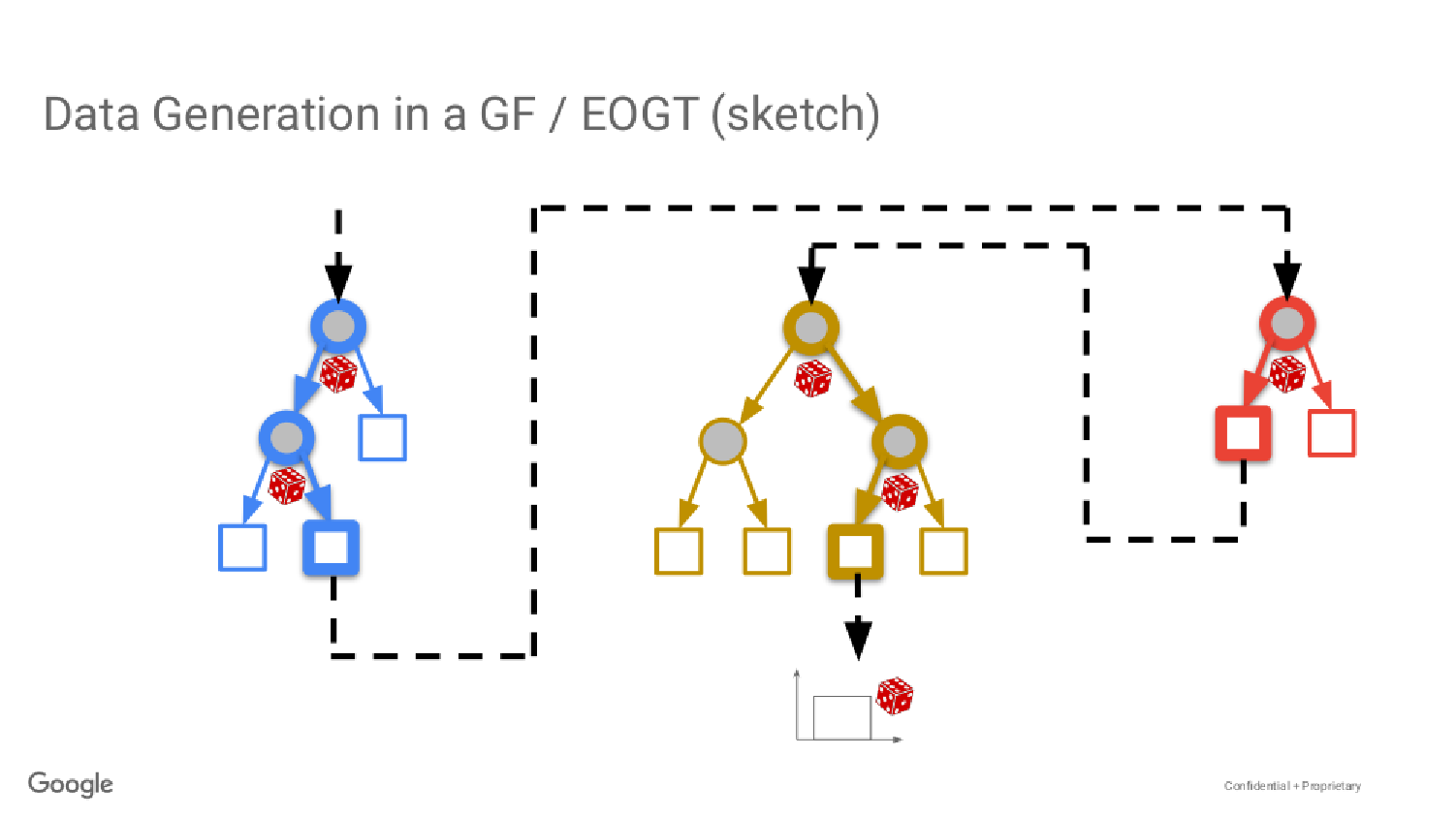} \\ \Xhline{2pt}
    \end{tabular}}
\negspace
    \caption{Sketch of comparison of two approaches to generate one observation, using Adversarial Random Forests \cite{wbkwAR} (left) and using generative forests, \geot~(right, this paper). In the case of Adversarial Random Forests, a tree is sampled uniformly at random, then a leaf is sampled in the tree and finally an observation is sampled according to the distribution "attached" to the leaf. Hence, only one tree is used to generate an observation. In our case, we leverage the combinatorial power of the trees in the forest: \textit{all trees} are used to generate one observation, as each is contributing to one leaf. Figure \ref{fig:generation-gt-eogt-smplesupport} provides more details on generation using \geot.}
    \label{fig:generation-sketch}
\bignegspace
\bignegspace
\negspace
  \end{figure}

\bignegspace
\section{Basic definitions}\label{sec-def}
\bignegspace

Perhaps surprisingly at first glance, this Section introduces generative and \textit{supervised} loss functions. Indeed, our algorithm, which trains a data generator and whose overall convergence shall be shown on a generative loss, turns out to locally optimize a \textit{supervised} loss. For the interested reader, the key link, which is of independent interest since it links in our context losses for supervised and generative learning, is established in Lemma \ref{lem-losses} (\supplement).\\
\noindent $\forall k\in \mathbb{N}_{>0}$, let $[k] \defeq \{1, 2, ..., k\}$. Our notations follow \cite{rwID}. Let $\binartask \defeq (\prior, \meas{A}, \meas{B})$ denote a \textit{binary task}, where $\meas{A}, \meas{B}$ (and any other measure defined hereafter) are probability measures with the same support, also called \textit{domain}, $\mathcal{X}$, and $\prior \in [0,1]$ is a \textit{prior}. $\meas{M} \defeq \prior \cdot \meas{A} + (1-\prior) \cdot \meas{B}$ is the corresponding mixture measure. For the sake of simplicity, we assume $\mathcal{X}$ bounded hereafter, and note that tricks can be used to remove this assumption \cite[Remark 3.3]{ngGT}. In tabular data, each of the $\mathrm{dim}(\mathcal{X})$ features can be of various \textit{types}, including categorical, numerical, etc., and associated to a natural measure (counting, Lebesgue, etc.) so we naturally associate $\mathcal{X}$ to the product measure, which can thus be of mixed type. We also write $\mathcal{X} \defeq \times_{i=1}^d \mathcal{X}_i$, where $\mathcal{X}_i$ is the set of values that can take on variable $i$. Several essential measures will be used in this paper, including $\coloru$, the uniform measure, $\colorg$, the measure associated to a generator that we learn, $\colorr$, the \textit{empirical} measure corresponding to a training sample of observations. Like in \cite{ngGT}, we do not investigate generalisation properties.\\
\noindent \textbf{Loss functions} There is a natural problem associated to binary task $\binartask$, that of estimating the probability that an arbitrary observation $\ve{x} \in \mathcal{X}$ was sampled from $\meas{A}$ -- call such \textit{positive} -- or $\meas{B}$ -- call these \textit{negative} --. To learn a \textit{supervised} model $\mathcal{X} \rightarrow [0,1]$ for such a \textit{class probability estimation} (\cpe) problem, one usually has access to a set of examples where each is a couple (observation, class), the class being in set $\mathcal{Y} \defeq \{-1,1\}$ (=$\{$negative, positive$\}$). Examples are drawn i.i.d. according to $\binartask$. Learning a model is done by minimizing a loss function: when it comes to \cpe, any such \cpe~loss \cite{bssLF} is some $\ell : \mathcal{Y} \times [0,1]
\rightarrow \mathbb{R}$ whose expression can be split according to
\textit{partial} losses $\partialloss{1}, \partialloss{-1}$, $\ell(y,u) \defeq \iver{y=1}\cdot \partialloss{1}(u) +
                     \iver{y =-1}\cdot \partialloss{-1}(u)$. Its (pointwise) \emph{Bayes risk} function
is
the best achievable loss
when labels are drawn with a particular positive base-rate,
\begin{eqnarray}
\poibayesrisk(p) \defeq
\inf_u \E_{\Y\sim \bernoulli(p)} \properloss(\Y, u),\label{defPOIBAYESRISK}
\end{eqnarray}
where $\bernoulli\defeq$ Bernoulli random variable for class 1. A fundamental property for a \cpe~loss is \textit{\textbf{properness}}, encouraging to guess ground truth: $\ell$ is proper iff $\poibayesrisk(p) = \E_{\Y\sim \bernoulli(p)} \properloss(\Y, p), \forall p
\in [0,1]$,
and \textit{strictly} proper if $\poibayesrisk(p) < \E_{\Y\sim \bernoulli(p)} \properloss(\Y, u), \forall u \neq p$. Strict properness implies strict concavity of Bayes risk. For example, the square loss has $\partialsqloss{1}(u) \defeq (1-u)^2$, $\partialsqloss{-1}(u) \defeq u^2$ , and, being strictly proper, Bayes risk $\bayessqrisk(u) \defeq u(1-u)$. Popular strictly proper ML include the log and Matusita's losses. All these losses are \textit{\textbf{symmetric}} since $\partialsqloss{1}(u) = \partialsqloss{-1}(1-u), \forall u \in (0,1)$ and \textit{\textbf{differentiable}} because both partial losses are differentiable. \\
In addition to \cpe~losses, we introduce a set of losses relevant to generative approaches, that are popular in density ratio estimation \cite{moLL,skML}. For any differentiable and convex $F: \mathbb{R} \rightarrow \mathbb{R}$, the Bregman divergence with generator $F$ is $D_F(z\|z') \defeq F(z) - F(z') - (z-z')F'(z')$. Given function $g : \mathbb{R} \rightarrow \mathbb{R}$, the generalized perspective transform of $F$ given $g$ is $\check{F}(z) \defeq g(z) \cdot F  \left(z/g(z)\right)$, $g$ being implicit in notation $\check{F}$ \cite{mOAI,mOAII,nmoAS}.  The \textit{Likelihood ratio risk} of $\meas{B}$ with respect to $\meas{A}$ for loss $\ell$ is
  \begin{eqnarray}
    \likelihoodratiorisk_{\ell}\left(\meas{A}, \meas{B}\right) & \defeq & \prior\cdot \expect_{\coloru} \left[ D_{\perspectivepobayesrisk}\left(\frac{\dmeas{\meas{A}}}{\dmeas{\coloru}}\left\|\frac{\dmeas{\meas{B}}}{\dmeas{\coloru}}\right. \right) \right],\label{eqLRR}
  \end{eqnarray}
  with $g(z) \defeq z + (1-\prior)/\prior$ in the generalized perspective transform. The prior multiplication is for technical convenience. $\likelihoodratiorisk_{\ell}$ is non-negative; strict properness is necessary for a key property of $\likelihoodratiorisk_{\ell}$: $\likelihoodratiorisk_{\ell} = 0$ iff $\meas{A} = \meas{B}$ almost everywhere \cite{ngGT}.

\bignegspace
\section{Generative forests: models and data generation}\label{sec-mod}
\begin{figure}
  \centering
\includegraphics[trim=0bp 30bp 0bp 130bp,clip,width=0.6\columnwidth]{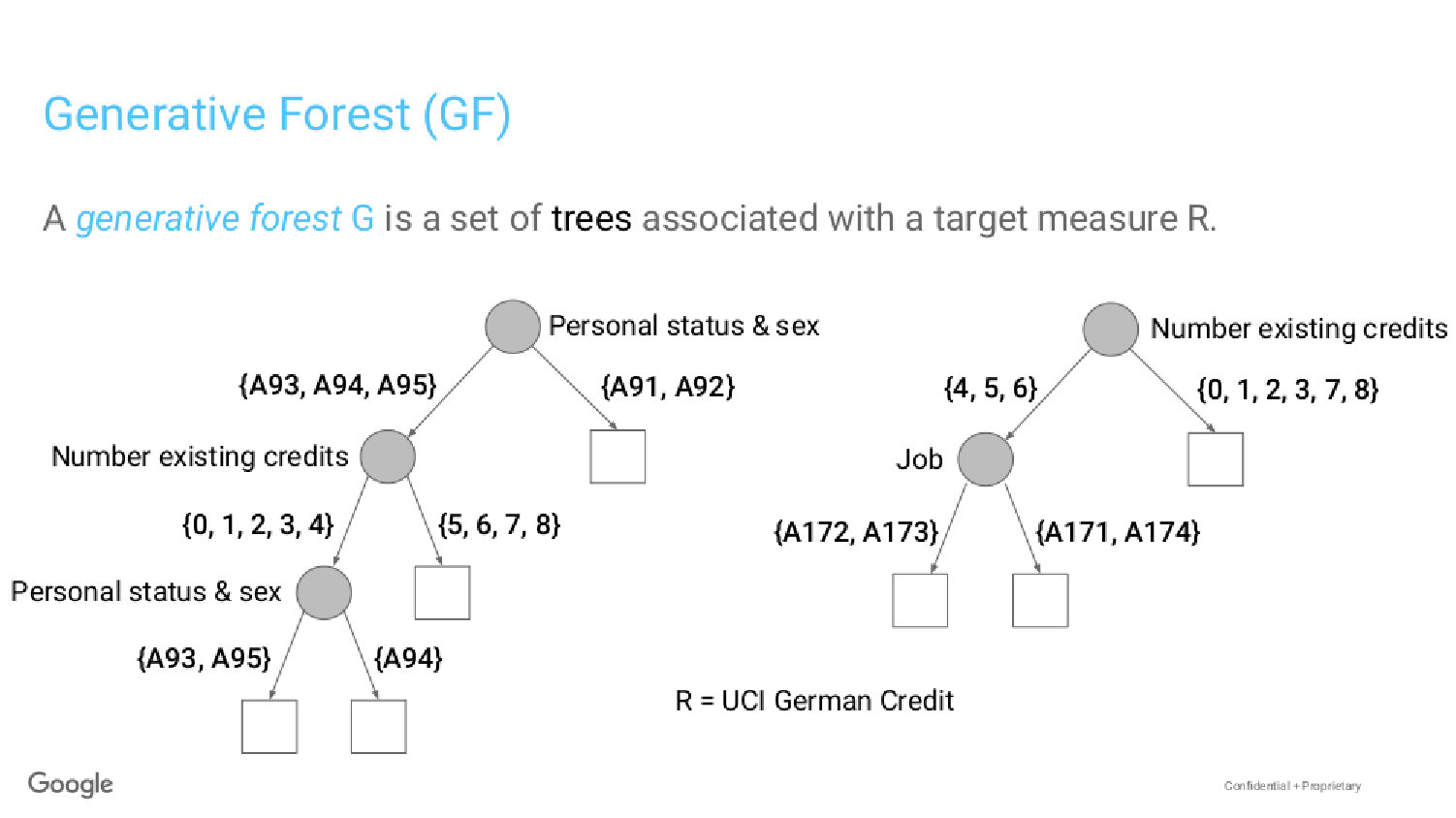} 
\negspace
    \caption{A \geot~($T=2$) associated to UCI German Credit. Constraint \textbf{(C)} (see text) implies that the domain of "Number existing credits" is $\{0, 1, ..., 8\}$, that of "Job" is $\{$A171, A172, A173, A174$\}$, etc. .}
    \label{fig:gf-1}
\bignegspace
  \end{figure}
\begin{table*}
\centering
\makebox[0pt][c]{\parbox{\textwidth}{%
    \resizebox{.9\columnwidth}{!}{    \begin{minipage}[t]{0.45\hsize}\centering
        \begin{algorithm}[H]
\caption{\initsampling($\{\tree_t\}_{t=1}^{T}$)}\label{alg-initsampling}
\begin{algorithmic}
  \STATE  \textbf{Input:} Trees $\{\tree_t\}_{t=1}^{T}$ of a \geot;
  \STATE  Step 1 : \textbf{for} $t\in [T]$
  \STATE  \hspace{1cm} $\tree_t.\node^\star \leftarrow \mathrm{root}_t$;
  \STATE  \hspace{1cm} $\tree_t.\texttt{done}\leftarrow \iver{\tree_t.\node^\star \in \leafset(\tree_t)};$
\end{algorithmic}
\end{algorithm}
    \end{minipage}
    \hfill
    \begin{minipage}[t]{0.55\hsize}\centering
        \begin{algorithm}[H]
\caption{\starupdate($\Upsilon, \mathcal{C}, \colorr$)}\label{alg-starupdate}
\begin{algorithmic}
    \STATE  \textbf{Input:} Tree $\Upsilon$, subset $\mathcal{C} \subseteq \mathcal{X}$, measure $\colorr$; 
  \STATE  Step 1 : \textbf{if} $\bigbernoulli\left(p_{\colorr}\left[\xxset{\Upsilon.\rightchild{\node}^\star} \cap \mathcal{C} | \mathcal{C} \right]\right) = 1$ \textbf{then}
 \STATE  \hspace{1.5cm} $\mathcal{C} \leftarrow  \xxset{\Upsilon.\rightchild{\node}^\star} \cap \mathcal{C} $; $\Upsilon.\node^\star \leftarrow \Upsilon.\rightchild{\node}^\star$;
 \STATE  \hspace{1.1cm} \textbf{else}
 \STATE  \hspace{1.5cm} $\mathcal{C} \leftarrow \xxset{\Upsilon.\leftchild{\node}^\star} \cap \mathcal{C} $; $\Upsilon.\node^\star \leftarrow \Upsilon.\leftchild{\node}^\star$;
 \STATE  Step 2 : \textbf{if} $\Upsilon.\node^\star\in \Upsilon.\leafset$ \textbf{then} $\Upsilon.\texttt{done}\leftarrow \texttt{true};$
\end{algorithmic}
\end{algorithm}
    \end{minipage}}%
}}
\bignegspace
\end{table*}
\noindent\textbf{Architecture} We first introduce the basic building block of our models, \textit{trees}.
\negspace
\begin{definition}\label{def-tree}
    A \textbf{tree} $\tree$ is a binary directed tree whose internal nodes are labeled with an observation variable and arcs are \emph{consistently} labeled with subsets of their tail node's variable domain.
  \end{definition}
\bignegspace
\textit{Consistency} is an important generative notion; informally, it postulates that the arcs' labels define a partition of the measure's support. To make this notion formal, we proceed need a key definition. For any node $\node\in\nodeset(\tree)$ (the whole set of nodes of $\tree$, including leaves), we denote $\xxset{\node} \subseteq \mathcal{X}$ the \textit{support} of the node. The root has $\xxset{\node} = \mathcal{X}$. To get $\xxset{\node}$ for any other node, we initialize it to $\mathcal{X}$ and then descend the tree from the root, progressively updating $\xxset{\node}$ by intersecting an arc's observation variable's domain in $\xxset{\node}$ with the sub-domain labelling the arc until we reach $\node$. Then, a labeling of arcs in a tree is \textit{consistent} iff it complies with one constraint \textbf{(C)}:
\negspace
\negspace
  \begin{itemize}[noitemsep,topsep=0pt]
\item [\textbf{(C)}] for each internal node $\node$ and its left and right children $\leftchild{\node}, \rightchild{\node}$ (respectively), $\xxset{\node} = \xxset{\leftchild{\node}} \cup \xxset{\rightchild{\node}}$ and the measure of $\xxset{\leftchild{\node}} \cap \xxset{\rightchild{\node}}$ with respect to $\colorg$ is zero.
  \end{itemize}
For example, the first split at the root of a tree is such that the union of the domains at the two arcs equals the domain of the feature labeling the split (see Figure \ref{fig:gf-1}). We define our generative models.
 \begin{definition}\label{def-df-gf}
  A \textbf{generative forest} (\geot), $\colorg$, is a set of trees, $\{\tree_t\}_{t=1}^T$, associated to measure $\colorr$ (implicit in notation).
\end{definition}
\bignegspace
Figure \ref{fig:gf-1} shows an example of \geot. Following the consistency requirement, any single tree defines a recursive partition of $\mathcal{X}$ according to the splits induced by the inner nodes. Such is \textit{also} the case for a set of trees, where \textit{intersections} of the supports of tuples of leaves (1 for each tree) define the subsets:
\begin{eqnarray}
  \mathcal{P} (\colorg) \defeq \left\{\cap_{i=1}^{T}{\mathcal{X}}_{\leaf_i} \mbox{ s.t. } \leaf_i \in \leafset(\Upsilon_i), \forall i\right\} \label{defPART}
\end{eqnarray}
($\leafset(\Upsilon) \subseteq \nodeset(\tree)$ is the set of leaves of $\tree$). Notably, we can construct the elements of $\mathcal{P} (\colorg)$ using the same algorithm that would compute it for 1 tree. First, considering a first tree $\tree_1$, we compute the support of a leaf, say $\mathcal{X}_{\leaf_1}$, using the algorithm described for the consistency property above. Then, we start again with a second tree $\tree_2$ \textit{but} replacing the initial $\mathcal{X}$ by $\mathcal{X}_{\leaf_1}$, yielding $\mathcal{X}_{\leaf_1}\cap \mathcal{X}_{\leaf_2}$. Then we repeat with a third tree $\tree_3$ replacing $\mathcal{X}$ by $\mathcal{X}_{\leaf_1}\cap \mathcal{X}_{\leaf_2}$, and so on until the last tree is processed. This yields one element of $\mathcal{P} (\colorg)$.\\
\noindent \textbf{Generating one observation} Generating one observation relies on a \textit{stochastic} version of the procedure just described. It ends up in an element of $\mathcal{P} (\colorg)$ of positive measure, from which we sample uniformly one observation, and then repeat the process for another observation. To describe the process at length, we make use of two key routines, \initsampling~and \starupdate, see Algorithms \ref{alg-initsampling} and \ref{alg-starupdate}. 
\initsampling~initializes "special" nodes in each tree, that are called \textit{star nodes}, to the root of each tree (notation for a variable $v$ relative to a tree $\tree$ is $\tree.v$). Stochastic activation, performed in \starupdate, progressively makes star nodes descend in trees. When all star nodes have reached a leaf in their respective tree, an observation is sampled from the intersection of the leaves' domains (which is an element of $\mathcal{P} (\colorg)$). A Boolean flag, \texttt{done} takes value \texttt{true} when the star node is in the leaf set of the tree, indicating no more \starupdate~calls for the tree ($\iver{.}$ is Iverson's bracket, \cite{kTN}).\\
\noindent \starupdate~is called with a tree of the \geot~for which $\texttt{done}$ is false, a subset $\mathcal{C}$ of the whole domain and measure $\colorr$. The first call of this procedure is done with $\mathcal{C} = \mathcal{X}$. When all trees are marked \texttt{done}, $\mathcal{C}$ has been "reduced" to some $\samplesupport\in \mathcal{P} (\colorg)$, where the index reminds that this is the last $\mathcal{C}$ we obtain, from which we \textbf{s}ample an observation, uniformly at random in $\samplesupport$. Step 1 in \starupdate~is fundamental: it relies on tossing an unfair coin (a Bernoulli event noted $\bigbernoulli\left(p\right)$), where the head probability $p$ is just the mass of $\colorr$ in $\xxset{\Upsilon.\rightchild{\node}^\star} \cap \mathcal{C}$ \textit{relative to} $\mathcal{C}$. Hence, if $\xxset{\Upsilon.\rightchild{\node}^\star} \cap \mathcal{C} = \mathcal{C}$, $p=1$. There is a simple but important invariant (proof omitted).
\negspace
\begin{lemma}\label{lemsub}
In \starupdate, it always holds that the input $\mathcal{C}$ satisfies $\mathcal{C} \subseteq \xxset{\Upsilon.\node^\star}$.
  \end{lemma}
\negspace
We have made no comment about the \textit{sequence} of tree choices over which \starupdate~is called. Let us call \textit{admissible} such a sequence that ends up with \textit{some} $\samplesupport \subseteq \mathcal{X}$. $T$ being the number of trees (see \initsampling), for any sequence $\ve{v}\in [T]^{D(\samplesupport)}$, where $D(\samplesupport)$ is the sum of the depths of all the star \textit{leaves} whose support intersection is $\samplesupport$, we say that $\ve{v}$ is \textit{admissible for} $\samplesupport$ if there exits a sequence of branchings in Step 1 of \starupdate, whose corresponding sequence of trees follows the indexes in $\ve{v}$, such that at the end of the sequence all trees are marked \texttt{done} and the last $\mathcal{C} = \samplesupport$. Crucially, the probability to end up in $\samplesupport$ using \starupdate, given any of its admissible sequences, is the \textit{same} and equals its mass with respect to $\colorr$.
\negspace
\begin{figure*}
  \centering
  \resizebox{0.8\textwidth}{!}{{\tiny\begin{tabular}{?c?c?c?}\Xhline{2pt}
    After \initsampling & \starupdate~on $\node^\star = {\color{blue} \symqueen}$ & \starupdate~on $\node^\star = {\color{red} \symrook}$ \\
    \includegraphics[trim=0bp 50bp 40bp 80bp,clip,width=0.3\columnwidth]{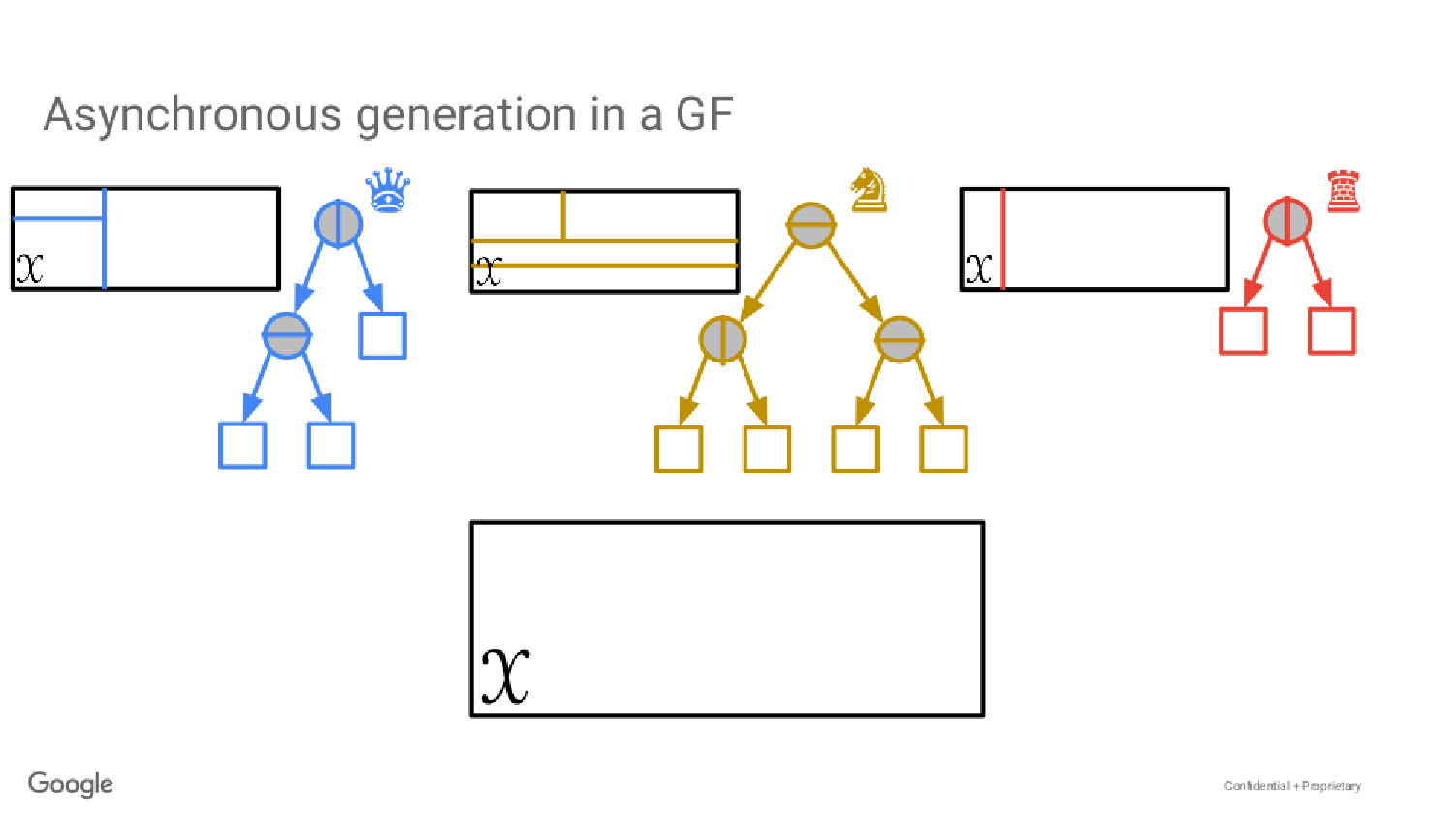} & \includegraphics[trim=0bp 50bp 40bp 80bp,clip,width=0.3\columnwidth]{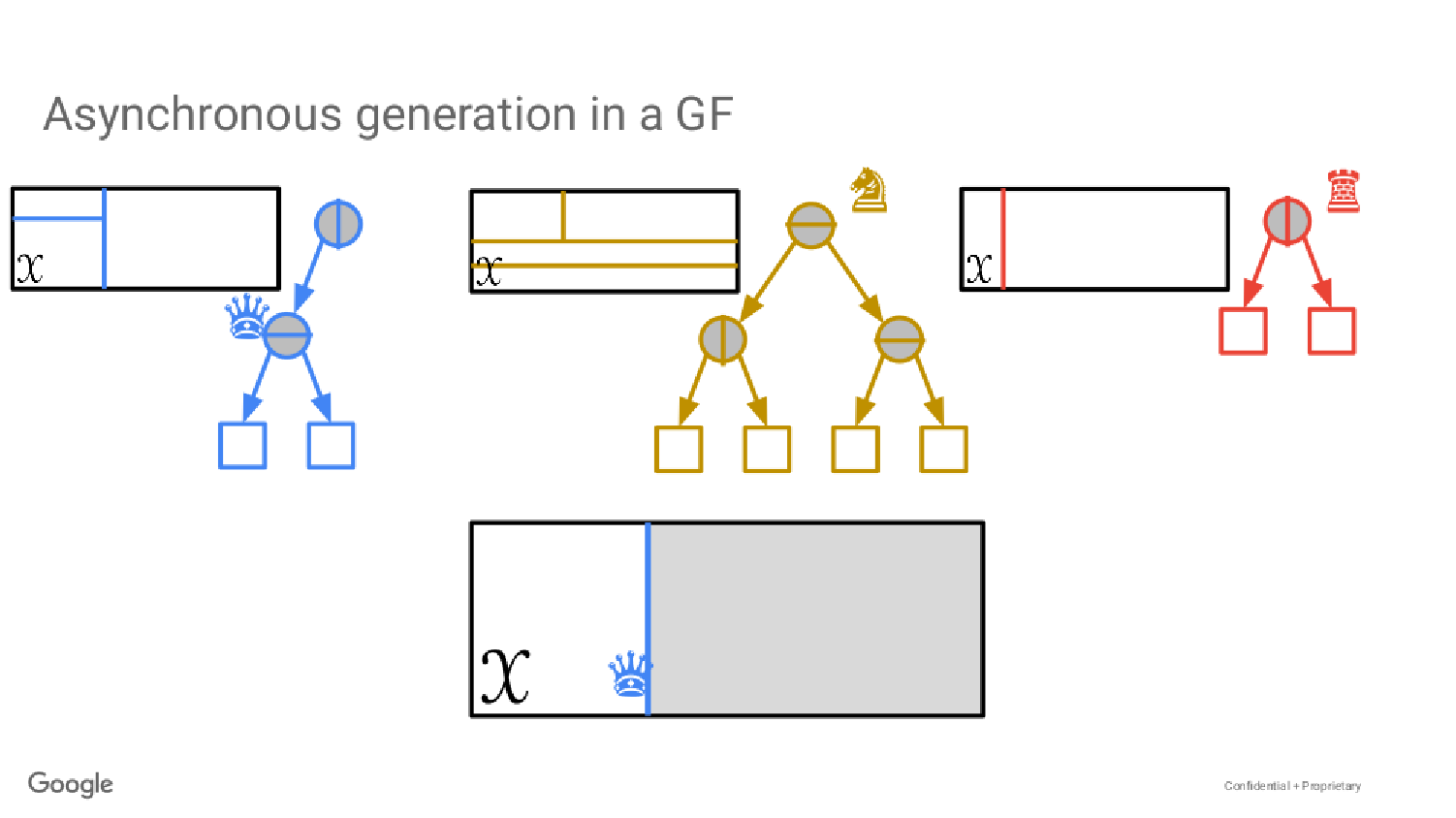} & \includegraphics[trim=0bp 50bp 40bp 80bp,clip,width=0.3\columnwidth]{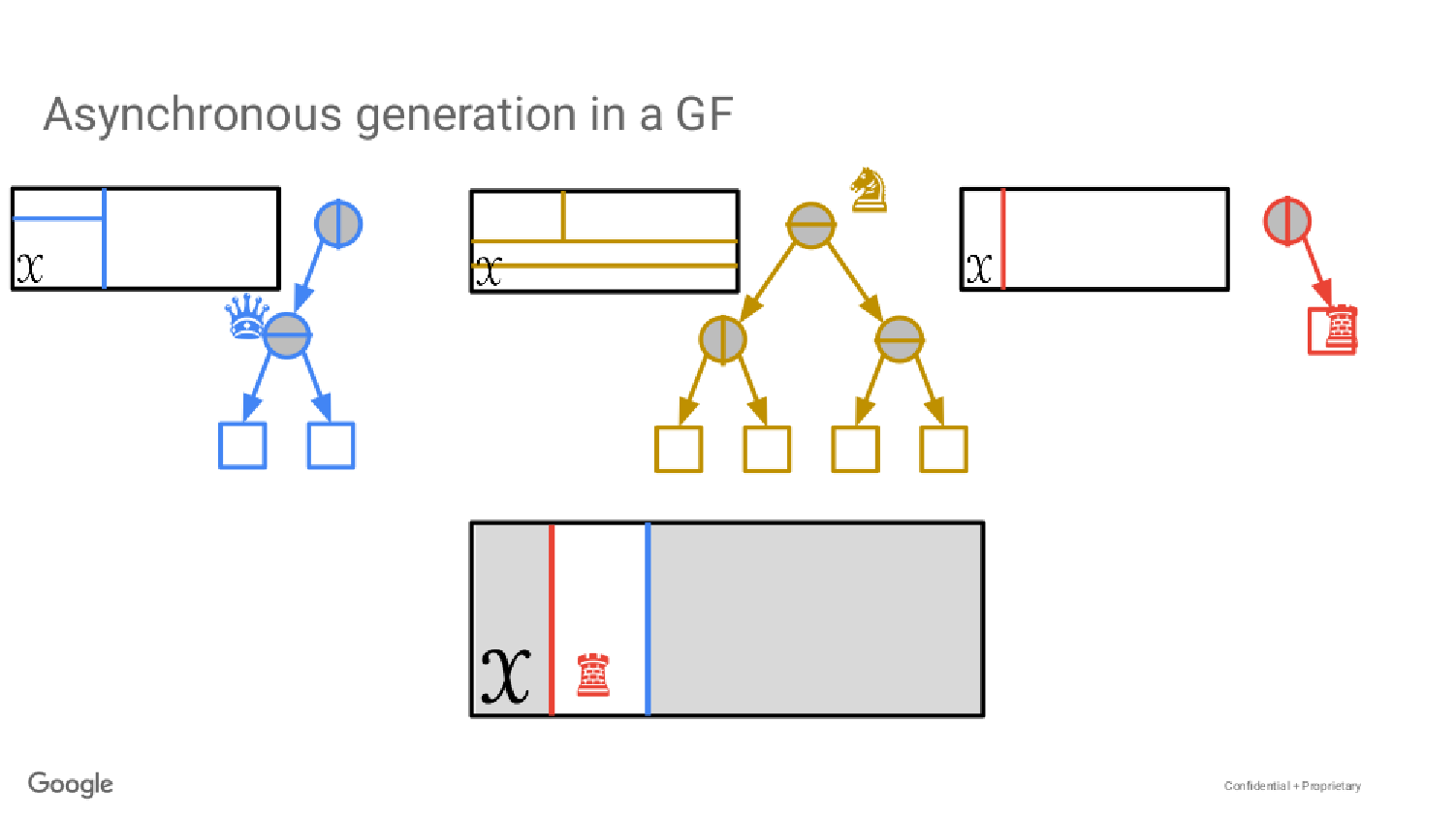} \\\Xhline{2pt}
    \starupdate~on $\node^\star = {\color{orange} \symknight}$ & \starupdate~on $\node^\star = {\color{orange} \symknight}$  & \starupdate~on $\node^\star = {\color{blue} \symqueen}$ \\
    \includegraphics[trim=0bp 50bp 40bp 80bp,clip,width=0.3\columnwidth]{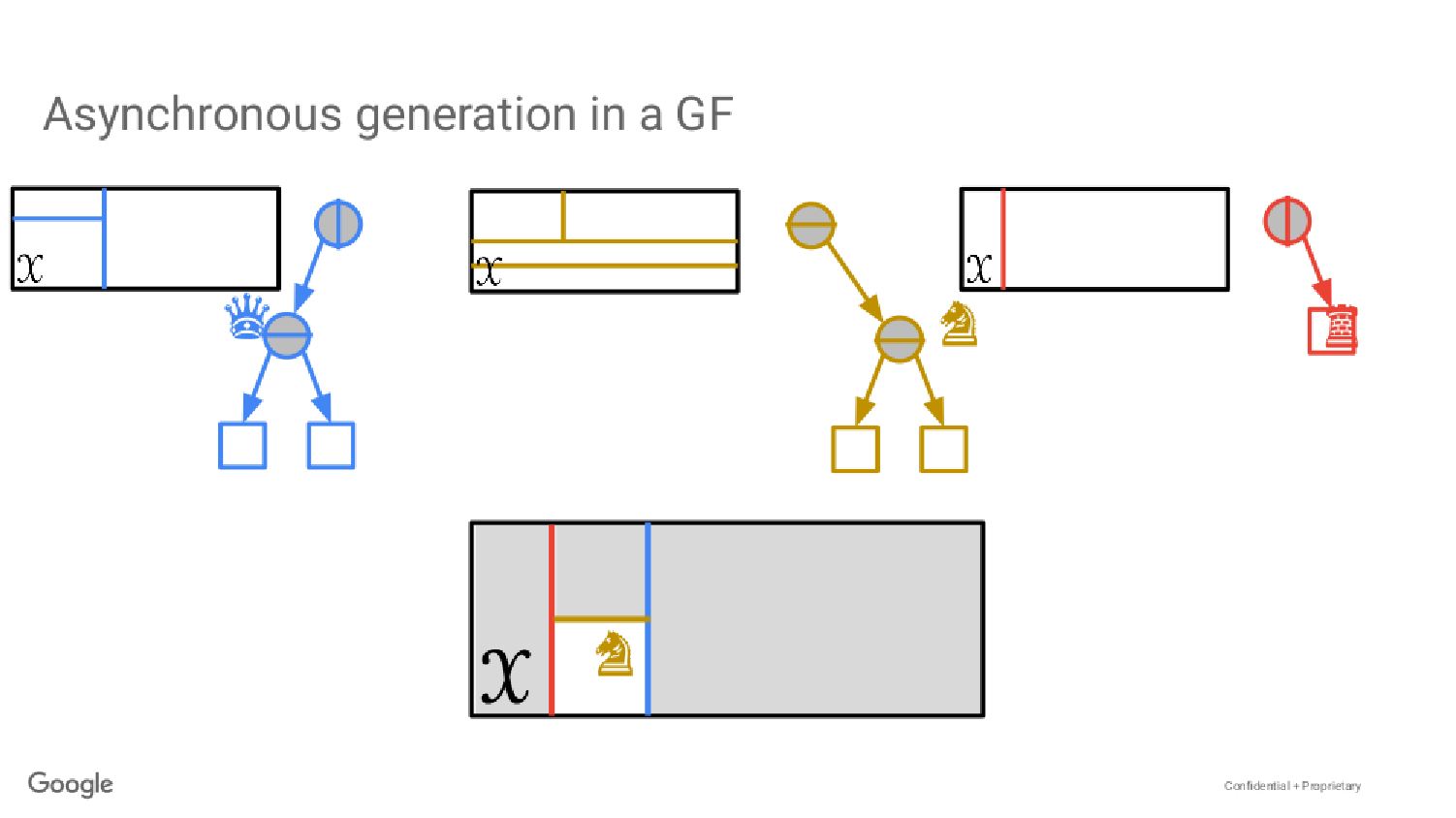} & \includegraphics[trim=0bp 50bp 40bp 80bp,clip,width=0.3\columnwidth]{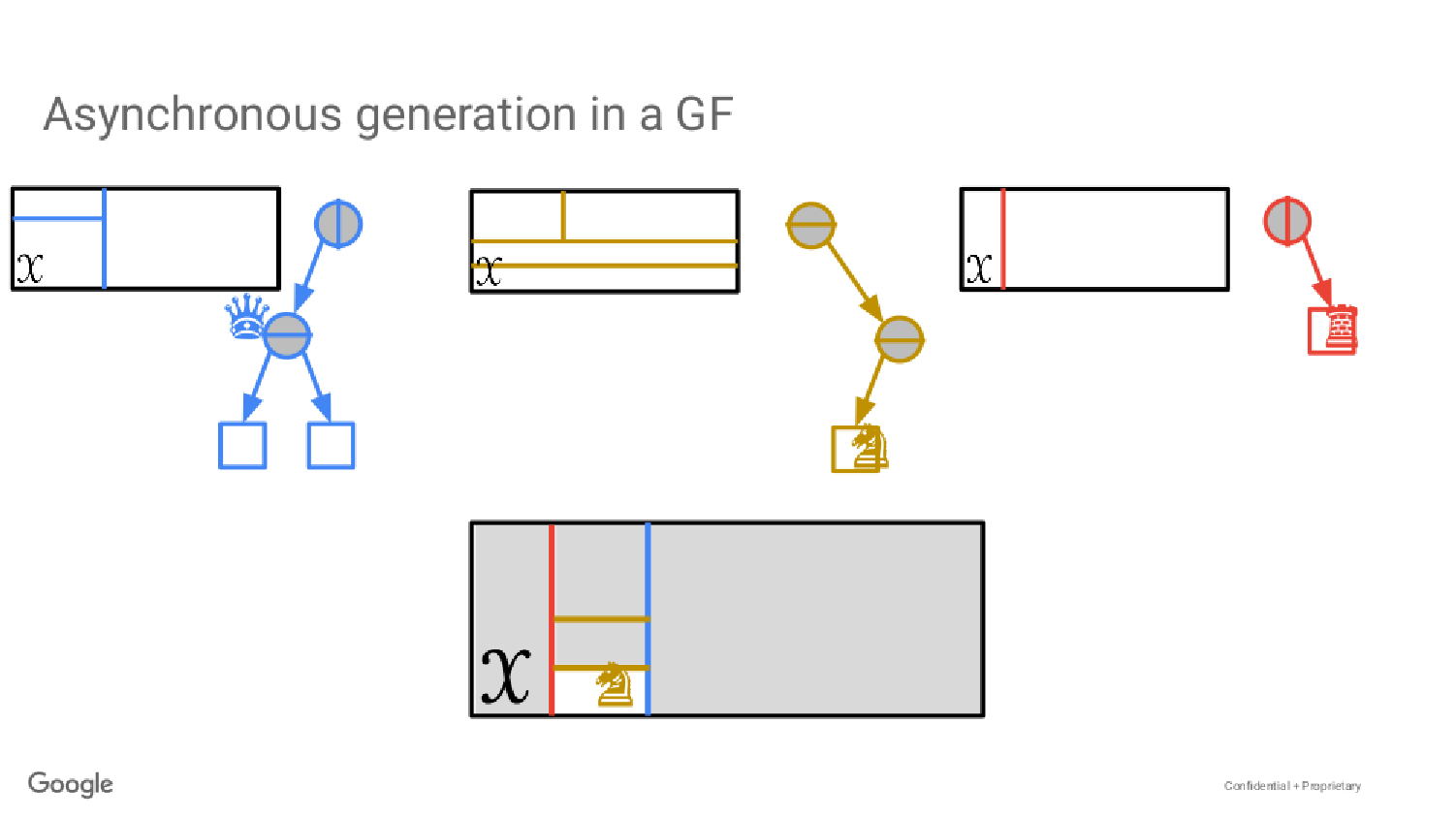} & \includegraphics[trim=0bp 50bp 40bp 80bp,clip,width=0.3\columnwidth]{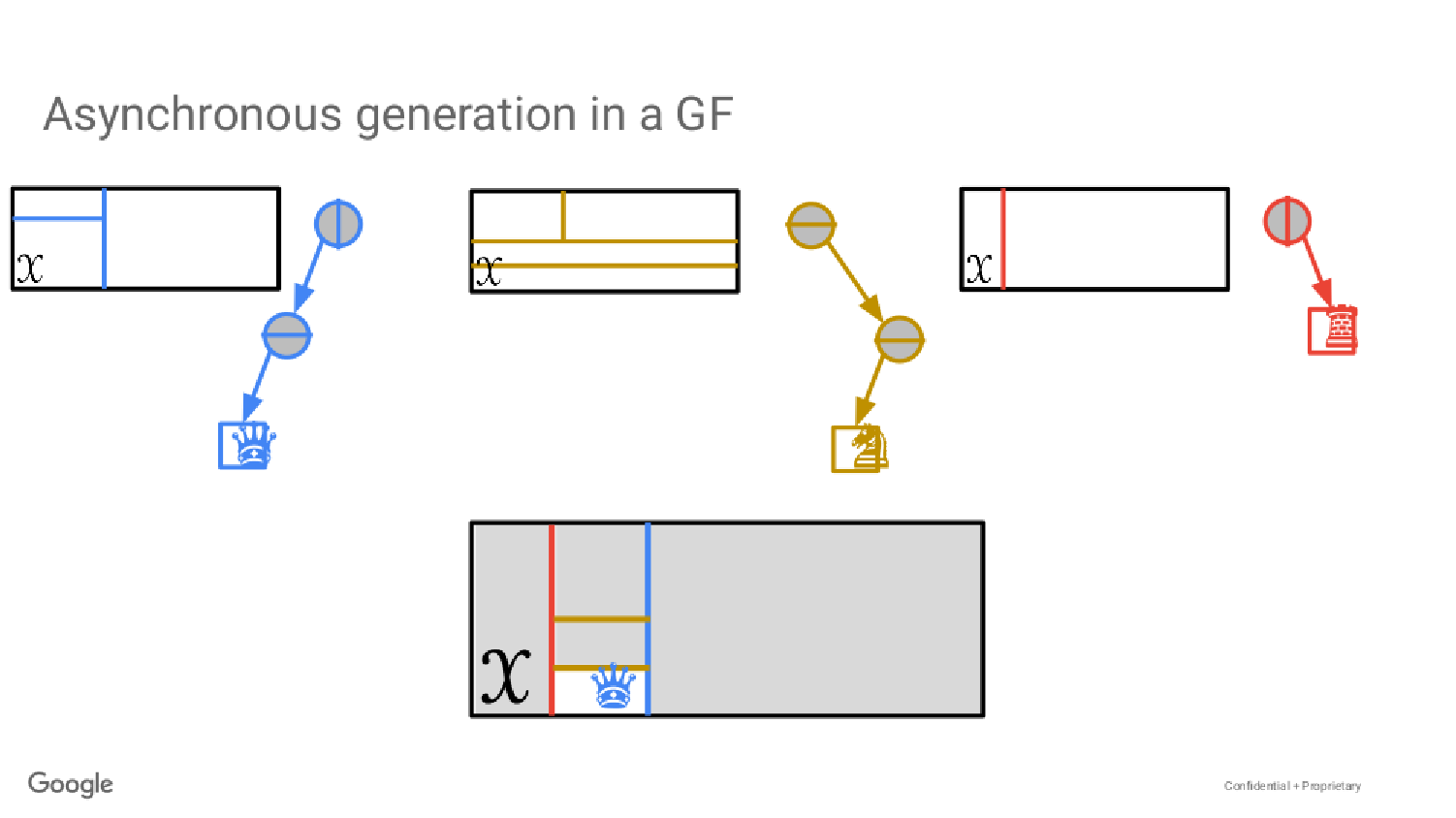} \\ \Xhline{2pt}
    \end{tabular}}}
\negspace
    \caption{From left to right and top to bottom: updates of the argument $\mathcal{C}$ of \starupdate~through a sequence of run of \starupdate~in a generative forest consisting of three trees (the partition of the domain induced by each tree is also depicted, alongside the nature of splits, vertical or horizontal, at each internal node) whose star nodes are indicated with chess pieces  ({\color{blue} \symqueen}, {\color{red} \symrook}, {\color{orange} \symknight}). In each picture, $\mathcal{C}$ is represented at the bottom of the picture (hence, $\mathcal{C} = \mathcal{X}$ after \initsampling). In the bottom-right picture, all star nodes are leaves and thus $\samplesupport=\mathcal{C}$ displays the portion of the domain in which an observation is sampled. Remark that the last star node update produced no change in $\mathcal{C}$.}
    \label{fig:generation-gt-eogt-smplesupport}
\bignegspace
\bignegspace
\end{figure*}
\begin{lemma}\label{leminv}
  For any $\samplesupport \in \mathcal{P} (\colorg)$ and admissible sequence $\ve{v}\in [T]^{D(\samplesupport)}$ for $\samplesupport$, $p_{{\color{blue}{\meas{G}}}}\left[\samplesupport | \ve{v}\right] = p_{\colorr}\left[\samplesupport\right]$.
\end{lemma}
\bignegspace
\negspace
\negspace
The Lemma is simple to prove but fundamental in our context as the way one computes the sequence -- and thus the way one picks the trees -- does not bias generation: the sequence of tree choices could thus be iterative, randomized, concurrent (\textit{e.g.} if trees were distributed), etc., this would not change generation's properties from the standpoint of Lemma \ref{leminv}. We defer to \supplement~(Section \ref{add-cont}) three examples of sequence choice. Figure \ref{fig:generation-gt-eogt-smplesupport} illustrates a sequence and the resulting $\samplesupport$.\\
\noindent  \textbf{Missing data imputation and density estimation with \geot s} A generative forest is not just a generator: it models a density and the exact value of this density at any observation is easily available. Figure \ref{fig:density-estimation-mda} (top row) shows how to compute it. Note that if carried out in parallel, the complexity to get this density is cheap, of order $O(\max_t \mathrm{depth}(\tree_t))$. If one wants to prevent predicting zero density, one can stop the descent if the next step creates a support with zero empirical measure. A \geot~thus also models an easily tractable density, but this fact alone is not enough to make it a \textit{good} density estimator. To get there, one has to factor the loss minimized during training. In our case, as we shall see in the following Section, we train a \geot~by minimizing an information-theoretic loss directly formulated on this density \eqref{eqLRR}. So, using a \geot~also for density estimation can be a reasonable additional benefit of training \geot s.\\
  The process described above that finds leaves reached by an observation can be extended to missing values in the observation using standard procedures for classification using decision trees. We obtain a simple procedure to carry out \textit{missing data imputation} instead of density estimation: once a tuple of leaves is reached, one uniformly samples the missing features in the corresponding support. This is fast, but at the expense of a bit more computations, we can have a much better procedure, as explained in Figure \ref{fig:density-estimation-mda} (bottom row). In a first step, we compute the density of each support subset corresponding to \textit{all} (not just 1 as for density estimation) tuples of leaves reached in each tree. This provides us with the \textit{full density} over the missing values \textit{given} (i) a partial assignment of the tabular domain's variables and (ii) the \geot. We then keep the elements corresponding to the maximal density value and finally simultaneously sample all missing features uniformly in the corresponding domain. Overall, the whole procedure is $O(d\cdot (\sum_t \mathrm{depth}(\tree_t))^2)$.\\
\begin{figure}[t]
  \centering
  \resizebox{0.8\columnwidth}{!}{\begin{tabular}{?c?c?c?}\Xhline{2pt}
    \rotatebox{90}{{\scriptsize density estimation}} & \includegraphics[trim=0bp 50bp 40bp 80bp,clip,width=0.45\columnwidth]{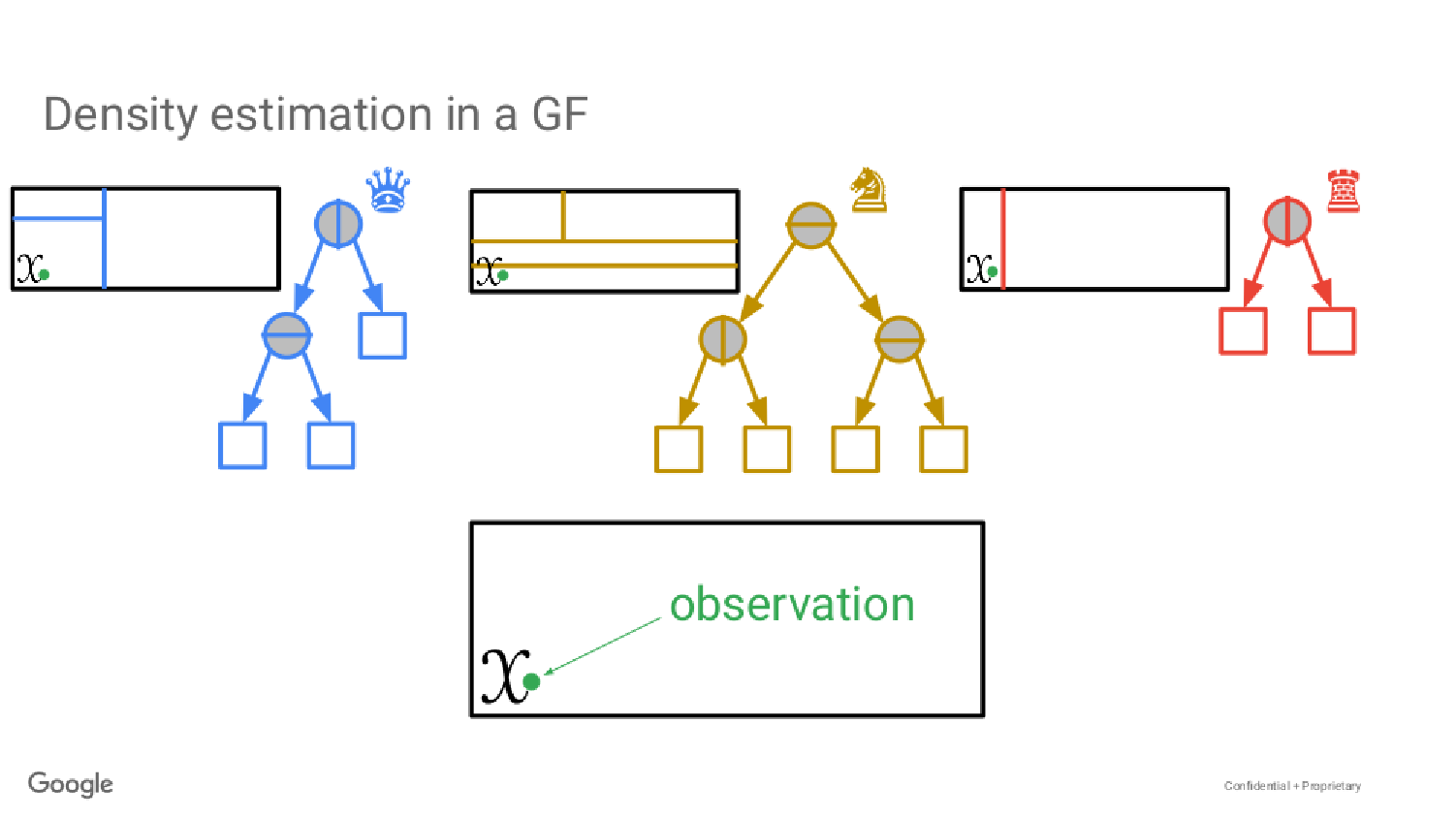} & \includegraphics[trim=0bp 50bp 40bp 80bp,clip,width=0.45\columnwidth]{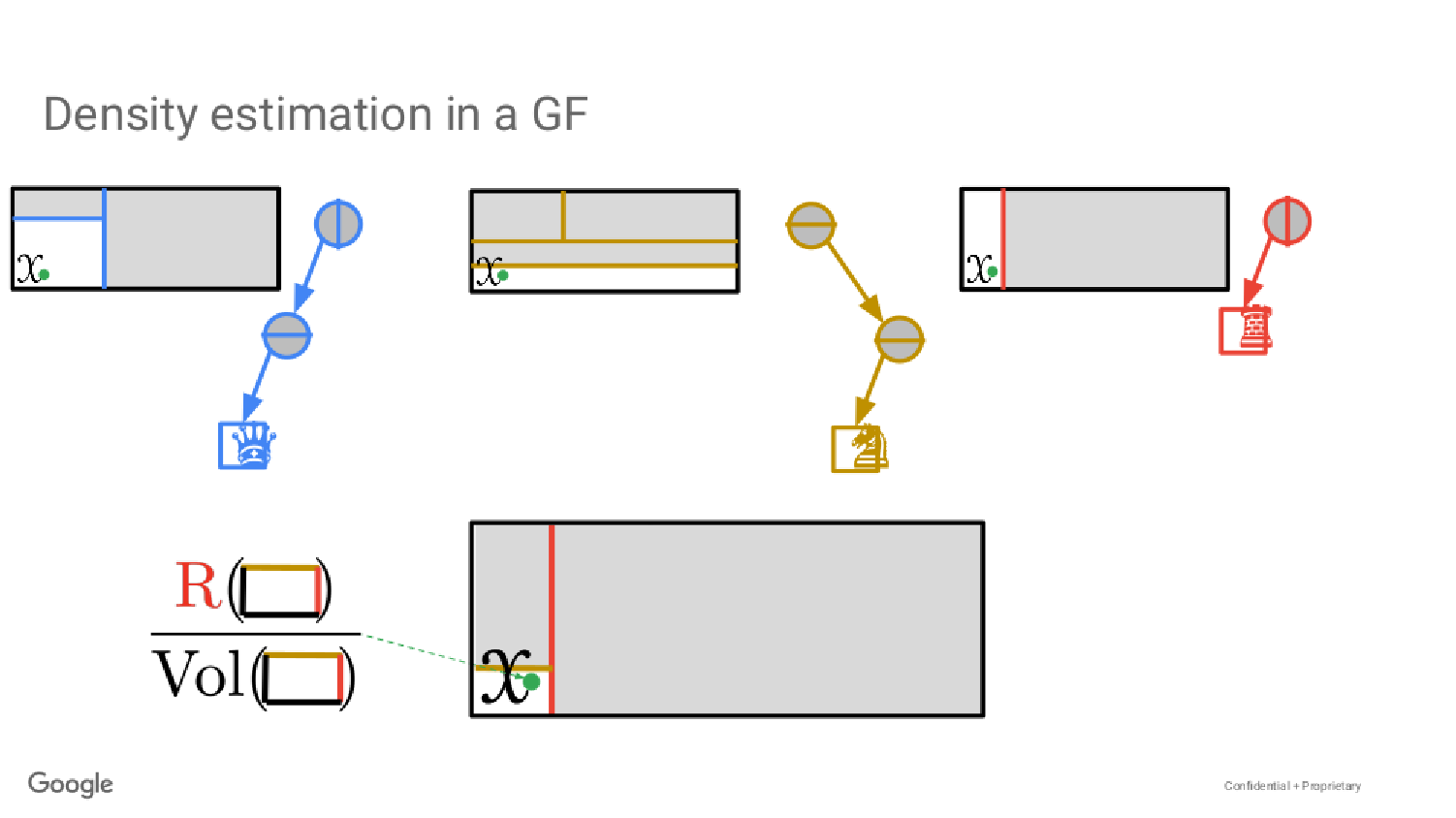} \\ \Xhline{2pt}
    \rotatebox{90}{{\scriptsize missing data imputation}} & \includegraphics[trim=0bp 27bp 40bp 80bp,clip,width=0.45\columnwidth]{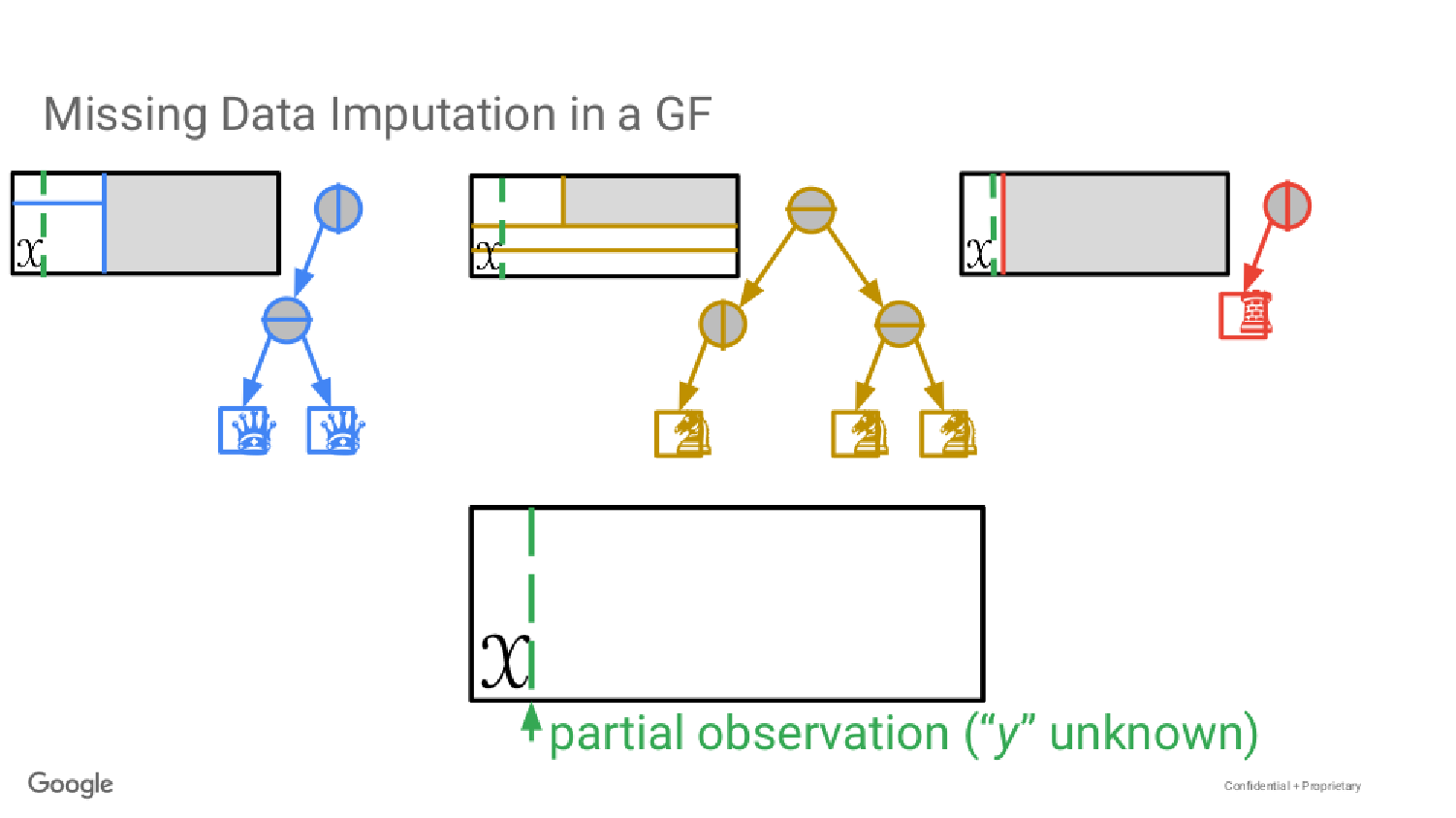} & \includegraphics[trim=0bp 27bp 40bp 80bp,clip,width=0.45\columnwidth]{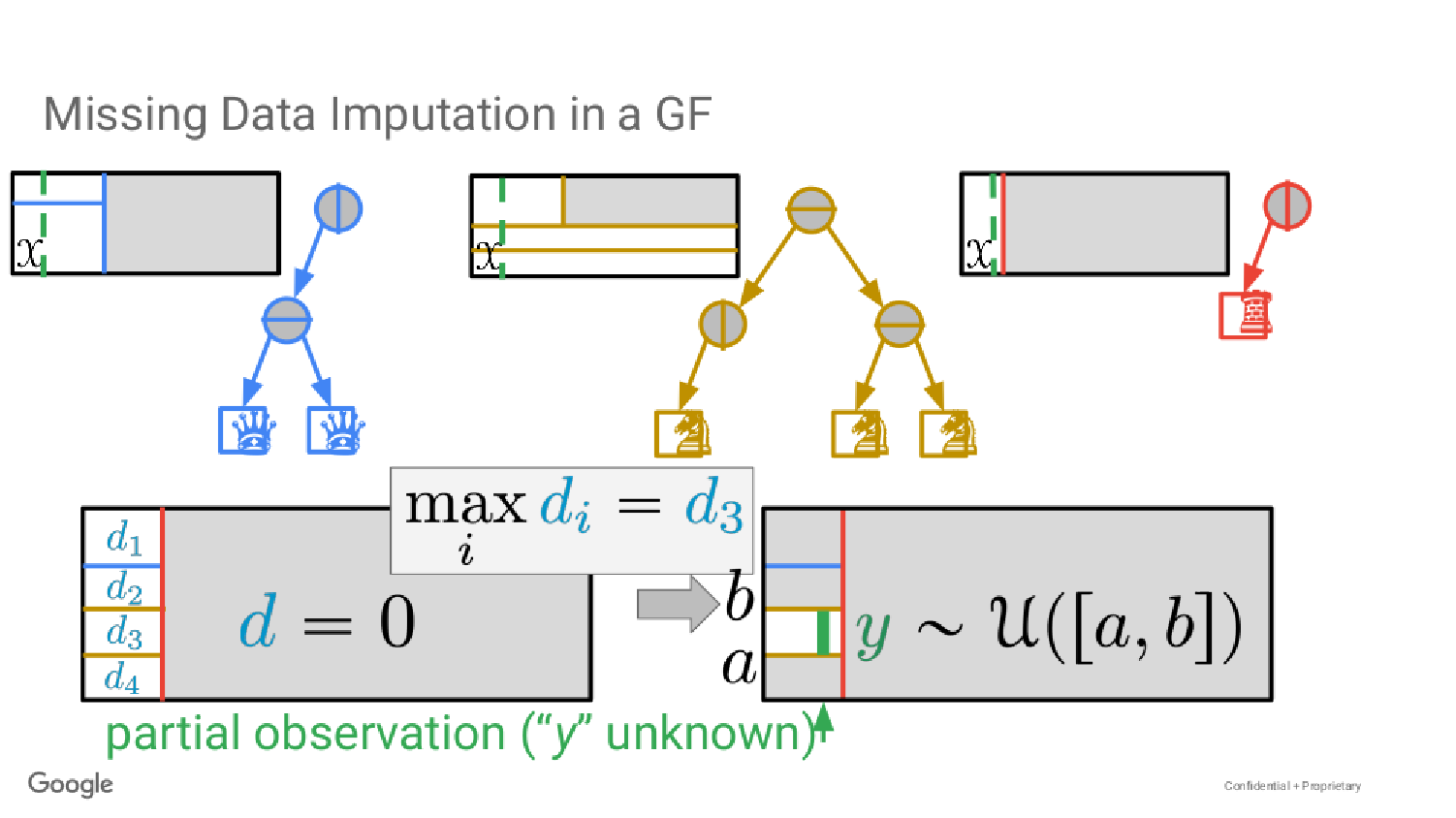} \\ \Xhline{2pt}
    \end{tabular}}
\bignegspace
    \caption{(\textbf{Top row}) Density estimation using a \geot, on an observation indicated by ${\color{darkgreen} \bullet}$ (\textit{Left}). In each tree, the leaf reached by the observation is found and the intersection of all leaves' supports is computed. The estimated density at ${\color{darkgreen} \bullet}$ is computed as the empirical measure in this intersection over its volume (\textit{Right}). (\textbf{Bottom row}) Missing data imputation using the same \geot, and an observation with one missing value (if $\mathcal{X} \subset \mathbb{R}^2$, then $y$ is missing). We first proceed like in density estimation, finding in each tree \textit{all} leaves \textit{potentially} reached by the observation if $y$ were known (\textit{Left}); then, we compute the density in \textit{each} non-empty intersection of all leaves' supports; among the corresponding elements with maximal density, we get the missing value(s) by uniform sampling (\textit{Right}).}
    \label{fig:density-estimation-mda}
 \bignegspace
\end{figure}

\bignegspace
\bignegspace
\bignegspace
\section{Learning generative forests using supervised boosting}\label{sec-boost}
\bignegspace

\begin{algorithm}[t]
\resizebox{.9\columnwidth}{!}{    \caption{\topdownGT($\colorr$, $J$, $T$)}\label{alg-topdownGT}
\begin{minipage}{\columnwidth}\begin{algorithmic}
  \STATE  \textbf{Input:} measure $\colorr$,  $\#$iters $J$, $\#$trees $T$; 
    \STATE  \textbf{Output:} trees $\{\tree_t\}_{t=1}^T$ of \geot~$\df{\colorg}$; 
  \STATE  Step 1 : $\mathcal{T} \leftarrow \{\tree_t = \mbox{ (root)}\}_{t=1}^T$; 
  \STATE  Step 2 : \textbf{for} $j=1$ \textbf{to} $J$
  \STATE  \hspace{0.3cm} Step 2.1 : $\tree_* \leftarrow$ \texttt{tree} ($\mathcal{T}$);
  \STATE  \hspace{0.3cm} Step 2.2 : $\candidateleaf \leftarrow $ \texttt{leaf} ($\tree_*$);
 \STATE  \hspace{0.3cm} Step 2.3 : $\texttt{p} \leftarrow$ \texttt{splitPred} $(\candidateleaf, \mathcal{T}, \colorr)$; 
 \STATE  \hspace{0.3cm} Step 2.4 : $\texttt{split}(\tree_*,\candidateleaf,\texttt{p})$; 
    \STATE  \textbf{return} $\mathcal{T}$; 
\end{algorithmic}
\end{minipage}}
\end{algorithm}

\noindent \textbf{From the GAN framework to a fully supervised training of generative models} It can appear quite unusual to train generative models using a supervised learning framework, so before introducing our algorithm, we provide details on its filiation in the generative world, starting from the popular GAN training framework \cite{gpmxwocbGA}. In this framework, one trains a generative model against a discriminator model which has no purpose other than to parameterize the generative loss optimized. As shown in \cite{nctFG}, there is a generally inevitable slack between the generator and the discriminator losses with neural networks, which translates into uncertainty for training. \cite{ngGT} show that the slack disappears for calibrated models, a property satisfied by their generative trees (and also by our generative forests). Moreover, \cite{ngGT} also show that the GAN training can be simplified and made much more efficient for generative trees by having the discriminator (a decision tree) copy the tree structure of the generator, a setting defined as \textit{copycat}. Hence, one gets reduced uncertainty and more efficient training. Training still involves two models but it becomes arguably very close to the celebrated boosting framework for top-down induction of decision trees \cite{kmOT}, up to the crucial detail that the generator turns out to implements boosting's leveraged "hard" distribution. This training gives guarantees on the likelihood ratio risk we define in \eqref{eqLRR}. Our paper closes the gap with boosting: copycat training can be equivalently simplified to training a \textit{single generative model} in a \textit{supervised} (2 classes) framework. The two classes involved are the observed data and the uniform distribution. Using the uniform distribution is allowed by the assumption that the domain is closed, which is reasonable for tabular data but can also be alleviated by reparameterization \cite[Remark 3.3]{ngGT}. The (supervised) loss involved for training reduces to the popular concave "entropic"-style losses used in CART, C4.5 and in popular packages like \texttt{scikit-learn}. And of course, minimizing such losses \textit{still} provides guarantee on the generative loss \eqref{eqLRR}. While it is out of scope to show how copycat training does ultimately simplify, we provide all components of the "end product": algorithm, loss optimized and the link between the minimization of the supervised and generative losses via boosting. \\
\noindent \textbf{The algorithm} To learn a \geot, we just have to learn its set of trees. Our training algorithm, \topdownGT~(Algorithm \ref{alg-topdownGT}), performs a greedy top-down induction. In Step 1, we initialize the set of $T$ trees to $T$ roots. Steps 2.1 and 2.2 choose a tree ($\tree_*$) and a leaf to split ($\candidateleaf$) in the tree. In our implementation, we pick the leaf among all trees which is the heaviest with respect to $\colorr$. Hence, we merge Steps 2.1 and 2.2. Picking the heaviest leaf is standard to grant boosting in decision tree induction \cite{kmOT}; in our case, there is a second benefit: we tend to learn size-balanced models. For example, during the first $J=T$ iterations, each of the $T$ roots gets one split because each root has a larger mass (1) than any leaf in a tree with depth $>0$. Step 2.4 splits $\tree_*$ by replacing $\candidateleaf$ by a stump whose corresponding splitting predicate, \texttt{p}, is returned in Step 2.3 using a weak splitter oracle called \texttt{splitPred}. "weak" refers to boosting's weak/strong learning setting \cite{kTO} and means that we shall only require lightweight assumptions about this oracle; in decision tree induction, this oracle is the key to boosting from such weak assumptions \cite{kmOT}. This will also be the case for our generative models. We now investigate Step 2.3 and \texttt{splitPred}.\\
\noindent \textbf{The weak splitter oracle \texttt{splitPred}} In decision tree induction, a splitting predicate is chosen to reduce an expected Bayes risk \eqref{defPOIBAYESRISK} (\textit{e.g.} that of the log loss \cite{qC4}, square loss \cite{bfosCA}, etc.). In our case, \texttt{splitPred} does about the same \textit{with a catch in the binary task it adresses}, which is\footnote{The prior is chosen by the user: without reasons to do otherwise, a balanced approach suggests $\prior = 0.5$.} $\binartaskgen \defeq (\prior, \colorr, \coloru)$ (our mixture measure is thus $\colorm{\small} \defeq\prior \cdot \colorr + (1-\prior) \cdot \coloru$). The corresponding expected Bayes risk that \texttt{splitPred} seeks to minimize is just:
\begin{eqnarray}
\poprisk(\mathcal{T}) & \defeq & \sum_{\mathcal{C} \in \mathcal{P}(\mathcal{T})} p_{\colorm{\scriptsize}}[\mathcal{C}] \cdot \poibayesrisk\left(\frac{\prior p_{\meas{\colorr}}[\mathcal{C}]}{p_{\colorm{\scriptsize}}[\mathcal{C}]}\right).\label{defPOPRISK}
\end{eqnarray}
The concavity of $\poibayesrisk$ implies $\poprisk(\mathcal{T}) \leq \poibayesrisk(\prior)$. Notation $\mathcal{P}(.)$ overloads that in \eqref{defPART} by depending explicitly on the set of trees of $\colorg$ instead of $\colorg$ itself. Regarding the minimization of \eqref{defPOPRISK}, there are three main differences with classical decision tree induction. The first is computational: in the latter case, \eqref{defPOPRISK} is optimized over a single tree: there is thus a single element in $\mathcal{P}(\mathcal{T})$ which is affected by the split, $\mathcal{X}_{\candidateleaf}$. In our case, multiple elements in $\mathcal{P}(\mathcal{T})$ can be affected by one split, so the optimisation is more computationally demanding but a simple trick allows to keep it largely tractable: we do not need to care about keeping in $\mathcal{P}(\mathcal{T})$ support elements with zero empirical measure since they will generate no data. Keeping only elements with strictly positive empirical measure guarantees a size $|\mathcal{P}(\mathcal{T})|$ never bigger than the number of observations defining $\colorr$. The second difference plays to our advantage: compared to classical decision tree induction, a single split generally buys a substantially bigger slack in $\poprisk(\mathcal{T})$ in our case. To see this, remark that for the candidate leaf $\candidateleaf$,
\begin{eqnarray*}
   \sum_{{\tiny \begin{array}{c}
          \mathcal{C} \in \mathcal{P}(\mathcal{T}) \\
          \mathcal{C} \subseteq \mathcal{X}_{\candidateleaf}
          \end{array}}}   p_{\colorm{\scriptsize}}[\mathcal{C}] \cdot \poibayesrisk\left(\frac{\prior p_{\meas{\colorr}}[\mathcal{C}]}{p_{\colorm{\scriptsize}}[\mathcal{C}]}\right) & = &  p_{\colorm{\scriptsize}}[\mathcal{X}_{\candidateleaf}]  \cdot \sum_{{\tiny \begin{array}{c}
          \mathcal{C} \in \mathcal{P}(\mathcal{T}) \\
          \mathcal{C} \subseteq \mathcal{X}_{\candidateleaf} \end{array}}}   \frac{p_{\colorm{\scriptsize}}[\mathcal{C}]}{p_{\colorm{\scriptsize}}[\mathcal{X}_{\candidateleaf}]} \cdot \poibayesrisk\left(\frac{\prior p_{\meas{\colorr}}[\mathcal{C}]}{p_{\colorm{\scriptsize}}[\mathcal{C}]}\right) \\
  & \leq & p_{\colorm{\scriptsize}}[\mathcal{X}_{\candidateleaf}] \cdot \poibayesrisk\left(\frac{\prior p_{\meas{\colorr}}[\mathcal{X}_{\candidateleaf}]}{p_{\colorm{\scriptsize}}[\mathcal{X}_{\candidateleaf}]}\right)
\end{eqnarray*}
(because $\poibayesrisk$ is concave and $\sum_{\mathcal{C} \in \mathcal{P}(\mathcal{T}), \mathcal{C} \subseteq \mathcal{X}_{\candidateleaf}} p_{\meas{\colorr}}[\mathcal{C}] = p_{\meas{\colorr}}[\mathcal{X}_{\candidateleaf}]$). The top-left term is the contribution of $\candidateleaf$ to $\poprisk(\mathcal{T})$, the bottom-right one its contribution to $\poprisk(\{\tree_*\})$ (the decision tree case), so the slack gives a proxy to our potential advantage after splitting. The third difference with decision tree induction is the possibility to converge much faster to a good solution in our case. Scrutinizing the two terms, we indeed get that in the decision tree case,  the split only gets two new terms. In our case however, there can be up to $2\cdot \mathrm{Card}(\{\mathcal{C} \in \mathcal{P}(\mathcal{T}) : \mathcal{C} \subseteq \mathcal{X}_{\candidateleaf}\})$ new elements in $\mathcal{P}(\mathcal{T})$.\\
  \noindent\textbf{Boosting} Two questions remain: can we quantify the slack in decrease and of course what quality guarantee does it bring for the \textit{generative} model $\colorg$ whose set of trees is learned by \topdownGT~?  We answer both questions in a single Theorem, which relies on a weak learning assumption that parallels the classical weak learning assumption of boosting:
\begin{definition}\label{defWLAG}
 \textbf{(WLA$(\upgamma, \upkappa)$)} There exists $\upgamma > 0, \upkappa >0$ such that at each iteration of \topdownGT, the couple $(\candidateleaf, \texttt{p})$ chosen in Steps 2.2, 2.3 of \topdownGT~satisfies the following properties: \textbf{(a)} $\candidateleaf$ is not skinny: $p_{\colorm{\scriptsize}}[\mathcal{X}_{\candidateleaf}] \geq 1/\mathrm{Card}(\leafset(\tree_*))$, \textbf{(b)} truth values of $\texttt{p}$ moderately correlates with $\binartaskgen$ at $\candidateleaf$: $\left| p_{\meas{\colorr}}[\mbox{\texttt{p}}|\mathcal{X}_{\candidateleaf}] - p_{\meas{\coloru}}[\mbox{\texttt{p}}|\mathcal{X}_{\candidateleaf}] \right| \geq \upgamma$, and finally \textbf{(c)} there is a minimal proportion of real data at $\candidateleaf$: $\prior p_{\meas{\colorr}}[\mathcal{X}_{\candidateleaf}] / p_{\colorm{\scriptsize}}[\mathcal{X}_{\candidateleaf}] \geq \upkappa$.
\end{definition}
\bignegspace
The convergence proof of \cite{kmOT} reveals that both \textbf{(a)} and \textbf{(b)} are jointly made at any split, where our \textbf{(b)} is equivalent to their \textit{weak hypothesis assumption} where their $\gamma$ parameter is twice ours. \textbf{(c)} postulates that the leaf split has empirical measure at least a fraction of its uniform measure -- thus, of its relative volume. \textbf{(c)} is important to avoid splitting leaves that would essentially be useless for our tasks: a leaf for which \textbf{(c)} is invalid would indeed locally model comparatively tiny density values.
\negspace
\begin{theorem}\label{th-boost}
  Suppose the loss $\ell$ is \textbf{strictly proper} and \textbf{differentiable}. Let $\colorg_0$ (= $\coloru$) denote the initial \geot~with $T$ roots in its trees (Step 1, \topdownGT) and $\colorg_J$ the final \geot, assuming wlog that the number of boosting iterations $J$ is a multiple of $T$. Under WLA$(\upgamma, \upkappa)$, we get the following on likelihood ratio risks: $\likelihoodratiorisk_{\ell}\left(\colorr, \colorg_J\right) \leq \likelihoodratiorisk_{\ell}\left(\colorr, \colorg_{0}\right) - \frac{\kappa \upgamma^2 \upkappa^2}{8} \cdot T\log\left(1 + \frac{J}{T}\right)$, for some constant $\kappa > 0$.
\end{theorem}
\bignegspace
It comes from \cite[Remark 1]{mnwRC} that we can choose $\kappa = \inf\{\partialloss{-1}' - \partialloss{1}'\}$, which is $>0$ if $\ell$ is strictly proper and differentiable. Strict properness is also essential for the loss to guarantee that minimization gets to a good generator (Section \ref{sec-def}).

\bignegspace
\bignegspace
 \setlength\tabcolsep{4pt}

\setlength\tabcolsep{0.5pt}

\begin{table}[t]
  \centering
  \resizebox{0.8\columnwidth}{!}{{\tiny\begin{tabular}{cccc?c}\Xhline{2pt}
$T$=$J$=20($^*$) & $T$=$J$=100($^*$) & $T$=200,$J$=500 & $T$=200,$J$=1500 & {Ground truth}\\ \Xhline{2pt}
\includegraphics[trim=0bp 0bp 0bp 0bp,clip,width=\cinq\columnwidth]{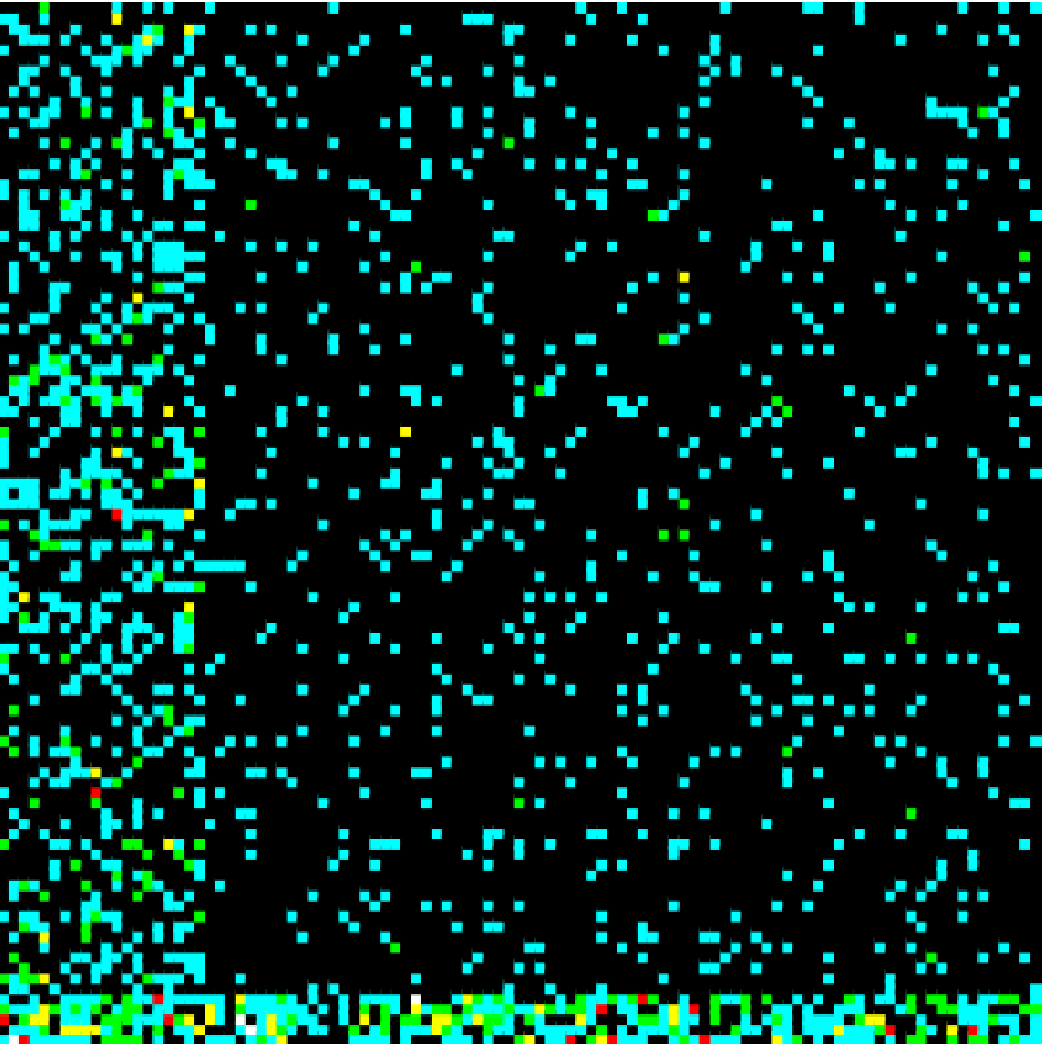} & \includegraphics[trim=0bp 0bp 0bp 0bp,clip,width=\cinq\columnwidth]{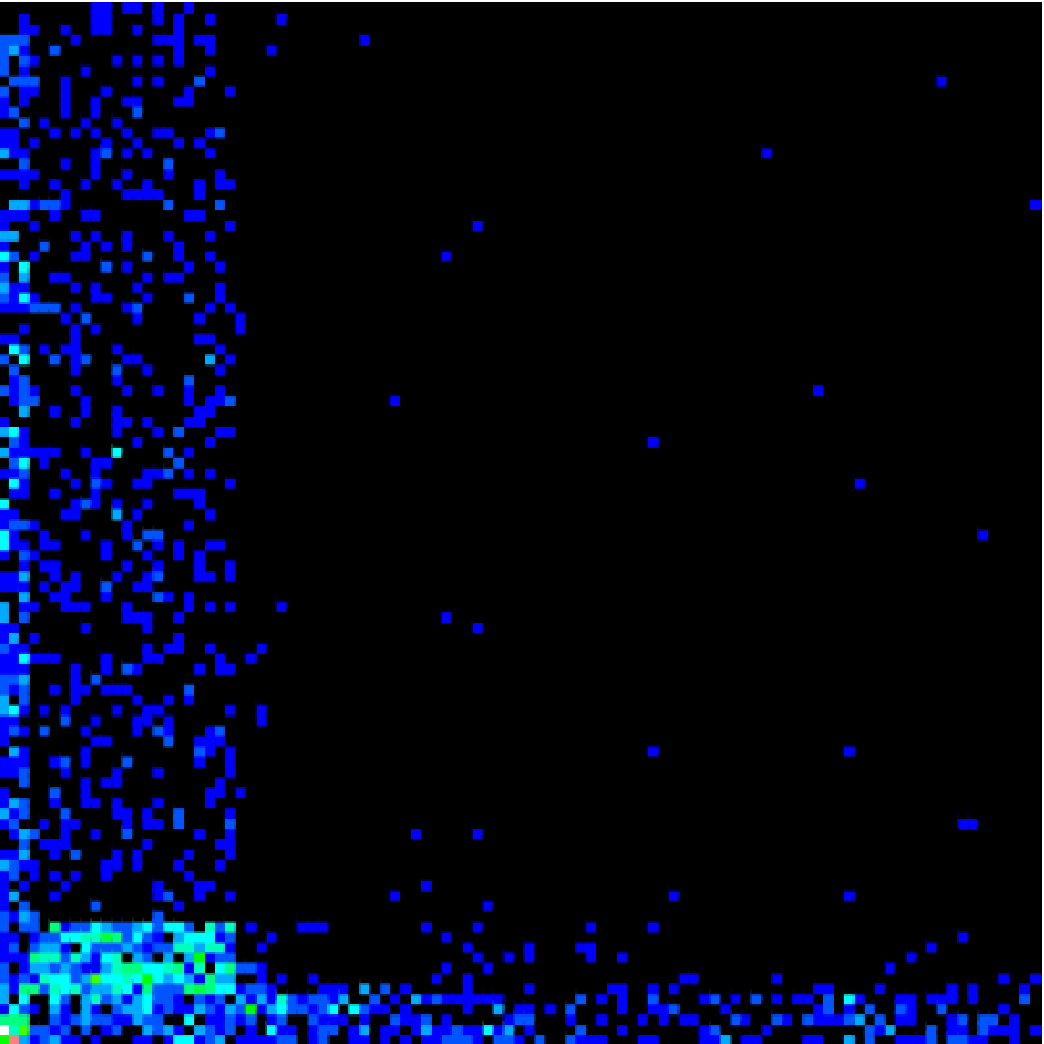} & \includegraphics[trim=0bp 0bp 0bp 0bp,clip,width=\cinq\columnwidth]{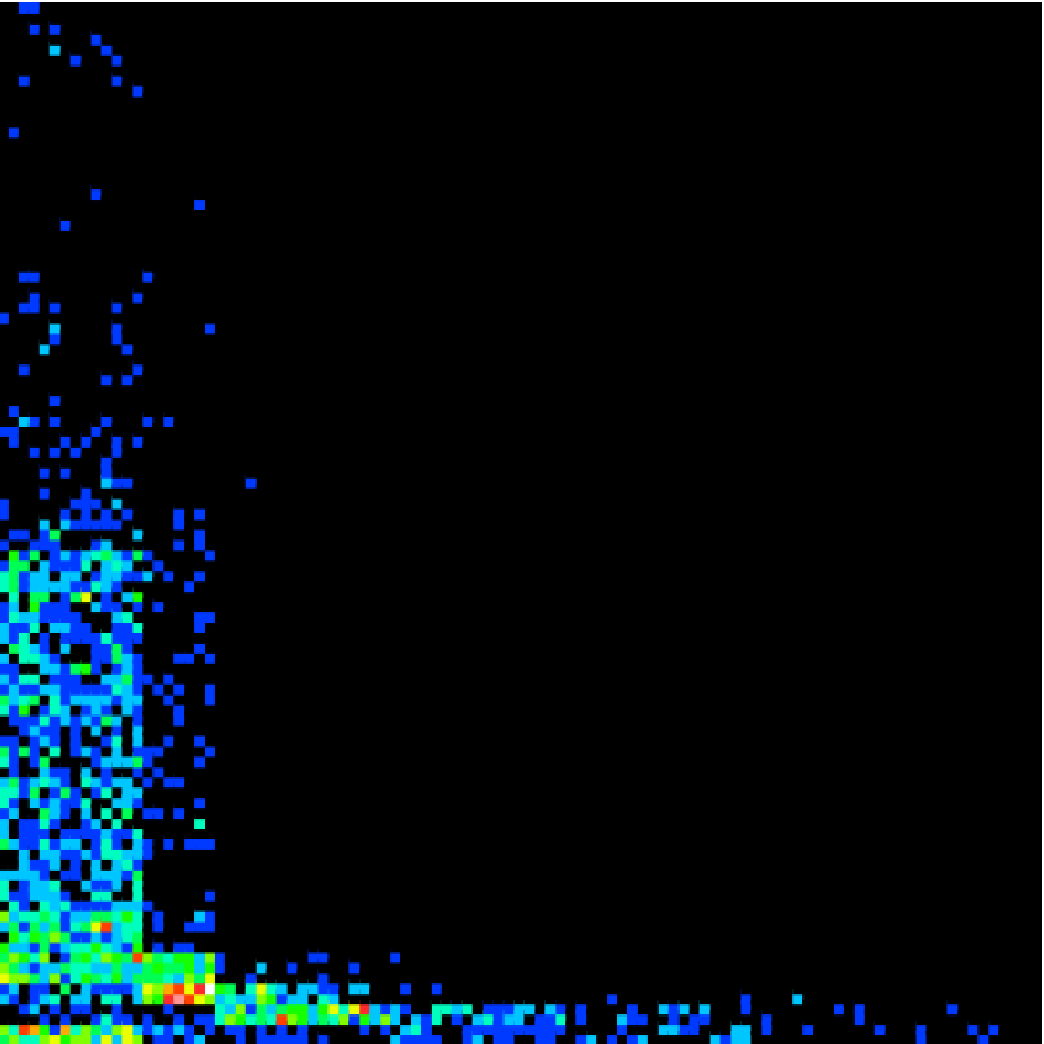} & \includegraphics[trim=0bp 0bp 0bp 0bp,clip,width=\cinq\columnwidth]{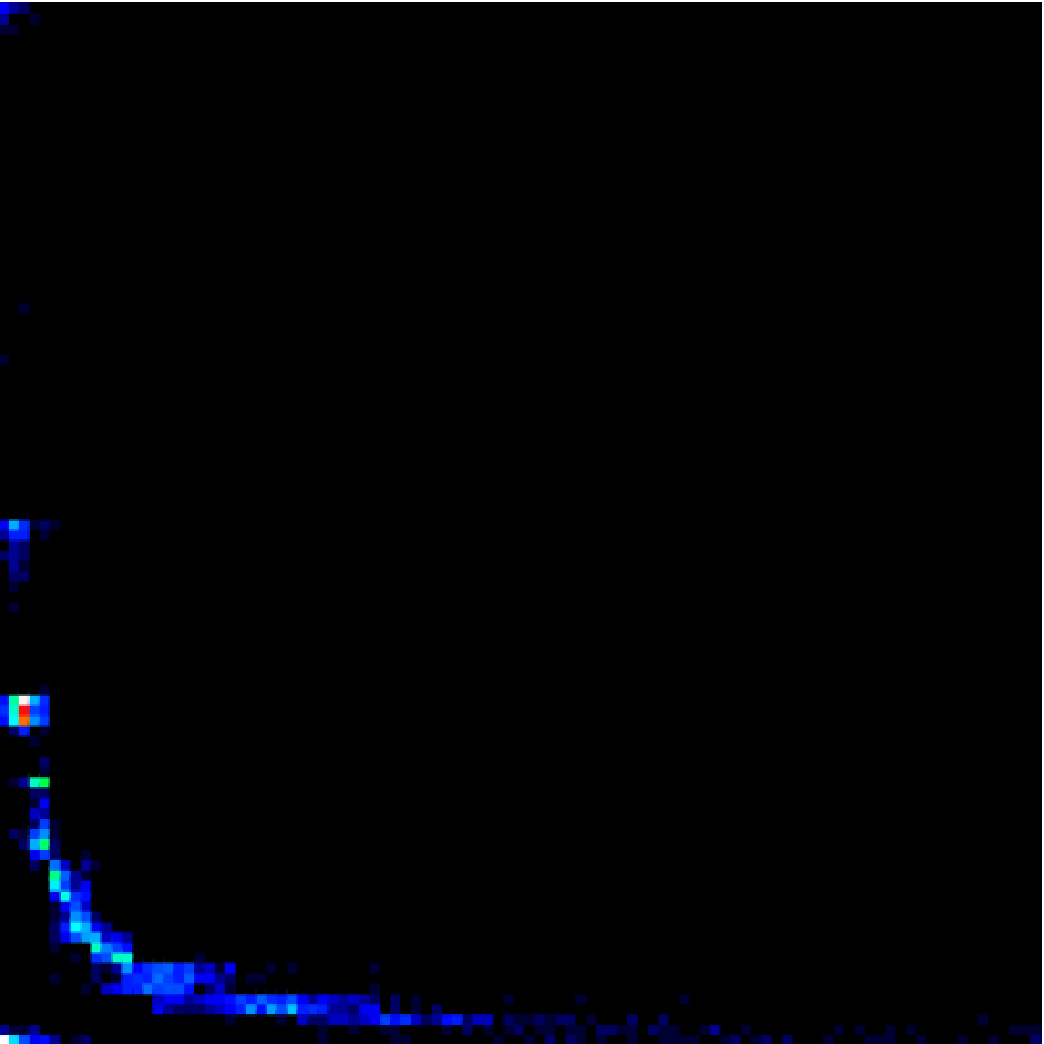} & \includegraphics[trim=0bp 0bp 0bp 0bp,clip,width=\cinq\columnwidth]{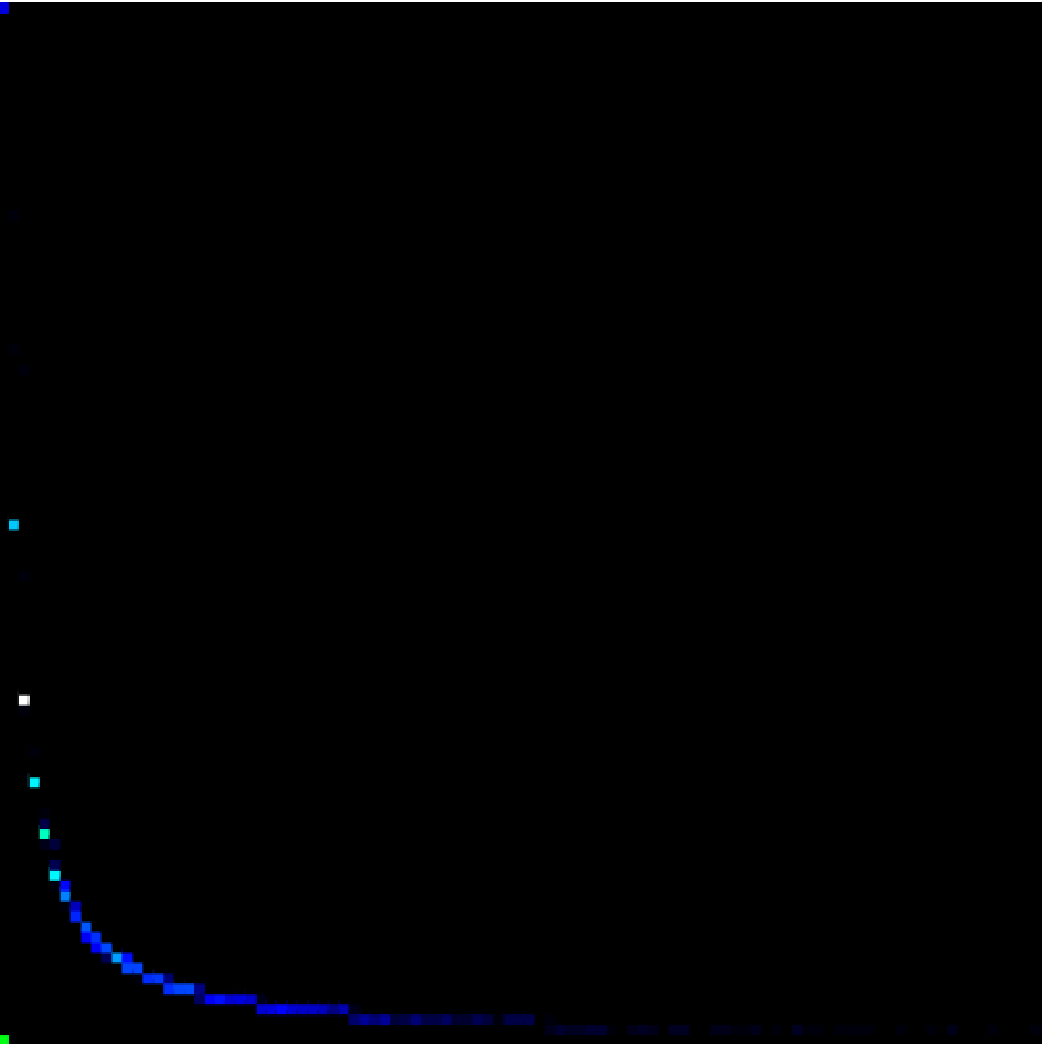} \\\Xhline{2pt}
  \end{tabular}}}
  \caption{2D density plots generated on OpenML \texttt{kc1} ($x=d$ and $y=l$) with \geot, for varying number of trees $T$ and number of splits $J$ (columns). "$^*$" = all trees are stumps. The rightmost column recalls the domain ground truth for comparison. Each generated dataset contains $m=2000$ observations. The two variables (among $d=22$) were chosen because of their deterministic dependence.}
    \label{tab:gen-full-kc1}
 \bignegspace
 \bignegspace
 \bignegspace
 \negspace
\end{table}
\begin{table}[t]
  \centering
  \begin{tabular}{c|c|c}\Xhline{2pt}
    \geot~$\gg$ KDE & neither & KDE $\gg$ \geot\\\hline
    9 & 5 & 3 \\\Xhline{2pt}
  \end{tabular}
  \caption{\textsc{density}: comparison between us and KDE (summary), counting the number of domains for which we are statistically better (left), or worse (right). The central column counts the remaining domains, for which no statistical significance ever holds.}
    \label{tab:density}
 \bignegspace
 \bignegspace
  \end{table}
\setlength\tabcolsep{1pt}

\begin{table*}[t]
  \centering
 \resizebox{0.92\textwidth}{!}{
  {\tiny \begin{tabular}{c|cc?cc}\Xhline{2pt}
    & \multicolumn{2}{c?}{UCI \domainname{abalone}} & \multicolumn{2}{c}{OpenML \domainname{analcatdata$\_$supreme}} \\ 
    & \textbf{perr} (cat. variables) & \textbf{rmse} (num. variables) & \textbf{perr} & \textbf{rmse} \\ \hline
\rotatebox{90}{vs \textsc{gf}, $J=20$} &   \includegraphics[trim=30bp 0bp 20bp 0bp,clip,width=\quatre\columnwidth]{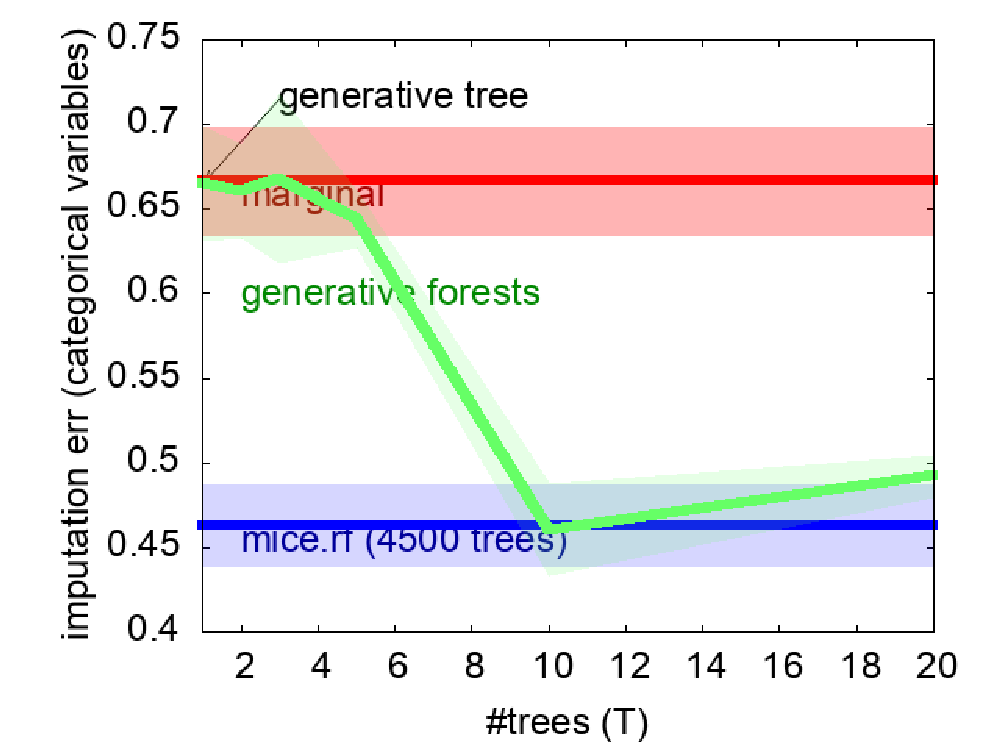} & \includegraphics[trim=30bp 0bp 20bp 0bp,clip,width=\quatre\columnwidth]{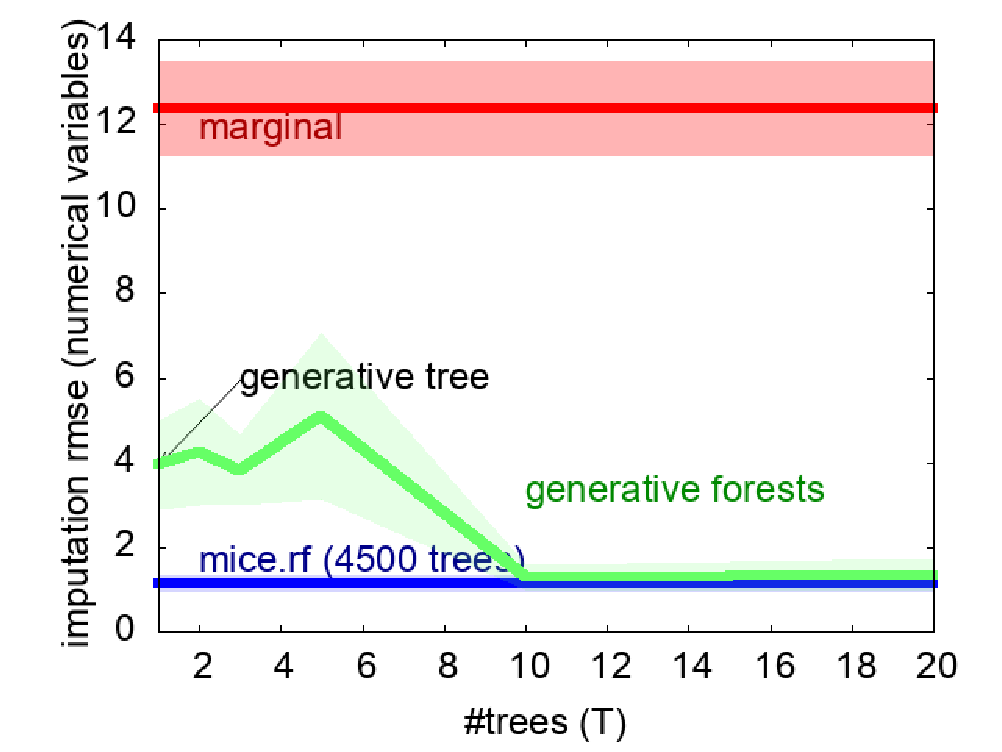} & \includegraphics[trim=30bp 0bp 20bp 0bp,clip,width=\quatre\columnwidth]{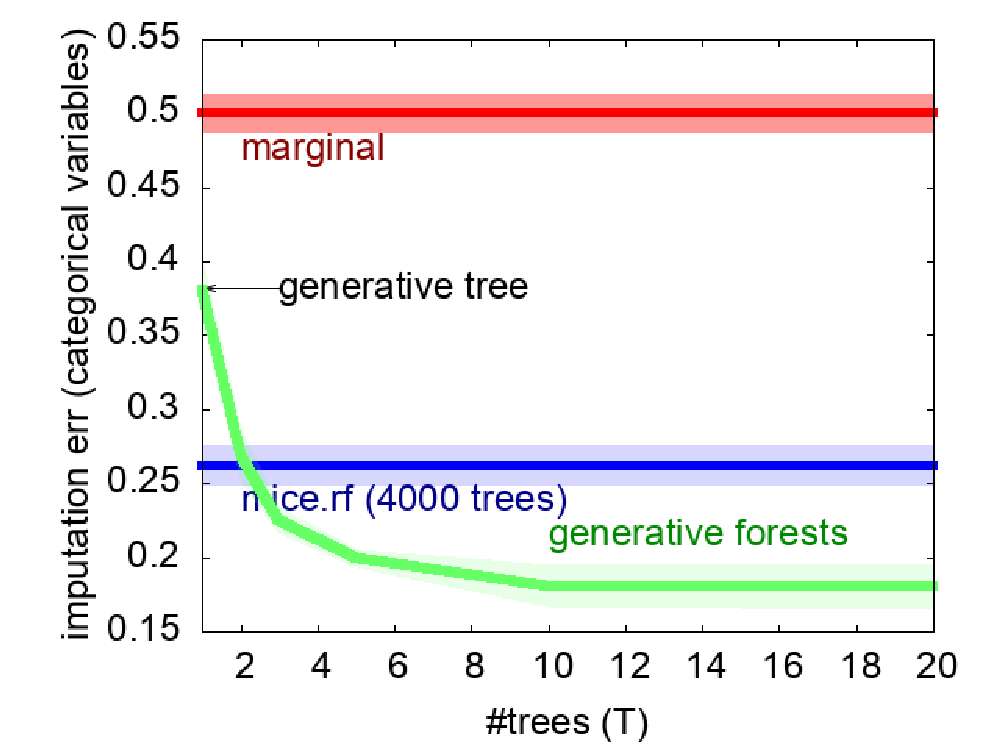} & \includegraphics[trim=30bp 0bp 20bp 0bp,clip,width=\quatre\columnwidth]{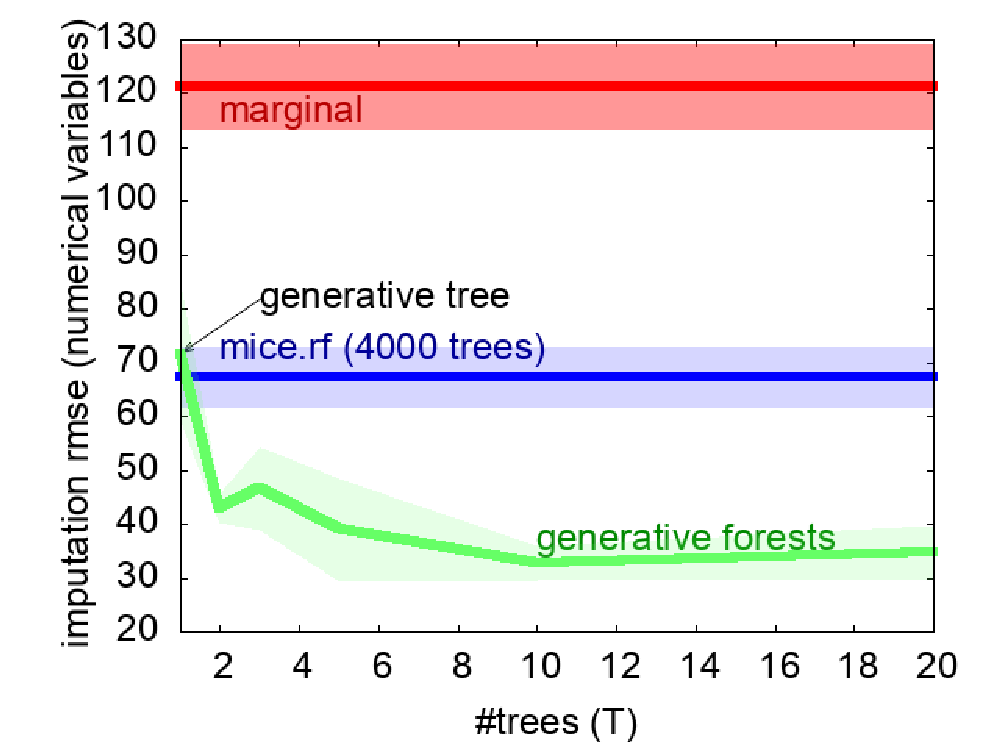} \\\Xhline{2pt}
  \end{tabular}}
  }
\bignegspace
  \caption{\textsc{impute}: results on two domains (left and right pane). Since both domain contain categorical and numerical variables, we compute for each \textbf{perr} (categorical variables) and the \textbf{rmse} (numerical variables). In each plot, we display both the corresponding results of \topdownGT, the results of \mice.\textsc{rf}~(indicating the total number of trees used to impute the dataset) and the result of a fast imputation baseline using the \texttt{marginal}s computed from the training sample. The $x$ axis displays the number of trees $T$ in \topdownGT~and each curve is displayed with its average $\pm$ standard deviation in shaded color. The result for $T=1$ equivalently indicates the performance of a single generative tree (\gt) with $J$ splits \cite{ngGT}, shown with an arrow (see text).}
    \label{tab:mdi-abaanalcat}
\end{table*}

\section{Experiments}\label{sec-exp}

\setlength\tabcolsep{4pt}

\begin{table*}[h]
    \centering
    \resizebox{1.0\textwidth}{!}{
      \begin{tabular}{l?rrr?rrr?rrr?rrr}\Xhline{2pt}
       Domain & \multicolumn{3}{c?}{Sinkhorn$\downarrow$} & \multicolumn{3}{c?}{Coverage$\uparrow$} & \multicolumn{3}{c?}{Density$\uparrow$} & \multicolumn{3}{c}{F1 measure$\downarrow$} \\
 & \multicolumn{1}{c}{us (\geot)} & \multicolumn{1}{c}{ARF} & pval & \multicolumn{1}{c}{us (\geot)} & \multicolumn{1}{c}{ARF} & pval & \multicolumn{1}{c}{us (\geot)} & ARF  & pval & \multicolumn{1}{c}{us (\geot)} & \multicolumn{1}{c}{ARF} & pval \\\hline
\domainname{ring} & \better{0.285$\pm$0.008} & 0.286$\pm$0.007 & 0.62 & \better{0.968$\pm$0.010} & 0.960$\pm$0.008 & 0.05 & \better{1.031$\pm$0.049} & 0.976$\pm$0.030 & 0.11 & \better{0.070$\pm$0.030} & 0.086$\pm$0.047 & 0.54\\\hline
\domainname{circ} & \better{0.351$\pm$0.005} & 0.355$\pm$0.005 & 0.33 & 0.945$\pm$0.015 & \better{0.962$\pm$0.021} & 0.35 & \better{0.993$\pm$0.045} & 0.954$\pm$0.011 & 0.05 & 0.514$\pm$0.529 & \better{0.507$\pm$0.032} & 0.92\\\hline
\domainname{grid} & \better{0.390$\pm$0.002} & 0.394$\pm$0.002 & 0.02 & \better{0.975$\pm$0.009} & 0.908$\pm$0.010 & \tugreen & \better{0.995$\pm$0.077} & 0.630$\pm$0.050 & \tugreen & \better{0.017$\pm$0.015} & 0.043$\pm$0.012 & 0.04\\\hline
\domainname{rand}  & \better{0.286$\pm$0.003} & 0.288$\pm$0.002 & 0.36 & \better{0.961$\pm$0.012} & 0.953$\pm$0.004 & 0.29 & \better{0.979$\pm$0.023} & 0.940$\pm$0.027 & \tugreen & \better{0.014$\pm$0.013} & 0.029$\pm$0.010 & \tugreen \\\hline
\domainname{wred}  & \better{0.980$\pm$0.032} & 1.099$\pm$0.031 & \tugreen & \better{0.962$\pm$0.028} & 0.929$\pm$0.010 & 0.14 & \better{0.961$\pm$0.026} & 0.801$\pm$0.028 & \tugreen & \better{0.459$\pm$0.031} & 0.531$\pm$0.028 & \tugreen \\\hline
\domainname{wwhi} & \better{1.064$\pm$0.003} & 1.150$\pm$0.021 & \tugreen & \better{0.963$\pm$0.002} & 0.946$\pm$0.012 & 0.10 & \better{0.989$\pm$0.017} & 0.941$\pm$0.013 & \tugreen & \better{0.492$\pm$0.037} & 0.500$\pm$0.027 & 0.04\\\hline
\domainname{comp} & \better{0.532$\pm$0.008} & 0.535$\pm$0.033 & 0.86 & 0.556$\pm$0.021 & \better{0.560$\pm$0.037} & 0.75 & 0.430$\pm$0.008 & \better{0.440$\pm$0.011} & 0.11 & \better{0.496$\pm$0.021} & 0.520$\pm$0.020 & 0.26\\\hline
        \domainname{arti} & \better{0.821$\pm$0.004} & 0.849$\pm$0.010 & \tugreen & \better{0.947$\pm$0.003} & 0.892$\pm$0.014 & \tugreen & \better{0.967$\pm$0.013} & 0.747$\pm$0.013 & \tugreen & \better{0.429$\pm$0.056} & 0.512$\pm$0.037 & \tugreen\\\Xhline{2pt}
         & \multicolumn{1}{c}{us (\geot)} & \multicolumn{1}{c}{CT} & pval & \multicolumn{1}{c}{us (\geot)} & \multicolumn{1}{c}{CT} & pval & \multicolumn{1}{c}{us (\geot)} & \multicolumn{1}{c}{CT}  & pval & \multicolumn{1}{c}{us (\geot)} & \multicolumn{1}{c}{CT} & pval \\\hline
\domainname{ring} & \better{0.285$\pm$0.008} & 0.351$\pm$0.042 & 0.01 & \better{0.968$\pm$0.010} & 0.798$\pm$0.050 &\tugreen & \better{1.031$\pm$0.049} & 0.757$\pm$0.025 &\tugreen & \better{0.070$\pm$0.030} & 0.100$\pm$0.073 & 0.40 \\\hline
\domainname{circ} & \better{0.351$\pm$0.005} &  0.435$\pm$0.075 & 0.07 & \better{0.945$\pm$0.015} & 0.734$\pm$0.041 &\tugreen & \better{0.993$\pm$0.045} & 0.401$\pm$0.033 &\tugreen & \better{0.514$\pm$0.529} & 0.746$\pm$0.033 &\tugreen   \\\hline
\domainname{grid} & \better{0.390$\pm$0.002} & 0.408$\pm$0.019  & 0.01 & \better{0.975$\pm$0.009} & 0.828$\pm$0.053 &\tugreen & \better{0.995$\pm$0.077} & 0.649$\pm$0.058 &\tugreen & \better{0.017$\pm$0.015} & 0.034$\pm$0.011 & 0.21 \\\hline
\domainname{rand} & \better{0.286$\pm$0.003} & 0.327$\pm$0.024 & 0.03 & \better{0.961$\pm$0.012} & 0.659$\pm$0.049 &\tugreen & \better{0.979$\pm$0.023} & 0.582$\pm$0.035 &\tugreen & \better{0.014$\pm$0.013} & 0.079$\pm$0.053 & 0.06 \\\hline
\domainname{wred} & \better{0.980$\pm$0.032} & 1.384$\pm$0.047 & \tugreen & \better{0.962$\pm$0.028} & 0.808$\pm$0.016 &\tugreen & \better{0.961$\pm$0.026} & 0.589$\pm$0.129 &\tugreen & \better{0.459$\pm$0.031} & 0.654$\pm$0.030 &\tugreen \\\hline
        \domainname{wwhi} & \better{1.064$\pm$0.003} & 1.158$\pm$0.009 & \tugreen & \better{0.963$\pm$0.002} & 0.894$\pm$0.026 &\tugreen & \better{0.989$\pm$0.017} & 0.953$\pm$0.035 & 0.05 & \better{0.492$\pm$0.037} & 0.581$\pm$0.043 &\tugreen \\\Xhline{2pt}
         & \multicolumn{1}{c}{us (\geot)} & \multicolumn{1}{c}{FF} & pval & \multicolumn{1}{c}{us (\geot)} & \multicolumn{1}{c}{FF} & pval & \multicolumn{1}{c}{us (\geot)} & \multicolumn{1}{c}{FF}  & pval & \multicolumn{1}{c}{us (\geot)} & \multicolumn{1}{c}{FF} & pval \\\hline
\domainname{ring} & 0.285$\pm$0.008 & \better{0.276$\pm$0.001} & 0.07 & \better{0.968$\pm$0.010} & 0.957$\pm$0.020 & 0.09 & 1.031$\pm$0.049 & \better{1.045$\pm$0.020} & 0.28 & 0.070$\pm$0.030 & \better{0.051$\pm$0.031} & 0.07  \\\hline
\domainname{circ} & \better{0.351$\pm$0.005} & 0.354$\pm$0.003 & 0.27 & 0.945$\pm$0.020 & \better{0.956$\pm$0.015} & 0.25 & \better{0.993$\pm$0.045} & 0.989$\pm$0.027 & 0.79 & \better{0.514$\pm$0.529} & 0.530$\pm$0.028 & 0.46  \\\hline
\domainname{grid} & \better{0.390$\pm$0.002} & 0.393$\pm$0.001 & 0.04 & \better{0.975$\pm$0.009} & 0.954$\pm$0.012 & 0.01 & 0.995$\pm$0.077 & \better{1.013$\pm$0.050} & 0.21 & \better{0.017$\pm$0.015} & 0.045$\pm$0.007 & \tugreen \\\hline
\domainname{rand} & \better{0.286$\pm$0.003} & 0.287$\pm$0.001 & 0.43 & \better{0.961$\pm$0.012} & 0.927$\pm$0.011 & 0.02 & 0.979$\pm$0.023 & \better{1.000$\pm$0.008} & 0.05 & \better{0.014$\pm$0.013} & 0.028$\pm$0.008 & 0.09  \\\hline
\domainname{wred} & \better{0.980$\pm$0.032} & 1.030$\pm$0.029 & \tugreen & \better{0.962$\pm$0.028} & 0.954$\pm$0.021 & 0.71 & 0.961$\pm$0.026 & \better{1.001$\pm$0.034} & 0.04 & 0.459$\pm$0.031 & \better{0.458$\pm$0.052} & 0.97  \\\hline
\domainname{wwhi} & \better{1.064$\pm$0.003} & 1.097$\pm$0.007 & \tugreen & \better{0.963$\pm$0.002} & 0.945$\pm$0.007 & \tugreen & \better{0.989$\pm$0.017} & 0.970$\pm$0.023 & 0.13 & \better{0.492$\pm$0.037} & 0.498$\pm$0.041 & 0.78  \\\hline
\domainname{comp} & \better{0.532$\pm$0.007} & 0.891$\pm$0.007 & \tugreen & \better{0.556$\pm$0.021} & 0.548$\pm$0.027 & 0.32 & \better{0.430$\pm$0.008} & 0.404$\pm$0.013 & \tugreen & \better{0.496$\pm$0.021} & 0.526$\pm$0.018 & 0.04  \\\hline
        \domainname{arti}  & \better{0.821$\pm$0.004} & 0.834$\pm$0.016 & 0.12 & \better{0.947$\pm$0.003} & 0.879$\pm$0.008 &\tugreen &  \better{0.967$\pm$0.013} & 0.774$\pm$0.016 &\tugreen &  \better{0.429$\pm$0.056} & 0.530$\pm$0.032 & \tugreen  \\\Xhline{2pt}
        & \multicolumn{1}{c}{us (\geot)} & \multicolumn{1}{c}{VC-G} & pval & \multicolumn{1}{c}{us (\geot)} & \multicolumn{1}{c}{VC-G} & pval & \multicolumn{1}{c}{us (\geot)} & \multicolumn{1}{c}{VC-G}  & pval & \multicolumn{1}{c}{us (\geot)} & \multicolumn{1}{c}{VC-G} & pval \\\hline
\domainname{ring} & \better{0.285$\pm$0.008} & 0.392$\pm$0.005 & \tugreen & \better{0.968$\pm$0.010} &  0.364$\pm$0.037 &\tugreen & \better{1.031$\pm$0.049} & 0.132$\pm$0.013  &\tugreen & \better{0.070$\pm$0.030} &  0.291$\pm$0.035 & \tugreen \\\hline
\domainname{circ} & \better{0.351$\pm$0.005} & 0.415$\pm$0.012  & \tugreen & \better{0.945$\pm$0.015} & 0.620$\pm$0.020  &\tugreen & \better{0.993$\pm$0.045} & 0.265$\pm$0.015  &\tugreen & \better{0.514$\pm$0.529} &  0.805$\pm$0.012 &\tugreen   \\\hline
\domainname{grid} & \better{0.390$\pm$0.002} & 0.400$\pm$0.003  & \tugreen  & \better{0.975$\pm$0.009} &  0.165$\pm$0.037 &\tugreen & \better{0.995$\pm$0.077} & 0.042$\pm$0.012 &\tugreen & \better{0.017$\pm$0.015} &  0.065$\pm$0.002 & \tugreen \\\hline
\domainname{rand} & \better{0.286$\pm$0.003} & 0.414$\pm$0.003 & \tugreen & \better{0.961$\pm$0.012} &  0.317$\pm$0.033 &\tugreen & \better{0.979$\pm$0.023} &  0.156$\pm$0.018 &\tugreen & \better{0.014$\pm$0.013} & 0.224$\pm$0.019 & \tugreen \\\hline
\domainname{wred} & \better{0.980$\pm$0.032} & 1.105$\pm$0.042 & \tugreen & \better{0.962$\pm$0.028} &  0.913$\pm$0.013 & 0.04 & \better{0.961$\pm$0.026} & 0.821$\pm$0.041  &\tugreen & \better{0.459$\pm$0.031} &  0.537$\pm$0.007 &\tugreen \\\hline
        \domainname{wwhi} & \better{1.064$\pm$0.003} & 1.181$\pm$0.019 & \tugreen & \better{0.963$\pm$0.002} & 0.938$\pm$0.008  &\tugreen & \better{0.989$\pm$0.017} & 0.907$\pm$0.026  & \tugreen & \better{0.492}$\pm$0.037 &  0.527$\pm$0.028 &\tugreen \\\Xhline{2pt}
        & \multicolumn{1}{c}{us (\geot)} & \multicolumn{1}{c}{VC-D} & pval & \multicolumn{1}{c}{us (\geot)} & \multicolumn{1}{c}{VC-D} & pval & \multicolumn{1}{c}{us (\geot)} & \multicolumn{1}{c}{VC-D}  & pval & \multicolumn{1}{c}{us (\geot)} & \multicolumn{1}{c}{VC-D} & pval \\\hline
\domainname{ring} & \better{0.285$\pm$0.008} & 0.390$\pm$0.011 & \tugreen & \better{0.968$\pm$0.010} & 0.331$\pm$0.067  &\tugreen & \better{1.031$\pm$0.049} & 0.122$\pm$0.028  &\tugreen & \better{0.070$\pm$0.030} &  0.319$\pm$0.037 & \tugreen \\\hline
\domainname{circ} & \better{0.351$\pm$0.005} & 0.411$\pm$0.004  & \tugreen & \better{0.945$\pm$0.015} &  0.649$\pm$0.055 &\tugreen & \better{0.993$\pm$0.045} & 0.269$\pm$0.026  &\tugreen & \better{0.514$\pm$0.529} &  0.813$\pm$0.018 &\tugreen   \\\hline
\domainname{grid} & \better{0.390$\pm$0.002} &  0.398$\pm$0.002 & \tugreen  & \better{0.975$\pm$0.009} & 0.162$\pm$0.034  &\tugreen & \better{0.995$\pm$0.077} & 0.043$\pm$0.009 &\tugreen & \better{0.017$\pm$0.015} & 0.064$\pm$0.000 & \tugreen \\\hline
\domainname{rand} & \better{0.286$\pm$0.003} & 0.414$\pm$0.003 & \tugreen & \better{0.961$\pm$0.012} &  0.312$\pm$0.040 &\tugreen & \better{0.979$\pm$0.023} & 0.149$\pm$0.018  &\tugreen & \better{0.014$\pm$0.013} & 0.225$\pm$0.017 & \tugreen \\\hline
\domainname{wred} & \better{0.980$\pm$0.032} & 1.383$\pm$0.037 & \tugreen & \better{0.962$\pm$0.028} & 0.868$\pm$0.042  & 0.02 & \better{0.961$\pm$0.026} & 0.738$\pm$0.030  &\tugreen & \better{0.459$\pm$0.031} &  0.587$\pm$0.019 &\tugreen \\\hline
        \domainname{wwhi} & \better{1.064$\pm$0.003} & 1.484$\pm$0.056 & \tugreen & \better{0.963$\pm$0.002} & 0.876$\pm$0.027  &\tugreen & \better{0.989$\pm$0.017} & 0.891$\pm$0.010  & \tugreen & \better{0.492$\pm$0.037} &  0.608$\pm$0.037 &\tugreen \\\Xhline{2pt}
         & \multicolumn{1}{c}{us (\geot)} & \multicolumn{1}{c}{VC-R} & pval & \multicolumn{1}{c}{us (\geot)} & \multicolumn{1}{c}{VC-R} & pval & \multicolumn{1}{c}{us (\geot)} & \multicolumn{1}{c}{VC-R}  & pval & \multicolumn{1}{c}{us (\geot)} & \multicolumn{1}{c}{VC-R} & pval \\\hline
\domainname{ring} & \better{0.285$\pm$0.008} & 0.388$\pm$0.004 & \tugreen & \better{0.968$\pm$0.010} &  0.331$\pm$0.065 &\tugreen & \better{1.031$\pm$0.049} &  0.124$\pm$0.022 &\tugreen & \better{0.070$\pm$0.030} & 0.322$\pm$0.023  & \tugreen \\\hline
\domainname{circ} & \better{0.351$\pm$0.005} & 0.402$\pm$0.003  & \tugreen & \better{0.945$\pm$0.015} & 0.664$\pm$0.017  &\tugreen & \better{0.993$\pm$0.045} & 0.274$\pm$0.006  &\tugreen & \better{0.514$\pm$0.529} &  0.806$\pm$0.022 &\tugreen   \\\hline
\domainname{grid} & \better{0.390$\pm$0.002} & 0.399$\pm$0.003  & \tugreen  & \better{0.975$\pm$0.009} & 0.153$\pm$0.032  &\tugreen & \better{0.995$\pm$0.077} & 0.038$\pm$0.006 &\tugreen & \better{0.017$\pm$0.015} & 0.065$\pm$0.001  & \tugreen \\\hline
\domainname{rand} & \better{0.286$\pm$0.003} & 0.417$\pm$0.009 & \tugreen & \better{0.961$\pm$0.012} & 0.315$\pm$0.023  &\tugreen & \better{0.979$\pm$0.023} &  0.145$\pm$0.021 &\tugreen & \better{0.014$\pm$0.013} & 0.229$\pm$0.022 & \tugreen \\\hline
\domainname{wred} & \better{0.980$\pm$0.032} & 1.365$\pm$0.098 & \tugreen & \better{0.962$\pm$0.028} & 0.839$\pm$0.027  &\tugreen & \better{0.961$\pm$0.026} & 0.716$\pm$0.072  &\tugreen & \better{0.459$\pm$0.031} &  0.596$\pm$0.047 &\tugreen \\\hline
\domainname{wwhi} & \better{1.064$\pm$0.003} & 1.353$\pm$0.038 & \tugreen & \better{0.963$\pm$0.002} & 0.747$\pm$0.014  &\tugreen & \better{0.989$\pm$0.017} & 0.942$\pm$0.029  & 0.02 & \better{0.492$\pm$0.037} &  0.729$\pm$0.024 &\tugreen \\\Xhline{2pt}  
\end{tabular} 
  }
 \bignegspace
 \caption{\textsc{lifelike}: comparison of Generative Forests (us, \geot), with $T=500$ trees, $J=$2 000 total splits to contenders: Adversarial Random Forests (ARF, 200 trees \cite{wbkwAR}), CT-GAN (CT, 1 000 epochs \cite{xscvMT}), Forest Flows (FF, 50 trees \cite{jmfkGA}) and Vine copulas autoencoders (VCAE, VC-* where *=G for Gaussian, D for Direct and R for Regular \cite{tavCA}). Metrics used are regularized OT (Sinkhorn), Coverage, Density and F1 measure. Values shown = average over the 5-folds $\pm$ std dev. . The best average for us vs contender is shown with a star "$\star$". Convention for $p$-values: computed using a Student pared $t$-test; if $p<0.01$, value is replaced by \tugreen (we beat the contender) or \tdred (the contender beats us). See text for details. }
    \label{tab:gen-us-vs-all-big}
 \negspace
\end{table*}

Our code is provided and commented in \supplement, Section \ref{sec-algos}. The main setting of our experiments is realistic data generation ('\textsc{lifelike}'), but we have also tested our method for missing data imputation ('\textsc{impute}') and density estimation ('\textsc{density}'): for this reason, we have selected a broad panel of state of the art approaches to test against, relying on models as diverse as (or mixing elements of) trees, neural nets, kernels or graphical models, with \mice~\cite{vgMM}, adversarial random forests (ARFs \cite{wbkwAR}), CT-GANs \cite{xscvMT}, Forest Flows \cite{jmfkGA}, Vine copulas auto-encoders (VCAE, \cite{tavCA}) and Kernel density estimation (KDE, \cite{cATOK,pOE}). All algorithms \textit{not} using neural networks were ran on a low-end CPU laptop -- this was purposely done for our technique. Neural network-based algorithms are run on a bigger machine (technical specs in Appendix, Section \ref{sec-algos}). We carried out experiments on a total of 21 datasets, from UCI \cite{dgUM}, Kaggle, OpenML, the Stanford Open Policing Project, or simulated. All are presented in \supplement, Section \ref{sec-doms}. We summarize results that are presented \textit{in extenso} in \supplement. Before starting, we complete the 2D heatmap of Table \ref{circgauss-intro} by another one showing our models can also learn deterministic dependences in real-world domains (Table \ref{tab:gen-full-kc1}). The Appendix also provides an example experiment on interpreting our models for a sensitive domain (in Section \ref{sec-interpreting}).\\
\noindent \textbf{Generation capabilities of our models: \textsc{lifelike}} The objective of the experiment is to evaluate whether a generative model is able to create "realistic" data. The evaluation pipeline is simple: we create for each domain a 5-fold stratified experiment. Evaluation is done via four metrics: a regularized optimal transport ("Sinkhorn", \cite{cSD}), coverage and density \cite{noucyRF} and the F1 score \cite{jmfkGA}. All metrics are obtained after generating a sample of the same size as the test fold. Sinkhorn evaluates an optimal transport distance between generated and real and F1 estimates the error of a classifier (a 5-nearest neighbor) to distinguish generated vs real (smaller is better for both). Coverage and density are refinements of precision and recall for generative models (the higher, the better). Due to size constraint, we provide results on one set of parameters for "medium-sized" generative forests with $T=500$ trees, $J=$2 000 total splits (Table \ref{tab:gen-us-vs-all-big}), thus learning very small trees with an average 4 splits per tree. In the Appendix, we provide additional results on even smaller models ($T=200$, $J=500$) and additional results on one metric (Sinkhorn) against contenders with various parameterizations on more domains. In Table \ref{tab:gen-us-vs-all-big}, contenders are parameterized as follows: ARFs learn sets of 200 trees. Since each tree has to be separately a good generative model, we end up with big models in general. CT-GANs are trained for 1 000 epochs. Forest Flows and VCAEs are run with otherwise default parameters. Table \ref{tab:gen-us-vs-all-big} contains just a subset of all our domains, furthermore not the same for each contender. This benchmark was crafted to compare against Forest Flows, our best \textsc{lifelike} contender but a method requiring to re-encode categorical variables. Since our method and ARFs process data natively, we selected only domains with numerical or binary variables (8 domains). On some of these domains, CT-GANs or VCAE crashed on some folds so we do not provide the results for the related domains. Globally, Generative Forests are very competitive against each contender, significantly beating all VCAEs on all metrics and CT-GANs on coverage and density. Tree-based contender methods appear to perform substantially better than neural networks, with FFs performing better than ARFs. Ultimately, our models compare favorably or very favorably against both FFs and ARFs, while on average being much smaller -- for example, each FF model can contain up to 6 750 total splits. As our experiments demonstrate (Appendix, Table \ref{tab:gen-us-vs-all-small}), learning much smaller generative forests with 500 total splits can still buy an advantage in most cases, a notable counterexample being density for FFs. From the memory standpoint, our code (Java) is not optimized yet managed to crunch domains of up to $m \times d \approx$1.5 M while always using less than 1.5Gb memory to train our biggest models. Finally, it clearly emerges from our experiments that there is a domain-dependent "ideal" size ($T,J$) for the best generative forests. "Guessing" the right size is a prospect for future work.\\
\noindent \textbf{Missing data imputation with our models: \textsc{impute}}  We compared our method against a powerful contender, \mice~ \cite{vgMM}, which relies on round-robin prediction of missing values using supervised models. We optimized \mice~by choosing as supervised models trees (\textsc{cart}) and random forests (\textsc{rf}s, we increased the number of trees to 100 for better results). We create a 5-fold experiment; in each fold, we remove a fixed $\%$ of observed values (data missing completely at random, MCAR). We then use the data \textit{with missing values} as input to \mice~or us to learn models which are then used to predict the missing values. We compute a per-feature discrepancy, the average error probability (\textbf{perr}, for categorical variables), and the root mean square error (\textbf{rmse}, numerical variables). We also compute one of the simplest baselines, which consists in imputing each variable from its empirical \texttt{marginal}.\\
\noindent \textit{Results summary} Table \ref{tab:mdi-abaanalcat} puts the spotlight on two domains. The \supplement~provides many more plots. From all results, we conclude that generative forests can compete or beat \mice.\textsc{rf}~while using hundred times less trees (eventually using just stumps, when $J=T=20$). From all our results, we also conclude that there is a risk of overfitting if $J$ and / or $T$ are too large. This could motivate further work on pruning generative models. For \geot, an unexpected pattern is that the pattern "small number of small trees works well" can work on real-world domains as well (see also \supplement), thereby demonstrating that the nice result displayed in Table \ref{circgauss-intro} generalises to bigger domains.\\
  \noindent \textbf{Density estimation: \textsc{density}} We compared \geot~vs  kernel density estimation (KDE) \cite{cATOK,lrNE,pOE}. The experimental setting is the same as for \textsc{lifelike}: in each of the 5-fold stratified cross validation fold, we use the training sample to learn a model (with KDE or \topdownGT) which is then used to predict the observation's density in the the test sample. The higher the density, the better the model. The \geot~models we consider are the same as in \textsc{lifelike} ($T=500$, $J=$2 000). \\
  \noindent \textit{Results summary} Table (\ref{tab:density}, right) summarizes the results (plots: \supplement, Section \ref{sec-full-comp-kde}). The leftmost column ("\geot~$\gg$ KDE") counts domains for which there exists an iteration $j$ for \geot~after which we tend to be statistically better (and never statistically worse) than KDE. The rightmost column ("KDE $\gg$ \geot") counts domains for which KDE is always statistically better than \geot. The last type of plots appears to be those for which there is no such statistical difference (central column). The results support the conclusion that \geot~can also be good models to carry out density estimation.

 \bignegspace
 \bignegspace
\bignegspace
\bignegspace
\section{Conclusion}\label{sec-conc}
 \bignegspace

 In this paper, we have introduced new generative models based on ensemble of trees and a top-down boosting-compliant training algorithm which casts training in a supervised 2-classes framework. Those generative forests allow not just efficient data generation: they also allow efficient missing data imputation and density estimation. Experiments have been carried out on all three problems against a variety of contenders, demonstrating the ability of small generative forests to perform well against potentially much more complex contenders. A responsible use of such models necessarily includes restriction to tabular data (Section \ref{sec-intro}). Experiments clearly show that there is a domain-dependent "best fit" size for our models; estimating it can be done via cross-validation but an important question, left for future work, is how to directly "guess" the right model size. A pruning algorithm with good generalization seems like a very promising direction.

\section*{Acknowledgments}

The authors thank the reviewers for engaging discussions and many suggestions that helped to improve the paper's content. Code availability: see \url{https://richardnock.github.io/}


\bibliographystyle{plain}
\bibliography{bibgen}


\clearpage
\pagebreak
\appendix

\newcommand{\eogt}{\textsc{eogt}}
\newcommand{\eogtp}{\eogt.\textsc{p}}

\counterwithin{theorem}{section}
\counterwithin{figure}{section}
\counterwithin{table}{section}

\renewcommand\thesection{\Roman{section}}
\renewcommand\thesubsection{\thesection.\arabic{subsection}}
\renewcommand\thesubsubsection{\thesection.\thesubsection.\arabic{subsubsection}}

\renewcommand*{\thetheorem}{\Alph{theorem}}
\renewcommand*{\thelemma}{\Alph{lemma}}
\renewcommand*{\thecorollary}{\Alph{corollary}}

\renewcommand{\thetable}{A\arabic{table}}

\begin{center}
\Huge{Appendix}
\end{center}

To
differentiate with the numberings in the main file, the numbering of
Theorems, etc. is letter-based (A, B, ...).

\section*{Table of contents}

\noindent \textbf{Related work} \hrulefill Pg \pageref{add-relwork}\\

\noindent \textbf{Additional content} \hrulefill Pg \pageref{add-cont}\\

\noindent \textbf{Supplementary material on proofs} \hrulefill Pg
\pageref{sec-sup-pro}\\
\noindent $\hookrightarrow$ Proof of Lemma \ref{leminv} \hrulefill Pg \pageref{proof-leminv}\\
\noindent $\hookrightarrow$ Proof of Theorem \ref{th-boost} \hrulefill Pg \pageref{proof-th-boost}\\

\noindent \textbf{Simpler models: ensembles of generative trees} \hrulefill Pg \pageref{sec-modtwo}\\

\noindent \textbf{Supplementary material on experiments} \hrulefill Pg
\pageref{app-exps}\\
\noindent $\hookrightarrow$ Domains \hrulefill Pg \pageref{sec-doms}\\
\noindent $\hookrightarrow$ Algorithms configuration and choice of parameters \hrulefill Pg \pageref{sec-algos}\\
\noindent $\hookrightarrow$ Interpreting our models: '\textsc{scrutinize}' \hrulefill Pg \pageref{sec-interpreting}\\
\noindent $\hookrightarrow$ More examples of Table \ref{circgauss-intro} (\mainfile) \hrulefill Pg \pageref{sec-tables}\\
\noindent $\hookrightarrow$ The generative forest of Table \ref{circgauss-intro} (\mainfile) developed further \hrulefill Pg \pageref{sec-gf}\\
\noindent $\hookrightarrow$ Full comparisons with \mice~on missing data imputation \hrulefill Pg \pageref{sec-full-comp-mdi}\\
\noindent $\hookrightarrow$ Experiment \textsc{lifelike} \textit{in extenso} \hrulefill Pg \pageref{sec-tables-lifelike}\\
\noindent $\hookrightarrow$ Comparison with "the optimal generator": \textsc{gen-discrim} \hrulefill Pg \pageref{sec-full-comp-ideal}

\newpage

\section{Related work}\label{add-relwork}

The typical Machine Learning (ML) problem usually contains at least three parts: (i) a training algorithm minimizes (ii) a loss function to output an object whose key part is (iii) a model. The main problem we are interested in is data generation as generally captured by "Generative AI".  The type of data we are interested in still represents one of the most prevalent form of data: tabular data \cite{cmm+NF}. When it comes to tabular data, a singular phenomenon of the data world can be observed: there is a fierce competition on part (iii) above, the \textit{models}. When data has other forms, like images, the ML community has generally converged to a broad idea of what the best models look like at a high-level for many generic tasks: neural networks\footnote{Whether such a perception is caused by the models themselves or the dazzling amount of optimisation that has progressively wrapped the models, as advocated for the loss in \cite{lmvpkEP}, is not the focus of our paper.}. Tabular data offers no such consensus yet, even on well-defined tasks like supervised learning \cite{govWD}. In fact, even on such "simple tasks" the consensus is rather that there is no such consensus \cite{mkvprgwWD}. It is an understatement to state that in the broader context of all tasks of interest to us, a sense of a truly intense competition emerges, whose "gradient" clearly points towards simultaneously the most complex / expressive / tractable models, as shown in \cite[Slides 27, 53]{vpcPC}. One can thus end up finding models based on trees \cite{cpdJI,ngGT,wbkwAR}, neural networks \cite{gpmxwocbGA,geBG,kwAE,xscvMT}, probabilistic circuits \cite{cvvPC,vpcPC,spdSP}, kernel methods \cite{cATOK,pOE}, graphical models \cite{tavCA} (among others: note that some are in fact hybrid models).

So the tabular data world is blessed with a variety of possible models to solve problems like the ones we are interested in, \textit{but} -- and this is another singular phenomenon of the tabular data world --, getting the best solutions is not necessarily a matter of competing on size or compute. In fact, it can be the opposite: striving for model simplicity or (non exclusive) lightweight tuning can substantially pay off \cite{mkvprgwWD}. In relative terms, this phenomenon is not new in the tabular data world: it has been known for more than two decades \cite{hVS}. That it has endured all major revolutions in ML points to the fact that lasting solutions can be conveniently addressing \textit{all} three parts (i--iii) above at once on models, algorithms and losses.

From this standpoint, the closest approaches to ours are \cite{wbkwAR} and \cite{ngGT}, first and foremost because the models include trees with a stochastic activation of edges to pick leaves, and a leaf-dependent data generation process. While \cite{ngGT} learn a single tree, \cite{wbkwAR} use a way to generate data from a set of trees -- called an adversarial random forest -- which is simple: sample a tree, and then sample an observation from the tree. Hence, the distribution is a convex combination of the trees' density. This is simple but it suffers several drawbacks: (i) each tree has to be accurate enough and thus big enough to model "tricky" data for tree-based models (tricky can be low-dimensional, see the 2D data of Table \ref{circgauss-intro}, main file); (ii) if leaves' samplers are simple, which is the case for \cite{ngGT} and our approach, it makes it tricky to learn sets of simple models, such as when trees are stumps (we do not have this issue, see Table \ref{circgauss-intro}). In our case, while our models include sets of trees, generating one observation makes use of leaves in \textit{all} trees instead of just one as in \cite{wbkwAR} (Figure \ref{fig:generation-sketch}, main file). We note that the primary goal of that latter work is in fact not to generate data \cite[Section 4]{wbkwAR}. 

Theoretical results that are relevant to data generation are in general scarce compared to the flurry of existing methods if we omit the independent convergence rates of the toolbox often used, such as for (stochastic) gradient descent. Specific convergence results are given in \cite{wbkwAR}, but they are relevant to statistical consistency (infinite sample) and they also rely on assumptions that are not realistic for real world domains, such as Lipschitz continuity of the target density, with second derivative continuous, square integrable and monotonic. The assumption made on models, that splitting probabilities on trees is lowerbounded by a constant, is also impeding to model real world data.

  In the generative trees of \cite{ngGT}, leaf generation is the simplest possible: it is uniform. This requires big trees to model real-world or tricky data. On the algorithms side, \cite{ngGT} introduce two training algorithms in the generative adversarial networks (GAN) framework \cite{gpmxwocbGA}, thus involving the generator to train but also a discriminator, which is a decision tree in \cite{ngGT}. The GAN framework is very convenient to tackle (ii) above in the context of our tasks because it allows to tie using first principles the problem of learning a density or a sampler and that of training a model (a "discriminator") to distinguish fakes from real, model that parameterizes the loss optimized. An issue with neural networks is that this parameterization has an uncontrollable slack unless the discriminator is extremely complex \cite{nctFG}. The advantage of using trees as in \cite{ngGT} is that for such classifiers, the slack disappears because the models are calibrated, so there is a tight link between training the generator and the discriminator. \cite{ngGT} go further, showing that one can replace the adversarial training by a \textit{copycat} training, involving copying parts of the discriminator in the generator to speed-up training (also discussed in \cite{hTF} for neural nets), with strong rates on training in the original boosting model. There is however a limitation in the convergence analysis of \cite{ngGT} as losses have to be symmetric, a property which is not desirable for data generation since it ties the misclassification costs of real and fakes with no argument to do so in general.

  Our paper lifts the whole setting of \cite{ngGT} to models that can fit more complex densities using simpler models (Table \ref{circgauss-intro}), using sets of trees that can be much smaller than those of \cite{wbkwAR} (Figure \ref{fig:generation-sketch}); training such models is achieved by merging the two steps of copycat training into a single one where only the generator is trained, furthermore keeping strong rates on training via the original boosting model, all this while getting rid of the undesirable symmetry assumption of \cite{ngGT} for the loss at hand.

\section{Additional content}\label{add-cont}

In this additional content, we provide the three ways to pick the trees in a Generative Forest~to generate one observation (sequential, concurrent, randomized), and then give a proof that the optimal splitting threshold on a continuous variable when training a generative forest using \topdownGT~is always an observed value if there is one tree, but can be another value if there are more (thus highlighting some technical difficulties of one wants to stick to the optimal choice of splitting).

\subsection{Sequentially choosing trees in a \geot~for the generation of observations}

\begin{algorithm}[t]
\caption{\iterativeupdatesupport($\{\Upsilon_t\}_{t=1}^{T}$)}\label{alg-iterativeupdatesupport}
\begin{algorithmic}
    \STATE  \textbf{Input:} Trees $\{\Upsilon_t\}_{t=1}^{T}$ of a \geot; 
    \STATE  \textbf{Output:} sampling support $\samplesupport$ for one observation; 
    \STATE  Step 1 : $\samplesupport \leftarrow \mathcal{X}$;
    \STATE  Step 2 : \initsampling($\{\Upsilon_t\}_{t=1}^{T}$);
    \STATE  Step 3 : \textbf{for} $t\in [T]$
  \STATE  \hspace{0.5cm} Step 2.1 : \textbf{while} !\texttt{$\Upsilon_t$.done}
  \STATE  \hspace{1cm} Step 2.1.1 : \starupdate($\Upsilon_t, \samplesupport, \colorr$);
    \STATE  \textbf{return} $\samplesupport$; 
\end{algorithmic}
\end{algorithm}

\subsection{Concurrent generation of observations}

\begin{algorithm}[t]
\caption{\concurrentupdatesupport}\label{alg-concurrentupdatesupport}
\begin{algorithmic}
  \STATE  Step 1 : \textbf{while} !\texttt{done}
  \STATE  \hspace{0.5cm} Step 1.1 : $\mathsf{P}$[\texttt{accessible}] // locks $\samplesupport$
  \STATE  \hspace{0.5cm} Step 1.2 : \starupdate(\texttt{this},$\samplesupport$, $\colorr$);
 \STATE  \hspace{0.5cm} Step 1.3 : $\mathsf{V}$[\texttt{accessible}] // unlocks $\samplesupport$
\end{algorithmic}
\end{algorithm}

We provide a concurrent generation using Algorithm
\ref{alg-concurrentupdatesupport}, which differs from Algorithm \ref{alg-iterativeupdatesupport} (main file). In concurrent generation, each tree runs concurrently algorithm \updatesupport~(hence the use of the Java-style \texttt{this} handler), with an additional global variables (in addition to $\samplesupport$, initialized to $\mathcal{X}$): a Boolean semaphore \texttt{accessible} implementing a lock, whose value 1 means $\mathcal{C}$ is available for an update (and otherwise it is locked by a tree in Steps 1.2/1.3 in the set of trees of the \geot). We assume that $\initsampling$ has been run beforehand (eventually locally).

\subsection{Randomized generation of observations}
\begin{algorithm}[t]
\caption{\randomizedupdatesupport}\label{alg-randomizedupdatesupport}
\begin{algorithmic}
  \STATE  Step 1 : \initsampling($\{\Upsilon_t\}_{t=1}^{T}$); $\mathbb{I} \leftarrow \{1, 2, ..., T\}$;
  \STATE  Step 2 : \textbf{while} ($\mathbb{I}$ != $\emptyset$)
  \STATE  \hspace{0.5cm} Step 2.1 : $i \leftarrow \textsc{random}(\mathbb{I})$;
  \STATE  \hspace{0.5cm} Step 2.2 : \starupdate($\tree_i$,$\samplesupport$, $\colorr$);
 \STATE  \hspace{0.5cm} Step 2.3 : \textbf{if} ($\tree_i.\texttt{done}$) \textbf{then} $\mathbb{I} \leftarrow \mathbb{I} \backslash \{i\}$;
\end{algorithmic}
\end{algorithm}

We provide a general randomized choice of the sequence of trees for
generation, in Algorithm \ref{alg-randomizedupdatesupport}.

\section{Supplementary material on proofs}\label{sec-sup-pro}


\subsection{Proof of Lemma \ref{leminv}}\label{proof-leminv}

Given sequence $\ve{v}$ of dimension $\mathrm{dim}(\ve{v})$, denote $\{\mathcal{C}_j\}_{j \in [1+\mathrm{dim}(\ve{v})]}$ the sequence of subsets of the domain appearing in the parameters of \updatesupport~through sequence $\ve{v}$, to which we add a last element, $\samplesupport$ (and its first element is $\mathcal{X}$). If we let $\mathcal{X}_{j}$ denote the support of the corresponding star node whose Bernoulli event is triggered at index $j$ in sequence $\ve{v}$ (for example, $\mathcal{X}_{1} = \mathcal{X}$), then we easily get
\begin{eqnarray}
\mathcal{C}_j & \subseteq & \mathcal{X}_{j}, \forall j \in [\mathrm{dim}(\ve{v})],
\end{eqnarray}
indicating the Bernoulli probability in Step 1.1 of \starupdate~is always defined. We then compute
\begin{eqnarray}
  p_{{\color{blue}{\meas{G}}}}\left[\samplesupport | \ve{v}\right] & = & p_{{\color{blue}{\meas{G}}}}(\cap_j  \mathcal{C}_j)\\
  & = & \prod_{j=1}^N p_{{\color{blue}{\meas{G}}}}\left[\mathcal{C}_{j+1} | \mathcal{C}_j\right] \label{eq2}\\
                                                        & = & \prod_{j=1}^N \frac{p_{{\color{red}{\meas{R}}}}\left[\mathcal{C}_{i+1}\right]}{p_{{\color{red}{\meas{R}}}}\left[\mathcal{C}_{i}\right]} \label{eq3}\\
                                                        & = & \frac{p_{{\color{red}{\meas{R}}}}\left[\mathcal{C}_{N+1}\right]}{p_{{\color{red}{\meas{R}}}}\left[\mathcal{C}_{0}\right]}\label{eq4}\\
                                                        & = & \frac{p_{{\color{red}{\meas{R}}}}\left[\samplesupport\right]}{p_{{\color{red}{\meas{R}}}}\left[\mathcal{X}\right]}\nonumber\\
  & = & p_{{\color{red}{\meas{R}}}}\left[\samplesupport\right].
  \end{eqnarray}
  \eqref{eq2} holds because updating the support generation is Markovian in a \geot, \eqref{eq3} holds because of Step 3.2 and \eqref{eq4} is a simplification through cancelling terms in products.

\subsection{Proof of Theorem  \ref{th-boost}}\label{proof-th-boost}

Notations: we iteratively learn generators $\colorg_0, \colorg_1, ..., \colorg_J$ where $\colorg_0$ is just a root. We let $\mathcal{P}(\colorg_j)$ denote the partition induced by the set of trees of $\colorg_j$, recalling that each element is the (non-empty) intersection of the support of a set of leaves, one for each tree (for example, $\mathcal{P}(\colorg_0) = \{\mathcal{X}\}$). The criterion minimized to build $\colorg_{j+1}$ from $\colorg_j$ is
\begin{eqnarray}
\poprisk(\colorg_{j}) & \defeq & \sum_{\mathcal{C} \in \mathcal{P}(\colorg_j)} p_{\colorm{\scriptsize}}[\mathcal{C}] \cdot \poibayesrisk\left(\frac{\prior p_{\meas{\colorr}}[\mathcal{C}]}{p_{\colorm{\scriptsize}}[\mathcal{C}]}\right).
\end{eqnarray}
The proof entails fundamental notions on loss functions, models and boosting. We start with loss functions. A key function we use is
\begin{eqnarray*}
  \gprior (t) \defeq  (\prior t + 1 - \prior)\cdot\poibayesrisk\left(\frac{\prior t}{\prior t + 1 - \prior}\right), \forall t\in \mathbb{R}_+,\label{defgprior}
\end{eqnarray*}
which is concave \cite[Appendix A.3]{rwID}. 
\begin{lemma}\label{lem-losses}
\begin{eqnarray}
   \likelihoodratiorisk_{\ell}\left(\colorr, \colorg_{J-1}\right) -  \likelihoodratiorisk_{\ell}\left(\colorr, \colorg_J\right) & = & \poprisk(\colorg_{J-1}) - \poprisk(\colorg_{J}).
\end{eqnarray}
\end{lemma}
\begin{proof}
  We make use of the following important fact about \geot s:
  \begin{enumerate}
\item [\textbf{(F1)}] For any $\mathcal{X}' \in \mathcal{P}_J$, $\int_{\mathcal{X}'} \dmeas{\colorg}_J = \int_{\mathcal{X}'} \dmeas{\colorr}$.
    \end{enumerate}
  We have from the proof of \cite[Lemma 5.6]{ngGT},
\begin{eqnarray}
   \likelihoodratiorisk_{\ell}\left(\colorr, \colorg_J\right) & = & \sum_{\mathcal{X}' \in \mathcal{P}_J} \left\{\gprior\left(\expect_{\coloru|\mathcal{X}'}\left[\frac{\dmeas{\colorg}_J}{\dmeas{\coloru}}\right]\right) - \expect_{\coloru|\mathcal{X}'}\left[\gprior\left(\frac{\dmeas{\colorr}}{\dmeas{\coloru}}\right)\right]\right\}.\label{eqDIFF}
\end{eqnarray}
We note that \textbf{(F1)} implies \eqref{eqDIFF} is just a sum of slacks of Jensen's inequality. We observe
\begin{eqnarray*}
  \lefteqn{\likelihoodratiorisk_{\ell}\left(\colorr, \colorg_{J-1}\right) - \likelihoodratiorisk_{\ell}\left(\colorr, \colorg_J\right)}\\
  & = & \sum_{\mathcal{X}' \in \mathcal{P}_{J-1}} \left\{\gprior\left(\expect_{\coloru|\mathcal{X}'}\left[\frac{\dmeas{\colorg}_{J-1}}{\dmeas{\coloru}}\right]\right) - \expect_{\coloru|\mathcal{X}'}\left[\gprior\left(\frac{\dmeas{\colorr}}{\dmeas{\coloru}}\right)\right]\right\} \\
  & & - \sum_{\mathcal{X}' \in \mathcal{P}_J} \left\{\gprior\left(\expect_{\coloru|\mathcal{X}'}\left[\frac{\dmeas{\colorg}_J}{\dmeas{\coloru}}\right]\right) - \expect_{\coloru|\mathcal{X}'}\left[\gprior\left(\frac{\dmeas{\colorr}}{\dmeas{\coloru}}\right)\right]\right\}\\
  & = & \sum_{(\xxset{\texttt{t}}, \xxset{\texttt{f}}) \in \mathcal{P}^{\mbox{\tiny{split}}}_{J}} \left\{\gprior\left(\expect_{\coloru|\xxset{\texttt{t}}\cup \xxset{\texttt{f}}}\left[\frac{\dmeas{\colorg}_{J-1}}{\dmeas{\coloru}}\right]\right) - \sum_{\texttt{v} \in \{\texttt{t}, \texttt{f}\}} \gprior\left(\expect_{\coloru|\xxset{\texttt{v}}}\left[\frac{\dmeas{\colorg}_{J}}{\dmeas{\coloru}}\right]\right)\right\},
\end{eqnarray*}
where $\mathcal{P}^{\mbox{\tiny{split}}}_{J}$ contains all couples $(\xxset{\texttt{t}}, \xxset{\texttt{f}})$ such that $\xxset{\texttt{t}}, \xxset{\texttt{f}} \in \mathcal{P}_J$
and $\xxset{\texttt{t}} \cup \xxset{\texttt{f}} \in \mathcal{P}_{J-1}$. These unions were subsets of the partition of $\mathcal{X}$ induced by the set of trees of $\colorg_{J-1}$ \textit{and} that were cut by the predicate $\texttt{p}$ put at the leaf $\leaf_*$ that created $\colorg_{J}$ from $\colorg_{J-1}$. To save space, denote $\xxset{\texttt{tf}} \defeq \xxset{\texttt{t}}\cup \xxset{\texttt{f}}$.
\begin{eqnarray}
  \lefteqn{\gprior\left(\expect_{\coloru|\xxset{\texttt{tf}}}\left[\frac{\dmeas{\colorg}_{J-1}}{\dmeas{\coloru}}\right]\right) - \sum_{\texttt{v} \in \{\texttt{t}, \texttt{f}\}} \gprior\left(\expect_{\coloru|\xxset{\texttt{v}}}\left[\frac{\dmeas{\colorg}_{J}}{\dmeas{\coloru}}\right]\right)}\nonumber\\
  & = & \left(\prior \int_{\xxset{\texttt{tf}}} \dmeas{\colorg}_{J-1} + (1 - \prior)  \int_{\xxset{\texttt{tf}}} \dmeas{\coloru}\right) \cdot\poibayesrisk\left(\frac{\prior \int_{\xxset{\texttt{tf}}} \dmeas{\colorg}_{J-1}}{\prior \int_{\xxset{\texttt{tf}}} \dmeas{\colorg}_{J-1} + (1 - \prior)  \int_{\xxset{\texttt{tf}}} \dmeas{\coloru}}\right)\nonumber\\
  & & - \left(\prior \int_{\xxset{\texttt{t}}} \dmeas{\colorg}_{J} + (1 - \prior)  \int_{\xxset{\texttt{t}}} \dmeas{\coloru}\right) \cdot\poibayesrisk\left(\frac{\prior \int_{\xxset{\texttt{t}}} \dmeas{\colorg}_{J}}{\prior \int_{\xxset{\texttt{t}}} \dmeas{\colorg}_{J} + (1 - \prior)  \int_{\xxset{\texttt{t}}} \dmeas{\coloru}}\right)\nonumber\\
  & & - \left(\prior \int_{\xxset{\texttt{f}}} \dmeas{\colorg}_{J} + (1 - \prior)  \int_{\xxset{\texttt{f}}} \dmeas{\coloru}\right) \cdot\poibayesrisk\left(\frac{\prior \int_{\xxset{\texttt{f}}} \dmeas{\colorg}_{J}}{\prior \int_{\xxset{\texttt{f}}} \dmeas{\colorg}_{J} + (1 - \prior)  \int_{\xxset{\texttt{f}}} \dmeas{\coloru}}\right)\label{keyID}.
\end{eqnarray}
We now work on \eqref{keyID}. Using \textbf{(F1)}, we note $\int_{\xxset{\texttt{tf}}} \dmeas{\colorg}_{J-1} = \int_{\xxset{\texttt{tf}}} \dmeas{\colorr}$ since $\xxset{\texttt{tf}} \in \mathcal{P}_{J-1}$. Similarly, $\int_{\xxset{\texttt{v}}} \dmeas{\colorg}_{J} = \int_{\xxset{\texttt{v}}} \dmeas{\colorr}, \forall \texttt{v} \in \{\texttt{t}, \texttt{f}\}$, so we make appear $\colorr$ in \eqref{keyID} and get:
\begin{eqnarray}
  \lefteqn{\gprior\left(\expect_{\coloru|\xxset{\texttt{tf}}}\left[\frac{\dmeas{\colorg}_{J-1}}{\dmeas{\coloru}}\right]\right) - \sum_{\texttt{v} \in \{\texttt{t}, \texttt{f}\}} \gprior\left(\expect_{\coloru|\xxset{\texttt{v}}}\left[\frac{\dmeas{\colorg}_{J}}{\dmeas{\coloru}}\right]\right)}\nonumber\\
  & = & p(\texttt{tf}) \cdot\left\{\poibayesrisk\left(\frac{\prior \int_{\xxset{\texttt{tf}}} \dmeas{\colorr}}{p(\texttt{tf})}\right) - \frac{p(\texttt{t})}{p(\texttt{tf})} \cdot \poibayesrisk\left(\frac{\prior \int_{\xxset{\texttt{t}}} \dmeas{\colorr}}{p(\texttt{t})}\right) - \frac{p(\texttt{f})}{p(\texttt{tf})} \cdot \poibayesrisk\left(\frac{\prior \int_{\xxset{\texttt{f}}} \dmeas{\colorr}}{p(\texttt{f})}\right) \right\}\label{keyID2},
\end{eqnarray}
where we let for short
\begin{eqnarray}
p(\texttt{v}) & \defeq & \prior \int_{\xxset{\texttt{v}}} \dmeas{\colorr} + (1 - \prior)  \int_{\xxset{\texttt{v}}} \dmeas{\coloru}, \forall \texttt{v} \in \{\texttt{t}, \texttt{f}, \texttt{tf}\}.
\end{eqnarray}
We finally get
\begin{eqnarray*}
  \lefteqn{ \likelihoodratiorisk_{\ell}\left(\colorr, \colorg_{J-1}\right) -  \likelihoodratiorisk_{\ell}\left(\colorr, \colorg_J\right)}\\
  & = & \sum_{\mathcal{C} \in \mathcal{P}_{J-1}} p_{\colorm{\scriptsize}}[\mathcal{C}] \cdot \poibayesrisk\left(\frac{\prior p_{\meas{\colorr}}[\mathcal{C}]}{p_{\colorm{\scriptsize}}[\mathcal{C}]}\right)   - \sum_{\mathcal{C} \in \mathcal{P}_{J-1}} p_{\colorm{\scriptsize}}[\mathcal{C}] \cdot \sum_{\texttt{v} \in \{\texttt{t}, \texttt{f}\}} \frac{p_{\colorm{\scriptsize}}[\mathcal{C}_{\texttt{p}_{J-1}, \texttt{v}}]}{p_{\colorm{\scriptsize}}[\mathcal{C}]} \cdot \poibayesrisk\left(\frac{\prior p_{\meas{\colorr}}[\mathcal{C}_{\texttt{p}_{J-1}, \texttt{v}}]}{p_{\colorm{\scriptsize}}[\mathcal{C}_{\texttt{p}_{J-1}, \texttt{v}}]}\right)\\
  & = & \poprisk(\colorg_{J-1}) - \poprisk(\colorg_{J}),
\end{eqnarray*}
as claimed. The last identity comes from the fact that the contribution to $\poprisk(.)$ is the same outside $\mathcal{P}_{J-1}$ for both $\colorg_{J-1}$ and $\colorg_{J}$.
  \end{proof}
  We now come to models and boosting. Suppose we have $t$ trees $\{\tree_1, \tree_2, ..., \tree_t\}$ in $\colorg_j$. We want to split leaf $\leaf \in \tree_t$ into two new leaves, $\leftchild{\leaf}, \rightchild{\leaf}$. Let $\mathcal{P}_{\leaf}(\colorg_j)$ denote the subset of $\mathcal{P}(\colorg_j)$ containing only the subsets defined by intersection with the support of $\leaf$, $\mathcal{X}_\leaf$. The criterion we minimize can be reformulated as
  \begin{eqnarray*}
  \poprisk(\colorg_{j}) & = & \sum_{\leaf \in \leafset(\tree)} p_{\colorm{\scriptsize}}[\mathcal{X}_\leaf] \cdot \sum_{\mathcal{C} \in \mathcal{P}_{\leaf}(\colorg_j)} \frac{p_{\colorm{\scriptsize}}[\mathcal{C}]}{p_{\colorm{\scriptsize}}[\mathcal{X}_\leaf]} \cdot \poibayesrisk\left(\frac{\prior p_{\meas{\colorr}}[\mathcal{C}]}{p_{\colorm{\scriptsize}}[\mathcal{C}]}\right), \tree \in \{\tree_1, \tree_2, ..., \tree_t\}.
  \end{eqnarray*}
  Here, $\tree$ is any tree in the set of trees, since $\cup_{\leaf \in \leafset(\tree)} \mathcal{P}_{\leaf}(\colorg_j)$ covers the complete partition of $\mathcal{X}$ induced by $\{\tree_1, \tree_2, ..., \tree_t\}$. If we had a single tree, the inner sum would disappear since we would have $\mathcal{P}_{\leaf}(\colorg_j) = \{\mathcal{X}_\leaf\}$ and so one iteration would split one of these subsets. In our case however, with a set of trees, we still split $\mathcal{X}_\leaf$ but the impact on the reduction of $\poprisk$ can be substantially better as we may simultaneously split as many subsets as there are in $\mathcal{P}_{\leaf}(\colorg_j)$. The reduction in $\poprisk$ can be obtained by summing all reductions to which contribute each of the subsets.

  To analyze it, we make use of a reduction to the \topdowngen~algorithm of \cite{mnwRC}. We present the algorithm in Algorithm \ref{splitMF}.
\begin{algorithm}[t]
\caption{\topdowngen $(\colorr, \loss, \weaklearner, J, T)$}\label{splitMF}
\begin{algorithmic}
  \STATE  \textbf{Input:} measure $\colorr$, \spd~loss $\ell$, weak learner \weaklearner, iteration number $J\geq 1$, number of trees $T$;
  \STATE  \textbf{Output:} \pbls~(partition-linear model) $H_J$ with $T$ trees;
  \STATE  Step 1 : $\Gamma_0 \defeq \{\tree_t\}_{t=1}^{T} \leftarrow \{\mathrm{root}, \mathrm{root}, ..., \mathrm{root}\}$; $\mathcal{P}(\Gamma_{0}) \leftarrow \{\mathcal{X}\}$;
  \STATE  Step 2 : \textbf{for} $j=1, 2, ..., J$
  \STATE \hspace{1cm} Step 2.1 : pick $t \in [T]$;
  \STATE \hspace{1cm} Step 2.2 : pick $\leaf \in \leafset(\tree_t)$; 
  \STATE \hspace{1cm} Step 2.3 : $h_{j} \leftarrow \weaklearner (\mathcal{X}_{\leaf})$;
  \STATE \hspace{2.5cm} // call to the weak learner for some $h_j$ of the type $h_j \defeq 1_{\texttt{p}(.)} \cdot K_j$, 
   \STATE \hspace{2.5cm} // $K_j$ constant, $\texttt{p}(.)$ a splitting predicate (\textit{e.g.} $x_i \geq a$)
   \STATE \hspace{1cm} Step 2.4 : \textbf{for} $\mathcal{C} \in \mathcal{P}_{\leaf}(\Gamma_{j-1})$,
   \begin{eqnarray}
     \mathcal{C}^{\texttt{p}} & \leftarrow & \mathcal{C} \cap \{\ve{x} : \texttt{p}(\ve{x}) \mbox{ is \texttt{true}}\},\label{defcpi}\\
     \alpha(\mathcal{C}^{\texttt{p}}) & \leftarrow & ({-\poibayesrisk'})\left(\frac{\prior p_{\meas{\colorr}}[\mathcal{C}^{\texttt{p}}]}{p_{\colorm{\scriptsize}}[\mathcal{C}^{\texttt{p}}]}\right),\\
          \mathcal{C}^{\neg\texttt{p}} & \leftarrow & \mathcal{C} \cap \{\ve{x} : \texttt{p}(\ve{x}) \mbox{ is \texttt{false}}\},\label{defnotcpi}\\
     \alpha(\mathcal{C}^{\neg\texttt{p}}) & \leftarrow &- ({-\poibayesrisk'})\left(\frac{\prior p_{\meas{\colorr}}[\mathcal{C}^{\neg\texttt{p}}]}{p_{\colorm{\scriptsize}}[\mathcal{C}^{\neg\texttt{p}}]}\right).
   \end{eqnarray}
   \STATE \hspace{1cm} Step 2.5 : // update $\mathcal{P}(\Gamma_{.})$:
   \begin{eqnarray}
\mathcal{P}(\Gamma_{j}) & \leftarrow & \left(\mathcal{P}(\Gamma_{j-1}) \backslash \mathcal{P}_{\leaf}(\Gamma_{j-1}) \right) \cup \left(\cup_{\mathcal{C} \in \mathcal{P}_{\leaf}(\Gamma_{j-1}): \mathcal{C}^{\texttt{p}} \neq \emptyset} \mathcal{C}^{\texttt{p}}\right)\cup \left(\cup_{\mathcal{C} \in \mathcal{P}_{\leaf}(\Gamma_{j-1}): \mathcal{C}^{\neg \texttt{p}} \neq \emptyset} \mathcal{C}^{\neg \texttt{p}}\right)
   \end{eqnarray}
   \STATE \hspace{1cm} Step 2.6 : split leaf $\leaf$ at $\tree_t$ using predicate $\texttt{p}$;
        \STATE  \textbf{return} $H_J(\ve{x}) \defeq \alpha\left(\mathcal{C}(\ve{x})\right)$ with $\mathcal{C}(\ve{x}) \defeq \mathcal{C} \in \mathcal{P}(\Gamma_{J})$ such that $\ve{x}\in \mathcal{C}$;
\end{algorithmic}
\end{algorithm}
The original \topdowngen~algorithm trains \textit{partition-linear models}, \textit{i.e.} models whose output is defined by a sum of reals over a set of partitions to which the observation to be classified belongs to. Training means then both learning the organisation of partitions and the reals -- which are just the output of a weak classifier given by a weak learner, times a leveraging constant computed by \topdowngen. As in the classical boosting literature, the original \topdowngen~includes the computation and updates of \textit{weights}.

In our case, the structure of partition learned is that of a set of trees, each weak classifier $h$ is of the form
\begin{eqnarray*}
h(\ve{x}) & \defeq & K \cdot \texttt{p}(\ve{x}),
\end{eqnarray*}
where $K$ is a real and $ \texttt{p}$ is a Boolean predicate, usually doing an axis-parallel split of $\mathcal{X}$ on an observed feature, such as $\texttt{p}(\ve{x}) \equiv 1_{x_i \geq a}$. 

From now on, it will be interesting to dissociate our generator $\colorg_j$ -- which includes a model structure and an algorithm to generate data from this structure-- from its set of trees -- \textit{i.e.} its model structure --  which we denote $\Gamma_j \defeq \{\tree_t\}_{t=1}^T$. \topdowngen~learns both $\Gamma_j$ but also predictions for each possible outcome, predictions that we shall not use since the loss we are minimizing can be made to depend only on the model structure $\Gamma_j$. Given the simplicity of the weak classifiers $h$ and the particular nature of the partitions learned, the original \topdowngen~of \cite{mnwRC} can be simplified to our presentation in Algorithm \ref{splitMF}. Among others, the simplification discards the weights from the presentation of the algorithm. What we show entangles two objectives on models and loss functions as we show that
\begin{center}
\textit{\topdowngen~learns the same structure $\Gamma$ as our \topdownGT, and yields a guaranteed decrease of $\poprisk(\colorg)$,}
\end{center}
and it is achieved via the following Theorem.
\begin{figure}
  \centering
\includegraphics[trim=110bp 80bp 90bp 80bp,clip,width=0.8\columnwidth]{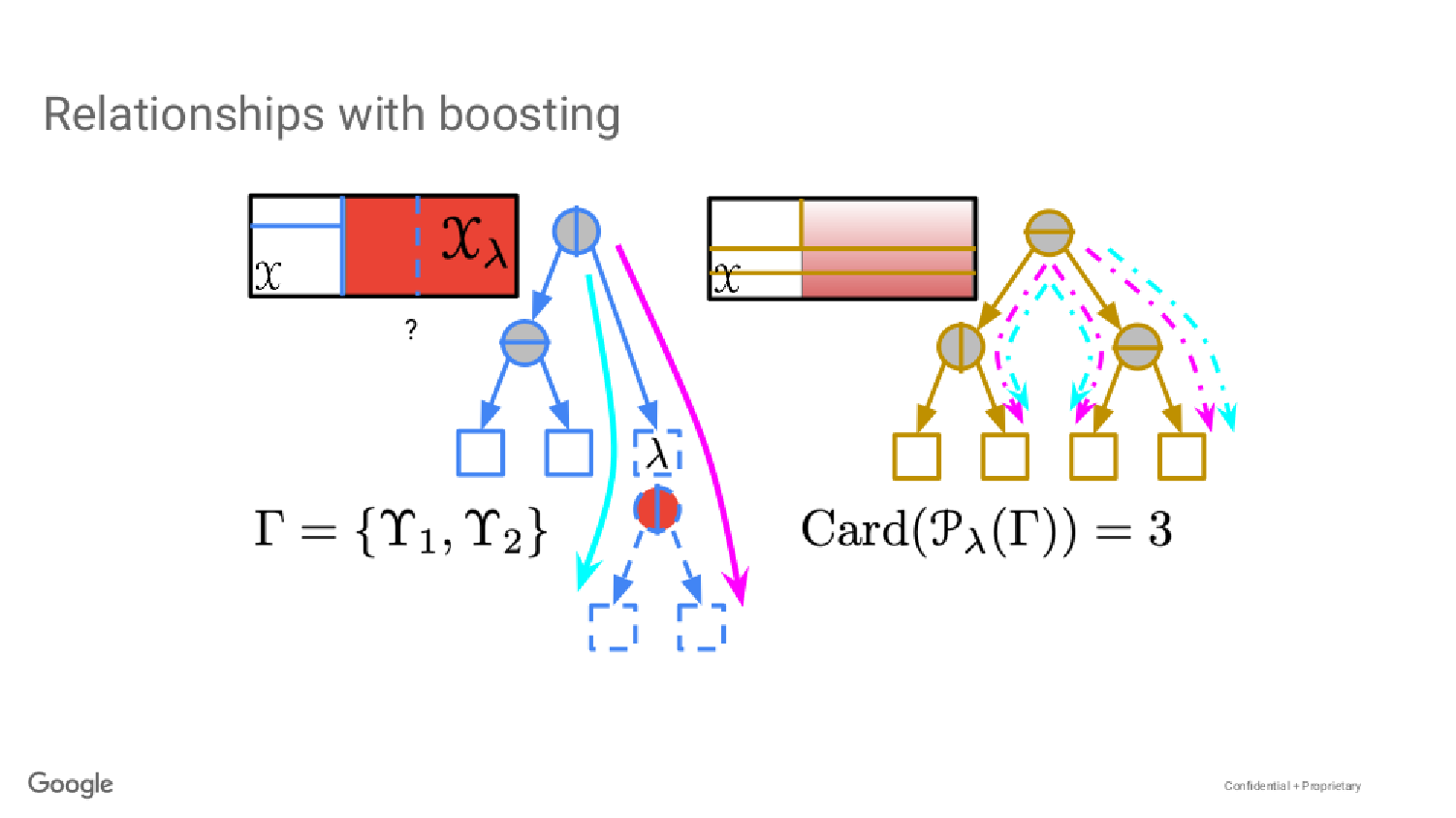} 
\caption{Learning a set of two trees $\Gamma$ with \topdowngen. If we want a split at leaf $\leaf$ indicated, then it would takes two steps (but a \textit{single} application of the weak learning assumption, \cite[Lemma 5]{mnwRC}) of the original \topdowngen~to learn $\tree_1$ alone \cite[Section 4]{mnwRC}. Each step, figured with a plain arrow, ends with adding a new leaf. In our case however, we have to take into account the partition induced by the leaves of $\tree_2$, which results in considering not one but three subsets in $\mathcal{P}_\leaf$, thus adding not one but three weak hypotheses simultaneously (one for each of the dashed arrows on $\tree_2$). Thanks to Jensen's inequality, a guaranteed decrease of the loss at hand can be obtained by a single application of the weak learning assumption at $\leaf$, just like in the original \topdowngen.}
    \label{fig:topdowngen}
  \end{figure}
\begin{theorem}
  \topdowngen~greedily minimizes the following loss function:
  \begin{eqnarray*}
  \poprisk(\Gamma_{j}) & = & \sum_{\leaf \in \leafset(\tree)} p_{\colorm{\scriptsize}}[\mathcal{X}_\leaf] \cdot \sum_{\mathcal{C} \in \mathcal{P}_{\leaf}(\Gamma_j)} \frac{p_{\colorm{\scriptsize}}[\mathcal{C}]}{p_{\colorm{\scriptsize}}[\mathcal{X}_\leaf]} \cdot \poibayesrisk\left(\frac{\prior p_{\meas{\colorr}}[\mathcal{C}]}{p_{\colorm{\scriptsize}}[\mathcal{C}]}\right), \tree \in \Gamma_j.
  \end{eqnarray*}
  Furthermore, suppose there exists $\upgamma > 0, \upkappa >0$ such that at each iteration $j$, the predicate $\texttt{p}$ splits $\mathcal{X}_{\leaf}$ of leaf $\leaf \in \leafset(\tree)$ into $\mathcal{X}^{\texttt{p}}_{\leaf}$ and $\mathcal{X}^{\neg \texttt{p}}_{\leaf}$ (same nomenclature as in \eqref{defcpi}, \eqref{defnotcpi}) such that:
  \begin{eqnarray}
    p_{\colorm{\scriptsize}}[\mathcal{X}_\leaf] & \geq & \frac{1}{\mathrm{Card}(\leafset(\tree))},\label{eqMSA}\\
    \left| \frac{\int_{\mathcal{X}^\texttt{p}_{\leaf}} \dmeas{\colorr}}{\int_{\mathcal{X}_{\leaf}} \dmeas{\colorr}} - \frac{\int_{\mathcal{X}^\texttt{p}_{\leaf}} \dmeas{\coloru}}{\int_{\mathcal{X}_{\leaf}} \dmeas{\coloru}} \right| & \geq & \upgamma, \label{eqWLA2}\\
    \min\left\{\frac{\prior p_{\meas{\colorr}}[\mathcal{X}_{\leaf}]}{p_{\colorm{\scriptsize}}[\mathcal{X}_{\leaf}]}, 1 - \frac{\prior p_{\meas{\colorr}}[\mathcal{X}_{\leaf}]}{p_{\colorm{\scriptsize}}[\mathcal{X}_{\leaf}]}\right\} & \geq & \upkappa. \label{eqWPA2}
  \end{eqnarray}
  Then we have the following guaranteed slack between two successive models:
  \begin{eqnarray}
\poprisk(\Gamma_{j+1})  - \poprisk(\Gamma_{j}) & \leq & - \frac{\kappa \upgamma^2 \upkappa^2}{8\mathrm{Card}(\leafset(\tree))},
  \end{eqnarray}
  where $\kappa$ is such that $0<\kappa \leq \inf\{\partialloss{-1}' - \partialloss{1}'\}$.
  \end{theorem}
  \begin{proofsketch}
The loss function comes directly from \cite[Lemma 7]{mnwRC}.  At each iteration, \topdowngen~makes a number of updates that guarantee, once Step 2.4 is completed (because the elements in $\mathcal{P}_\leaf(\Gamma_j)$ are disjoint)
\begin{eqnarray}
  \poprisk(\Gamma_{j+1})  - \poprisk(\Gamma_{j}) & \leq & - \frac{\kappa}{2}\cdot \sum_{\mathcal{C}\in \mathcal{P}_\leaf(\Gamma_j)} p_{\colorm{\scriptsize}}[\mathcal{C}] \cdot \expect_{\colorm{\scriptsize} | \mathcal{C}}\left[(w_{j+1} - w_j)^2\right]\label{eqsimpl1}\\
  & = & - \frac{\kappa}{2}\cdot p_{\colorm{\scriptsize}}[\mathcal{X}_\leaf] \expect_{\colorm{\scriptsize} | \mathcal{X}_\leaf}\left[(w_{j+1} - w_j)^2\right],
\end{eqnarray}
where the weights are given in \cite{mnwRC}. Each expression in the summand of \eqref{eqsimpl1} is exactly the guarantee of \cite[ineq. before (69)]{mnwRC}; all such expressions are not important; what is more important is \eqref{eqsimpl1}: all steps occurring in Step 2.4 are equivalent to a single step of \topdowngen~carried out over the whole $\mathcal{X}_\leaf$. Overall, the number of "aggregated" steps match the counter $j$ in \topdowngen. We then just have to reuse the proof of \cite[Theorem B]{mnwRC}, which implies, in lieu of their \cite[eq. (74)]{mnwRC}
\begin{eqnarray}
\poprisk(\Gamma_{j+1})  - \poprisk(\Gamma_{j}) & \leq & - 2 \kappa\cdot p_{\colorm{\scriptsize}}[\mathcal{X}_\leaf] \cdot \left(\frac{1}{2}\cdot \left(\frac{\int_{\mathcal{X}^\texttt{p}_{\leaf}} \dmeas{\colorr}}{\int_{\mathcal{X}_{\leaf}} \dmeas{\colorr}} - \frac{\int_{\mathcal{X}^\texttt{p}_{\leaf}} \dmeas{\coloru}}{\int_{\mathcal{X}_{\leaf}} \dmeas{\coloru}} \right)\right)^2 \cdot \bayessqrisk \left(\frac{\prior p_{\meas{\colorr}}[\mathcal{X}_{\leaf}]}{p_{\colorm{\scriptsize}}[\mathcal{X}_{\leaf}]}\right)^2,
\end{eqnarray}
with $\bayessqrisk(u) \defeq u(1-u)$. Noting $2 \bayessqrisk(u) \geq \min\{u, 1-u\}$, we then use \eqref{eqMSA} -- \eqref{eqWPA2}, which yields the statement of the Theorem.
\end{proofsketch}
    As a consequence, if we assume that the total number of boosting iterations $J$ is a multiple of the number of trees $T$, it comes from \cite[eq. (29)]{mnwRC} that after $J$ iterations, we have
    \begin{eqnarray}
\poprisk(\Gamma_{J})  - \poprisk(\Gamma_{0}) & \leq & - \frac{\kappa \upgamma^2 \upkappa^2}{8} \cdot T \log\left(1 + \frac{J}{T}\right).
    \end{eqnarray}
    Using Lemma \ref{lem-losses} and the fact that the induction of the sets of trees in \topdowngen~is done in the same way as the induction of the set of trees of our generator $\colorg$ in \topdownGT, we get:
\begin{eqnarray*}
  \prior \cdot \likelihoodratiorisk_{\ell}\left(\colorr, \colorg_J\right) & = & \prior \cdot \likelihoodratiorisk_{\ell}\left(\colorr, \colorg_{0}\right) + \left(\poprisk(\colorg_{J}) - \poprisk(\colorg_{0})\right)\\
                                                                                   & = & \prior \cdot \likelihoodratiorisk_{\ell}\left(\colorr, \colorg_{0}\right) + \left(\poprisk(\Gamma_{J})  - \poprisk(\Gamma_{0})\right)\\
  & \leq & \prior \cdot \likelihoodratiorisk_{\ell}\left(\colorr, \colorg_{0}\right) - \frac{\kappa \upgamma^2 \upkappa^2}{8} \cdot T\log\left(1 + \frac{J}{T}\right),
\end{eqnarray*}
    as claimed.

    \begin{remark}
      One may wonder what is the relationship between the loss that we minimize and "conventional" losses used to train generative models. The advantage of our Theorem  \ref{th-boost} is that by choosing different proper losses, one can get convergence guarantees for different "conventional" losses. Let us illustrate this with the \textsc{kl} divergence between measures $\meas{A}$ and $\meas{B}$, noted $\textsc{kl}(\meas{A} \| \meas{B})$.
      \begin{lemma}\label{lemKL1}
        Suppose the prior satisfies $\prior > 0$ and $\meas{\colorr}$ absolutely continuous with respect to $\meas{\colorg}$. Then for the choice $\ell \defeq \logloss$ = log-loss, we get
        \begin{eqnarray}
            \likelihoodratiorisk_{\logloss}\left(\colorr, \colorg\right) & = & \prior \cdot \textsc{kl} \left(\colorr \| \colorg\right) - \textsc{kl} \left(\prior \dmeas{\colorr} + (1-\prior)\dmeas{\coloru}  \| \prior \dmeas{\colorg} + (1-\prior)\dmeas{\coloru} \right) .
            \end{eqnarray}
        \end{lemma}
    \begin{proof}
      We first recall that for any convex $F$, we have with our choice of $g$ (assuming $\prior > 0$),
      \begin{eqnarray}
        \check{F}(z) & = & \frac{\prior z + 1-\prior}{\prior} \cdot F \left(\frac{\prior z}{\prior z + 1-\prior}\right),\\
        \check{F}'(z) & = & F \left(\frac{\prior z}{\prior z + 1-\prior}\right) + \frac{1-\prior}{\prior z + 1-\prior} \cdot F'\left(\frac{\prior z}{\prior z + 1-\prior}\right). \label{der2}
        \end{eqnarray}
        The log-loss has $\partiallogloss{1}(u) = -\log u, \partiallogloss{-1}(u) = -\log(1-u)$. It is strictly proper and so it comes
        \begin{eqnarray*}
\perspectivepobayesrisk(z) & = & z \cdot \log \left(\frac{\prior z}{\prior z + 1-\prior}\right) + \frac{1-\prior}{\prior} \cdot \log \left(\frac{1-\prior}{\prior z + 1-\prior}\right).
          \end{eqnarray*}
          We thus get, in our case
          \begin{eqnarray}
\perspectivepobayesrisk\left(\frac{\dmeas{\colorr}}{\dmeas{\coloru}}\right) & = & \frac{\dmeas{\colorr}}{\dmeas{\coloru}} \cdot \log\left(\frac{\prior \dmeas{\colorr}}{\prior \dmeas{\colorr} + (1-\prior)\dmeas{\coloru}}\right) + \frac{1-\prior}{\prior} \cdot \log\left(\frac{(1-\prior)\dmeas{\coloru}}{\prior \dmeas{\colorr} + (1-\prior)\dmeas{\coloru}}\right), \label{persplog1}
          \end{eqnarray}
          and using \eqref{der2} and the fact that for the log-loss $(-\poibayesrisk)'(u) = \log(u/(1-u))$,
          \begin{eqnarray}
            \perspectivepobayesrisk'\left(\frac{\dmeas{\colorg}}{\dmeas{\coloru}}\right) & = & \frac{\prior \dmeas{\colorg}}{\prior \dmeas{\colorg} + (1-\prior)\dmeas{\coloru}} \cdot \log\left(\frac{\prior \dmeas{\colorg}}{\prior \dmeas{\colorg} + (1-\prior)\dmeas{\coloru}}\right) \nonumber\\
            & & + \frac{(1-\prior)\dmeas{\coloru}}{\prior \dmeas{\colorg} + (1-\prior)\dmeas{\coloru}} \cdot \log\left(\frac{(1-\prior)\dmeas{\coloru}}{\prior \dmeas{\colorg} + (1-\prior)\dmeas{\coloru}}\right)\nonumber\\
            & & + \frac{(1-\prior)\dmeas{\coloru}}{\prior \dmeas{\colorg} + (1-\prior)\dmeas{\coloru}} \cdot \log\left(\frac{\prior\dmeas{\colorg}}{(1-\prior)\dmeas{\coloru}}\right)  = \log\left(\frac{\prior \dmeas{\colorg}}{\prior \dmeas{\colorg} + (1-\prior)\dmeas{\coloru}}\right). \label{poibaiesder}
          \end{eqnarray}
          Finally, we get, assuming $\meas{\colorr}$ absolutely continuous with respect to $\meas{\colorg}$,
          \begin{eqnarray}
            \lefteqn{\likelihoodratiorisk_{\logloss}\left(\colorr, \colorg\right)}\nonumber\\
            & \defeq & \prior \cdot \int_{\mathcal{X}} \dmeas{\coloru} \cdot \left(\perspectivepobayesrisk\left(\frac{\dmeas{\colorr}}{\dmeas{\coloru}}\right) - \perspectivepobayesrisk\left(\frac{\dmeas{\colorg}}{\dmeas{\coloru}}\right) - \left(\frac{\dmeas{\colorr}-\dmeas{\colorg}}{\dmeas{\coloru}}\right) \cdot \perspectivepobayesrisk'\left(\frac{\dmeas{\colorg}}{\dmeas{\coloru}}\right)\right)\nonumber\\
            & = & \int_{\mathcal{X}} \left\{ \begin{array}{l}
                                               \prior \dmeas{\colorr}\cdot \log\left(\frac{\prior \dmeas{\colorr}}{\prior \dmeas{\colorr} + (1-\prior)\dmeas{\coloru}}\right) + (1-\prior) \dmeas{\coloru} \cdot \log\left(\frac{(1-\prior)\dmeas{\coloru}}{\prior \dmeas{\colorr} + (1-\prior)\dmeas{\coloru}}\right)\\
                                               -\prior \dmeas{\colorg}\cdot \log\left(\frac{\prior \dmeas{\colorg}}{\prior \dmeas{\colorg} - (1-\prior)\dmeas{\coloru}}\right) - (1-\prior) \dmeas{\coloru} \cdot \log\left(\frac{(1-\prior)\dmeas{\coloru}}{\prior \dmeas{\colorg} + (1-\prior)\dmeas{\coloru}}\right)\\
                                               - \prior \dmeas{\colorr}\cdot \log\left(\frac{\prior \dmeas{\colorg}}{\prior \dmeas{\colorg} + (1-\prior)\dmeas{\coloru}}\right) +  \prior \dmeas{\colorg}\cdot \log\left(\frac{\prior \dmeas{\colorg}}{\prior \dmeas{\colorg} + (1-\prior)\dmeas{\coloru}}\right) 
                                             \end{array}\right.\\
            & = & \int_{\mathcal{X}} \left\{ \begin{array}{l}
                                               \prior \dmeas{\colorr}\cdot \log\left(\frac{\dmeas{\colorr}}{\dmeas{\colorg}}\right)\\
                                               - \left( \prior \dmeas{\colorr} + (1-\prior)\dmeas{\coloru} \right)\cdot \log\left(\prior \dmeas{\colorr} + (1-\prior)\dmeas{\coloru}\right) \\
                                               + \left( \prior \dmeas{\colorr} + (1-\prior)\dmeas{\coloru} \right)\cdot \log\left(\prior \dmeas{\colorg} + (1-\prior)\dmeas{\coloru}\right) 
                                             \end{array}\right.          \\
            & =& \prior \cdot \textsc{kl} \left(\colorr \| \colorg\right) - \textsc{kl} \left(\prior \dmeas{\colorr} + (1-\prior)\dmeas{\coloru}  \| \prior \dmeas{\colorg} + (1-\prior)\dmeas{\coloru} \right) ,
          \end{eqnarray}
          as claimed (end of the proof of Lemma \ref{lemKL1}).
        \end{proof}
        Because of the joint convexity of KL divergence, we always have
        \begin{eqnarray*}
          \textsc{kl} \left(\prior \dmeas{\colorr} + (1-\prior)\dmeas{\coloru}  \| \prior \dmeas{\colorg} + (1-\prior)\dmeas{\coloru} \right) & \leq & \prior \textsc{kl} \left(\colorr \| \colorg\right) +  (1-\prior) \textsc{kl} \left(\dmeas{\coloru} \| \dmeas{\coloru}\right) \\
                                                                                                                                              & & = \prior \textsc{kl} \left(\colorr \| \colorg\right),
        \end{eqnarray*}
        which yields (another proof) that our loss, $\likelihoodratiorisk_{\logloss}\left(\colorr, \colorg\right)$, is lowerbounded by 0. In general, if we let $\gamma > 0$ such that $\textsc{kl} \left(\prior \dmeas{\colorr} + (1-\prior)\dmeas{\coloru}  \| \prior \dmeas{\colorg} + (1-\prior)\dmeas{\coloru} \right) \leq (1-\gamma) \cdot \prior \textsc{kl} \left(\colorr \| \colorg\right)$ (The strict positivity of $\gamma$ is also a weak assumption as long as $\meas{\colorr}, \meas{\colorg}$ sufficiently differ from the uniform distribution), then
        \begin{eqnarray}
            \likelihoodratiorisk_{\logloss}\left(\colorr, \colorg\right) & \geq &  \gamma\prior\cdot \textsc{kl} \left(\colorr \| \colorg\right),
        \end{eqnarray}
        so any upperbound on $\likelihoodratiorisk_{\logloss}\left(\colorr, \colorg\right)$ (such as obtained from Theorem  \ref{th-boost}) translates to an upperbound on the KL divergence. Note that the condition for absolute continuity  in Lemma \ref{lemKL1} is always met with the models we learn (both generative forests and ensembles of generative trees). 
          \end{remark}

\section{Simpler models: ensembles of generative trees}\label{sec-modtwo}

\begin{figure}
  \centering
\includegraphics[trim=110bp 30bp 210bp 140bp,clip,width=0.4\columnwidth]{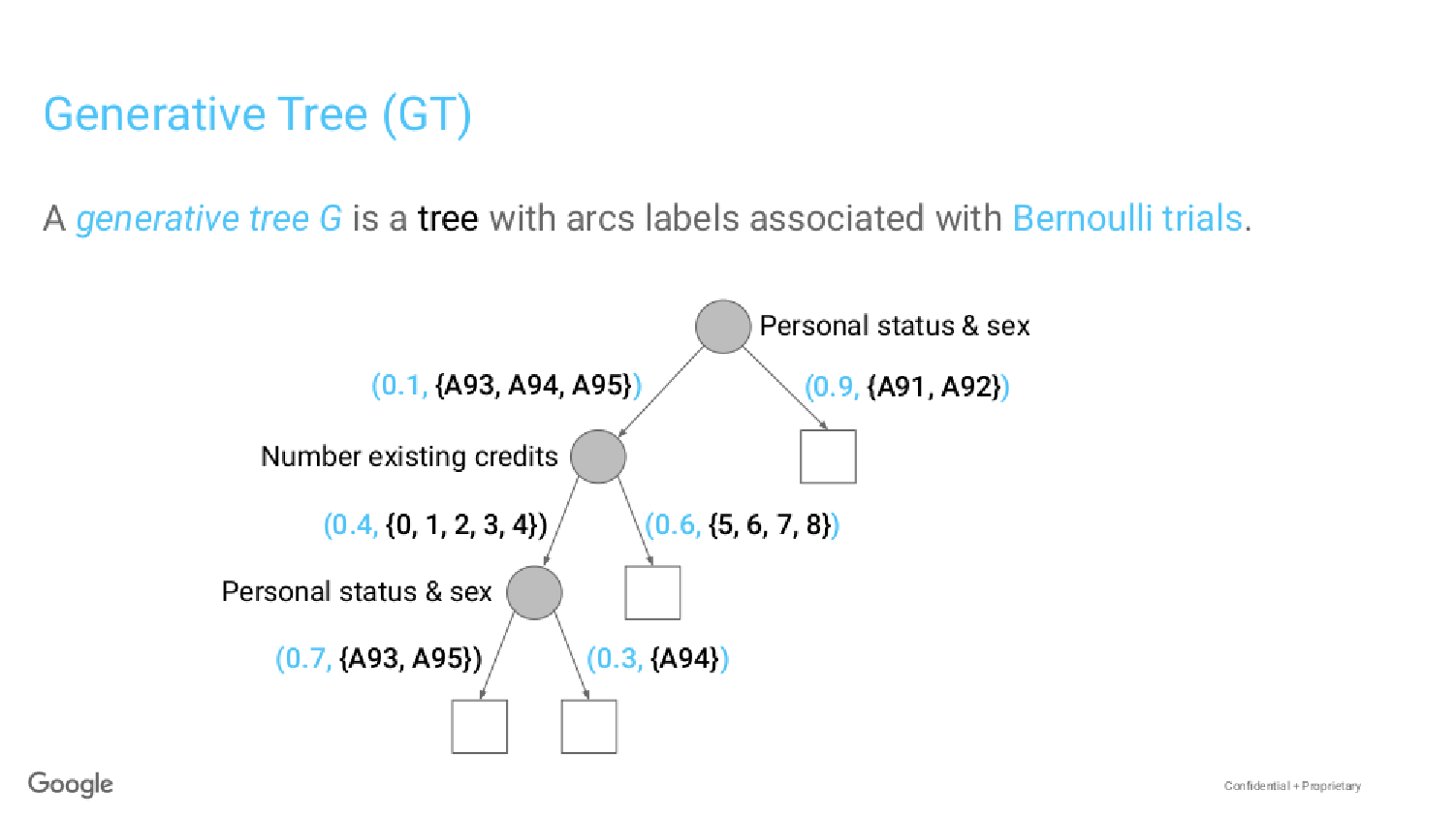} 
    \caption{A generative tree (\gt) associated to UCI German Credit.}
    \label{fig:gt-def-1}
  \end{figure}

Storing a \geot~requires keeping information about the empirical measure $\colorr$ to compute the branching probability in Step 1 of \starupdate. This does not require to store the full training sample, but requires at least an index table recording the $\{\mbox{leaves}\} \times \{\mbox{observations}\}$ association, for a storage cost in between $\Omega(m)$ and $O(mT)$ where $m$ is the size of the training sample. There is a simple way to get rid of this constraint and approximate the \geot~by a set of \textit{generative trees} (\gt s) of \cite{ngGT}.

\paragraph{Models} Simply put, a \gt~is a simplified equivalent representation of a generative forest with 1 tree only. In this case, the branching probabilities in Step 1 of \starupdate~depend only on the tree's star node position. Thus, instead of recomputing them from scratch each time an observation needs to be generated, we compute them beforehand and use them to label the arcs in addition to the features' domains, as shown in Figure \ref{fig:gt-def-1}. This representation is equivalent to that of generative trees \cite{ngGT}. If we have several trees, we do this process independently for each tree and end up with an \textit{ensemble of generative trees} (\eogt). The additional memory footprint for storing probabilities ($O(\sum_i |\leafset(\tree_i)|)$) is offset by the fact that we do not have anymore to store associations between leaves and observations for generative forests, for a potentially significant reduction is storing size. However, we cannot use anymore \starupdate~as is for data generation since we cannot compute exactly the probabilities in Step 1. Two questions need to be addressed: is it possible to use an ensemble of generative trees to generate data with good approximation guarantees (and if so, how) and of course how do we train such models. 

\paragraph{Data generation} We propose a simple solution based on how well an \eogt~can approximate a generative forest. It relies on a simple assumption about $\colorr$ and $\mathcal{X}$. Taking as references the parameters in \starupdate, we now make two simplifications on notations, first replacing $\tree.\anychild{\node}^\star$ by $\anychild{\node}^\star$ (tree implicit), and then using notation $\mathcal{C}_{\texttt{v}} \defeq \xxset{\anychild{\node}^\star} \cap \mathcal{C}$ for any $\texttt{v} \in \{\texttt{t}, \texttt{f}\}$, the reference to the tree / star node being implicit. The assumption is as follows.
\begin{assumption}\label{assum-u}
There exists $\varkappa \in (0,\infty)$ such that at any Step 1 of \starupdate, for any $\texttt{v} \in \{\texttt{t}, \texttt{f}\}$, we have 
\begin{eqnarray}
\frac{p_{{\color{red}{\meas{R}}}}\left[\mathcal{C}_{\texttt{v}} | \xxset{\anychild{\node}^\star}\right]}{p_{\coloru}\left[\mathcal{C}_{\texttt{v}} | \xxset{\anychild{\node}^\star}\right]} & \in & \left[\exp(-\varkappa), \exp(\varkappa)\right]. \label{defVarkappa}
\end{eqnarray}
\end{assumption}
A simple condition to ensure the existence of $\varkappa$ is a requirement weaker than \textbf{(c)} in Definition \ref{defWLAG}: at all Steps 1 of \starupdate, $p_{\colorr}\left[\mathcal{C}_{\texttt{t}} | \mathcal{C} \right] \in (0,1)$, which also guarantees $p_{\colorr}\left[\mathcal{C}_{\texttt{f}} | \mathcal{C} \right] \in (0,1)$ and thus postulates that the branching in Step 1 of \starupdate~never reduces to one choice only. The next Lemma shows that it is indeed possible to combine the generation of an ensemble of generative trees with good approximation properties, and provides a simple algorithm to do so, which reduces to running \starupdate~with a specific approximation to branching probabilities in Step 1. For any \geot~$\colorg$ and $\mathcal{C} \in \mathcal{P} (\colorg)$, $\mathrm{depth}_{\colorg}(\mathcal{C})$ is the sum of depths of the leaves in each tree whose support intersection is $\mathcal{C}$. We need the definition of the \textit{expected} depth of $\colorg$.
\negspace
\begin{definition}\label{def-exdepth}
  The expected depth of \geot~$\colorg$ is $\exdepth(\meas{\colorg}) \defeq \sum_{\mathcal{C} \in \mathcal{P} (\colorg)} p_{\meas{\colorg}}[\mathcal{C}] \cdot \mathrm{depth}_{\colorg}(\mathcal{C})$.
  \end{definition}
$\exdepth(\meas{\colorg})$ represents the expected complexity to sample an observation (see also Lemma \ref{lem-geot-eogt}).
\begin{lemma}\label{lem-geot-eogt}
In Step 1 of \starupdate, suppose $p_{\colorr}\left[\mathcal{C}_{\texttt{t}} \cap \mathcal{C} | \mathcal{C} \right]$ is replaced by
  \begin{eqnarray}
\hat{p}_{\node^\star} \defeq \frac{p_{\coloru}\left[\mathcal{C}_{\texttt{t}}| \xxset{\rightchild{\node}^\star}\right] \cdot p_{\node^\star}}{p_{\coloru}\left[\mathcal{C}_{\texttt{t}}| \xxset{\rightchild{\node}^\star}\right] \cdot p_{\node^\star} + p_{\coloru}\left[\mathcal{C}_{\texttt{f}}| \xxset{\leftchild{\node}^\star}\right]\cdot (1-p_{\node^\star})}, \label{new-p}
  \end{eqnarray}
  with $p_{\node^\star} \defeq p_{{\color{red}{\meas{R}}}}\left[\xxset{\rightchild{\node}^\star} | \xxset{\node^\star}\right]$ the probability labeling arc $(\node^\star, \rightchild{\node})$ in its generative tree. Under Assumption \ref{assum-u}, if we denote $\colorg$ the initial \geot~and $\hat{\colorg}$ the \eogtp~using \eqref{new-p}, the following bound holds on the $\kl$ divergence between $\colorg$ and $\hat{\colorg}$:
  \begin{eqnarray}
  \kl(\meas{\colorg} \| \hat{\meas{\colorg}}) \leq 2\varkappa \cdot \exdepth(\meas{\colorg}), \quad \mbox{where $\exdepth(\meas{\colorg})$ is given in Definition \ref{def-exdepth}}\label{boundGGhat}.
  \end{eqnarray}
\end{lemma}
\begin{proof}
Because $\mathcal{C} \subseteq \xxset{\Upsilon.\node^\star}$ at any call of \starupdate, we have
\begin{eqnarray}
  \lefteqn{p_{{\color{red}{\meas{R}}}}\left[\xxset{\Upsilon.\rightchild{\node}^\star} \cap \mathcal{C} | \mathcal{C} \right]}\nonumber\\
  & = & p_{\meas{R}}\left[\mathcal{C}_{\texttt{t}} | \mathcal{C} \right]\nonumber\\
    & = & \frac{p_{{\color{red}{\meas{R}}}}\left[\mathcal{C}_{\texttt{t}}\right]}{p_{{\color{red}{\meas{R}}}}\left[\mathcal{C}_{\texttt{t}}\right]+p_{{\color{red}{\meas{R}}}}\left[\mathcal{C}_{\texttt{f}}\right]}\\
                                                                                                     & \leq & \frac{\exp(\varkappa) p_{\coloru}\left[\mathcal{C}_{\texttt{t}}| \xxset{\Upsilon.\rightchild{\node}^\star}\right] p_{{\color{red}{\meas{R}}}}\left[\xxset{\Upsilon.\rightchild{\node}^\star}\right]}{\exp(-\varkappa) p_{\coloru}\left[\mathcal{C}_{\texttt{t}}| \xxset{\Upsilon.\rightchild{\node}^\star}\right] p_{{\color{red}{\meas{R}}}}\left[\xxset{\Upsilon.\rightchild{\node}^\star}\right] + \exp(-\varkappa) p_{\coloru}\left[\mathcal{C}_{\texttt{f}}| \xxset{\Upsilon.\leftchild{\node}^\star}\right]p_{{\color{red}{\meas{R}}}}\left[\xxset{\Upsilon.\leftchild{\node}^\star}\right]}\label{bsup1}\\
                                                                                                     & & = \exp(2\varkappa) \cdot \frac{p_{\coloru}\left[\mathcal{C}_{\texttt{t}}| \xxset{\Upsilon.\rightchild{\node}^\star}\right] p_{{\color{red}{\meas{R}}}}\left[\xxset{\Upsilon.\rightchild{\node}^\star} | \xxset{\Upsilon.\node^\star}\right]}{p_{\coloru}\left[\mathcal{C}_{\texttt{t}}| \xxset{\Upsilon.\rightchild{\node}^\star}\right] p_{{\color{red}{\meas{R}}}}\left[\xxset{\Upsilon.\rightchild{\node}^\star} | \xxset{\Upsilon.\node^\star}\right] + p_{\coloru}\left[\mathcal{C}_{\texttt{f}}| \xxset{\Upsilon.\leftchild{\node}^\star}\right]p_{{\color{red}{\meas{R}}}}\left[\xxset{\Upsilon.\leftchild{\node}^\star} | \xxset{\Upsilon.\node^\star}\right]}\label{bsup2}.
  \end{eqnarray}
  \eqref{bsup1} is due to the fact that, since $\mathcal{C} \subseteq \xxset{\Upsilon.\node^\star}$, $p_{{\color{red}{\meas{R}}}}\left[\mathcal{C}_{\texttt{v}}\right] = p_{{\color{red}{\meas{R}}}}\left[\mathcal{C}_{\texttt{v}} \cap \xxset{\Upsilon.\anychild{\node}^\star}\right] = p_{{\color{red}{\meas{R}}}}\left[\mathcal{C}_{\texttt{v}} | \xxset{\Upsilon.\anychild{\node}^\star}\right] p_{{\color{red}{\meas{R}}}}\left[\xxset{\Upsilon.\anychild{\node}^\star}\right]$, and then using \eqref{defVarkappa}. \eqref{bsup2} is obtained by dividing numerator and denominator by $p_{{\color{red}{\meas{R}}}}\left[\xxset{\Upsilon.\rightchild{\node}^\star}\right]+p_{{\color{red}{\meas{R}}}}\left[\xxset{\Upsilon.\leftchild{\node}^\star}\right] = p_{{\color{red}{\meas{R}}}}\left[\xxset{\Upsilon.\node^\star}\right]$.

  $\colorg$ and $\hat{\colorg}$ satisfy $\mathcal{P}({\colorg}) = \mathcal{P}(\hat{\colorg})$ so
  \begin{eqnarray}
  \kl(\meas{\colorg} \| \hat{\meas{\colorg}}) & = & \int \dmeas{\colorg} \log \frac{\dmeas{\colorg}}{\dmeas{\hat{\colorg}}}\\
  & = & \sum_{\mathcal{C} \in  \mathcal{P}({\colorg})} p_{\meas{\colorg}}[\mathcal{C}] \log \frac{p_{\meas{\colorg}}[\mathcal{C}]}{p_{\meas{\hat{\colorg}}}[\mathcal{C}]}.
\end{eqnarray}
Finally, $p_{\meas{\colorg}}[\mathcal{C}]$ and $p_{\meas{\hat{\colorg}}}[\mathcal{C}]$ are just the product of the branching probabilities in any admissible sequence. If we use the same admissible sequence in both generators, we can write $p_{\meas{\colorg}}[\mathcal{C}] = \prod_{j=1}^{n(\mathcal{C})} p_j$ and $p_{\meas{\hat{\colorg}}}[\mathcal{C}] = \prod_{j=1}^{n(\mathcal{C})} \hat{p}_j$, and \eqref{bsup2} directly yields $p_j \leq \exp(2\varkappa) \cdot \hat{p}_j, \forall j \in [n(\mathcal{C})]$, $n(\mathcal{C})$ being a shorthand for $\mathrm{depth}_{\colorg}(\mathcal{C}) = \mathrm{depth}_{\hat{\colorg}}(\mathcal{C})$ (see main file). So, for any $\mathcal{C} \in  \mathcal{P}({\colorg})$,
\begin{eqnarray}
\frac{p_{\meas{\colorg}}[\mathcal{C} ]}{p_{\meas{\hat{\colorg}}}[\mathcal{C} ]} & \leq & \exp(2\varkappa \cdot \mathrm{depth}_{\colorg}(\mathcal{C} )),
\end{eqnarray}
and finally, replacing back $n(\mathcal{C})$ by notation $\mathrm{depth}(\mathcal{C})$,
\begin{eqnarray}
  \kl(\meas{\colorg} \| \hat{\meas{\colorg}}) & \leq & 2\varkappa \cdot \sum_{\mathcal{C} \in  \mathcal{P}({\colorg})}  p_{\meas{\colorg}}[\mathcal{C} ] \cdot \mathrm{depth}(\mathcal{C} )\nonumber\\
  & & = 2\varkappa \cdot \exdepth(\meas{\colorg}),
\end{eqnarray}
as claimed. 
\end{proof}

Note the additional leverage for training and generation that stems from Assumption \ref{assum-u}, not just in terms of space: computing $p_{\coloru}\left[\mathcal{C}_{\texttt{v}} | \xxset{\anychild{\node}^\star}\right]$ is $O(d)$ and does not necessitate data so the computation of key conditional probabilities \eqref{new-p} drops from $\Omega(m)$ to $O(d)$ for training and generation in an \eogt. Lemma \ref{lem-geot-eogt} provides the change in \starupdate~to generate data.

\paragraph{Training} To train an \eogt, we cannot rely on the idea that we can just train a \geot~and then replace each of its trees by generative trees. To take a concrete example of how this can be a bad idea in the context of missing data imputation, we have observed empirically that a generative forest (\geot) can have many trees whose node's observation variables are the \textit{same} within the tree. Taken independently of the forest, such trees would only model marginals, but in a \geot, sampling is dependent on the other trees and Lemma \ref{leminv} guarantees the global accuracy of the forest. \textit{However}, if we then replace the \geot~by an \eogt~with the same trees and use them for missing data imputation, this results in imputation at the mode of the marginal(s) (exclusively), which is clearly suboptimal. To avoid this, we have to use a specific training for an \eogt, and to do this, it is enough to change a key part of training in \texttt{splitPred}, the computation of probabilities $p_{\meas{\colorr}}[.]$ used in \eqref{defPOPRISK}. Suppose $\varkappa$ very small in Assumption \ref{assum-u}. To evaluate a split at some leaf, we observe (suppose we have a single tree)
  \begin{eqnarray*}
p_{{\color{red}{\meas{R}}}}[\xxset{\leftchild{\leaf}}] = p_{{\color{red}{\meas{R}}}}[\xxset{\leaf}] \cdot p_{\meas{\colorr}}[\xxset{\leftchild{\leaf}} | \xxset{\leaf}] \approx p_{{\color{red}{\meas{R}}}}[\xxset{\leaf}] \cdot p_{\meas{\coloru}}[\xxset{\leftchild{\leaf}} | \xxset{\leaf}]
  \end{eqnarray*}
  (and the same holds for $\rightchild{\leaf}$). Extending this to multiple trees and any $\mathcal{C} \in \mathcal{P}(\mathcal{T})$, we get the computation of any $p_{{\color{red}{\meas{R}}}}[\mathcal{C}_{\texttt{f}}]$ and $p_{{\color{red}{\meas{R}}}}[\mathcal{C}_{\texttt{t}}]$ needed to measure the new \eqref{defPOPRISK} for a potential split. Crucially, it does not necessitate to split the empirical measure at $\mathcal{C}$ but just relies on computing $p_{\meas{\coloru}}[\mathcal{C}_{\texttt{v}} | \mathcal{C}], \texttt{v} \in \{\texttt{f}, \texttt{t}\}$: with the product measure and since we make axis-parallel splits, it can be done in $O(1)$, thus substantially reducing training time. 

\paragraph{More with ensembles of generative trees} we can follow the exact same algorithmic steps as for generative trees to perform missing data imputation and density estimation, but use branching probabilities in the trees to decide branching, instead of relying on the empirical measure at tuples of nodes, which we do not have access to anymore. For this, we rely on the approximation \eqref{new-p} in Lemma \ref{lem-geot-eogt}. A key difference with a \geot~is that no tuple can have zero density because branching probabilities are in $(0,1)$ in each generative tree.

\section{Appendix on experiments}\label{app-exps}

\subsection{Domains}\label{sec-doms}

\begin{table}[t]
\begin{center}
\begin{tabular}{|cccc|r|r|r|r|}
\hline \hline
  Domain & Tag & Source & Missing data $?$ & \multicolumn{1}{c|}{$m$} & \multicolumn{1}{c|}{$d$} & \multicolumn{1}{c|}{$\#$ Cat.}  & \multicolumn{1}{c|}{$\#$ Num.} \\ \hline
  \domainname{iris} & -- & UCI & No & 150 & 5 & 1 & 4\\
  $^*$\domainname{ringGauss} & \domainname{ring} & -- & No & 1 600 & 2 & -- & 2\\
  $^*$\domainname{circGauss} & \domainname{circ} & -- & No & 2 200 & 2 & -- & 2\\
  $^*$\domainname{gridGauss} & \domainname{grid} & -- & No & 2 500 & 2 & -- & 2\\
  \domainname{forestfires} & \domainname{for} & OpenML & No & 517 & 13 & 2 & 11 \\
  $^*$\domainname{randGauss} & \domainname{rand} & -- & No & 3 800 & 2 & -- & 2\\
  \domainname{tictactoe} & \domainname{tic} & UCI & No & 958 & 9 & 9 & -- \\
  \domainname{ionosphere} & \domainname{iono} & UCI & No & 351 & 34 & 2 & 32 \\
  \domainname{student$\_$performance$\_$mat} & \domainname{stm} & UCI & No & 396 & 33 & 17 & 16 \\
  \domainname{winered} & \domainname{wred} & UCI & No & 1 599 & 12 & -- & 12 \\
  \domainname{student$\_$performance$\_$por} & \domainname{stp} & UCI & No & 650 & 33 & 17 & 16 \\
  \domainname{analcatdata$\_$supreme} & \domainname{ana} & OpenML & No & 4 053 & 8 & 1 & 7 \\
  \domainname{abalone} & \domainname{aba} & UCI & No & 4 177 & 9 & 1 & 8 \\
  \domainname{kc1} & -- & OpenML & No & 2 110 & 22 & 1 & 21 \\
  \domainname{winewhite} & \domainname{wwhi} & UCI & No & 4 898 & 12 & -- & 12\\
  \domainname{sigma-cabs} & -- & Kaggle & Yes & 5 000 & 13 & 5 & 8 \\
  \domainname{compas} & \domainname{comp} & OpenML & No & 5 278 & 14 & 9 & 5 \\
  \domainname{artificial$\_$characters} & \domainname{arti} & OpenML & No & 10 218 & 8 & -- & 8 \\
  \domainname{jungle$\_$chess} & \domainname{jung} & OpenML & No & 44 819 & 8 & 1 & 6 \\
  \domainname{open-policing-hartford} & -- & SOP & Yes & 18 419 & 20 & 16 & 4 \\
    \domainname{electricity} & \domainname{elec} & OpenML & No & 45 312 & 9 & 2 & 7 \\
\hline\hline
\end{tabular}
\end{center}
\caption{Public domains considered in our experiments ($m=$ total number
  of examples, $d=$ number of features), ordered in
  increasing $m \times d$. "Cat." is a shorthand for categorical (nominal / ordinal / binary); "Num." stands for numerical (integers / reals). ($^*$) = simulated, URL for OpenML: \texttt{https://www.openml.org/search?type=data\&sort=runs\&id=504\&status=active}; SOP = Stanford Open Policing project, \texttt{https://openpolicing.stanford.edu/} (see text). "Tag" refers to tag names used in Tables \ref{tab:gen-us-vs-arf-big} and \ref{tab:gen-us-vs-ctgan-big}.}
  \label{t-s-uci}
\end{table}

\domainname{ringGauss} is the seminal 2D ring Gaussians appearing in numerous GAN papers \cite{xzzBG}; those are eight (8) spherical Gaussians with equal covariance, sampling size and centers located on  regularly spaced (2-2 angular distance) and at equal distance from the origin. \domainname{gridGauss} was generated as a decently hard task from \cite{dbplamcAL}: it consists of 25 2D mixture spherical Gaussians with equal variance and sampled sizes, put on a regular grid. \domainname{circGauss} is a Gaussian mode surrounded by a circle, from \cite{xzzBG}. \domainname{randGauss} is a substantially harder version of \domainname{ringGauss} with 16 mixture components, in which covariance, sampling sizes and distances on sightlines from the origin are all random, which creates very substantial discrepancies between modes.

\subsection{Algorithms configuration and choice of parameters}\label{sec-algos}

\paragraph{\topdownGT} We have implemented \topdownGT~in Java, following Algorithm \ref{alg-topdownGT}'s blueprint. Our implementation of \texttt{tree} and \texttt{leaf} in Steps 2.1, 2.2 is simple: we pick the heaviest leaf among all trees (with respect to $\colorr$). The search for the best split is exhaustive unless the variable is categorical with more than a fixed number of distinct modalities (22 in our experiments), above that threshold, we pick the best split among a random subset. We follow \cite{ngGT}'s experimental setting: in particular, the input of our algorithm to train a generator is a .csv file containing the training data \textit{without any further information}. Each feature's domain is learned from the training data only; while this could surely and trivially be replaced by a user-informed domain for improved results (\textit{e.g.} indicating a proportion's domain as $[0\%,100\%]$, informing the complete list of socio-professional categories, etc.) --- and is in fact standard in some ML packages like \texttt{weka}'s ARFF files, we did not pick this option to alleviate all side information available to the GT learner. Our software automatically recognizes three types of variables: nominal, integer and floating point represented.

\noindent\textbf{Comments on implementation} For the interested reader, we give here some specific implementation details regarding choices mades for two classical bottlenecks on tree-based methods (this also applies to the supervised case of learning decision trees): handling features with large number of modalities and handling continuous features.\\

We first comment the case where the number of modalities of a feature is large. The details can be found in the code in \texttt{File Algorithm.java} (class \texttt{Algorithm}), as follows:
\begin{itemize}
  \item method \texttt{public GenerativeModelBasedOnEnsembleOfTrees learn$\_$geot()} implements both GF.Boost (for generative forests) and its extension to training ensembles of generative trees (Appendix) as a single algorithm.
\item method \texttt{Node choose$\_$leaf(String how$\_$to$\_$choose$\_$leaf)} is the method to choose a leaf (Step 2.2). Since it picks the heaviest leaf among all trees, it also implements Step 2.1 (we initialize all trees to their roots; having a tree = root for generation does not affect generation).
\item method \texttt{public boolean split$\_$generative$\_$forest(Node leaf, HashSet <MeasuredSupportAtTupleOfNodes> tol)} finds the split for a Generative Forest, given a node chosen for the split. Here is how it works:
  \begin{enumerate}
    \item It first computes the complete description of possible splits (not their quality yet), using method \texttt{public static Vector<FeatureTest> ALL$\_$FEATURE$\_$TESTS(Feature f, Dataset ds)} in \texttt{Feature.java}. The method is documented: for continuous features, the list of splits is just a list of evenly spaced splits.
\item then it calls \texttt{public SplitDetails split$\_$top$\_$down$\_$boosting$\_$generative$\_$forest$\_$fast(Node leaf, boolean [] splittable$\_$feature, FeatureTest [][] all$\_$feature$\_$tests, HashSet <MeasuredSupportAtTupleOfNodes> tol)}, which returns the best split as follows: (i) it shuffles the potential list of splits and then (ii) picks the best one in the sublist containing the first \texttt{Algorithm.MAXIMAL$\_$NUMBER$\_$OF$\_$SPLIT$\_$TESTS$\_$TRIES$\_$PER$\_$BOOSTING$\_$ITERATION} elements (if the list is smaller, the search for the best is exhaustive); this class variable is currently = 1000.
\end{enumerate}
\end{itemize}

Last, we comment how we handle continuous variables: the boosting theory tells us that finding a moderately good split (condition (b) Definition \ref{defWLAG}, main file) is enough to get boosting-compliant convergence (Theorem \ref{th-boost}, main file), so we have settled for a trivial and efficient mechanism: cut-points are evenly spaced and the number is fixed beforehand. See method \texttt{public static Vector<FeatureTest> ALL$\_$FEATURE$\_$TESTS(Feature f, Dataset ds)} in \texttt{Feature.java} for the details. It turns out that this works very well and shows our theory was indeed concretely used in the design/implementation of the training algorithm.
  
  \paragraph{\mice} We have used the R \mice~package V 3.13.0 with two choices of methods for the round robin (column-wise) prediction of missing values: \textsc{cart} \cite{bfosCA} and random forests (\textsc{rf}) \cite{vgMM}. In that last case, we have replaced the default number of trees (10) by a larger number (100) to get better results. We use the default number of round-robin iterations (5). We observed that random forests got the best results so, in order not to multiply the experiments reported and perhaps blur the comparisons with our method, we report only \mice's results for random forests.

  \paragraph{\textsc{TensorFlow}} To learn the additional Random Forests involved in experiments \textsc{gen-discrim}, we used Tensorflow Decision Forests library\footnote{\url{https://github.com/google/yggdrasil-decision-forests/blob/main/documentation/learners.md}}. We use 300 trees with max depth 16. Attribute sampling: sqrt(number attributes) for classification problems, number attributes / 3 for regression problems (Breiman rule of thumb);  the min $\#$examples per leaf is 5.

  \paragraph{\textsc{CT-GAN}} We used the Python implementation\footnote{\url{https://github.com/sdv-dev/CTGAN}} with default values \cite{xscvMT}.

  \paragraph{Adversarial Random Forests} We used the R code of the generator \forge~made available from the paper \cite{wbkwAR}\footnote{\url{https://github.com/bips-hb/arf_paper}}, learning forests containing a variable number of trees in $\{10, 50, 100, 200\}$. We noted that the code does not run when the dataset has missing values and we also got an error when trying to run the code on \texttt{kc1}.

  \paragraph{Vine Copulas AutoEncoders} We used the Python code available from the paper \cite{tavCA}\footnote{\url{https://github.com/sdv-dev/Copulas}.}, which processes only fully numerical datasets. We got a few errors when trying to run the code on \texttt{compas}.

  \paragraph{Forest Flows} We used the R code available from the paper \cite{jmfkGA}\footnote{\url{https://github.com/SamsungSAILMontreal/ForestDiffusion}.}. Hyperparameters used were default parameters\footnote{\url{https://htmlpreview.github.io/?https://github.com/SamsungSAILMontreal/ForestDiffusion/master/R-Package/Vignette.html}.}. 
  
  \paragraph{Kernel Density Estimation} We used the R code available in the package \texttt{npudens} with default values following the approach of \cite{lrNE}. We tried several kernels but ended up with sticking to the default choices, that seemed to provide overall some of the best results. Compared to us with \geot, KDE's code is extremely compute-intensive as on domains below \domainname{tictactoe} in Table \ref{t-s-uci}: it took orders of magnitude more than ours to get the density values on all folds, so we did not run it on the biggest domains. Notice that in our case, computation time includes not just training the generative model, but also, \textit{at each applicable iteration} $J$ in \topdownGT, the computation of the same number of density values as for KDE.
  
  \paragraph{Computers used} We ran part of the experiments on a Mac Book Pro 16 Gb RAM w/ 2 GHz Quad-Core Intel Core i5 processor, and part on a desktop Intel(R) Xeon(R) 3.70GHz with 12 cores and 64 Gb RAM. CT-GANs and VCAEs were run on the desktop, the other algorithms on the laptop.
  
\subsection{Supplementary results}\label{sec-exp-suppl}

Due to the sheer number of tables to follow, we shall group them according to topics

\subsubsection{Interpreting our models: '\textsc{scrutinize}'}\label{sec-interpreting}

\newsavebox\treeone
\begin{lrbox}{\treeone}\begin{minipage}{\textwidth}
\begin{verbatim}
------------------------------------------------------------------- 
(name = #0.0 | depth = 5 | #nodes = 11)
[#0:root] internal {5278} (juv_fel_count <= 0 ? #1 : #2)
|-[#1] internal {4823} (age_cat_Greaterthan45 in {1} ? #3 : #4)
| |-[#3] leaf {1081}
| \-[#4] internal {3742} (race_Caucasian in {1} ? #5 : #6)
|   |-[#5] leaf {1373}
|   \-[#6] internal {2369} (race_African-American in {1} ? #7 : #8)
|     |-[#7] internal {2266} (two_year_recid in {1} ? #9 : #10)
|     | |-[#9] leaf {1203}
|     | \-[#10] leaf {1063}
|     \-[#8] leaf {103}
\-[#2] leaf {455}
------------------------------------------------------------------- 
\end{verbatim}
    \end{minipage}\end{lrbox}

\newsavebox\treethree
\begin{lrbox}{\treethree}\begin{minipage}{\textwidth}
\begin{verbatim}
--------------------------------------------------------------------------
(name = #0.2 | depth = 8 | #nodes = 17)
[#0:root] internal {5278} (juv_other_count <= 1 ? #1 : #2)
|-[#1] internal {4947} (age <= 56 ? #3 : #4)
| |-[#3] internal {4596} (sex in {1} ? #5 : #6)
| | |-[#5] internal {3632} (age_cat_Lessthan25 in {0} ? #7 : #8)
| | | |-[#7] internal {2759} (age <= 24 ? #9 : #10)
| | | | |-[#9] leaf {34}
| | | | \-[#10] internal {2725} (age_cat_25-45 in {0} ? #11 : #12)
| | | |   |-[#11] leaf {620}
| | | |   \-[#12] internal {2105} (age <= 43 ? #13 : #14)
| | | |     |-[#13] internal {2016} (c_charge_degree_F in {1} ? #15 : #16)
| | | |     | |-[#15] leaf {1330}
| | | |     | \-[#16] leaf {686}
| | | |     \-[#14] leaf {89}
| | | \-[#8] leaf {873}
| | \-[#6] leaf {964}
| \-[#4] leaf {351}
\-[#2] leaf {331}
--------------------------------------------------------------------------
\end{verbatim}
    \end{minipage}\end{lrbox}

\newsavebox\treetwo
\begin{lrbox}{\treetwo}\begin{minipage}{\textwidth}
\begin{verbatim}
------------------------------------------------------------------ 
(name = #0.1 | depth = 7 | #nodes = 15)
[#0:root] internal {5278} (juv_misd_count <= 1 ? #1 : #2)
|-[#1] internal {4961} (priors_count <= 23 ? #3 : #4)
| |-[#3] internal {4837} (priors_count <= 6 ? #5 : #6)
| | |-[#5] internal {3945} (priors_count <= 3 ? #7 : #8)
| | | |-[#7] internal {3293} (c_charge_degree_M in {0} ? #9 : #10)
| | | | |-[#9] internal {2012} (priors_count <= 2 ? #11 : #12)
| | | | | |-[#11] internal {1761} (priors_count <= 0 ? #13 : #14)
| | | | | | |-[#13] leaf {839}
| | | | | | \-[#14] leaf {922}
| | | | | \-[#12] leaf {251}
| | | | \-[#10] leaf {1281}
| | | \-[#8] leaf {652}
| | \-[#6] leaf {892}
| \-[#4] leaf {124}
\-[#2] leaf {317}
------------------------------------------------------------------ 
\end{verbatim}
    \end{minipage}\end{lrbox}

\newsavebox\treefour
\begin{lrbox}{\treefour}\begin{minipage}{\textwidth}
\begin{verbatim}
------------------------------------------------------------------------
(name = #0.0 | depth = 8 | #nodes = 17)
[#0:root] internal {5278} (juv_fel_count <= 0 ? #1 : #2)
|-[#1] internal {5087} (age_cat_Greaterthan45 in {0} ? #3 : #4)
| |-[#3] internal {4004} (age <= 41 ? #5 : #6)
| | |-[#5] internal {3760} (race_African-American in {0} ? #7 : #8)
| | | |-[#7] leaf {1321}
| | | \-[#8] internal {2439} (age <= 19 ? #9 : #10)
| | |   |-[#9] leaf {13}
| | |   \-[#10] internal {2426} (age <= 39 ? #11 : #12)
| | |     |-[#11] internal {2347} (age <= 34 ? #13 : #14)
| | |     | |-[#13] internal {2024} (two_year_recid in {0} ? #15 : #16)
| | |     | | |-[#15] leaf {886}
| | |     | | \-[#16] leaf {1138}
| | |     | \-[#14] leaf {323}
| | |     \-[#12] leaf {79}
| | \-[#6] leaf {244}
| \-[#4] leaf {1083}
\-[#2] leaf {191}
------------------------------------------------------------------------
\end{verbatim}
  \end{minipage}\end{lrbox}
\newsavebox\treefive
\begin{lrbox}{\treefive}\begin{minipage}{\textwidth}
\begin{verbatim}
---------------------------------------------------------------------------
(name = #0.1 | depth = 8 | #nodes = 17)
[#0:root] internal {5278} (juv_misd_count <= 2 ? #1 : #2)
|-[#1] internal {5236} (juv_other_count <= 2 ? #3 : #4)
| |-[#3] internal {5193} (age <= 52 ? #5 : #6)
| | |-[#5] internal {4660} (priors_count <= 26 ? #7 : #8)
| | | |-[#7] internal {4651} (sex in {1} ? #9 : #10)
| | | | |-[#9] internal {3731} (juv_misd_count <= 0 ? #11 : #12)
| | | | | |-[#11] internal {3491} (race_Caucasian in {0} ? #13 : #14)
| | | | | | |-[#13] internal {2186} (age_cat_Lessthan25 in {1} ? #15 : #16)
| | | | | | | |-[#15] leaf {573}
| | | | | | | \-[#16] leaf {1613}
| | | | | | \-[#14] leaf {1305}
| | | | | \-[#12] leaf {240}
| | | | \-[#10] leaf {920}
| | | \-[#8] leaf {9}
| | \-[#6] leaf {533}
| \-[#4] leaf {43}
\-[#2] leaf {42}
---------------------------------------------------------------------------
\end{verbatim}
  \end{minipage}\end{lrbox}

\setlength\tabcolsep{2pt}
\begin{table}[t]
  \centering
  \resizebox{\columnwidth}{!}{\begin{tabular}{cc}
                   \resizebox{0.21\textwidth}{!}{\begin{minipage}{\textwidth}
                     \includegraphics[trim=150bp 60bp 290bp 110bp,clip,width=\columnwidth]{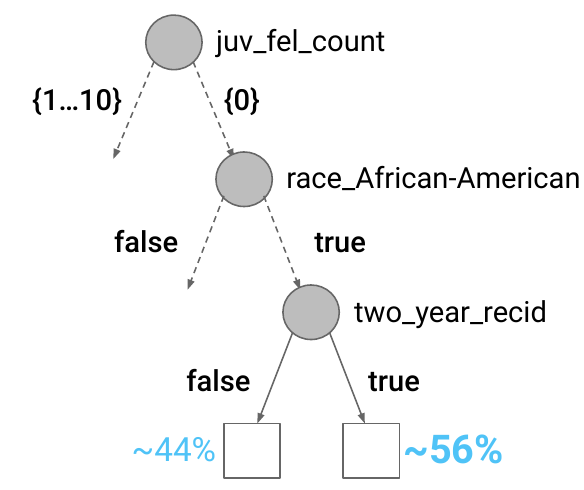}
                     \end{minipage}} & \resizebox{0.35\textwidth}{!}{\usebox\treefour} 
  \end{tabular}}
 \bignegspace
  \caption{\textsc{scrutinize}: Tree in an example of Generative Forest on sensitive domain \domainname{compas}, with 17 nodes and part of its graph sketched following Figure \ref{fig:gf-1}'s convention, modeling a bias learned from data. Ensemble of such small trees can be \textit{very} accurate: on missing data imputation, they compete with or beat  \mice~comparatively using 7 000 trees for imputation (see text for details).}
  \label{tab:compass-gf-1}
   \bignegspace
 \bignegspace
  \end{table}

Table \ref{tab:compass-gf-1} provides experimental results on sensitive real world domain \domainname{compas}. It shows that it is easy to flag potential bias / fairness issues in data by directly estimating conditional probabilities: considering the \geot~of this example and ignoring variable \texttt{age} for simplicity, conditionally to having 0 felony count and race being African-American, the probability of generating an observation with \texttt{two$\_$year$\_$recid = true} is $.562$ (=1138/2024).

\subsubsection{More examples of Table \ref{circgauss-intro} (\mainfile)}\label{sec-tables}

Tables \ref{tab:gen-full-sigma_cabs} completes Tables \ref{tab:gen-full-kc1} and \ref{tab:compass-gf-1} (main file), displaying that even a stochastic dependence can be accurately modeled.

\setlength\tabcolsep{0.5pt}

\begin{table}[t]
  \centering
  \resizebox{\columnwidth}{!}{{\tiny\begin{tabular}{c|cccc?c}\Xhline{2pt}
& $T$=$J$=20($^*$) & $T$=$J$=100($^*$) & $T$=100,$J$=200 & $T$=100,$J$=500 & Ground truth\\ \Xhline{2pt}
    \rotatebox{90}{Ens. of gen. trees}  &  \includegraphics[trim=0bp 0bp 0bp 0bp,clip,width=\cinq\columnwidth]{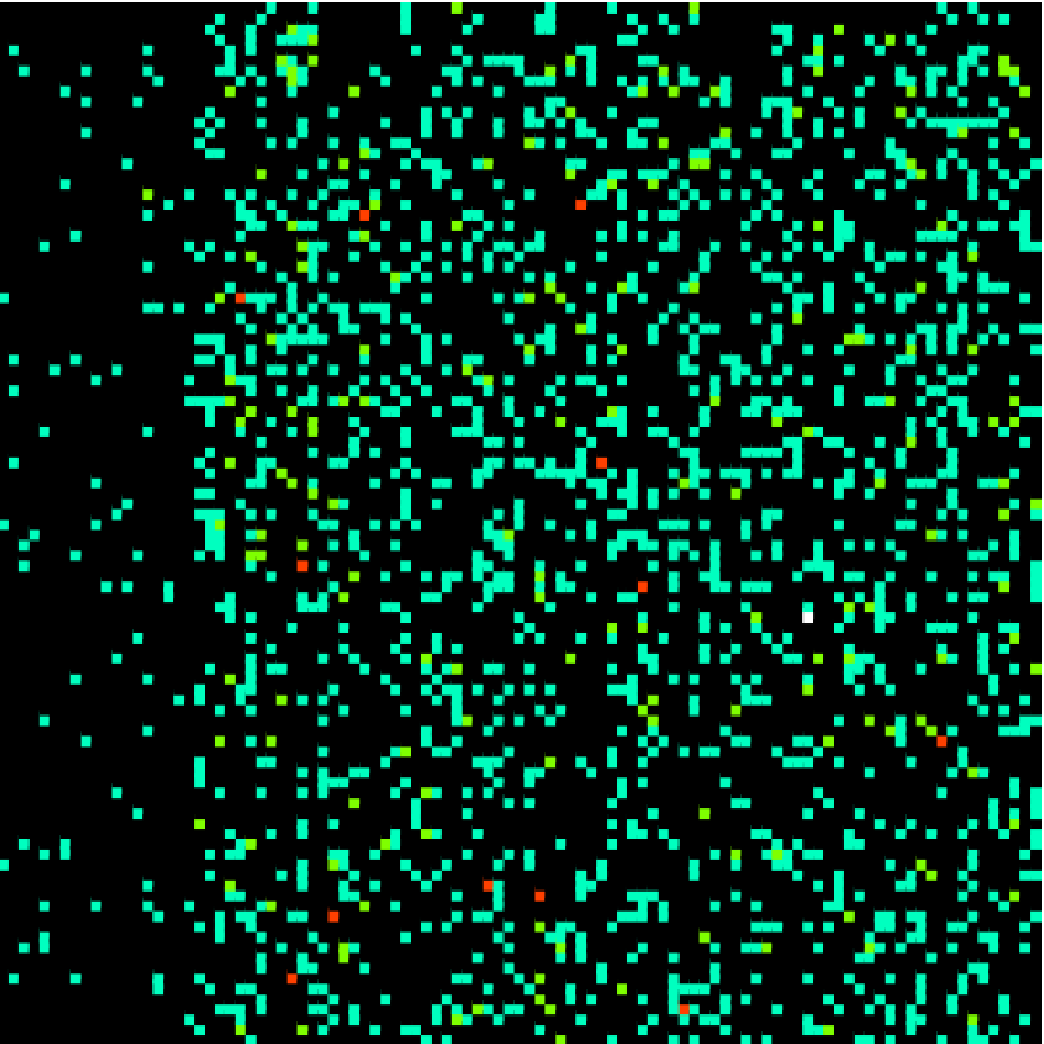}  &  \includegraphics[trim=0bp 0bp 0bp 0bp,clip,width=\cinq\columnwidth]{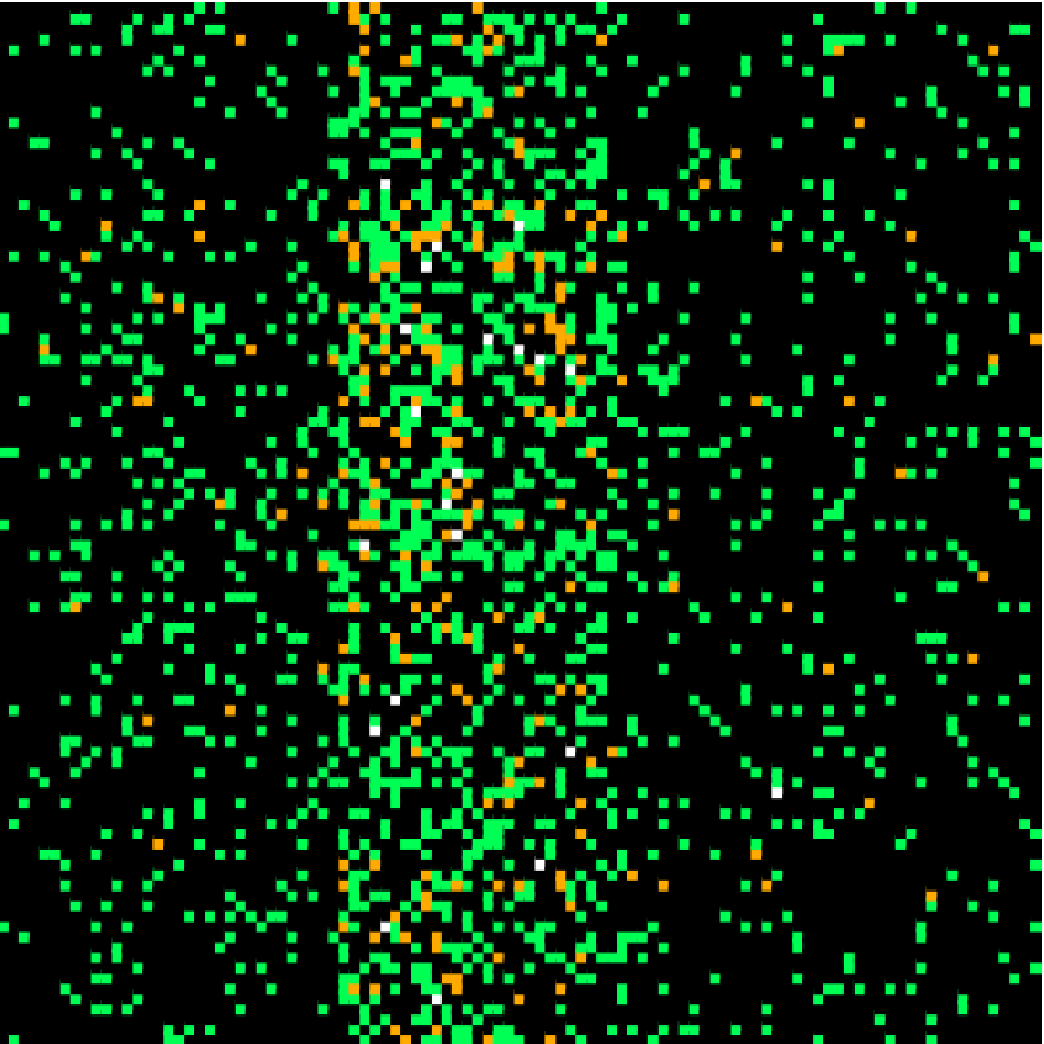}  &  \includegraphics[trim=0bp 0bp 0bp 0bp,clip,width=\cinq\columnwidth]{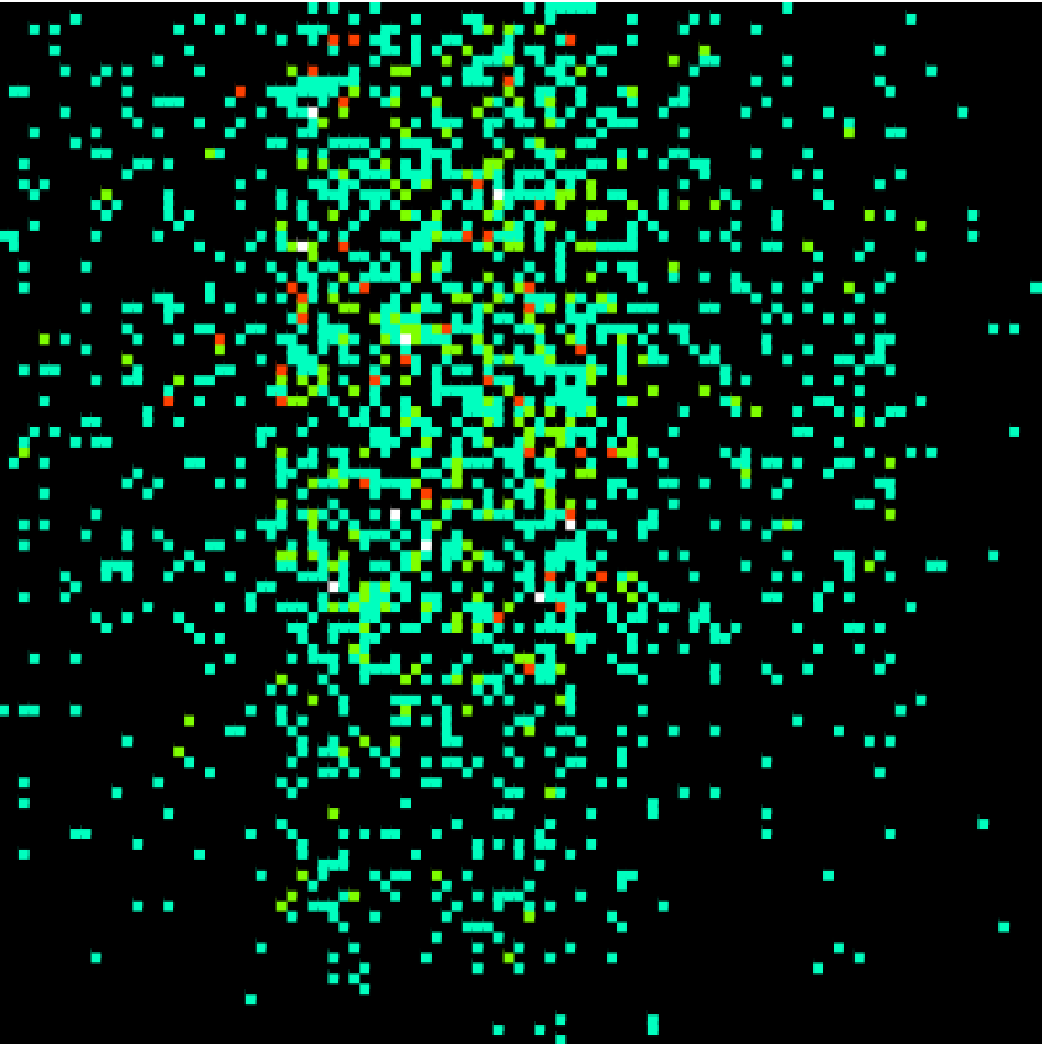}  &  \includegraphics[trim=0bp 0bp 0bp 0bp,clip,width=\cinq\columnwidth]{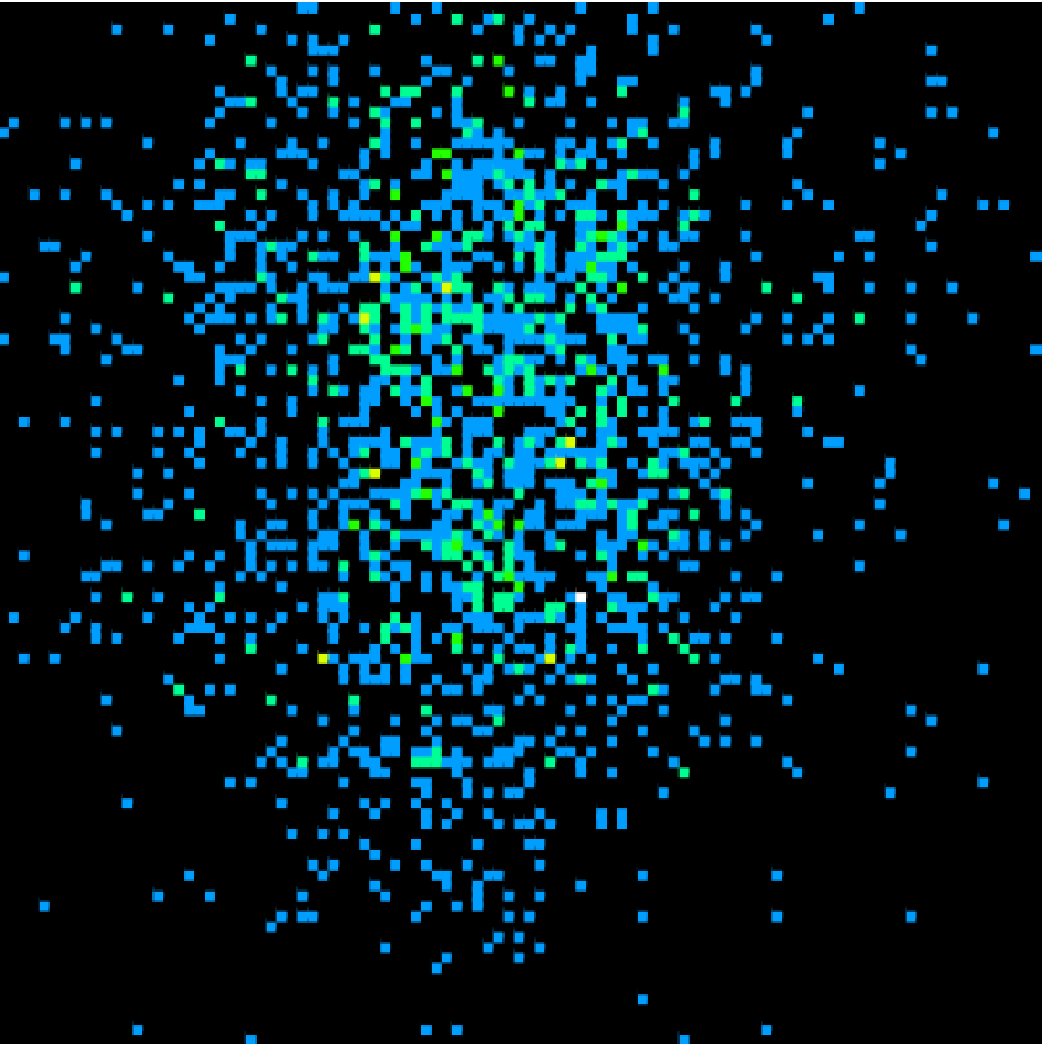}  &  \includegraphics[trim=0bp 0bp 0bp 0bp,clip,width=\cinq\columnwidth]{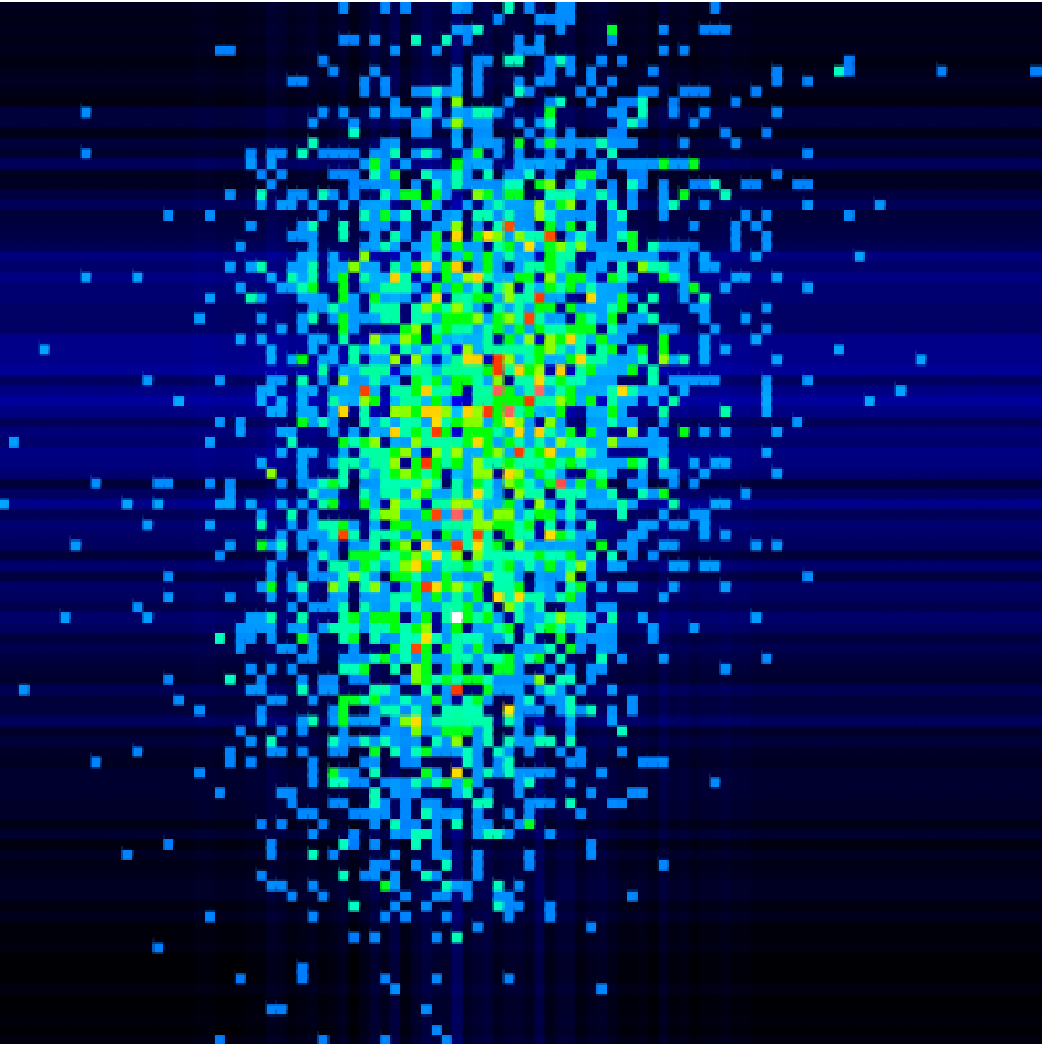}  \\\Xhline{2pt}
    \rotatebox{90}{Generative forests}  &  \includegraphics[trim=0bp 0bp 0bp 0bp,clip,width=\cinq\columnwidth]{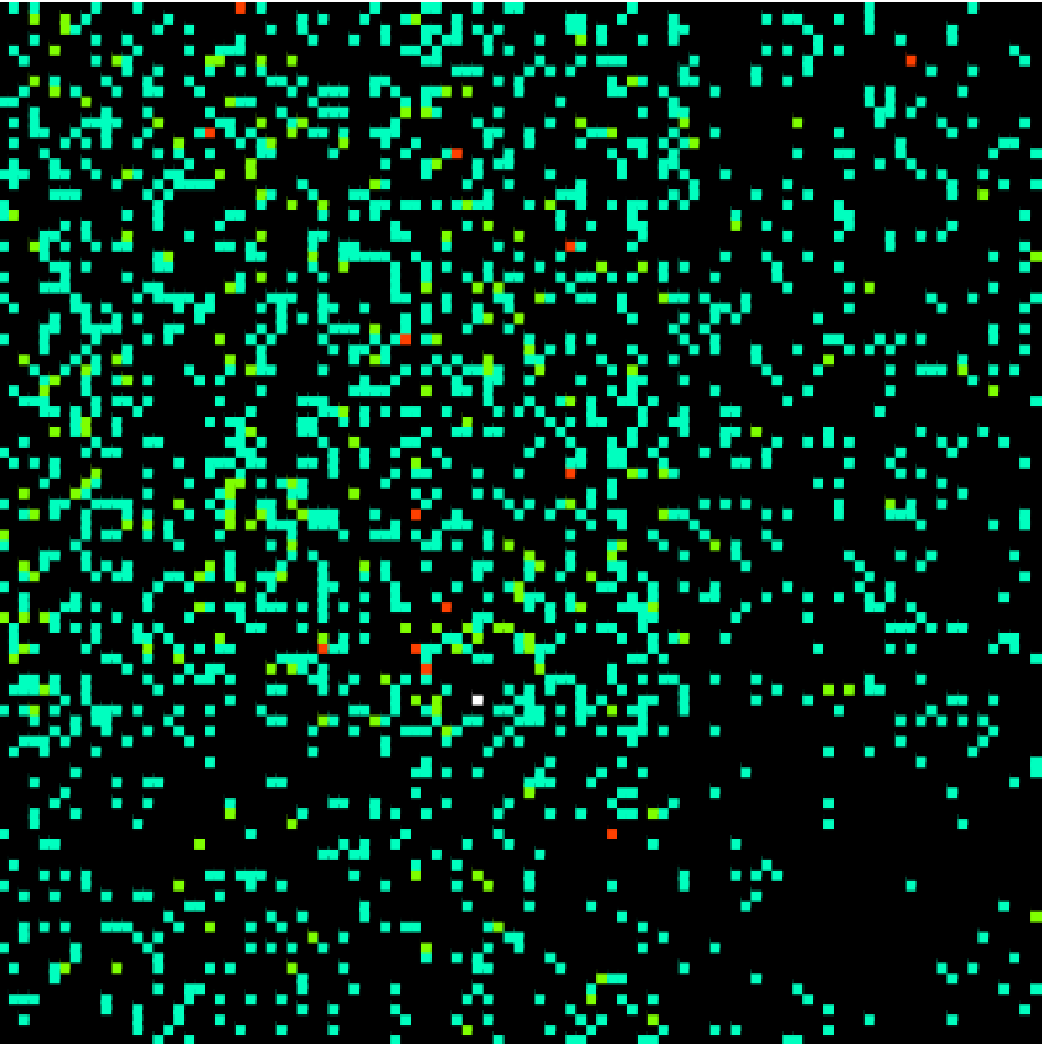}  &  \includegraphics[trim=0bp 0bp 0bp 0bp,clip,width=\cinq\columnwidth]{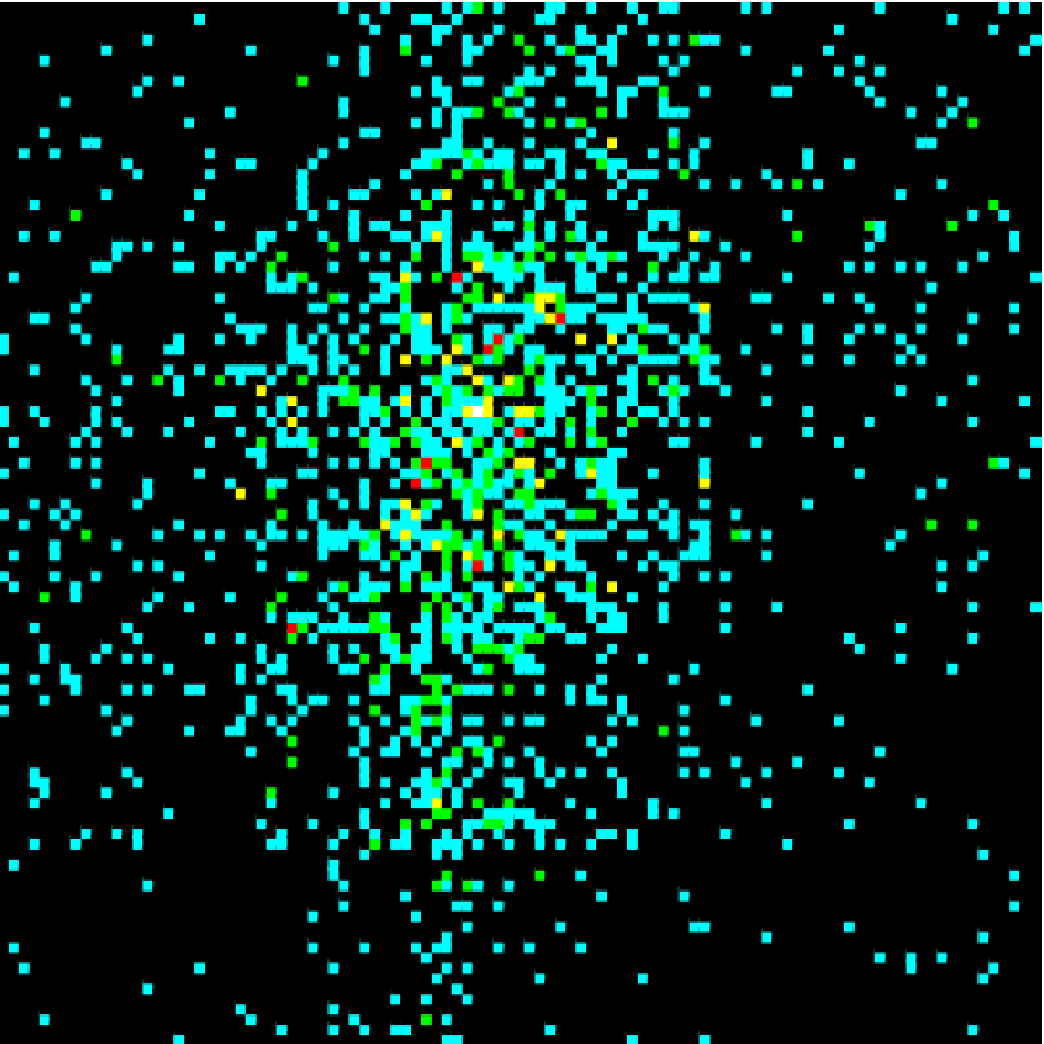}  &  \includegraphics[trim=0bp 0bp 0bp 0bp,clip,width=\cinq\columnwidth]{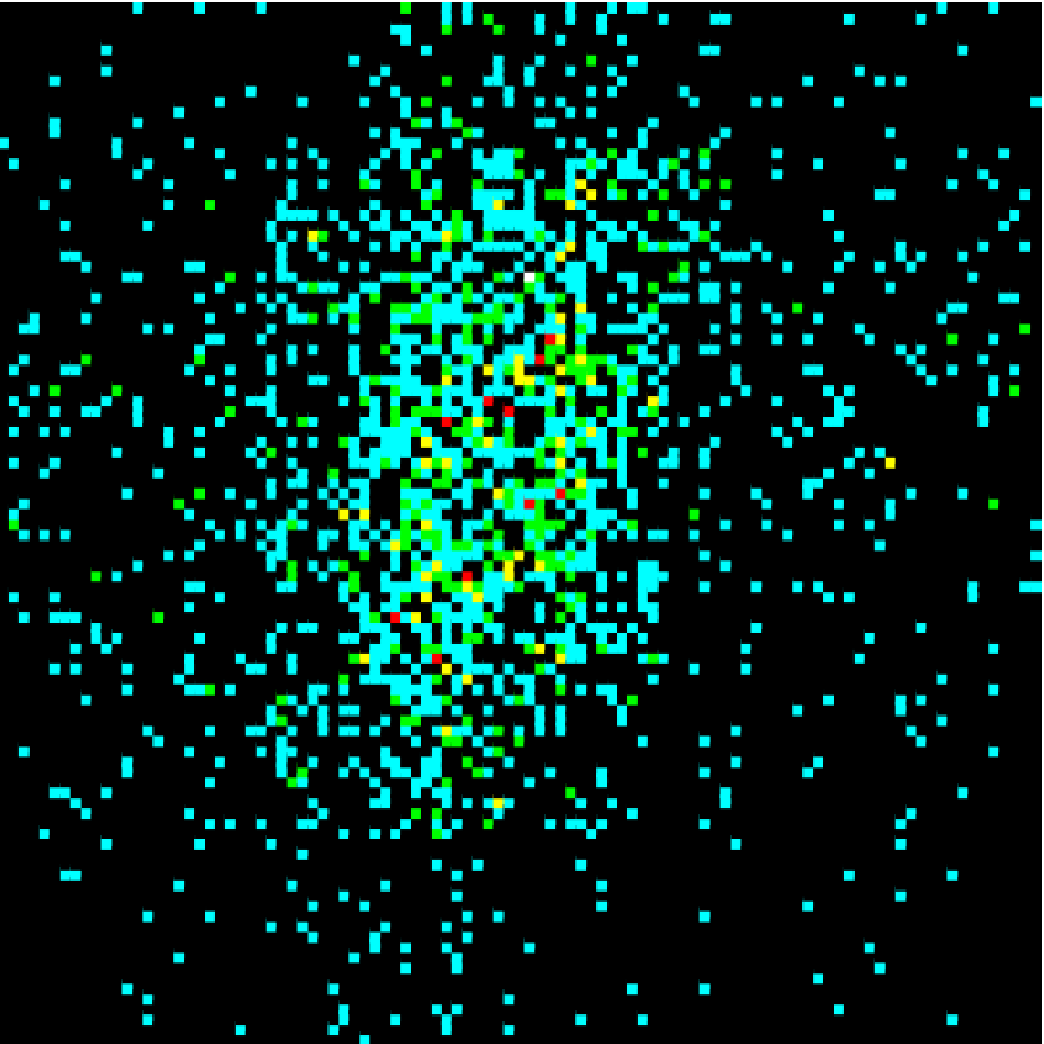}  &  \includegraphics[trim=0bp 0bp 0bp 0bp,clip,width=\cinq\columnwidth]{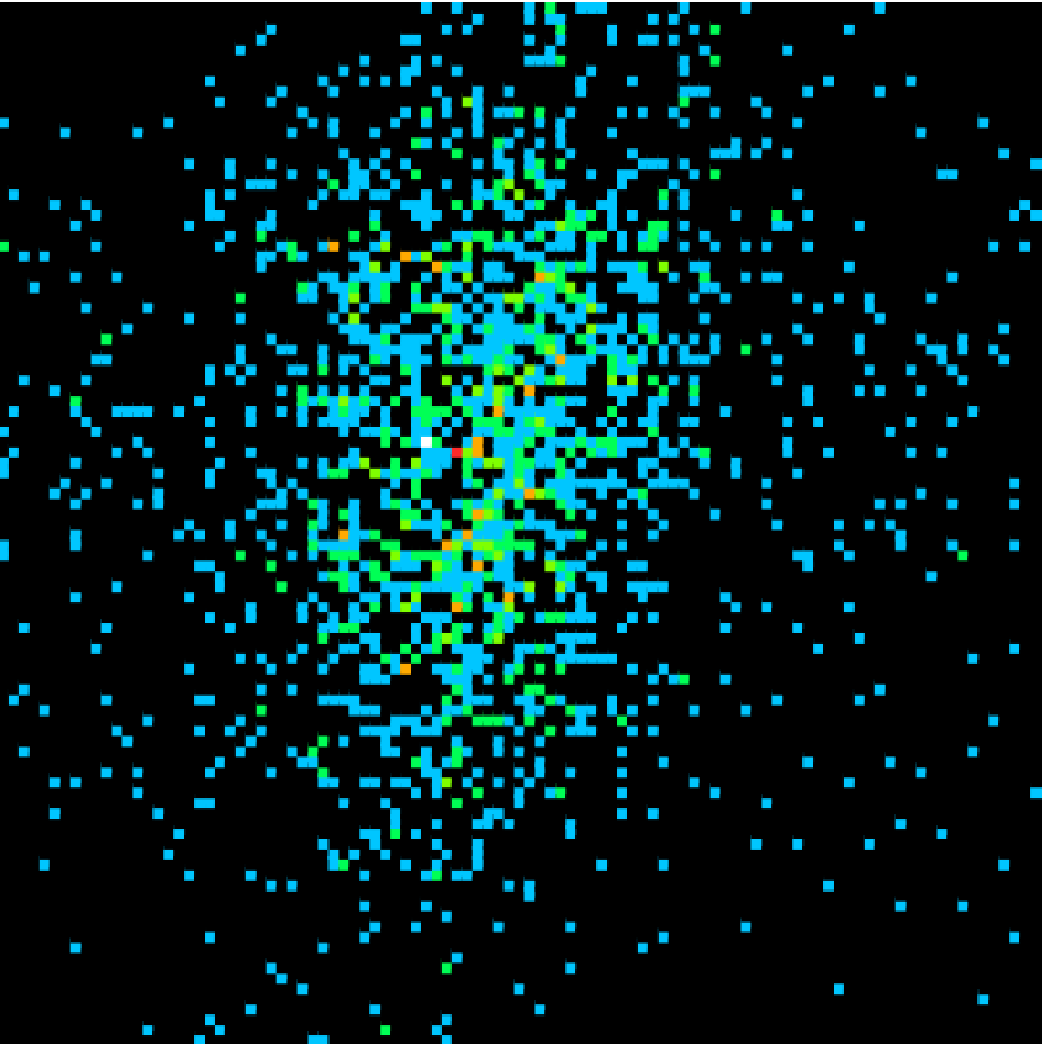}  &  \includegraphics[trim=0bp 0bp 0bp 0bp,clip,width=\cinq\columnwidth]{Experiments/generation_plots/sigma_cabs/sigma_cabs_impute5_geot_generated_S1_I100_T100_GEOT_S1_I500_T100_GEOT1__X_Life_Style_Index_Y_Customer_Rating_jointdensity_plot_domaindensity.eps}  \\\Xhline{2pt}
                                 \end{tabular}}
                                 }
  \caption{2D density plots ($x=$ Life Style Index and $y=$ Customer Rating) for data generated on \texttt{sigma$\_$cabs}, from models learned with $5\%$ missing features (MCAR), using using ensembles of generative trees (top) and generative forests (bottom), for varying total number of trees $T$ and total number of splits $J$ (columns). "$^*$" = all trees are stumps. The rightmost column recalls the domain ground truth for comparison. Each generated dataset contains $m=2000$ observations.}
    \label{tab:gen-full-sigma_cabs}
\end{table}

We provide in Table \ref{tab:tree-vs-stumps} the density learned on all our four 2D (for easy plotting) simulated domains and not just \domainname{circgauss} as in Table \ref{circgauss-intro} (\mainfile). We see that the observations made in Table \ref{circgauss-intro} (\mainfile) can be generalized to all domains. Even when the results are less different between 50 stumps in a generative forest and 1 generative tree with 50 splits for \domainname{ringgauss} and \domainname{randgauss}, the difference on \domainname{gridgauss} is stark, 1 generative tree merging many of the Gaussians.

\setlength\tabcolsep{1pt}
\newcommand{\trois}{0.25}
\begin{table}
  \centering
  \includegraphics[trim=0bp 0bp 0bp 0bp,clip,width=0.8\columnwidth]{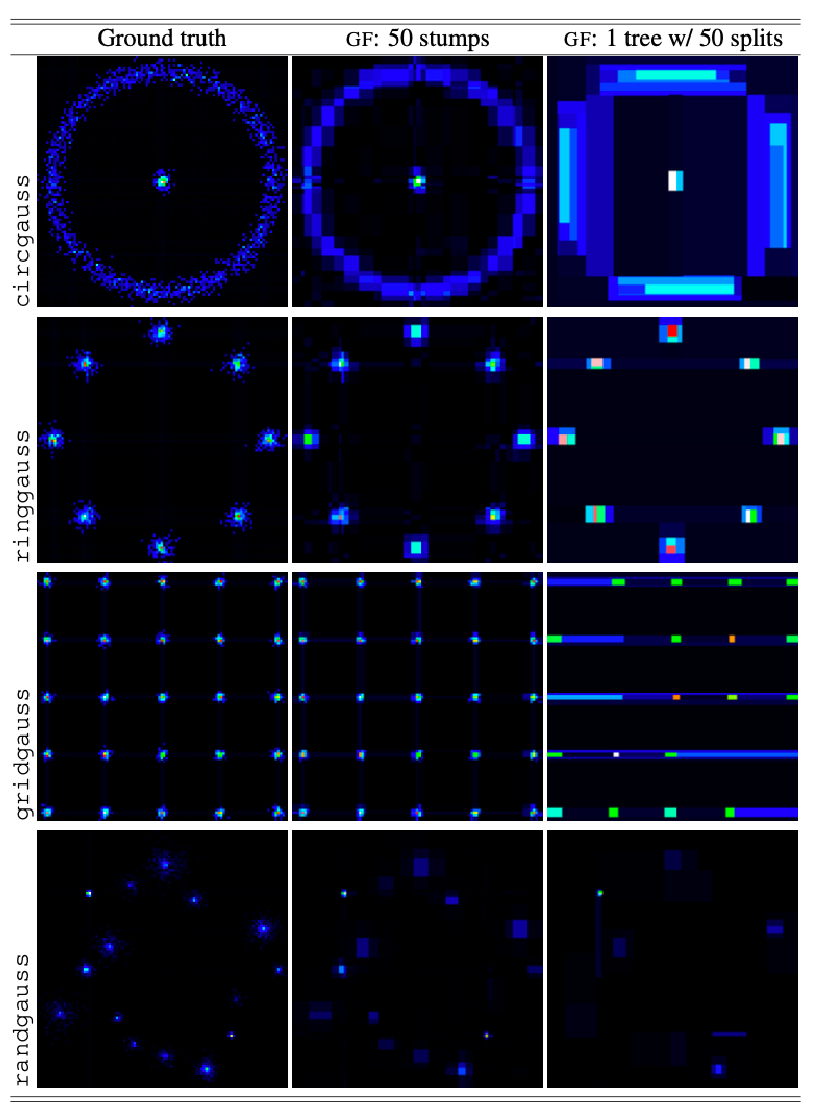}
    \caption{Comparison of the density learned by a Generative Forest (\geot) consisting of a single tree learned by \topdownGT~with $n$ total splits (\textit{right}) and a set of $n$ \geot~stumps (1 split per tree) learned by \topdownGT~(\textit{center}). The left picture is the domain, represented using the same heatmap. On each domain, 5\% of the data is missing (MCAR = missing completely at random).}
    \label{tab:tree-vs-stumps}
  \vspace{-0.2cm}
\end{table}

\newpage

\subsubsection{The generative forest of Table \ref{circgauss-intro} (\mainfile) developed further} \label{sec-gf}
One might wonder how a set of stumps gets to accurately fit the domain as the number of stumps increases. Table \ref{tab:gen_succ_circgauss} provides an answer. From this experiment, we can see that 16 stumps are enough to get the center mode. The ring shape takes obviously more iterations to represent but still, is takes a mere few dozen stumps to clearly get an accurate shape, the last iterations just finessing the fitting.

\setlength\tabcolsep{0pt}

\newcommand{\dix}{0.1}

\begin{table}
  \centering
  \includegraphics[trim=0bp 0bp 0bp 0bp,clip,width=\columnwidth]{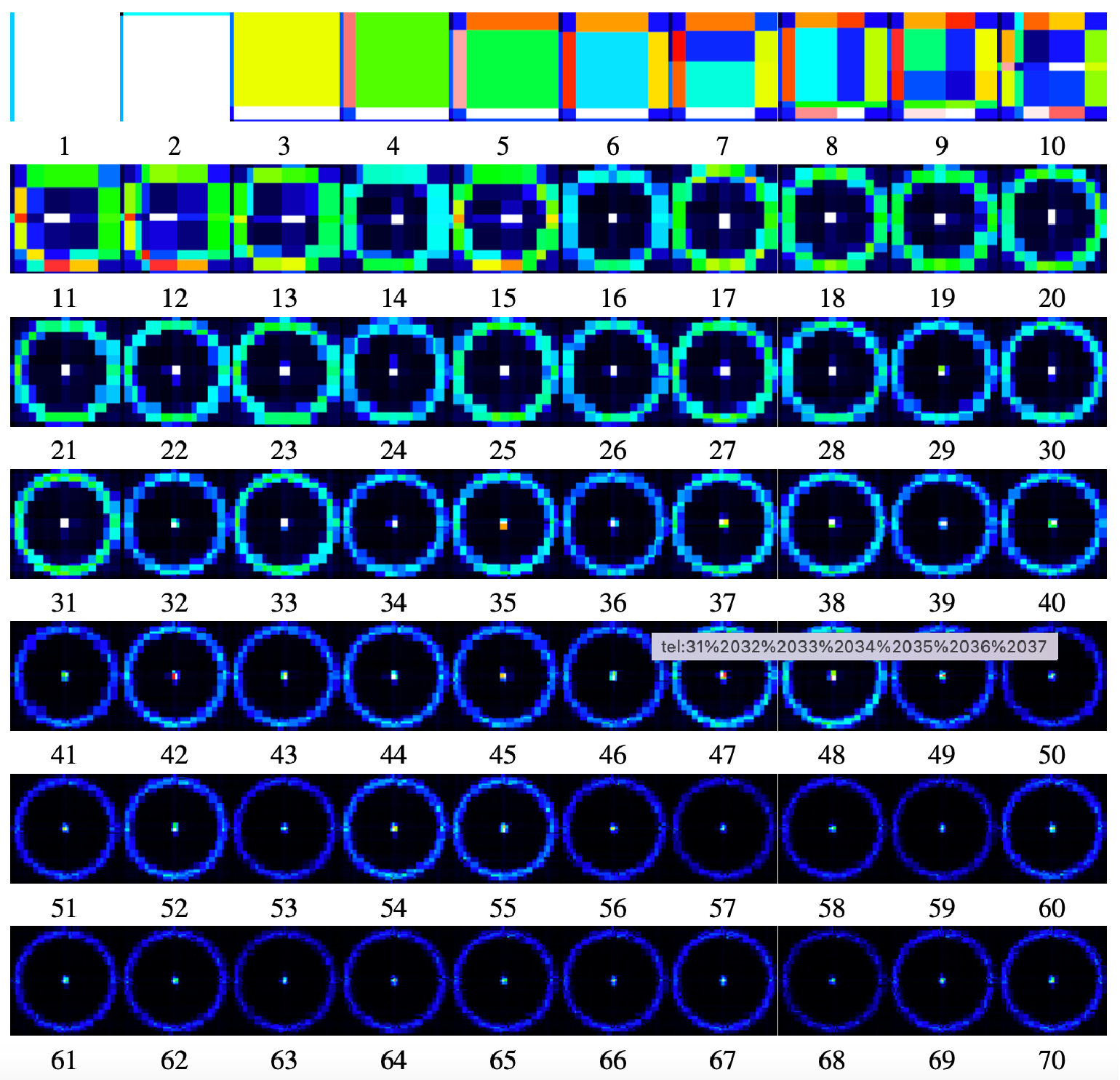}
    \caption{Densities learned on \texttt{circgauss} by \topdownGT~with generative forests consisting of stumps, for values of $T=J$ in $\{1, 2, ... 70\}$. The models quickly capture the ring and the middle dense Gaussian. The boxy-shape of the densities, due to the axis-parallel splits, get less noticeable as $T=J$ exceeds a few dozen.}
    \label{tab:gen_succ_circgauss}
\end{table}

\newpage

\subsubsection{Experiment \textsc{lifelike} \textit{in extenso}} \label{sec-tables-lifelike}

  \setlength\tabcolsep{4pt}

\begin{table*}[h]
    \centering
    \resizebox{1.0\textwidth}{!}{
      \begin{tabular}{l?rrr?rrr?rrr?rrr}\Xhline{2pt}
       Domain & \multicolumn{3}{c?}{Sinkhorn$\downarrow$} & \multicolumn{3}{c?}{Coverage$\uparrow$} & \multicolumn{3}{c?}{Density$\uparrow$} & \multicolumn{3}{c}{F1 measure$\downarrow$} \\
 & \multicolumn{1}{c}{us (\geot)} & \multicolumn{1}{c}{ARF} & pval & \multicolumn{1}{c}{us (\geot)} & \multicolumn{1}{c}{ARF} & pval & \multicolumn{1}{c}{us (\geot)} & ARF  & pval & \multicolumn{1}{c}{us (\geot)} & \multicolumn{1}{c}{ARF} & pval \\\hline
\domainname{ring} & \better{0.283$\pm$0.005} & 0.286$\pm$0.007 & 0.74 & \better{0.964$\pm$0.022} & 0.9600$\pm$0.008 & 0.53 & \better{0.999$\pm$0.055} & 0.976$\pm$0.030 & 0.42 & \better{0.060$\pm$0.047} & 0.086$\pm$0.047 & 0.23\\\hline
\domainname{circ} & 0.356$\pm$0.002& \better{0.355$\pm$0.005} & 0.74 & 0.954$\pm$0.014 & \better{0.962$\pm$0.021} & 0.27 & \better{0.968$\pm$0.027} & 0.954$\pm$0.011 & 0.04 & 0.520$\pm$0.021 & \better{0.507$\pm$0.032} & 0.67\\\hline
\domainname{grid} & \better{0.392$\pm$0.002} & 0.394$\pm$0.002 & 0.34 & \better{0.964$\pm$0.006} & 0.908$\pm$0.010 & \tugreen & \better{0.953$\pm$0.062} & 0.630$\pm$0.050 & \tugreen & \better{0.034$\pm$0.010} & 0.043$\pm$0.012 & 0.26\\\hline
\domainname{rand}  & 0.290$\pm$0.007 & \better{0.288$\pm$0.002} & 0.65 & 0.954$\pm$0.010 & \better{0.953$\pm$0.004} & 0.94 & \better{0.936$\pm$0.029} & 0.940$\pm$0.027 & 0.74 & \better{0.020$\pm$0.005} & 0.029$\pm$0.010 & 0.14 \\\hline
\domainname{wred}  & \better{1.027$\pm$0.037} & 1.099$\pm$0.031 & \tugreen & \better{0.954$\pm$0.013} & 0.929$\pm$0.010 & 0.04 & \better{0.874$\pm$0.042} & 0.801$\pm$0.028 & \tugreen & \better{0.503$\pm$0.023} & 0.531$\pm$0.028 & 0.01 \\\hline
\domainname{wwhi} & \better{1.120$\pm$0.013} & 1.150$\pm$0.021 & 0.03 & \better{0.952$\pm$0.008} & 0.946$\pm$0.012 & 0.22 & 0.940$\pm$0.019 & \better{0.941$\pm$0.013} & 0.88 & 0.510$\pm$0.027 & \better{0.500$\pm$0.027} & 0.09\\\hline
\domainname{comp} & 0.538$\pm$0.015 & \better{0.535$\pm$0.033} & 0.90 & \better{0.546$\pm$0.026} & 0.560$\pm$0.037 & 0.10 & \better{0.424$\pm$0.014} & 0.440$\pm$0.011 & 0.03 & \better{0.500$\pm$0.025} & 0.520$\pm$0.020 & 0.33\\\hline
\domainname{arti} & \better{0.830$\pm$0.007} & 0.849$\pm$0.010 & \tugreen & \better{0.932$\pm$0.018} & 0.892$\pm$0.014 & \tugreen & \better{0.897$\pm$0.004} & 0.747$\pm$0.013 & \tugreen & \better{0.457$\pm$0.057} & 0.512$\pm$0.037 & 0.02\\\Xhline{2pt}
         & \multicolumn{1}{c}{us (\geot)} & \multicolumn{1}{c}{CT} & pval & \multicolumn{1}{c}{us (\geot)} & \multicolumn{1}{c}{CT} & pval & \multicolumn{1}{c}{us (\geot)} & \multicolumn{1}{c}{CT}  & pval & \multicolumn{1}{c}{us (\geot)} & \multicolumn{1}{c}{CT} & pval \\\hline
\domainname{ring} & \better{0.283$\pm$0.005} & 0.351$\pm$0.042 & 0.02 & \better{0.964$\pm$0.022} & 0.798$\pm$0.05 & \tugreen & \better{0.999$\pm$0.055} & 0.757$\pm$0.025 & \tugreen & \better{0.060$\pm$0.047} & 0.100$\pm$0.073 & 0.75 \\\hline
\domainname{circ} & \better{0.356$\pm$0.002} &  0.435$\pm$0.075 & 0.08 & \better{0.954$\pm$0.014} & 0.734$\pm$0.041 & \tugreen & \better{0.968$\pm$0.027} & 0.401$\pm$0.033 & \tugreen & \better{0.520$\pm$0.021} & 0.746$\pm$0.033 & \tugreen \\\hline
\domainname{grid} & \better{0.392$\pm$0.002} & 0.408$\pm$0.019  & 0.12 & \better{0.964$\pm$0.006} & 0.828$\pm$0.053 & \tugreen & \better{0.953$\pm$0.062} & 0.649$\pm$0.058 & \tugreen & 0.034$\pm$0.010 & 0.034$\pm$0.011 & 0.98 \\\hline
\domainname{rand} & \better{0.290$\pm$0.007} & 0.327$\pm$0.024 & 0.01 & \better{0.954$\pm$0.010} & 0.659$\pm$0.049 & \tugreen & \better{0.936$\pm$0.029} & 0.582$\pm$0.035 & \tugreen & \better{0.020$\pm$0.005} & 0.079$\pm$0.053 & 0.08 \\\hline
\domainname{wred} & \better{1.027$\pm$0.037} & 1.384$\pm$0.047 & \tugreen & \better{0.954$\pm$0.013} & 0.808$\pm$0.016 & \tugreen & \better{0.874$\pm$0.042} & 0.589$\pm$0.129 & \tugreen & \better{0.503$\pm$0.023} & 0.654$\pm$0.030 & \tugreen \\\hline
\domainname{wwhi} & \better{1.120$\pm$0.013} & 1.158$\pm$0.009 & 0.01 & \better{0.952$\pm$0.008} & 0.894$\pm$0.026 & \tugreen & 0.940$\pm$0.019 & \better{0.953$\pm$0.035} & 0.50 & \better{0.510$\pm$0.027} & 0.581$\pm$0.043 & \tugreen \\\Xhline{2pt}
              & \multicolumn{1}{c}{us (\geot)} & \multicolumn{1}{c}{FF} & pval & \multicolumn{1}{c}{us (\geot)} & \multicolumn{1}{c}{FF} & pval & \multicolumn{1}{c}{us (\geot)} & \multicolumn{1}{c}{FF}  & pval & \multicolumn{1}{c}{us (\geot)} & \multicolumn{1}{c}{FF} & pval \\\hline
        \domainname{ring} & 0.283$\pm$0.005 & \better{0.276$\pm$0.001} & 0.09 & \better{0.964$\pm$0.022} & 0.957$\pm$0.020 & 0.31 & 0.999$\pm$0.055 & \better{1.045$\pm$0.020} & 0.07 & 0.060$\pm$0.047 & \better{0.051$\pm$0.031} & 0.75 \\\hline
\domainname{circ} & 0.356$\pm$0.003 & \better{0.354$\pm$0.003} & 0.30 & 0.954$\pm$0.014 & \better{0.956$\pm$0.015} & 0.74 & 0.968$\pm$0.027 & \better{0.989$\pm$0.027} & 0.06 & \better{0.520$\pm$0.021} & 0.530$\pm$0.028 & 0.55  \\\hline
\domainname{grid} & \better{0.392$\pm$0.002} & 0.393$\pm$0.001 & 0.53 & \better{0.964$\pm$0.006} & 0.954$\pm$0.013 & 0.07 & 0.953$\pm$0.062 & \better{1.013$\pm$0.050} & 0.01 & \better{0.034$\pm$0.010} & 0.045$\pm$0.007 & 0.02 \\\hline
\domainname{rand} & 0.290$\pm$0.007 & \better{0.288$\pm$0.001} & 0.55 & \better{0.954$\pm$0.010} & 0.927$\pm$0.011 & 0.04 & 0.936$\pm$0.029 & \better{1.000$\pm$0.008} & 0.01 & \better{0.020$\pm$0.005} & 0.028$\pm$0.008 & 0.11\\\hline
\domainname{wred} & \better{1.027$\pm$0.037} & 1.030$\pm$0.029 & 0.84 & 0.954$\pm$0.013 & \better{0.955$\pm$0.018} & 0.95 & 0.874$\pm$0.042 & \better{1.009$\pm$0.034} & \tdred & 0.503$\pm$0.023 & \better{0.458$\pm$0.052} & 0.13  \\\hline
\domainname{wwhi} & 1.120$\pm$0.013 & \better{1.097$\pm$0.007} & 0.05 & \better{0.952$\pm$0.008} & 0.945$\pm$0.007 & 0.18 & 0.940$\pm$0.019 & \better{0.970$\pm$0.023} & 0.03 & 0.510$\pm$0.027 & \better{0.498$\pm$0.041} & 0.48 \\\hline
\domainname{comp} & \better{0.537$\pm$0.015} & 0.891$\pm$0.007 & \tugreen & 0.546$\pm$0.026 & \better{0.548$\pm$0.027} & 0.74 & \better{0.424$\pm$0.014} & 0.404$\pm$0.013 & 0.03 & \better{0.500$\pm$0.025} & 0.510$\pm$0.032 & 0.36  \\\hline
\domainname{arti}  & \better{0.829$\pm$0.006} & 0.834$\pm$0.016 & 0.51 & \better{0.932$\pm$0.018} & 0.879$\pm$0.008 &\tugreen & \better{0.897$\pm$0.004} & 0.774$\pm$0.016 &\tugreen & \better{0.457$\pm$0.057} & 0.530$\pm$0.032 & \tugreen  \\\Xhline{2pt}
              & \multicolumn{1}{c}{us (\geot)} & \multicolumn{1}{c}{VC-G} & pval & \multicolumn{1}{c}{us (\geot)} & \multicolumn{1}{c}{VC-G} & pval & \multicolumn{1}{c}{us (\geot)} & \multicolumn{1}{c}{VC-G}  & pval & \multicolumn{1}{c}{us (\geot)} & \multicolumn{1}{c}{VC-G} & pval \\\hline
        \domainname{ring} & \better{0.283$\pm$0.005} & 0.392$\pm$0.005  & \tugreen & \better{0.964$\pm$0.022} & 0.364$\pm$0.037  & \tugreen & \better{0.999$\pm$0.055} & 0.132$\pm$0.013  & \tugreen & \better{0.060$\pm$0.047} & 0.291$\pm$0.035 & \tugreen \\\hline
\domainname{circ} & \better{0.356$\pm$0.002} &  0.415$\pm$0.012  & \tugreen & \better{0.954$\pm$0.014} &  0.620$\pm$0.020 & \tugreen & \better{0.968$\pm$0.027} & 0.265$\pm$0.015  & \tugreen & \better{0.520$\pm$0.021} &  0.805$\pm$0.012 & \tugreen \\\hline
\domainname{grid} & \better{0.392$\pm$0.002} & 0.400$\pm$0.003  & \tugreen & \better{0.964$\pm$0.006} & 0.165$\pm$0.037  & \tugreen & \better{0.953$\pm$0.062} & 0.042$\pm$0.012  & \tugreen & \better{0.034$\pm$0.010} &  0.065$\pm$0.002 & \tugreen \\\hline
\domainname{rand} & \better{0.290$\pm$0.007} &  0.414$\pm$0.003 & \tugreen & \better{0.954$\pm$0.010} & 0.317$\pm$0.033  & \tugreen & \better{0.936$\pm$0.029} & 0.156$\pm$0.018  & \tugreen & \better{0.020$\pm$0.005} & 0.224$\pm$0.019  & \tugreen \\\hline
\domainname{wred} & \better{1.027$\pm$0.037} & 1.105$\pm$0.042  & 0.01 & \better{0.954$\pm$0.013} & 0.913$\pm$0.013  & 0.02 & \better{0.874$\pm$0.042} & 0.821$\pm$0.041  & 0.04 & \better{0.503$\pm$0.023} & 0.537$\pm$0.007  & 0.03 \\\hline
\domainname{wwhi} & \better{1.120$\pm$0.013} &  1.181$\pm$0.019 & \tugreen & \better{0.952$\pm$0.008} & 0.938$\pm$0.008   & 0.06 & \better{0.940$\pm$0.019} & 0.907$\pm$0.026  & 0.01 & \better{0.510$\pm$0.027} & 0.527$\pm$0.028  & 0.03 \\\Xhline{2pt}
              & \multicolumn{1}{c}{us (\geot)} & \multicolumn{1}{c}{VC-C} & pval & \multicolumn{1}{c}{us (\geot)} & \multicolumn{1}{c}{VC-C} & pval & \multicolumn{1}{c}{us (\geot)} & \multicolumn{1}{c}{VC-C}  & pval & \multicolumn{1}{c}{us (\geot)} & \multicolumn{1}{c}{VC-C} & pval \\\hline
        \domainname{ring} & \better{0.283$\pm$0.005} & 0.394$\pm$0.016  & \tugreen & \better{0.964$\pm$0.022} & 0.346$\pm$0.042  & \tugreen & \better{0.999$\pm$0.055} & 0.126$\pm$0.020  & \tugreen & \better{0.060$\pm$0.047} &  0.326$\pm$0.015& \tugreen \\\hline
\domainname{circ} & \better{0.356$\pm$0.002} &  0.409$\pm$0.005  & \tugreen & \better{0.954$\pm$0.014} & 0.638$\pm$0.033  & \tugreen & \better{0.968$\pm$0.027} & 0.275$\pm$0.018  & \tugreen & \better{0.520$\pm$0.021} & 0.800$\pm$0.018  & \tugreen \\\hline
\domainname{grid} & \better{0.392$\pm$0.002} & 0.399$\pm$0.004  & 0.03 & \better{0.964$\pm$0.006} & 0.169$\pm$0.013  & \tugreen & \better{0.953$\pm$0.062} & 0.045$\pm$0.004  & \tugreen & \better{0.034$\pm$0.010} & 0.065$\pm$0.002 & \tugreen \\\hline
\domainname{rand} & \better{0.290$\pm$0.007} &  0.415$\pm$0.005 & \tugreen & \better{0.954$\pm$0.010} & 0.321$\pm$0.034  & \tugreen & \better{0.936$\pm$0.029} & 0.152$\pm$0.014  & \tugreen & \better{0.020$\pm$0.005} & 0.226$\pm$0.015  & \tugreen \\\hline
\domainname{wred} & \better{1.027$\pm$0.037} &  1.544$\pm$0.024 & \tugreen & \better{0.954$\pm$0.013} & 0.600$\pm$0.052  & \tugreen & \better{0.874$\pm$0.042} & 0.909$\pm$0.053  & 0.18 & \better{0.503$\pm$0.023} & 0.793$\pm$0.040  & \tugreen \\\hline
\domainname{wwhi} & \better{1.120$\pm$0.013} & 1.437$\pm$0.048  & \tugreen & \better{0.952$\pm$0.008} & 0.572$\pm$0.017  & \tugreen & \better{0.940$\pm$0.019} & 0.920$\pm$0.045  & 0.41 & \better{0.510$\pm$0.027} & 0.837$\pm$0.010  & \tugreen \\\Xhline{2pt}
              & \multicolumn{1}{c}{us (\geot)} & \multicolumn{1}{c}{VC-D} & pval & \multicolumn{1}{c}{us (\geot)} & \multicolumn{1}{c}{VC-D} & pval & \multicolumn{1}{c}{us (\geot)} & \multicolumn{1}{c}{VC-D}  & pval & \multicolumn{1}{c}{us (\geot)} & \multicolumn{1}{c}{VC-D} & pval \\\hline
        \domainname{ring} & \better{0.283$\pm$0.005} & 0.390$\pm$0.011  & \tugreen & \better{0.964$\pm$0.022} & 0.331$\pm$0.067  & \tugreen & \better{0.999$\pm$0.055} &  0.122$\pm$0.028 & \tugreen & \better{0.060$\pm$0.047} & 0.319$\pm$0.037 & \tugreen \\\hline
\domainname{circ} & \better{0.356$\pm$0.002} & 0.411$\pm$0.004   & \tugreen & \better{0.954$\pm$0.014} & 0.649$\pm$0.055  & \tugreen & \better{0.968$\pm$0.027} &  0.269$\pm$0.026 & \tugreen & \better{0.520$\pm$0.021} &  0.813$\pm$0.018 & \tugreen \\\hline
\domainname{grid} & \better{0.392$\pm$0.002} & 0.398$\pm$0.002  & \tugreen & \better{0.964$\pm$0.006} & 0.162$\pm$0.034  & \tugreen & \better{0.953$\pm$0.062} & 0.043$\pm$0.009  & \tugreen & \better{0.034$\pm$0.010} & 0.064$\pm$0.001  & \tugreen \\\hline
\domainname{rand} & \better{0.290$\pm$0.007} & 0.414$\pm$0.003  & \tugreen & \better{0.954$\pm$0.010} & 0.312$\pm$0.040  & \tugreen & \better{0.936$\pm$0.029} & 0.149$\pm$0.018  & \tugreen & \better{0.020$\pm$0.005} & 0.225$\pm$0.017  & \tugreen \\\hline
\domainname{wred} & \better{1.027$\pm$0.037} & 1.383$\pm$0.037   & \tugreen & \better{0.954$\pm$0.013} & 0.868$\pm$0.042  & 0.02 & \better{0.874$\pm$0.042} & 0.738$\pm$0.030  & \tugreen & \better{0.503$\pm$0.023} & 0.587$\pm$0.019  & \tugreen \\\hline
\domainname{wwhi} & \better{1.120$\pm$0.013} & 1.484$\pm$0.056  & \tugreen & \better{0.952$\pm$0.008} &  0.876$\pm$0.027 & \tugreen & \better{0.940$\pm$0.019} & 0.891$\pm$0.010  & \tugreen & \better{0.510$\pm$0.027} &  0.608$\pm$0.037 & \tugreen \\\Xhline{2pt}
              & \multicolumn{1}{c}{us (\geot)} & \multicolumn{1}{c}{VC-R} & pval & \multicolumn{1}{c}{us (\geot)} & \multicolumn{1}{c}{VC-R} & pval & \multicolumn{1}{c}{us (\geot)} & \multicolumn{1}{c}{VC-R}  & pval & \multicolumn{1}{c}{us (\geot)} & \multicolumn{1}{c}{VC-R} & pval \\\hline
        \domainname{ring} & \better{0.283$\pm$0.005} & 0.388$\pm$0.004  & \tugreen & \better{0.964$\pm$0.022} & 0.331$\pm$0.065  & \tugreen & \better{0.999$\pm$0.055} & 0.124$\pm$0.022  & \tugreen & \better{0.060$\pm$0.047} & 0.322$\pm$0.023 & \tugreen \\\hline
\domainname{circ} & \better{0.356$\pm$0.002} & 0.402$\pm$0.003   & \tugreen & \better{0.954$\pm$0.014} & 0.664$\pm$0.017  & \tugreen & \better{0.968$\pm$0.027} & 0.274$\pm$0.006  & \tugreen & \better{0.520$\pm$0.021} & 0.806$\pm$0.022  & \tugreen \\\hline
\domainname{grid} & \better{0.392$\pm$0.002} &  0.399$\pm$0.003 & 0.02 & \better{0.964$\pm$0.006} & 0.153$\pm$0.032  & \tugreen & \better{0.953$\pm$0.062} & 0.038$\pm$0.006 & \tugreen & \better{0.034$\pm$0.010} & 0.065$\pm$0.001  & \tugreen \\\hline
\domainname{rand} & \better{0.290$\pm$0.007} &  0.417$\pm$0.009 & \tugreen & \better{0.954$\pm$0.010} & 0.315$\pm$0.023  & \tugreen & \better{0.936$\pm$0.029} & 0.145$\pm$0.021  & \tugreen & \better{0.020$\pm$0.005} & 0.229$\pm$0.022  & \tugreen \\\hline
\domainname{wred} & \better{1.027$\pm$0.037} & 1.365$\pm$0.098  & \tugreen & \better{0.954$\pm$0.013} & 0.839$\pm$0.027  & \tugreen & \better{0.874$\pm$0.042} & 0.716$\pm$0.072  & 0.01 & \better{0.503$\pm$0.023} & 0.596$\pm$0.047  & 0.02 \\\hline
\domainname{wwhi} & \better{1.120$\pm$0.013} &  1.353$\pm$0.038 & \tugreen & \better{0.952$\pm$0.008} &  0.747$\pm$0.014 & \tugreen & \better{0.940$\pm$0.019} & 0.942$\pm$0.029  & 0.82 & \better{0.510$\pm$0.027} & 0.729$\pm$0.024  & \tugreen \\\Xhline{2pt}
\end{tabular} 
  }
 \bignegspace
 \caption{\textsc{lifelike}: comparison of Generative Forests (us, \geot), using smaller generative forest models ($T=200$ trees, $J=$500 total splits) to the same contenders as in Table \ref{tab:gen-us-vs-all-big} (main file, we also add the option Center for VCAE; we did not push it in the main file for space constraints but its results are on par with the other VCAE results). Conventions are the same as in Table \ref{tab:gen-us-vs-all-big}. See text for details. }
    \label{tab:gen-us-vs-all-small}
 \negspace
\end{table*}
  
The objective of the experiment is to evaluate whether a generative model is able to create "realistic" data. Tabular data is hard to evaluate visually, unlike \textit{e,g.} images or text, so we have considered a simple evaluation pipeline: we create for each domain a 5-fold stratified experiment. After a generative model has been trained, we generate the same number of observations as in the test fold and compute the optimal transport (OT) distance between the generated sample and the fold's test sample. To fasten the computation of OT costs, we use Sinkhorn's algorithm \cite{cSD} with an $\epsilon$-entropic regularizer, for some $\epsilon = 0.5$ which we observed was the smallest in our experiments to run with all domains without leading to numerical instabilities. To balance the importance of categorical features (for which the cost is binary, depending on whether the guess is right or not) and numerical features, we normalize numerical features with the domain's mean and standard deviation prior to computing OT costs. We then compare our method to three possible contenders: Adversarial Random Forests (ARF) \cite{wbkwAR}, CT-GANs \cite{xscvMT}, Forest Flows \cite{jmfkGA} and Vine copula autoencoders (VCAE) \cite{tavCA}. All these approaches rely on models that are very different from each other. Adversarial Random Forests (ARF) and Forest Flows (FFs) are tree-based generators. ARFs work as follows to generate one observation: one first sample uniformly at random a tree, then samples a leaf in the tree. The leaf is attached to a distribution which is then used to sample the observation. As we already pointed out in Section \ref{sec-rel}, there are two main differences with our models. From the standpoint of the model's distribution, if one takes the (non-empty) support from a tuple of leaves, its probability is obtained from a weighted average of the tree's distributions in \cite{wbkwAR} while it is obtained from a product of theirs in our case. Hence, at similar tree sizes, the modelling capability of the set of trees tips in our favor, but it is counterbalanced by the fact that their approach equip leaves with potentially complex distributions (e.g. truncated Gaussians) whereas we stick to uniform distributions at the leaves. A consequence of the modeling of ARFs is that each tree has to separately code for a good generator: hence, it has to be big enough or the distributions used at the leaves have to be complex enough. In our case, as we already pointed out, a small number of trees (sometimes even \textit{stumps}) can be enough to get a fairly good generator, even on real-world domains. Forest Flows use tree-based models to estimate parameters of diffusion models. Their tree-based models are not necessarily big but they need a lot of them to carry out training and generation, which is a big difference with us. Our next contender, CT-GAN \cite{xscvMT}, fully relies on neural networks so it is \textit{de facto} substantially different from ours. The last, VCAE \cite{tavCA} relies on a subtle combination of deep nets and graphical models learned in the latent space.

\noindent \textbf{Table \ref{tab:gen-us-vs-all-big} (main file) with smaller generative forests} Table \ref{tab:gen-us-vs-all-small} provides the equivalent of Table \ref{tab:gen-us-vs-all-big} (with one additional contender, which did not fit in the space allotted for the main file), but in which our generative forests are much smaller: $T=$200 trees, $j=$500 total splits (so each tree is on average barely bigger than a decision stump). One can check that our models are still competitive with respect to all contenders. The best contender, Forest Flows, beat our models on density, but it must be remembered that Forest Flows uses many tree-based models to carry out both training and generation.

\noindent \textbf{Focus on Sinkhorn for different parameterizations of contenders} We now investigate how changing the parameterization of various contenders affect their performance against our approach. Here,  ARFs are trained with a number of trees $T \in \{10, 50, 100, 200\}$ (ARFs include an algorithm to select the tree size so we do not have to select it). CT-GANs are trained with a number of epochs $E \in \{10, 100, 300, 1000\}$. 

We compare those models with generative forests with $T=500$ trees and trained for a total of $J=2 000$ iterations, for all domains considered. Compared to the trees learned by ARFs, the total number of splits we use can still be small compared to theirs, in particular when they learn $T\geq 100$ trees. Each table against ARFs and CT-GANs has a different size and contains only a subset of the whole 21 domains: ARFs do not process domains with missing values and we also got issues running contenders on several domains: those are discussed in \supplement, Section \ref{sec-algos}.

  \setlength\tabcolsep{4pt}
  \begin{table*}[h]
    \centering
    \resizebox{0.9\textwidth}{!}{
  \begin{tabular}{c?c?cc|cc|cc|cc}\Xhline{2pt}
    domain & us (\geot) & \multicolumn{8}{c}{Adversarial Random Forests (ARFs)} \\
        tag   & $T$=500, $J$=2 000 & $T=10$ & $p$-val & $T=50$ & $p$-val & $T=100$ & $p$-val & $T=200$ & $p$-val \\\Xhline{2pt}
    \domainname{iris} & 0.423$\pm$0.033 & {0.530$\pm$0.064} & \green{0.0039} & 0.517$\pm$0.081 & \green{0.0866} & 0.466$\pm$0.050 & 0.1452 & 0.502$\pm$0.043 & 0.0044 \\ \hline
    \domainname{ring} & 0.285$\pm$0.008 & {0.289$\pm$0.006} & {0.1918} & 0.286$\pm$0.007 & 0.8542 & 0.288$\pm$0.010 & 0.7205 & 0.286$\pm$0.007 & 0.6213 \\ \hline
    \domainname{circ} & 0.351$\pm$0.005 & 0.354$\pm$0.002 & 0.2942 & 0.356$\pm$0.008 & 0.2063 & 0.350$\pm$0.002	& {0.9050} & 0.355$\pm$0.005 & 0.3304\\ \hline
    \domainname{grid} & 0.390$\pm$0.002 &  0.394$\pm$0.003 & 0.1376 & 0.391$\pm$0.002 & 0.3210 & 0.392$\pm$0.001 & 0.2118 & 0.394$\pm$0.002 & \green{0.0252} \\ \hline
    \domainname{for} & 1.108$\pm$0.105 & 1.272$\pm$0.281 & 0.2807 & 1.431$\pm$0.337 & 0.1215 & 1.311$\pm$0.255	& \green{0.0972} & {1.268$\pm$0.209} & \green{0.0616} \\ \hline
    \domainname{rand} & 0.286$\pm$0.003 & 0.286$\pm$0.003 & 0.9468 & 0.290$\pm$0.005 & 0.2359 & 0.289$\pm$0.005 & 0.3990 & 0.288$\pm$0.002 & 0.3685\\ \hline
    \domainname{tic} & 0.575$\pm$0.002 & {0.577$\pm$0.001} & {0.2133} & {0.578$\pm$0.001} & {0.1104} & {0.577$\pm$0.001} & {0.2739} & {0.577$\pm$0.001} & {0.2773}\\ \hline
    \domainname{iono} & 1.167$\pm$0.074 & 1.431$\pm$0.043 & \green{0.0005} & 1.469$\pm$0.079	& \green{0.0001} & 1.431$\pm$0.058 & \green{0.0015} & 1.420$\pm$0.056 & \green{0.0035}\\ \hline
    \domainname{stm} & 0.975$\pm$0.026 & 1.010$\pm$0.015 & \green{0.0304} & 1.015$\pm$0.018& \green{0.0231} & 1.001$\pm$0.010 & \green{0.0956} & 1.014$\pm$0.021 & \green{0.0261} \\ \hline
    \domainname{wred} & 0.980$\pm$0.032 & 1.086$\pm$0.036 & \green{0.0003} & 1.072$\pm$0.055	& \green{0.0133} & 1.086$\pm$0.030 & \green{0.0003} & 1.099$\pm$0.031 & \green{0.0001} \\ \hline
        \domainname{stp} & 0.976$\pm$0.018 & 0.999$\pm$0.011 & \green{0.0382} & 1.012$\pm$0.010 & \green{0.0250} & 1.006$\pm$0.011 & \green{0.0280} & 1.003$\pm$0.003 & \green{0.0277}\\ \hline
    \domainname{ana}  & 0.368$\pm$0.014 & 0.370$\pm$0.012 & 0.7951 & 0.382$\pm$0.015 & 0.1822 & {0.368$\pm$0.006} & 0.9161 & {0.369$\pm$0.007} & 0.8316\\\hline
    \domainname{aba}  & 0.490$\pm$0.028 & 0.505$\pm$0.046 & 0.5251 & 0.484$\pm$0.024 & 0.7011 & 0.503$\pm$0.035	& 0.6757 & 0.481$\pm$0.031 & 0.4632\\\hline
    \domainname{wwhi}  & 1.064$\pm$0.003 &  1.159$\pm$0.011	& \green{$\epsilon$} & 1.159$\pm$0.006 & \green{$\epsilon$} & 1.170$\pm$0.014 & \green{$\epsilon$} & 1.150$\pm$0.021 & \green{0.0005}\\\hline
    \domainname{comp}  & 0.532$\pm$0.008 &  0.531$\pm$0.012	& 0.8278 & 0.534$\pm$0.009 & 0.8438 & 0.551$\pm$0.023 & 0.1371 & 0.535$\pm$0.033 & 0.8642\\\hline
    \domainname{arti}  & 0.821$\pm$0.004 &  0.849$\pm$0.004	& \green{0.0007} & 0.847$\pm$0.010 & \green{0.0005} & 0.843$\pm$0.009 & \green{0.0006} & 0.849$\pm$0.010 & \green{0.0005}\\\hline
    \domainname{jung}  & 0.929$\pm$0.002 &  0.929$\pm$0.002	& 0.6044 & 0.930$\pm$0.002 & 0.6536 & 0.929$\pm$0.002 & 0.5937 & 0.930$\pm$0.001 & 0.2447\\\hline
    \domainname{elec}  & 1.483$\pm$0.033 &  1.577$\pm$0.088	& 0.1052 & 1.543$\pm$0.017 & \green{0.0081} & 1.552$\pm$0.039 & \green{0.0522} & 1.581$\pm$0.025 & \green{0.0095}\\\Xhline{2pt}
    \multicolumn{2}{r|}{wins / lose for us $\rightarrow$} & 15 / 1 & & 17 / 1 & & 16 / 1 & & 17 / 1 & \\\Xhline{2pt}
  \end{tabular} 
  }
 \bignegspace
  \caption{\textsc{lifelike}: comparison of ARFs \cite{wbkwAR} with different number $T$ of trees, and meduim-sized Generative Forests (us, \geot) where the number of trees and number of iterations are \textbf{fixed} ($T=500$ trees, total of $J=2000$ splits in trees). Values shown = average over the 5-folds $\pm$ std dev. . The $p$-values for the comparison of our results ("us") and ARFs are shown. Values appearing in \green{green} in $p$-vals mean (i) the average regularized OT cost for us (\geot) is smaller than that of ARFs and (ii) $p<0.1$, \textit{i.e.} our method outperforms ARFs (there is no instance of us being statistically significantly beaten by ARFs). $\varepsilon$ means $p$-val $<10^{-4}$. Domain details in Table \ref{t-s-uci}.}
    \label{tab:gen-us-vs-arf-big}
 \negspace
\end{table*}

  \paragraph{Results vs Adversarial Random Forests} Table \ref{tab:gen-us-vs-arf-big} digs into results obtained against adversarial random forests on the Sinkhorn metric. A first observation, not visible in the table, is that ARFs indeed tend to learn big models, typically with dozens of nodes in each tree. For the largest ARFs with 100 or 200 trees, this means in general a total of thousands of nodes in models. A consequence, also observed experimentally, is that there is sometimes little difference in performance in general between models with a different number of trees in ARFs as each tree is in fact an already good generator. In our case, this is obviously not the case. Noticeably, the performance of ARFs is not monotonic in the number of trees, so there could be a way to find the appropriate number of trees in ARFs to buy the slight (but sometimes significant) increase in performance in a domain-dependent way. In our case, there is obviously a dependency in the number of trees chosen as well. From the standpoint of the optimal transport metric, even when ARFs make use of distributions at their tree leaves that are much more powerful than ours and in fact fit to some of our simulated domains (Gaussians used in \texttt{circgauss}, \texttt{randgauss} and \texttt{gridgauss}), we still manage to compete or beat ARFs on those simulated domains. 

  Globally, we manage to beat ARFs on almost all runs, and very significantly on many of them. Since training generative forests does not include a mechanism to select the size of models, we have completed this experiment by another one on which we learn much smaller generative forests. The corresponding Table is in \supplement, Section \ref{sec-tables-lifelike}. Because we clearly beat ARFs on \domainname{student$\_$performance$\_$mat} and  \domainname{student$\_$performance$\_$por} in Table \ref{tab:gen-us-vs-arf-big} but are beaten by ARFs for much smaller generative forests, we conclude that there could exist mechanisms to compute the "right size" of our models, or to prune big models to get models with the right size. Globally however, even with such smaller models, we still manage to beat ARFs on a large majority of cases, which is a good figure given the difference in model sizes. The ratio "results quality over model size" tips in our favor and demonstrates the potential in using all trees in a forest to generate an observation (us) vs using a single of its trees (ARFs), see Figure \ref{fig:generation-sketch}.
  
  \paragraph{Results vs CT-GANs} Table \ref{tab:gen-us-vs-ctgan-big} summarizes our results against CT-GAN using various numbers of training epochs ($E$). In our case, our setting is the same as versus ARFs: we stick to $T=500$ trees trained for a total number of $J=2000$ iterations in the generative forest, for all domains. We see that generative forests consistently and significantly outperform CT-GAN on nearly all cases. Furthermore, while CT-GANs performances tend to improve with the number of epochs, we observe that on a majority of datasets, performances at 1 000 epochs are still far from those of generative forests and already induce training times that are far bigger than for generative forests: for example, more than 6 hours per fold for \domainname{stm} while it takes just a few seconds to train a generative forest. Just like we did for ARFs, we repeated the experiment vs CT-GAN using much smaller generative forests ($T=200, J=500$). The results are presented in  \supplement, Section \ref{sec-tables-lifelike}. There is no change in the conclusion: even with such small models, we still beat CT-GANs, regardless of the number of training epochs, on all cases (and very significantly on almost all of them).

  \begin{table}[t]
    \centering
    \resizebox{\textwidth}{!}{
  \begin{tabular}{c?c?cc|cc|cc|cc}\Xhline{2pt}
    domain & us (\geot) & \multicolumn{8}{c}{CT-GANs} \\
        tag   & $T$=500, $J$=2 000 & $E=10$ & $p$-val & $E=100$ & $p$-val & $E=300$ & $p$-val & $E=1000$ & $p$-val \\\Xhline{2pt}
    \domainname{ring} & 0.285$\pm$0.008 & {0.457$\pm$0.058} & \green{0.0032} & 0.546$\pm$0.079 & \green{0.0018} & 0.405$\pm$0.044 & \green{0.0049} & 0.351$\pm$0.042 & \green{0.0183}  \\ \hline
        \domainname{circ} & 0.351$\pm$0.005 & 0.848$\pm$0.300 & \green{0.0213} & 0.480$\pm$0.040 & \green{0.0019} & 0.443$\pm$0.014 & \green{$\varepsilon$} & 0.435$\pm$0.075 & \green{0.0711} \\ \hline
    \domainname{grid} & 0.390$\pm$0.002 & 0.749$\pm$0.417 & 0.1268 & 0.459$\pm$0.031 & \green{0.0085} & 0.426$\pm$0.022 & \green{0.0271} &  0.408$\pm$0.019 & 0.1000\\ \hline
    \domainname{for} & 1.108$\pm$0.105 & 9.796$\pm$5.454 & \green{0.0049} & 1.899$\pm$0.289 & \green{0.0045} & 1.532$\pm$0.178 & \green{0.0244} & 1.520$\pm$0.311 & \green{0.0410} \\ \hline
    \domainname{rand} & 0.286$\pm$0.003 & 0.746$\pm$0.236 & \green{0.0119} & 0.491$\pm$0.063 & \green{0.0021} & 0.368$\pm$0.015 & \green{0.0002} & 0.327$\pm$0.024 & \green{0.0288} \\ \hline
    \domainname{tic} & 0.575$\pm$0.002 & 0.601$\pm$0.003 & \green{0.0002} & 0.581$\pm$0.001 & \green{0.0082} & 0.586$\pm$0.006 & \green{0.0294} & 0.584$\pm$0.003 & \green{0.0198}  \\ \hline
    \domainname{iono} & 1.167$\pm$0.074 & 2.263$\pm$0.035 & \green{$\varepsilon$} & 2.084$\pm$0.052 & \green{$\varepsilon$} & 1.984$\pm$0.136 & \green{$\varepsilon$} & 1.758$\pm$0.137 & \green{0.0002} \\ \hline
    \domainname{stm} & 0.980$\pm$0.032 & 1.511$\pm$0.074 & \green{$\varepsilon$} & 1.189$\pm$0.045 & \green{0.0002} & 1.168$\pm$0.054 & \green{0.0056} & 1.167$\pm$0.041 & \green{0.0013}  \\ \hline
    \domainname{wred} & 0.980$\pm$0.032& 2.836$\pm$0.703 & \green{0.0044} & 2.004$\pm$0.090 & \green{$\varepsilon$} & 1.825$\pm$0.127 & \green{0.0001} & 1.384$\pm$0.047 & \green{$\varepsilon$}  \\ \hline
    \domainname{stp} & 0.976$\pm$0.018& 1.660$\pm$0.161 & \green{0.0006} & 1.141$\pm$0.018 & \green{0.0001} & 1.206$\pm$0.046 & \green{0.0007} & 1.186$\pm$0.052 & \green{0.0019}  \\ \hline
  \domainname{ana}  & 0.368$\pm$0.014 & 0.542$\pm$0.135 & \green{0.0050} & 0.439$\pm$0.056 & \green{0.0202} & 0.404$\pm$0.012 & \green{0.0397} & 0.436$\pm$0.036 & \green{0.0051}  \\ \hline
  \domainname{aba}  & 0.490$\pm$0.028 & 1.689$\pm$0.053 & \green{$\varepsilon$} & 1.463$\pm$0.094 & \green{$\varepsilon$} & 0.754$\pm$0.040 & \green{0.0009} & 0.657$\pm$0.026 & \green{0.0006} \\ \hline
    \domainname{wwhi}  & 1.064$\pm$0.003 & 1.770$\pm$0.160 & \green{0.0005} & 1.849$\pm$0.064 & \green{$\varepsilon$} & 1.284$\pm$0.029 & \green{$\varepsilon$} & 1.158$\pm$0.009 & \green{$\varepsilon$}  \\ \hline
    \domainname{arti}$^*$  & 0.821$\pm$0.004 &  1.388$\pm$0.109$_4$ 	& \green{0.0020} & 1.209$\pm$0.018$_4$  & \green{0.0005} & 1.026$\pm$0.013$_4$  & \green{0.0001} & 0.959$\pm$0.013$_3$ & \green{0.0038}\\ \Xhline{2pt}
    \multicolumn{2}{r|}{wins / lose for us $\rightarrow$} & 14 / 0 & & 14 / 0 & & 14 / 0 & & 14 / 0 & \\\Xhline{2pt}
  \end{tabular} 
  }
  \caption{Comparison of results with respect to CT-GANs \cite{xscvMT} trained with
    different number of epochs $E$. On some domains, indicated with a "$^*$", CT-GANs crashed on some folds. For such domains, we indicate in index to CT-GANs results the number of folds (out of 5) for which this did \textit{not} happen. In such cases, we restricted the statistical comparisons with us to the folds for which CT-GANs did not crash: the statistical tests take this into account; instead of providing all corresponding average performances for us, we keep giving the average performance for us on all five folds (leftmost column), which is thus indicative. Other conventions follow Table \ref{tab:gen-us-vs-arf-big}.}
    \label{tab:gen-us-vs-ctgan-big}
  \end{table}

  \begin{table}[t]
    \centering
    \resizebox{\textwidth}{!}{
  \begin{tabular}{c?c?cc|cc|cc|cc}\Xhline{2pt}
    domain & us (\geot) & \multicolumn{8}{c}{Adversarial Random Forests (ARFs)} \\
        tag   & $T$=200, $J$=500 & $T=10$ & $p$-val & $T=50$ & $p$-val & $T=100$ & $p$-val & $T=200$ & $p$-val \\\Xhline{2pt}
    \domainname{iris} & 0.529$\pm$0.099 & {0.530$\pm$0.064} & 0.9890 & 0.517$\pm$0.081 & 0.8645 & 0.466$\pm$0.050 & 0.1099 & 0.501$\pm$0.043 & 0.6548 \\ \hline
    \domainname{ring} & 0.283$\pm$0.005 & {0.289$\pm$0.006} & \green{0.0651} & 0.286$\pm$0.007 & 0.4903 & 0.288$\pm$0.010 & 0.4292 & 0.286$\pm$0.007 & 0.7442 \\ \hline
    \domainname{circ} & 0.356$\pm$0.002 & 0.354$\pm$0.002 & 0.4561 & 0.356$\pm$0.008 & 0.7882 & 0.350$\pm$0.002	& \magenta{0.0591} & 0.355$\pm$0.005 & 0.7446\\ \hline
    \domainname{grid} & 0.392$\pm$0.002 &  0.394$\pm$0.003 & 0.4108 & 0.391$\pm$0.002 & 0.3403 & 0.392$\pm$0.001 & 0.8370 & 0.394$\pm$0.002 & 0.3423 \\ \hline
    \domainname{for} & 1.086$\pm$0.124 & 1.272$\pm$0.281 & 0.1420 & 1.431$\pm$0.337 & 0.1142 & 1.311$\pm$0.255	& 0.1215 & {1.268$\pm$0.209} & \green{0.0370} \\ \hline
    \domainname{rand} & 0.290$\pm$0.007 & 0.286$\pm$0.003 & 0.4488 & 0.290$\pm$0.005 & 0.9553 & 0.289$\pm$0.005 & 0.9004 & 0.288$\pm$0.002 & 0.6542\\ \hline
    \domainname{tic} & 0.574$\pm$0.001 & {0.577$\pm$0.001} & \green{0.0394} & {0.578$\pm$0.001} & \green{0.0325} & {0.577$\pm$0.001} & \green{0.0657} & {0.577$\pm$0.001} & \green{0.0914}\\ \hline
    \domainname{iono} & 1.203$\pm$0.068 & 1.431$\pm$0.043 & 0.7307 & 1.469$\pm$0.079	& 0.3402 & 1.431$\pm$0.058 & 0.7046 & 1.420$\pm$0.056 & 0.8737\\ \hline
    \domainname{stm} & 1.044$\pm$0.013 & 1.010$\pm$0.015 & \magenta{0.0080} & 1.015$\pm$0.018& \magenta{0.0094} & 1.001$\pm$0.010 & \magenta{0.0037} & 1.014$\pm$0.021 & \magenta{0.0329} \\ \hline
    \domainname{wred} & 1.027$\pm$0.037 & 1.086$\pm$0.036 & \green{0.0327} & 1.072$\pm$0.055	& 0.1432 & 1.086$\pm$0.030 & \green{0.0095} & 1.099$\pm$0.031 & \green{0.0044} \\ \hline
    \domainname{stp} & 1.029$\pm$0.027 & 0.999$\pm$0.011 & \magenta{0.0826} & 1.012$\pm$0.010 & 0.3421 & 1.006$\pm$0.011 & \magenta{0.0542} & 1.003$\pm$0.003 & \magenta{0.0764}\\ \hline
  \domainname{ana}  & 0.367$\pm$0.018 & 0.370$\pm$0.012 & 0.8229 & 0.382$\pm$0.015 & 0.2318 & {0.368$\pm$0.006} & 0.9264 & {0.369$\pm$0.007} & 0.8722\\\hline
  \domainname{aba}  & 0.476$\pm$0.017 & 0.505$\pm$0.046 & 0.2264 & 0.484$\pm$0.024 & 0.2162 & 0.503$\pm$0.035	& 0.1847 & 0.481$\pm$0.031 & 0.7468\\\hline
    \domainname{wwhi}  & 1.120$\pm$0.013 &  1.159$\pm$0.011	& \green{0.0013} & 1.159$\pm$0.006 & \green{0.0014} & 1.170$\pm$0.014 & \green{0.0026} & 1.150$\pm$0.021 & \green{0.0327}\\\hline
    \domainname{comp}  &  0.538$\pm$0.015 &  0.531$\pm$0.012	& 0.5639 & 0.534$\pm$0.009 & 0.5706  & 0.551$\pm$0.023 & 0.4057  & 0.535$\pm$0.033 & 0.9017 \\\hline
    \domainname{arti}  & 0.830$\pm$0.007 &  0.849$\pm$0.004	&  \green{0.0078} & 0.847$\pm$0.010 & \green{0.0012}  & 0.843$\pm$0.009 & \green{0.0003}  & 0.849$\pm$0.010 & \green{0.0010}\\\hline
    \domainname{jung}  & 0.928$\pm$0.001 &  0.929$\pm$0.002	& 0.1720 & 0.930$\pm$0.002 & \green{0.0646} & 0.929$\pm$0.002 & 0.2675 & 0.930$\pm$0.001 & \green{0.0081}\\\hline
    \domainname{elec}  & 1.503$\pm$0.035  &  1.577$\pm$0.088	& 0.1222  & 1.543$\pm$0.017 & \green{0.0880}  & 1.552$\pm$0.039 & 0.1228  & 1.581$\pm$0.025 & \green{0.0027}  \\\Xhline{2pt}
    \multicolumn{2}{r|}{wins / lose for us $\rightarrow$} & 13 / 5 & & 11 / 5 & & 12 / 5 & & 12 / 6 & \\\Xhline{2pt}
  \end{tabular}
  }
  \caption{Comparison of results for ARFs \cite{wbkwAR} with different number $J$ of trees, and small Generative Forests (us, \geot) where the number of trees and number of iterations are \textbf{fixed} ($T=200$ trees, total of $J=500$ splits in trees). Values are shown on the form average over the 5-folds $\pm$ standard deviation. The $p$-values for the comparison of our results ("us") and ARFs are shown. Values appearing in \green{green} in $p$-vals mean (i) the average regularised OT cost for us (\geot) is smaller than that of ARFs and (ii) $p<0.1$, meaning we can conclude that our method outperforms ARFs. Values appearing in \magenta{magenta} in $p$-vals mean (i) the average regularised OT cost for us (\geot) is larger than that of ARFs and (ii) $p<0.1$, meaning we can conclude that our method is outperformed by ARFs. The last row summarizes the number of times we win / lose against ARFs according to the average metric used. To reduce size, domains are named by a tag: the correspondence with domains and their characteristics is in Table \ref{t-s-uci}.}
    \label{tab:gen-us-vs-arf-small}
  \end{table}

  \begin{table}[t]
    \centering
    \resizebox{\textwidth}{!}{
  \begin{tabular}{c?c?cc|cc|cc|cc}\Xhline{2pt}
    domain & us (\geot) & \multicolumn{8}{c}{CT-GANs} \\
        tag   & $T$=200, $J$=500 & $E=10$ & $p$-val & $E=100$ & $p$-val & $E=300$ & $p$-val & $E=1000$ & $p$-val \\\Xhline{2pt}
    \domainname{ring} & 0.283$\pm$0.005 & {0.457$\pm$0.058} & \green{0.0029} & 0.546$\pm$0.079 & \green{0.0020} & 0.405$\pm$0.044 & \green{0.0048} & 0.351$\pm$0.042 & \green{0.0299}  \\ \hline
        \domainname{circ} & 0.356$\pm$0.002 & 0.848$\pm$0.300 & \green{0.0217} & 0.480$\pm$0.040 & \green{0.0024} & 0.443$\pm$0.014 & \green{0.0001} & 0.435$\pm$0.075 & \green{0.0816} \\ \hline
    \domainname{grid} & 0.392$\pm$0.002 & 0.749$\pm$0.417 & 0.1279 & 0.459$\pm$0.031 & \green{0.0096} & 0.426$\pm$0.022 & \green{0.0233} &  0.408$\pm$0.019 & 0.1239\\ \hline
    \domainname{for} & 1.086$\pm$0.124 & 9.796$\pm$5.454 & \green{0.0241} & 1.899$\pm$0.289 & \green{0.0028} & 1.532$\pm$0.178 & \green{0.0020} & 1.520$\pm$0.311 & \green{0.0313} \\ \hline
    \domainname{rand} & 0.290$\pm$0.007 & 0.746$\pm$0.236 & \green{0.0129} & 0.491$\pm$0.063 & \green{0.0017} & 0.368$\pm$0.015 & \green{0.0003} & 0.327$\pm$0.024 & \green{0.0124} \\ \hline
    \domainname{tic} & 0.574$\pm$0.001 & 0.601$\pm$0.003 & \green{$\varepsilon$} & 0.581$\pm$0.001 & \green{0.0024} & 0.586$\pm$0.006 & \green{0.0299} & 0.584$\pm$0.003 & \green{0.0228}  \\ \hline
    \domainname{iono} & 1.203$\pm$0.068 & 2.263$\pm$0.035 & \green{$\varepsilon$} & 2.084$\pm$0.052 & \green{$\varepsilon$} & 1.98$\pm$0.136 & \green{0.0001} & 1.758$\pm$0.137 & \green{0.0005} \\ \hline
    \domainname{stm} & 1.044$\pm$0.013 & 1.511$\pm$0.074 & \green{0.0001} & 1.189$\pm$0.045 & \green{0.0106} & 1.168$\pm$0.054 & \green{0.0346} & 1.167$\pm$0.041 & \green{0.0036}  \\ \hline
    \domainname{wred} & 1.027$\pm$0.037& 2.836$\pm$0.703 & \green{0.0049} & 2.004$\pm$0.090 & \green{$\varepsilon$} & 1.825$\pm$0.127 & \green{$\varepsilon$} & 1.384$\pm$0.047 & \green{0.0002}  \\ \hline
    \domainname{stp} & 1.029$\pm$0.027& 1.660$\pm$0.161 & \green{0.0010} & 1.141$\pm$0.018 & \green{0.0004} & 1.206$\pm$0.046 & \green{0.0015} & 1.186$\pm$0.052 & \green{0.0057}  \\ \hline
  \domainname{ana}  & 0.367$\pm$0.018 & 0.542$\pm$0.135 & \green{0.0399} & 0.439$\pm$0.056 & \green{0.0196} & 0.404$\pm$0.012 & \green{0.0231} & 0.436$\pm$0.036 & \green{0.0018}  \\ \hline
  \domainname{aba}  & 0.476$\pm$0.017 & 1.689$\pm$0.053 & \green{$\varepsilon$} & 1.463$\pm$0.094 & \green{$\varepsilon$} & 0.754$\pm$0.040 & \green{$\varepsilon$} & 0.657$\pm$0.026 & \green{$\varepsilon$} \\ \hline
  \domainname{wwhi}  & 1.120$\pm$0.013 & 1.770$\pm$0.160 & \green{0.0009} & 1.849$\pm$0.064 & \green{$\varepsilon$} & 1.284$\pm$0.029 & \green{0.0006} & 1.158$\pm$0.009 & \green{0.0117}  \\\hline
  \domainname{arti}$^*$  & 0.830$\pm$0.007& 1.388$\pm$0.109$_4$ & \green{0.0021} & 1.209$\pm$0.018$_4$ & \green{$\varepsilon$} & 1.026$\pm$0.013$_4$ & \green{0.0002} & 0.959$\pm$0.013$_3$ & \green{0.0051}  \\\Xhline{2pt}
    \multicolumn{2}{r|}{wins / lose for us $\rightarrow$} & 14 / 0 & & 14 / 0 & & 14 / 0 & & 14 / 0 & \\\Xhline{2pt}
  \end{tabular} 
  }
  \caption{Comparison of results with respect to CT-GANs \cite{xscvMT} trained with
    different number of epochs $E$. $\varepsilon$ means $p$-val $<10^{-4}$. As already mentioned in Table \ref{tab:gen-us-vs-ctgan-big}, in some domains, indicated with a "$^*$", CT-GANs crashed on some folds. For such domains, we indicate in index to CT-GANs results the number of folds (out of 5) for which this did \textit{not} happen. In such cases, we restricted the statistical comparisons with us to the folds for which CT-GANs did not crash: the statistical tests take this into account; instead of providing all corresponding average performances for us, we keep giving the average performance for us on all five folds (leftmost column), which is thus indicative. Other conventions follow Table \ref{tab:gen-us-vs-arf-small}.}
    \label{tab:gen-us-vs-ctgan-small}
  \end{table}

\noindent \textbf{More results with small generative forests} Tables \ref{tab:gen-us-vs-arf-small} and \ref{tab:gen-us-vs-ctgan-small} present the results of generative forests vs adversarial random forests and CT-GANs, when the size of our models is substantially smaller than in the main file ($T=200, J=500$). We still manage to compete or beat ARFs on simulated domains for which modelling the leaf distributions in ARFs should represent an advantage (Gaussians used in \texttt{circgauss}, \texttt{randgauss} and \texttt{gridgauss}), even more using comparatively very small models compared to ARFs. We believe this could be due to two factors: (i) the fact that in our models all trees are used to generate each observation (which surely facilitates learning small and accurate models) and (ii) the fact that we do not generate data during training but base our splitting criterion on an \textit{exact} computation of the reduction of Bayes risk in growing trees. If we compute aggregated statistics comparing, for each number of trees in ARFs, the number of times we win / lose versus ARFs, either considering only significant $p$-values or disregarding them, our generative forests always beat ARFs on a majority of domains. Globally, this means we win on 34 out of 52 cases. The results vs CT-GANs are of the same look and feel as in the main file with larger generative forests: generative forests beat them in all cases, and very significantly in most of them. 

  \subsubsection{Comparison with "the optimal generator": \textsc{gen-discrim}} \label{sec-full-comp-ideal}

  Rather than compare our method with another one, the question we ask here is "can we generate data that looks like domain's data" ? We replicate the experimental pipeline of \cite{ngGT}. In short, we shuffle a 3-partition (say $\colorr_1, \colorr_2, \colorr_3$) of the training data in a $3! = 6$-fold CV, then train a generator $\colorg_1$ from $\colorr_1$ (we use generative forests with \topdownGT), then train a random forest (\textsc{rf}) to distinguish between $\colorg_1$ and $\colorr_2$, and finally estimate its test accuracy on $\colorg_1$ vs $\colorr_3$. The \textit{lower} this accuracy, the less distinguishable are fakes from real and thus the \textit{better} the generator. We consider 3 competing baselines to our models: (i) the "optimal" one, \textsc{copy}, which consists in replacing $\colorg_1$ by a set of real data, (ii) the "worst" one, \textsc{unif}(orm), which uniformly samples data, and (iii) \topdownGT~for $T=1$, which learns a generative tree (\gt). We note that this pipeline is radically different from the one used in \cite{wbkwAR}. In this pipeline, domains used are supervised (the data includes a variable to predict). A generator is trained from some training data, then generates an amount of data equal to the size of the training data. Then, two classifiers are trained, one from the original training data, one from the generated data, to predict for the variable to predict and their performances are compared on this basis (the higher the accuracy, the better). There is a risk in this pipeline that we have tried to mitigate with our sophisticated pipeline: if the generator consists just in copying its training data, it is guaranteed to perform well according to \cite{wbkwAR}'s metric. Obviously, this would not be a good generator however. Our pipeline would typically prevent this, $\colorr_2$ being different from $\colorr_3$. \\
  \noindent\textbf{Results} Tables \ref{tab:gen_discrim_full} and \ref{tab:gen_discrim_full_pval} provide the results we got, including statistical $p$-values for the test of comparing our method to \textsc{copy}. They display that it is possible to get \geot s generating realistically looking data: for \texttt{iris}, the Student $t$-tests show that we can keep $H_0$ = "\geot s perform identically to \textsc{copy}" for $J=T=40$ ($p>.2$). For \texttt{mat}, the result is quite remarkable not because of the highest $p$-val, which is not negligible ($>0.001$), but because to get it, it suffices to build a \geot~in which the total number of feature occurrences ($J=80$) is barely three times as big as the domain's dimension ($d=33$). The full tables display a clear incentive to grow further the \geot s as domains increase in size. However, our results show, pretty much like the experiment against Adversarial Random Forests, that there could be substantial value of learning the right model size: \domainname{iris} and \domainname{student$\_$performance$\_$por} clearly show nontrivial peaks of $p$-values, for which we would get the best models compared to \textsc{copy}.

\setlength\tabcolsep{1pt}
\newpage

\begin{table}
  \centering
  \includegraphics[trim=0bp 0bp 0bp 0bp,clip,width=0.9\columnwidth]{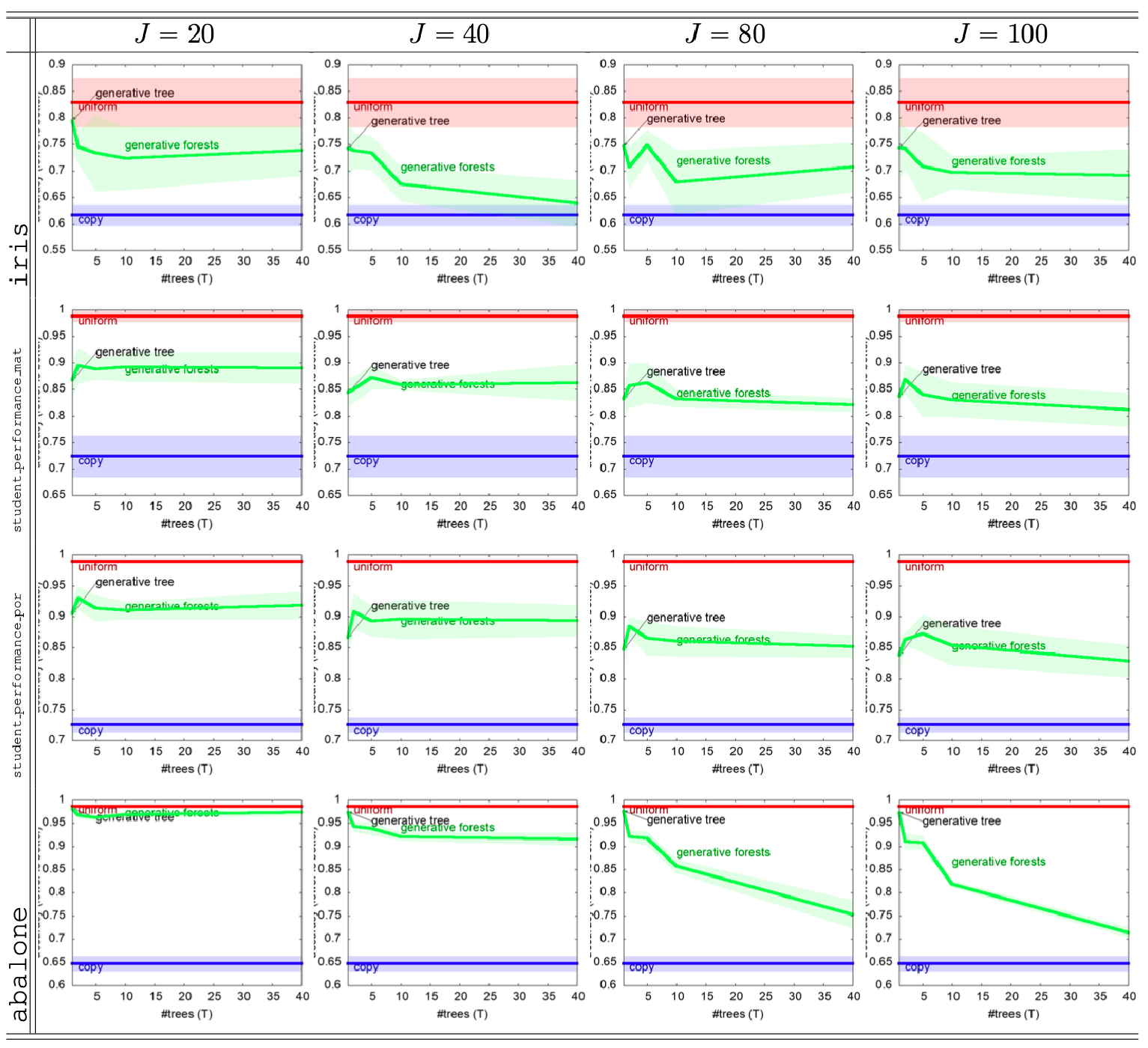}
  \caption{Results of generative forests on \textsc{gen-discrim} for four UCI domains, ordered, \textbf{from top to bottom, in increasing domain size}. In each row, the accuracy of distinguishing fakes from real is plotted (the \textit{lower}, the \textit{better} the technique) for four contenders: \textsc{uniform} (generates observation uniformly at random), \textsc{copy} ("optimal" contender using real data as generated), \topdownGT~inducing generative forests and generative trees (indicated for $T=1$). A clear pattern emerges, that as the domain size grows, increasing $J$ buys improvements and, if $J$ is large enough, increasing $T$ also improves results. See Table \ref{tab:gen_discrim_full_pval} for comparisons (us vs \textsc{copy}) in terms of $p$-values.}
    \label{tab:gen_discrim_full}
\end{table}

\newpage

\begin{table}
  \centering
  \includegraphics[trim=0bp 0bp 0bp 0bp,clip,width=0.9\columnwidth]{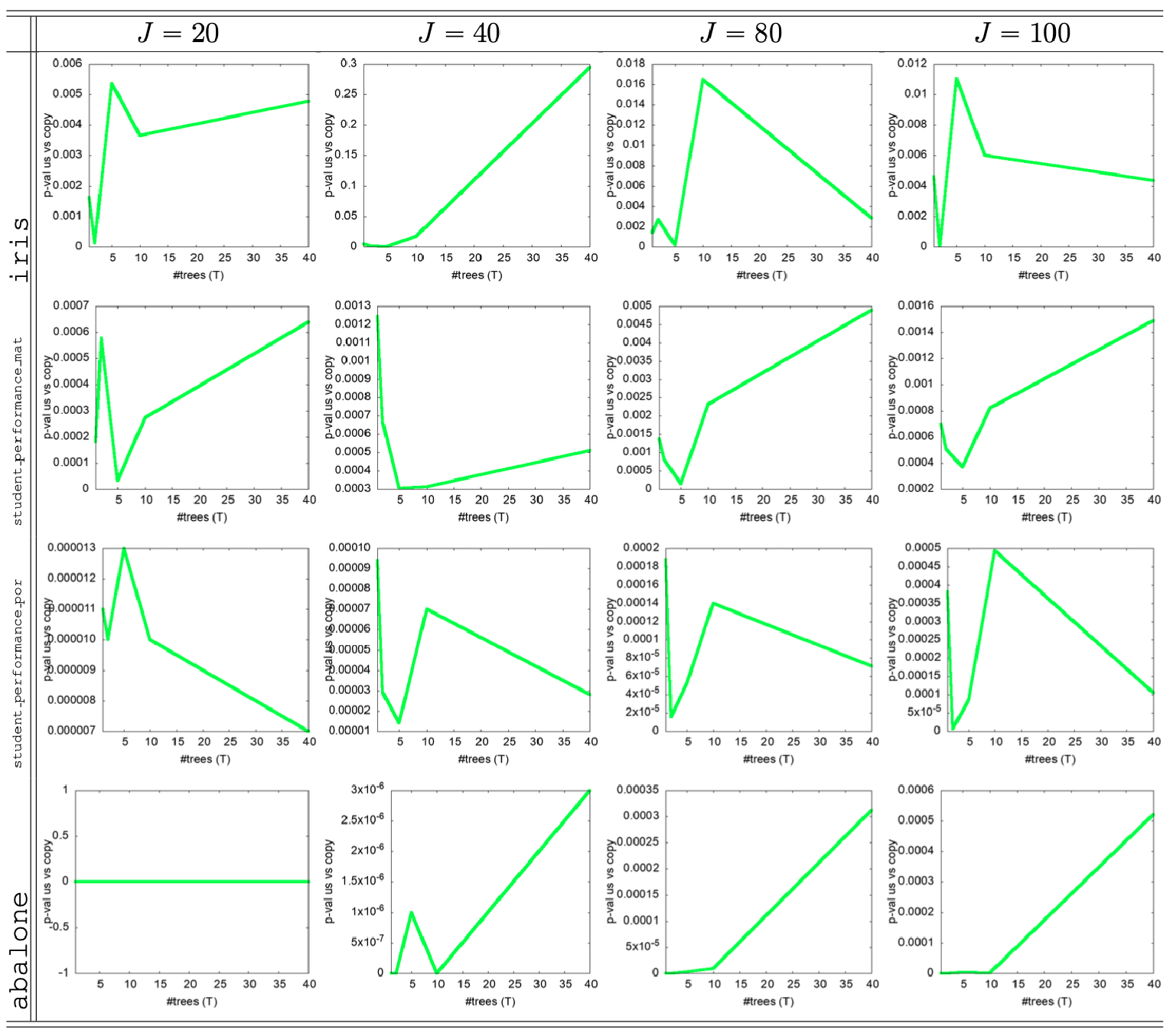}
  \caption{$p$-values for Student tests with $H_0$ being "\geot~and \textsc{copy} perform identically on \textsc{gen-discrim}'' (for the domains in Table \ref{tab:gen_discrim_full}). Large values indicate we can keep $H_0$ and thus consider that \geot~is successful at generating "realistic" data. Clearly, \geot~are successful for \texttt{iris} ($p$ increases with $T$ up to $p\sim 0.3$ for $J=40$). For the other domains, Table \ref{tab:gen_discrim_full} suggests increasing $J$ or $T$. The curves for \texttt{abalone} reveal a dramatic relative increase of $p$ as $(J,T)$ increases, starting from $p\sim 0, \forall T$ for $J=20$.}
    \label{tab:gen_discrim_full_pval}
\end{table}

\subsubsection{Full comparisons with \mice~on missing data imputation} \label{sec-full-comp-mdi}
The main file provides just two examples of domains for the comparison, \domainname{abalone} and \domainname{analcatdata$\_$supreme}. We provide here more complete results on the domain used in the main file and results on more domains in the form of one table for each additional domain:
\begin{itemize}
\item Table \ref{tab:mdi-full-analcatdata_supreme_gf_eogt_vs_mice_rf}: experiments on \domainname{analcatdata$\_$supreme} completing the results shown in Table \ref{tab:mdi-abaanalcat} (\mainfile);
\item Table \ref{tab:mdi-full-abalone_gf_eogt_vs_mice_rf}: experiments on \domainname{abalone} completing the results shown in Table \ref{tab:mdi-abaanalcat} (\mainfile);
  \item [] (tables below are ordered in increasing domain size, see Table \ref{t-s-uci})
\item Table \ref{tab:mdi-full-iris_gf_eogt_vs_mice_rf}: experiments on \domainname{iris};
\item Table \ref{tab:mdi-full-ringgauss_gf_eogt_vs_mice_rf}: experiments on \domainname{ringgauss};
\item Table \ref{tab:mdi-full-circgauss_gf_eogt_vs_mice_rf}: experiments on \domainname{circgauss};
\item Table \ref{tab:mdi-full-gridgauss_gf_eogt_vs_mice_rf}: experiments on \domainname{gridgauss};
  \item Table \ref{tab:mdi-full-randgauss_gf_eogt_vs_mice_rf}: experiments on \domainname{randgauss};
\item Table \ref{tab:mdi-full-student_performance_mat_gf_eogt_vs_mice_rf}: experiments on \domainname{student$\_$performance$\_$mat};
\item Table \ref{tab:mdi-full-student_performance_por_gf_eogt_vs_mice_rf}: experiments on \domainname{student$\_$performance$\_$por};
 \item Table \ref{tab:mdi-full-kc1_gf_eogt_vs_mice_rf}: experiments on \domainname{kc1};
 \item Table \ref{tab:mdi-full-sigma_cabs_gf_eogt_vs_mice_rf}: experiments on \domainname{sigma$\_$cabs};
 \item Table \ref{tab:mdi-full-compas_gf_eogt_vs_mice_rf}: experiments on \domainname{compas};
 \item Table \ref{tab:mdi-full-open_policing_hartford_gf_eogt_vs_mice_rf}: experiments on \domainname{open$\_$policing$\_$hartford};
 \end{itemize}
 The large amount of pictures to process may make it difficult to understand at first glance how our approaches behave against \mice, so we summarize here the key points:

 \begin{itemize}
 \item first and most importantly, there is not one single type of our models that perform better than the other. Both Generative Forests and Ensembles of Generative Trees can be useful for the task of missing data imputation. For example, \domainname{analcatdata$\_$supreme} gives a clear advantage to Generative Forests: they even beat \mice~with random forests having 4 000 trees (that is hundreds of times more than our models) on both categorical variables (perr) and numerical variables (rmse). However, on \domainname{circgauss}, \domainname{student$\_$performance$\_$por} and \domainname{sigma$\_$cabs}, Ensembles of Generative Trees obtain the best results. On \domainname{sigma$\_$cabs}, they can perform on par with \mice~if the number of tree is limited (which represents 100+ times less trees than \mice's models);
 \item as is already noted in the main file (\mainfile) and visible on several plots (see \textit{e,g,} \domainname{sigma$\_$cabs}), overfitting can happen with our models (both Generative Forests and Ensembles of Generative Trees), which supports the idea that a pruning mechanism or a mechanism to stop training would be a strong plus to training such models. Interestingly, in several cases where overfitting seems to happen, increasing the number of splits can reduce the phenomenon, at fixed number of trees: see for example \domainname{iris} ($50 \rightarrow 100$ iterations for \eogt s), \domainname{gridgauss}, \domainname{randgauss} and \domainname{kc1} (\geot s);
 \item some domains show that our models seem to be better at estimating continuous (regression) rather than categorical (prediction) variables: see for example \domainname{abalone}, \domainname{student$\_$performance$\_$mat}, \domainname{student$\_$performance$\_$por} and our biggest domain, \domainname{open$\_$policing$\_$hartford}. Note from that last domain that a larger number of iterations or models with a larger number of trees may be required for the best results, in particular for Ensembles of Generative Trees;
   \item small models may be enough to get excellent results on real world domains whose size would, at first glance, seem to require much bigger ones: on \domainname{compas}, we can compete (rmse) or beat (perr) \mice~whose models contain 7 000 trees with just 20 stumps;
   \item it is important to remember that our technique does not just offer the capability to do missing data imputation: our models can also generate data and compute the full density conditional to observed values, which is not the case of \mice's models, crafted with the sole purpose of solving the task of missing data imputation. Our objective was not to beat \mice, but rather show that Generative Forests and Ensembles of Generative Trees can also be useful for this task.
   \end{itemize}

\setlength\tabcolsep{0pt}
\begin{table}[H]
  \centering
  \includegraphics[trim=0bp 0bp 0bp 0bp,clip,width=\columnwidth]{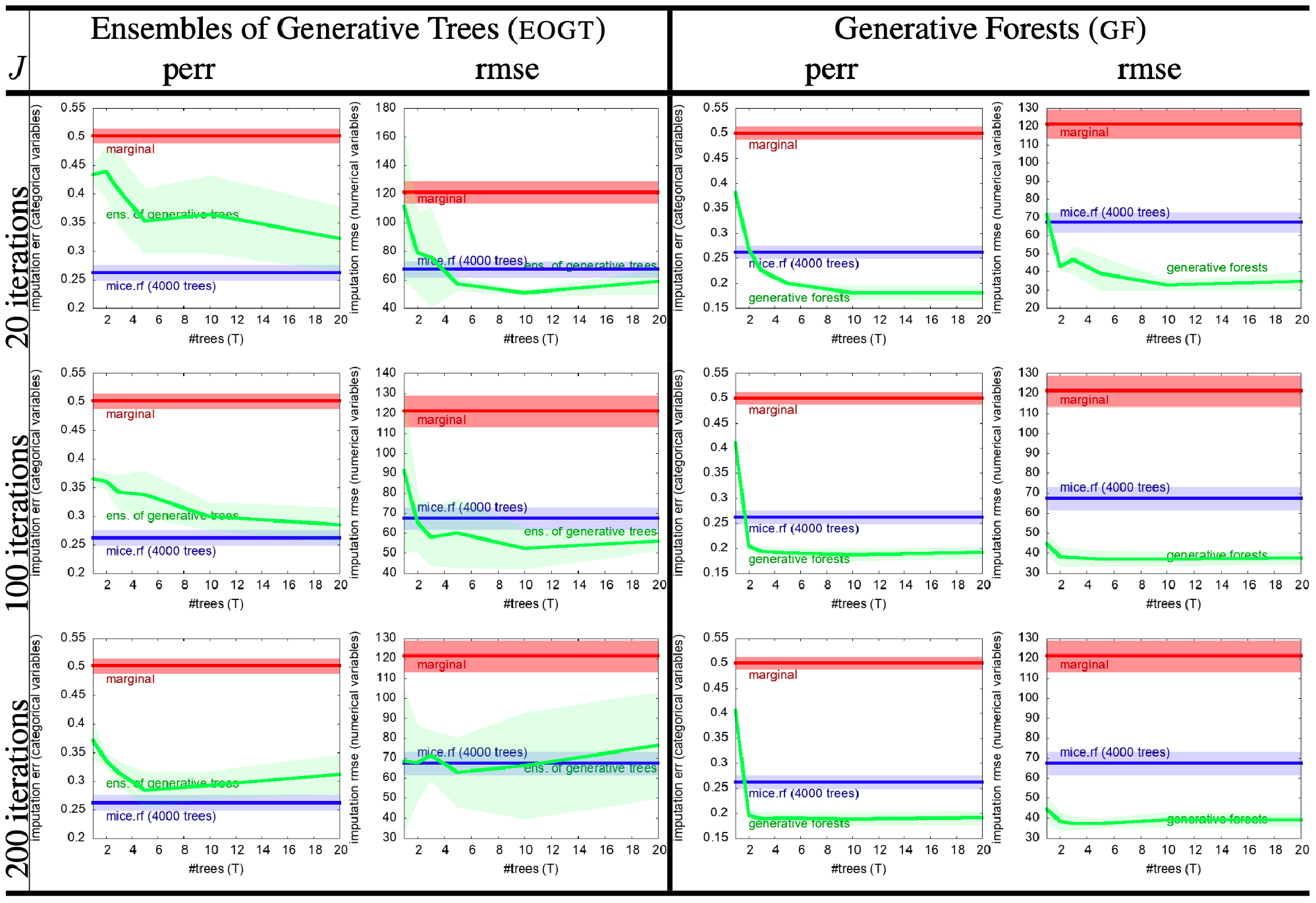}
    \caption{Missing data imputation: results on \texttt{analcatdata$\_$supreme} with $5\%$ missing features (MCAR), for \topdownGT~learning \geot s and \eogt s, vs \texttt{mice} using \texttt{RF} (random forests) with 100 trees each. Each row corresponds to a different total number of iterations $J$ in \topdownGT. Since the domain contains both numerical and categorical variables, de compute separately the errors on categorical variables (using the error probability, denoted perr) and the error on numerical variables (using the root mean square error, rmse). In each plot, we display both the corresponding results of \topdownGT, the results of \texttt{mice} (indicating also the total number of trees used to impute the dataset) and the result of a fast imputation baseline, the marginal computed from the training sample. In each plot, the $x$ axis displays the number of trees $T$ in \topdownGT~and each curve is displayed with its average $\pm$ standard deviation in shaded color. The result for $T=1$ equivalently indicates the performance of a single generative tree (\gt) with $J$ splits \cite{ngGT}.}
    \label{tab:mdi-full-analcatdata_supreme_gf_eogt_vs_mice_rf}
\end{table}

\newpage
  
\setlength\tabcolsep{0pt}
\newpage

\begin{table}[H]
  \centering
  \includegraphics[trim=0bp 0bp 0bp 0bp,clip,width=\columnwidth]{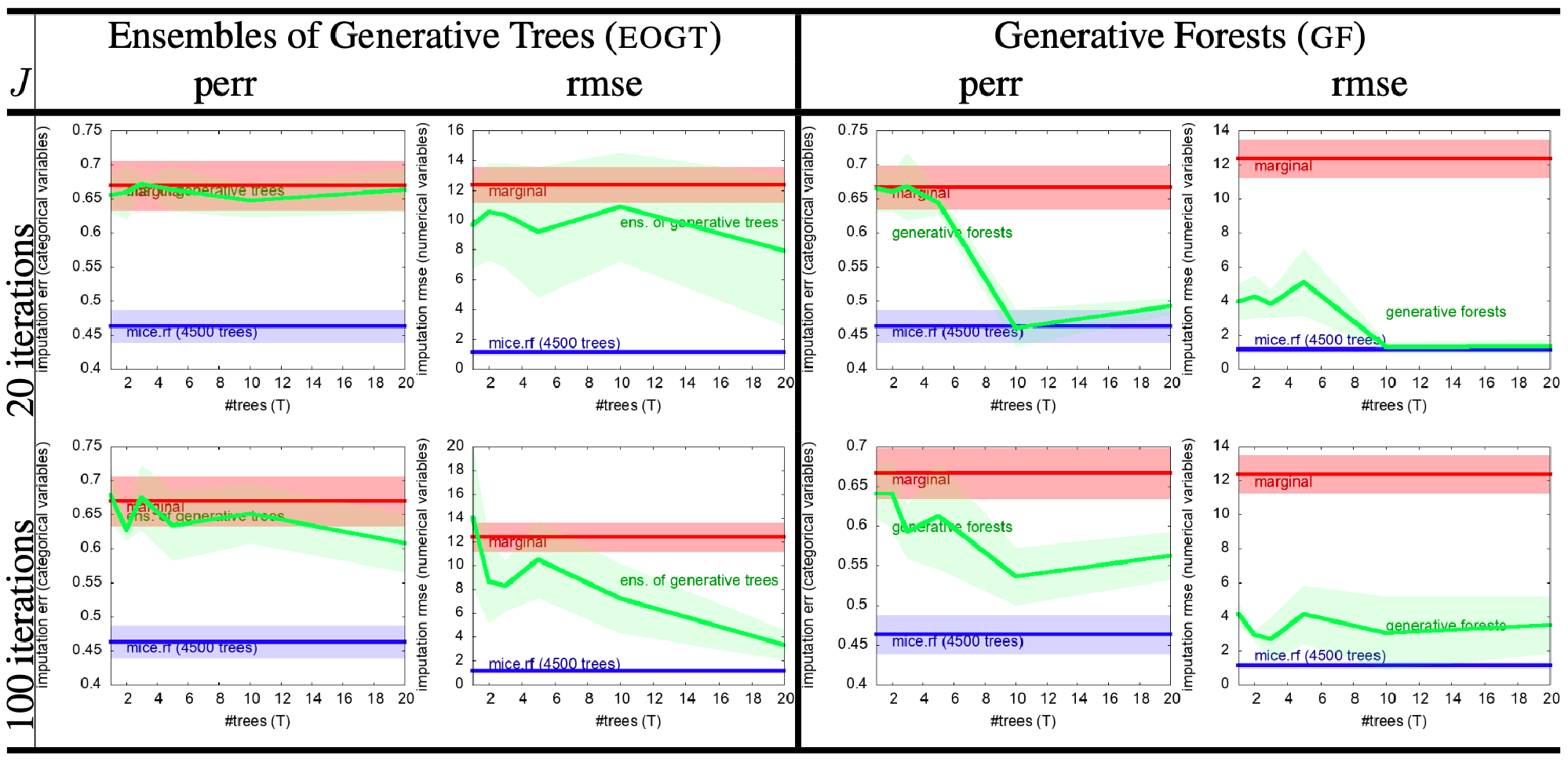}
    \caption{Missing data imputation: results on \texttt{abalone} with $5\%$ missing features (MCAR), for \topdownGT~learning \geot s and \eogt s, vs \texttt{mice} using \texttt{RF} (random forests). Conventions follow Table \ref{tab:mdi-full-analcatdata_supreme_gf_eogt_vs_mice_rf}.}
    \label{tab:mdi-full-abalone_gf_eogt_vs_mice_rf}
\end{table}

\newpage

\setlength\tabcolsep{0pt}

\begin{table}[H]
  \centering
  \includegraphics[trim=0bp 0bp 0bp 0bp,clip,width=\columnwidth]{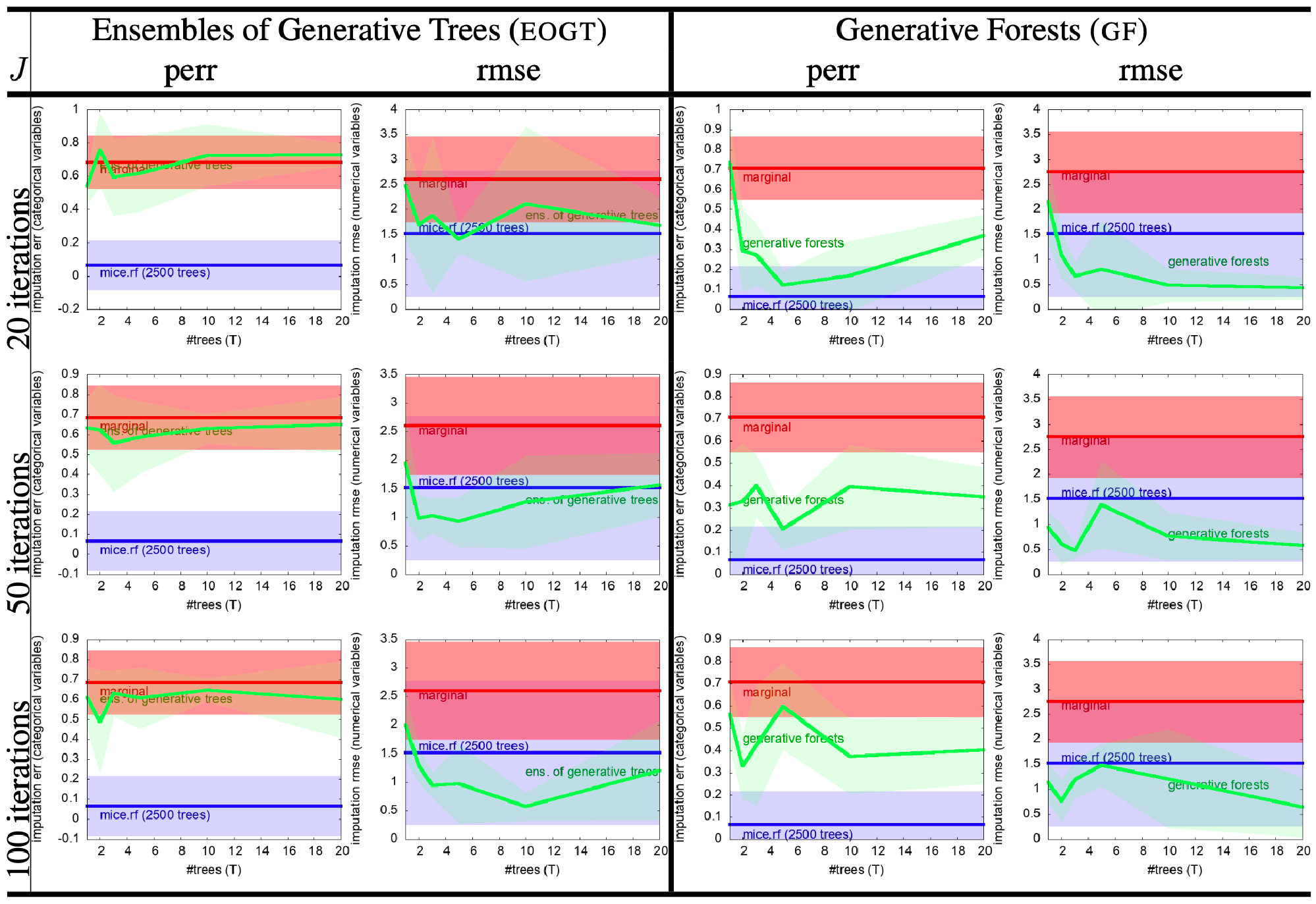}
    \caption{Missing data imputation: results on \texttt{iris} with $5\%$ missing features (MCAR), for \topdownGT~learning \geot s and \eogt s, vs \texttt{mice} using \texttt{RF} (random forests). Conventions follow Table \ref{tab:mdi-full-analcatdata_supreme_gf_eogt_vs_mice_rf}.}
    \label{tab:mdi-full-iris_gf_eogt_vs_mice_rf}
\end{table}

\newpage

\setlength\tabcolsep{0pt}

\begin{table}[H]
  \centering
  \includegraphics[trim=0bp 0bp 0bp 0bp,clip,width=0.7\columnwidth]{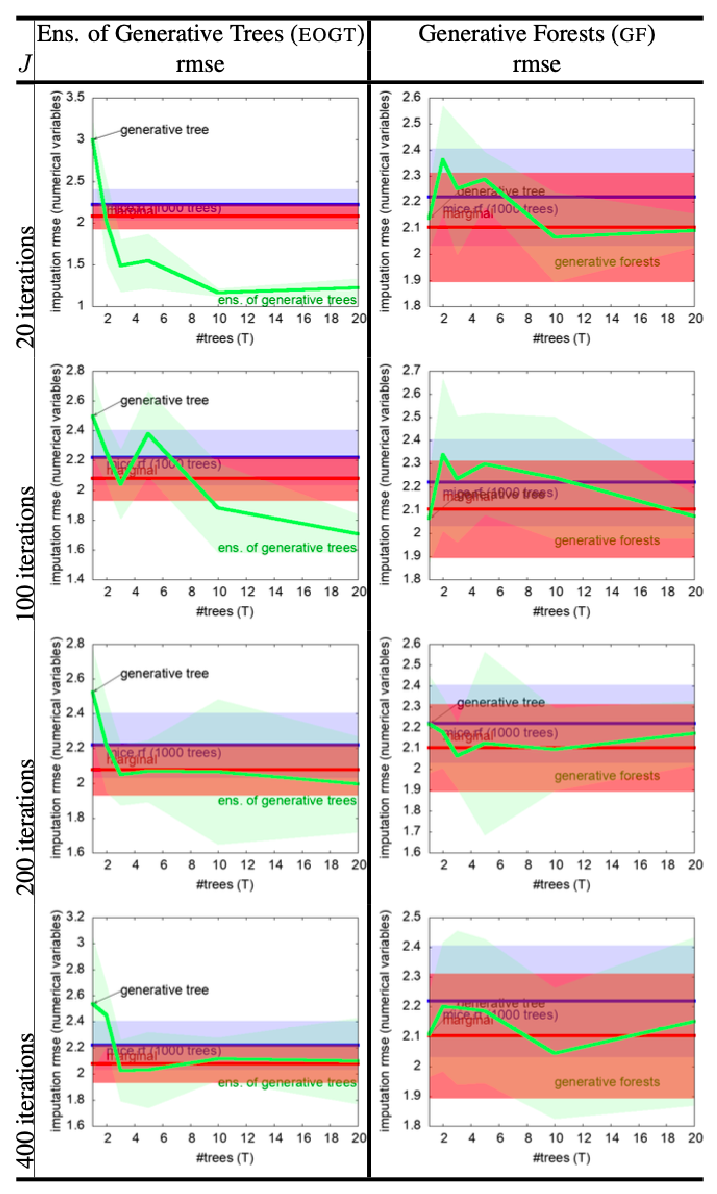}
    \caption{Missing data imputation: results on \texttt{ringgauss} with $5\%$ missing features (MCAR), for \topdownGT~learning \geot s and \eogt s, vs \texttt{mice} using \texttt{RF} (random forests). Since there are no nominal attributes in the domain, we have removed column perr. Conventions otherwise follow Table \ref{tab:mdi-full-analcatdata_supreme_gf_eogt_vs_mice_rf}.}
    \label{tab:mdi-full-ringgauss_gf_eogt_vs_mice_rf}
\end{table}

\newpage

\setlength\tabcolsep{0pt}

\begin{table}[H]
  \centering
  \includegraphics[trim=0bp 0bp 0bp 0bp,clip,width=0.7\columnwidth]{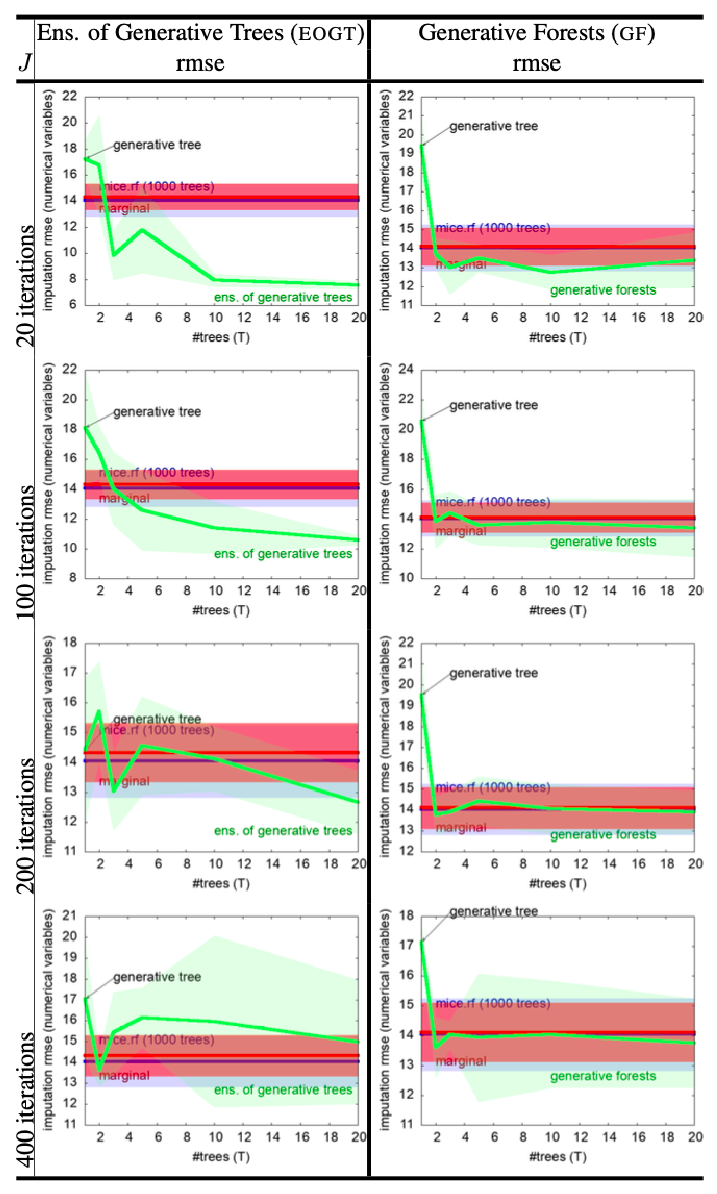}
    \caption{Missing data imputation: results on \texttt{circgauss} with $5\%$ missing features (MCAR), for \topdownGT~learning \geot s and \eogt s, vs \texttt{mice} using \texttt{RF} (random forests). Since there are no nominal attributes in the domain, we have removed column perr. Conventions otherwise follow Table \ref{tab:mdi-full-analcatdata_supreme_gf_eogt_vs_mice_rf}.}
    \label{tab:mdi-full-circgauss_gf_eogt_vs_mice_rf}
\end{table}

\newpage

\setlength\tabcolsep{0pt}

\begin{table}[H]
  \centering
  \includegraphics[trim=0bp 0bp 0bp 0bp,clip,width=0.7\columnwidth]{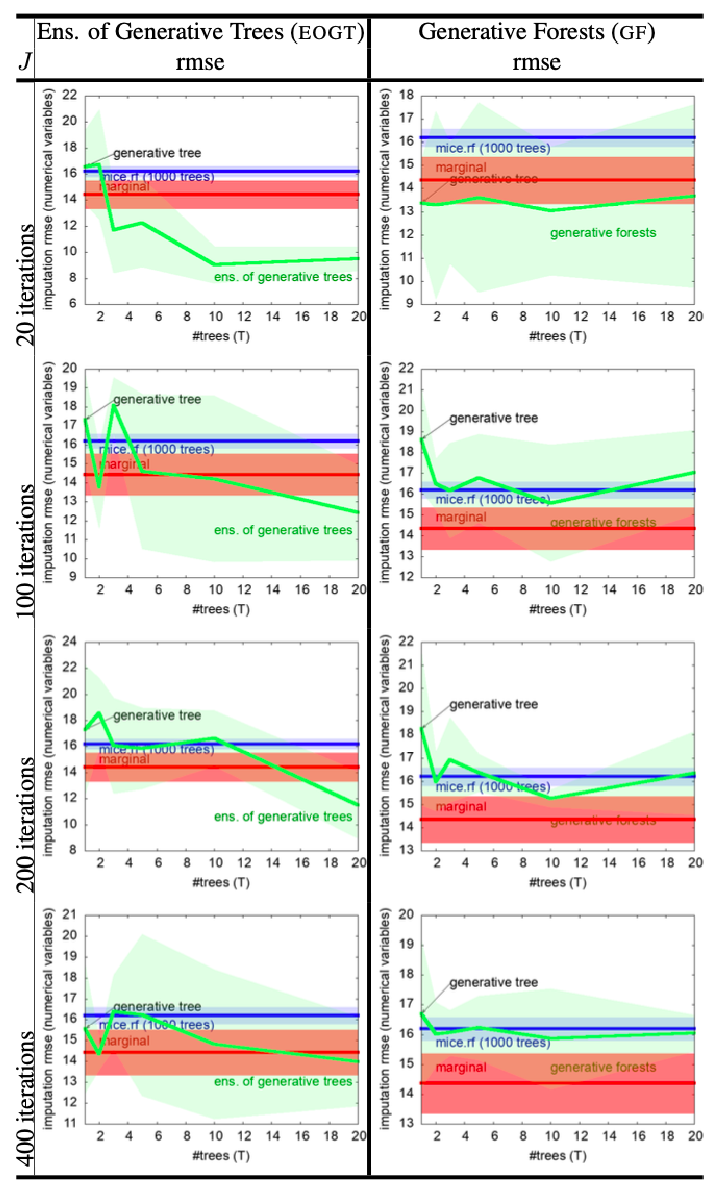}
    \caption{Missing data imputation: results on \texttt{gridgauss} with $5\%$ missing features (MCAR), for \topdownGT~learning \geot s and \eogt s, vs \texttt{mice} using \texttt{RF} (random forests). Since there are no nominal attributes in the domain, we have removed column perr. Conventions otherwise follow Table \ref{tab:mdi-full-analcatdata_supreme_gf_eogt_vs_mice_rf}.}
    \label{tab:mdi-full-gridgauss_gf_eogt_vs_mice_rf}
\end{table}

\newpage

\setlength\tabcolsep{0pt}

\begin{table}[H]
  \centering
  \includegraphics[trim=0bp 0bp 0bp 0bp,clip,width=0.7\columnwidth]{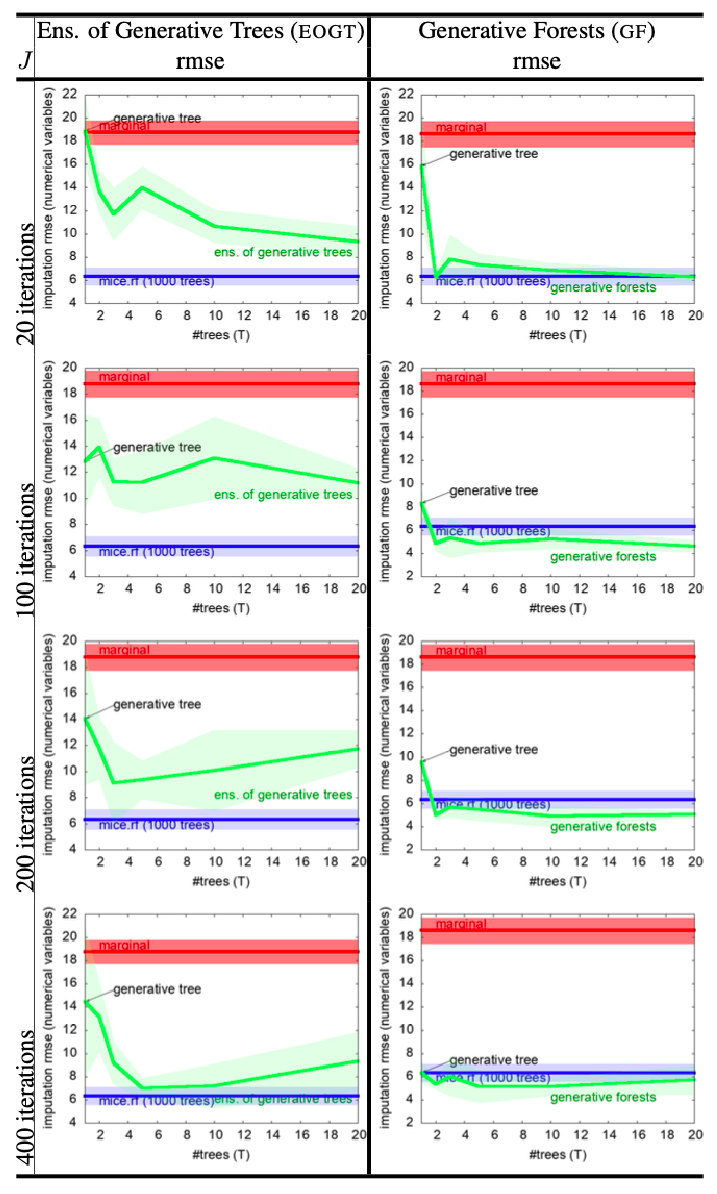}
    \caption{Missing data imputation: results on \texttt{randgauss} with $5\%$ missing features (MCAR), for \topdownGT~learning \geot s and \eogt s, vs \texttt{mice} using \texttt{RF} (random forests). Since there are no nominal attributes in the domain, we have removed column perr. Conventions otherwise follow Table \ref{tab:mdi-full-analcatdata_supreme_gf_eogt_vs_mice_rf}.}
    \label{tab:mdi-full-randgauss_gf_eogt_vs_mice_rf}
\end{table}

\newpage
  
\setlength\tabcolsep{0pt}

\begin{table}[H]
  \centering
  \includegraphics[trim=0bp 0bp 0bp 0bp,clip,width=\columnwidth]{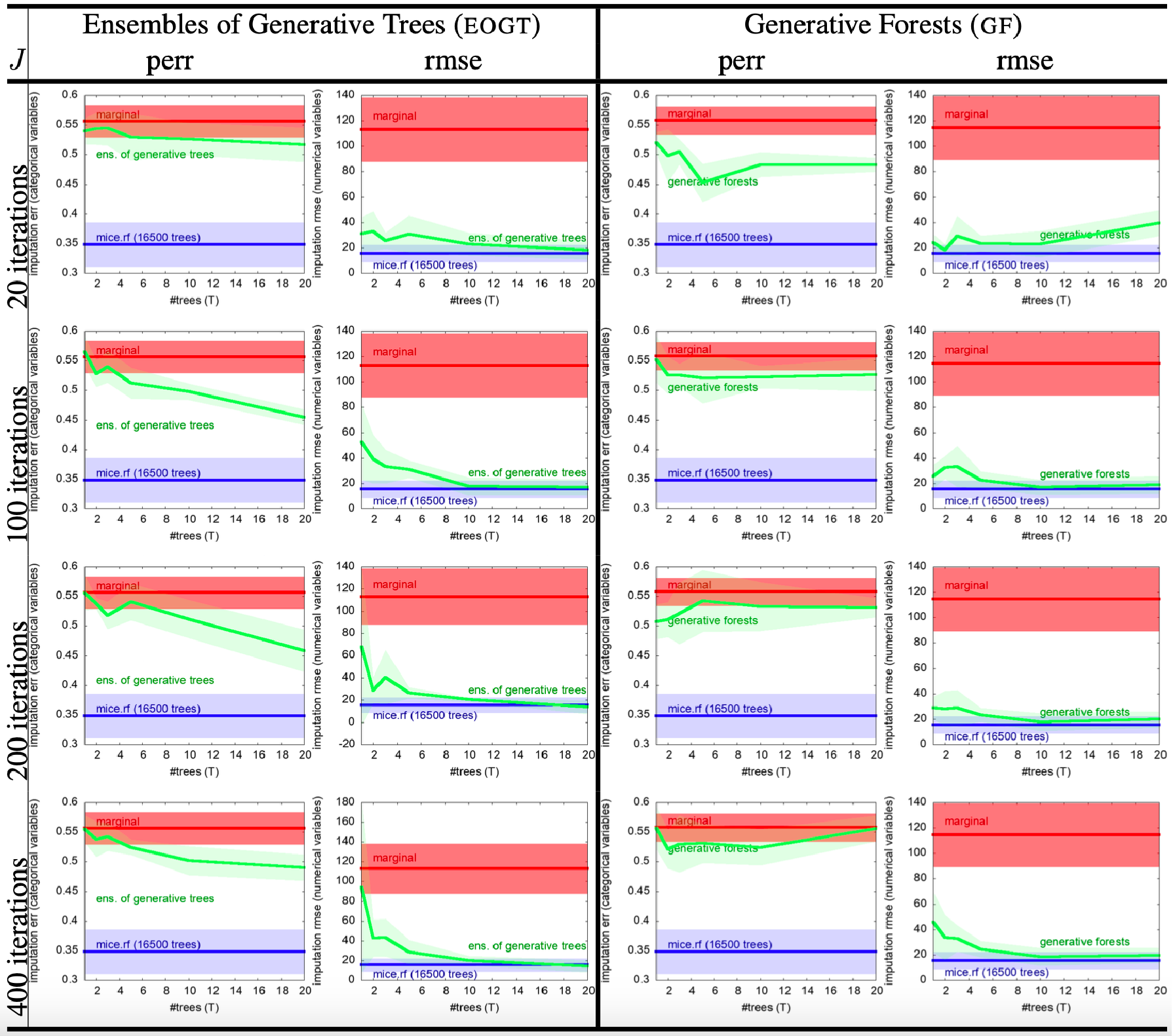}
    \caption{Missing data imputation: results on \texttt{student$\_$performance$\_$mat} with $5\%$ missing features (MCAR), for \topdownGT~learning \geot s and \eogt s, vs \texttt{mice} using \texttt{RF} (random forests). Conventions follow Table \ref{tab:mdi-full-analcatdata_supreme_gf_eogt_vs_mice_rf}.}
    \label{tab:mdi-full-student_performance_mat_gf_eogt_vs_mice_rf}
\end{table}

\newpage
  
\setlength\tabcolsep{0pt}

\begin{table}[H]
  \centering
  \includegraphics[trim=0bp 0bp 0bp 0bp,clip,width=\columnwidth]{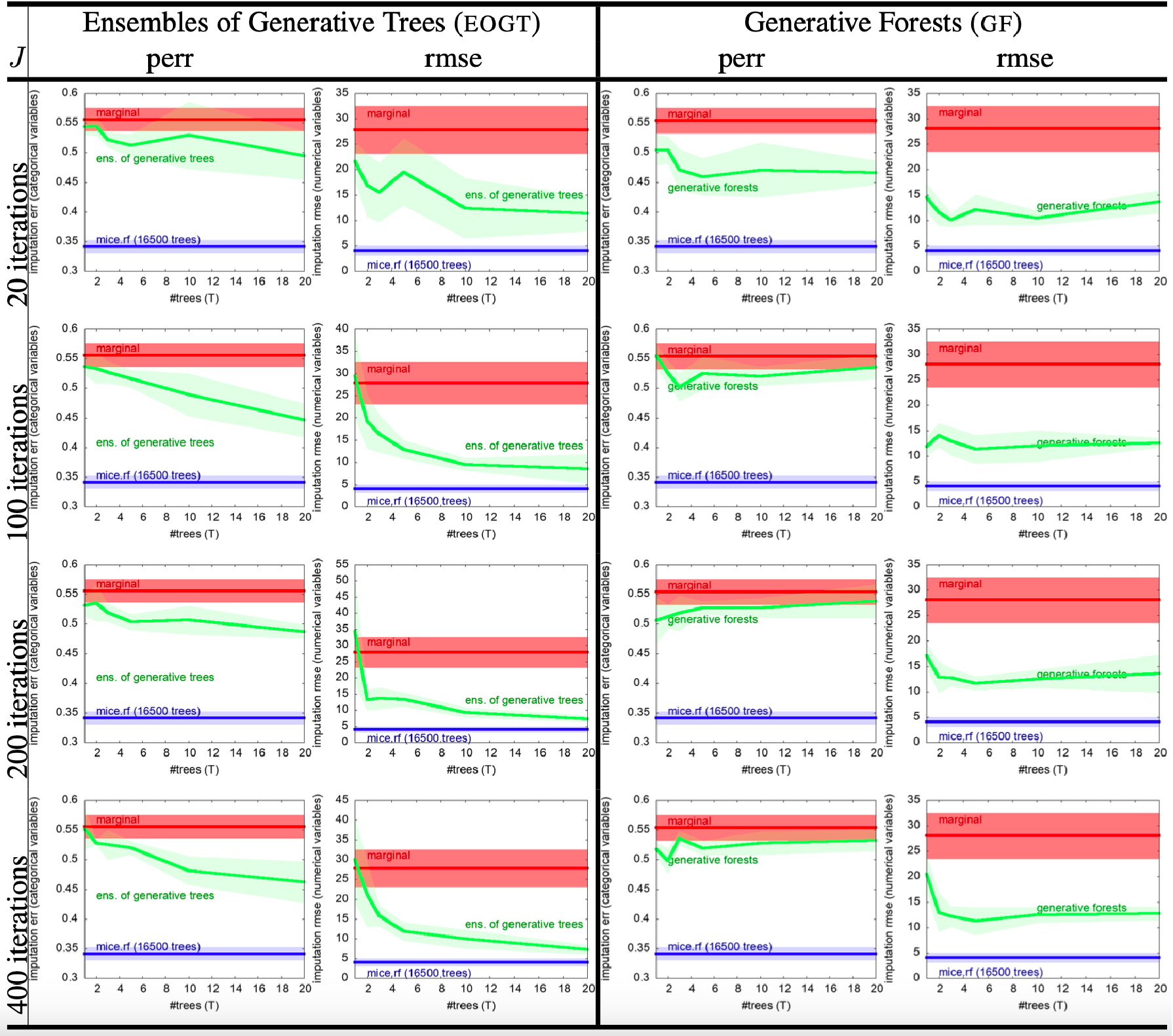}
    \caption{Missing data imputation: results on \texttt{student$\_$performance$\_$por} with $5\%$ missing features (MCAR), for \topdownGT~learning \geot s and \eogt s, vs \texttt{mice} using \texttt{RF} (random forests). Conventions follow Table \ref{tab:mdi-full-analcatdata_supreme_gf_eogt_vs_mice_rf}.}
    \label{tab:mdi-full-student_performance_por_gf_eogt_vs_mice_rf}
\end{table}

\newpage

\setlength\tabcolsep{0pt}

\begin{table}[H]
  \centering
  \includegraphics[trim=0bp 0bp 0bp 0bp,clip,width=\columnwidth]{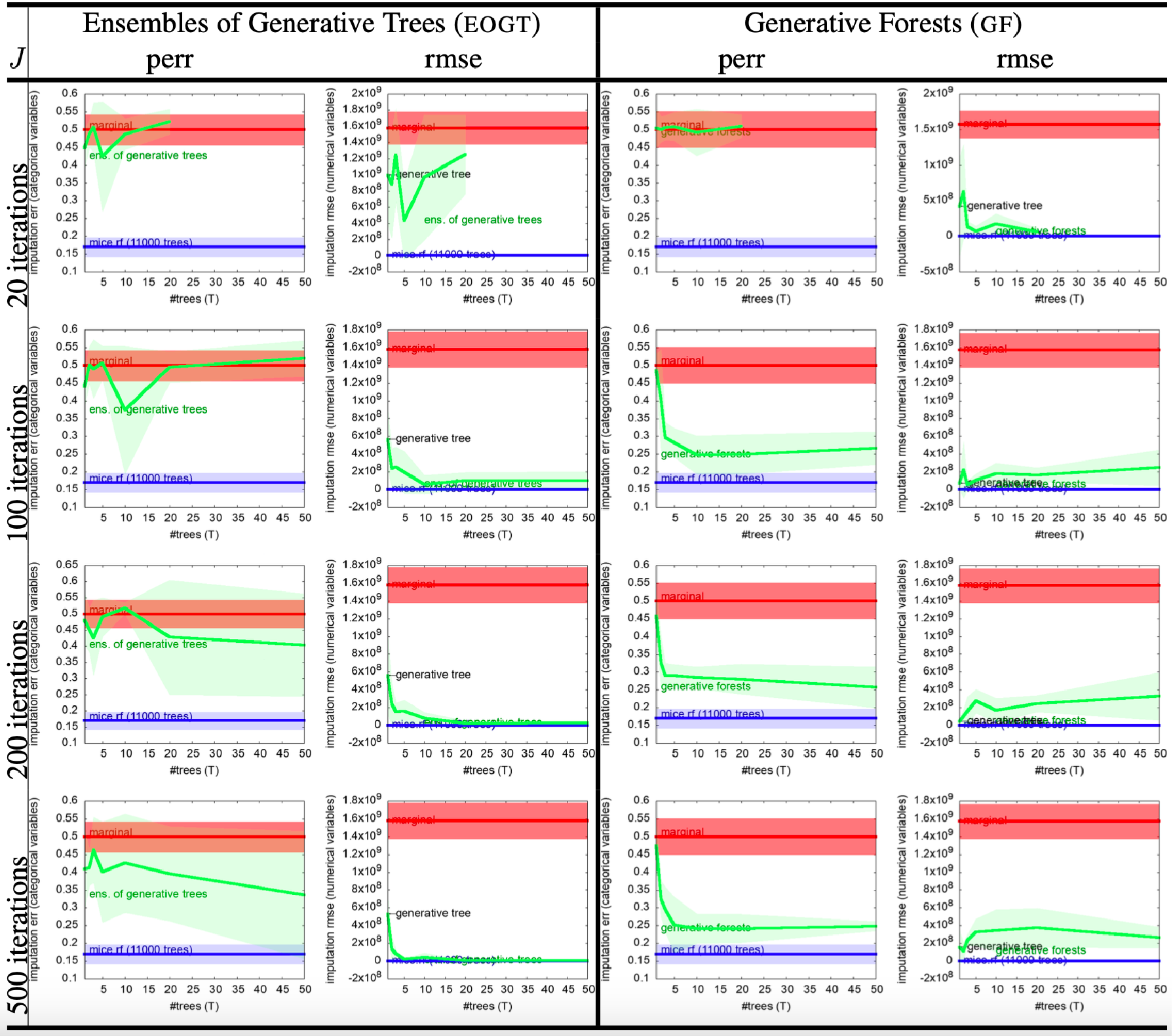}
    \caption{Missing data imputation: results on \texttt{kc1} with $5\%$ missing features (MCAR), for \topdownGT~learning \geot s and \eogt s, vs \texttt{mice} using \texttt{RF} (random forests). Conventions follow Table \ref{tab:mdi-full-analcatdata_supreme_gf_eogt_vs_mice_rf}.}
    \label{tab:mdi-full-kc1_gf_eogt_vs_mice_rf}
  \vspace{-0.2cm}
\end{table}

\newpage

\setlength\tabcolsep{0pt}

\begin{table}[H]
  \centering
  \includegraphics[trim=0bp 0bp 0bp 0bp,clip,width=\columnwidth]{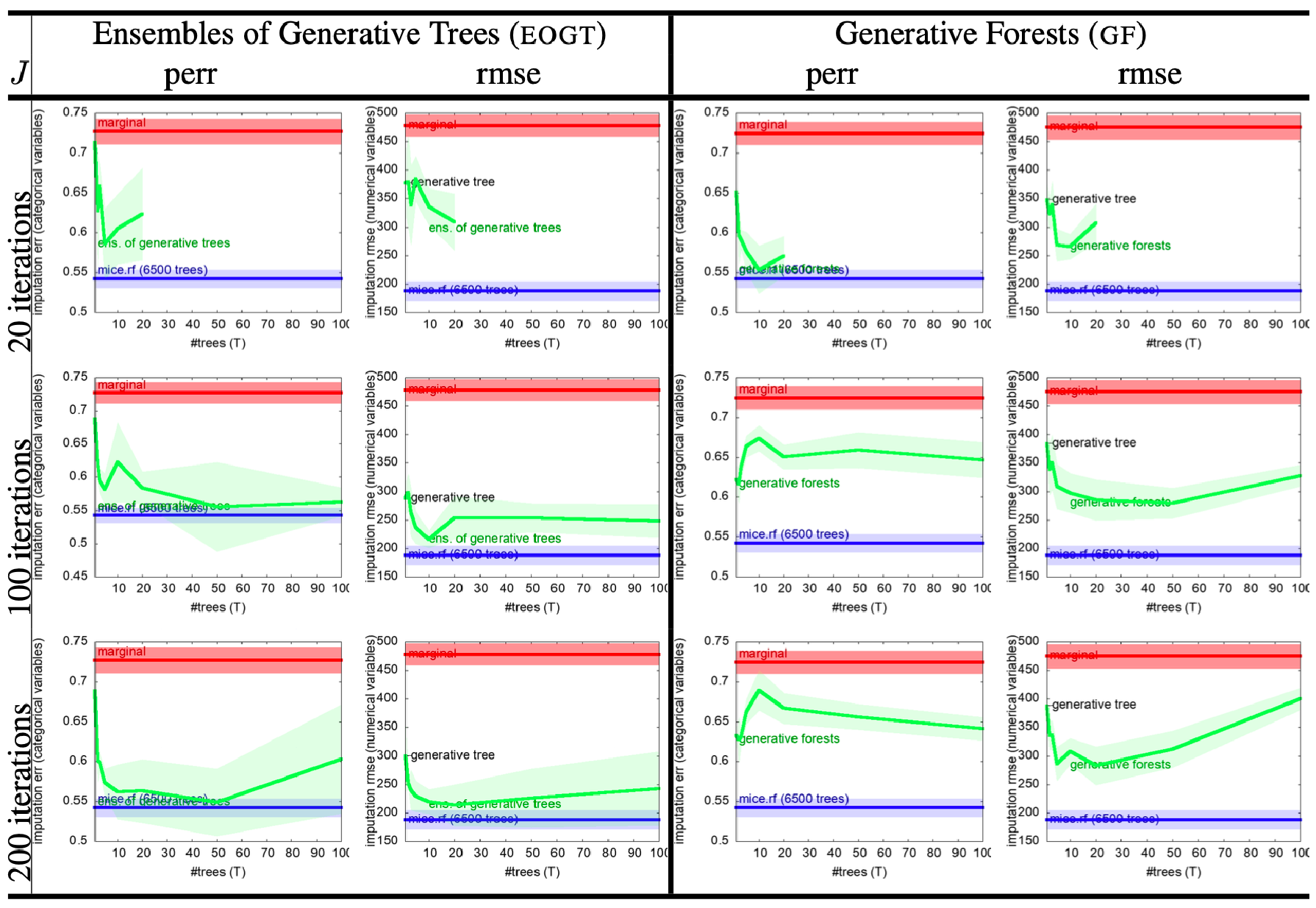}
    \caption{Missing data imputation: results on \texttt{sigma$\_$cabs} with $5\%$ missing features (MCAR), for \topdownGT~learning \geot s and \eogt s, vs \texttt{mice} using \texttt{RF} (random forests). Conventions follow Table \ref{tab:mdi-full-analcatdata_supreme_gf_eogt_vs_mice_rf}.}
    \label{tab:mdi-full-sigma_cabs_gf_eogt_vs_mice_rf}
\end{table}

\newpage

\setlength\tabcolsep{0pt}

\begin{table}[H]
  \centering
  \includegraphics[trim=0bp 0bp 0bp 0bp,clip,width=\columnwidth]{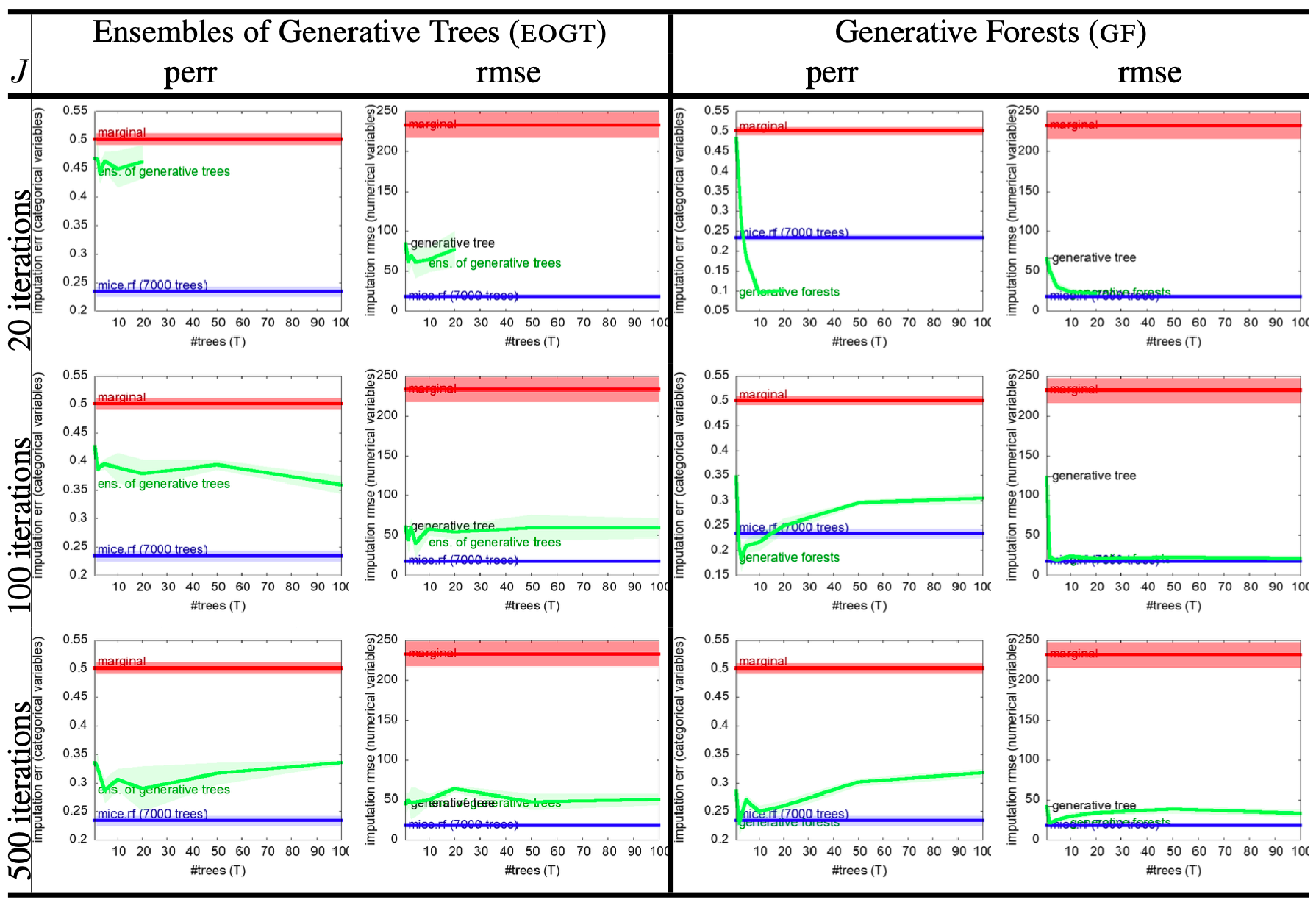}
    \caption{Missing data imputation: results on \texttt{compas} with $5\%$ missing features (MCAR), for \topdownGT~learning \geot s and \eogt s, vs \texttt{mice} using \texttt{RF} (random forests). Conventions follow Table \ref{tab:mdi-full-analcatdata_supreme_gf_eogt_vs_mice_rf}.}
    \label{tab:mdi-full-compas_gf_eogt_vs_mice_rf}
\end{table}

\newpage

\setlength\tabcolsep{0pt}

\begin{table}[H]
  \centering
  \includegraphics[trim=0bp 0bp 0bp 0bp,clip,width=\columnwidth]{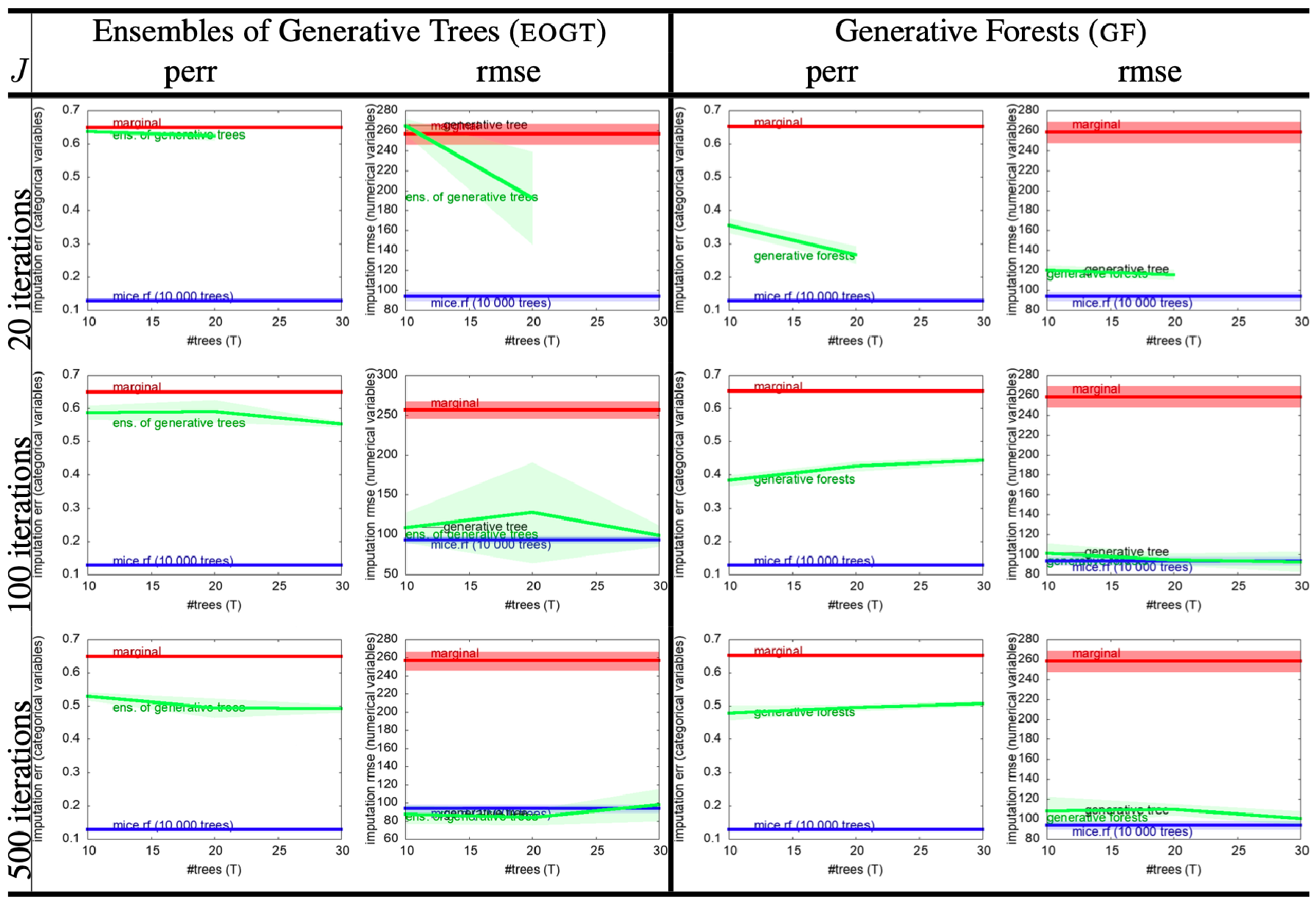}
    \caption{Missing data imputation: results on \texttt{open$\_$policing$\_$hartford} with $5\%$ missing features (MCAR), for \topdownGT~learning \geot s and \eogt s, vs \texttt{mice} using \texttt{RF} (random forests). Conventions follow Table \ref{tab:mdi-full-analcatdata_supreme_gf_eogt_vs_mice_rf}.}
    \label{tab:mdi-full-open_policing_hartford_gf_eogt_vs_mice_rf}
\end{table}

\newpage

\subsubsection{Full comparisons with KDE on density estimation} \label{sec-full-comp-kde}

For each domain we either use the expected density, or the expected negative log-density, as the metric to be maximized. We consider the expected density to cope with the eventuality of zero (pointwise) density predictions. Tables \ref{tab:gen-full-kde-1} and \ref{tab:gen-full-kde-2} present the results against KDE (details in Table \ref{tab:gen-full-kde-1}). Table \ref{tab:mdi-gridgauss} singles out a result on the domain \domainname{gridGauss}, which displays an interesting pattern. The leaf chosen for a split in \topdownGT~is always the heaviest leaf; hence, as long as the iteration number $j\leq T$ (the maximum number of trees, which is 500 in this experiment), \topdownGT~grows stumps. In the case of \domainname{gridGauss}, like in a few other domains, there is a clear plateau in performances for $j\leq 500$, which does not hold anymore as $j > 500$. This, we believe, characterizes the fact that on these domains, which all share the property that their dimension is small, training stumps attains a limit in terms of performances. Table \ref{tab:mdi-abalone} shows that by changing basic parameters -- here, decreasing the number of trees, increasing the total number of iterations --, one can obtain substantially better results.

\begin{table*}[t]
  \centering
  \includegraphics[trim=150bp 40bp 170bp 80bp,clip,width=0.6\columnwidth]{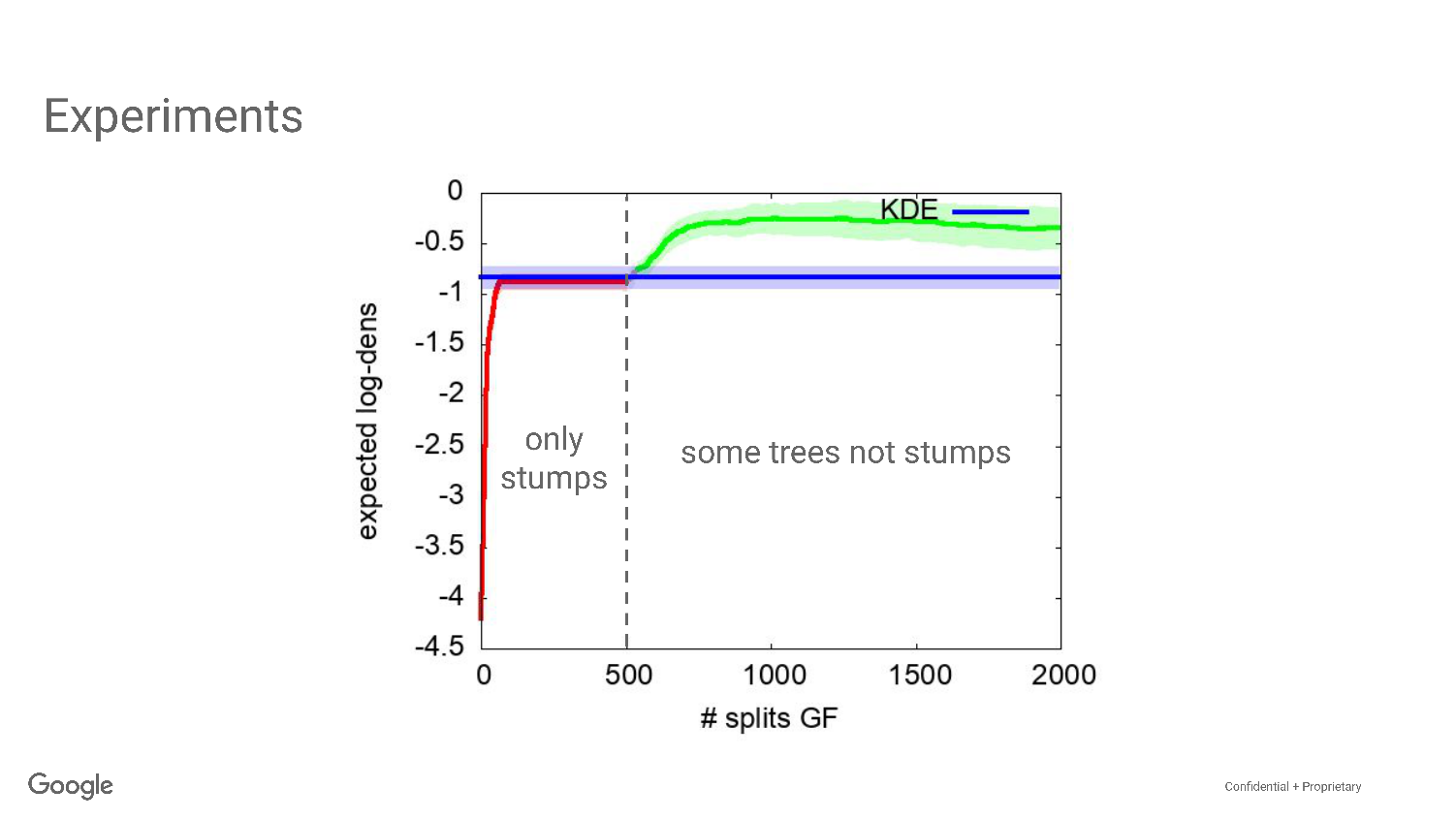} 
  \caption{\textsc{density}: zoom on the \domainname{gridGauss} domain: the \geot~is grown by \topdownGT~in a way that first learns a complete set of stumps (there are a total of $T=500$ trees in the final model) and then grows some of the trees. The plateau for $j\leq 500$ leaves displays that we quickly reach a limit in terms of maximal performances of stumps, this being probably dues to the fact that the dimension of the domain is small (2). We can however remark that with just stumps, we still attain close performances to KDE (note that we plot the expected log-densities; see text for details).}
    \label{tab:mdi-gridgauss}
\end{table*}

\begin{table*}[t]
  \centering
  \begin{tabular}{c|c}\Xhline{2pt}
    \includegraphics[trim=0bp 0bp 0bp 0bp,clip,width=\troisbis\columnwidth]{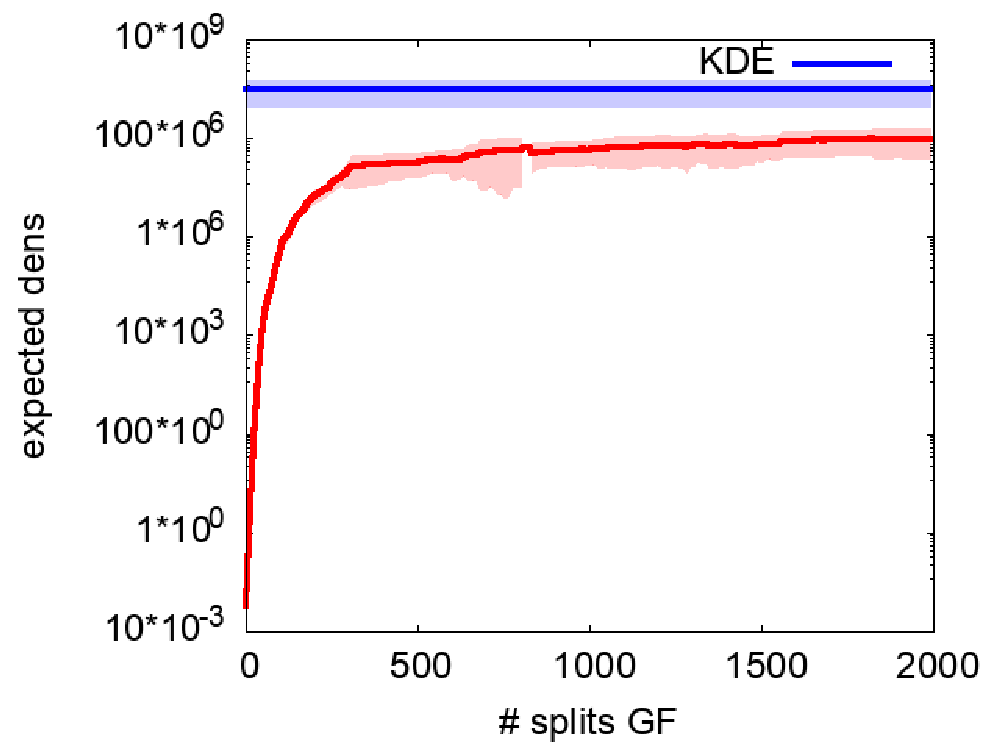}& \includegraphics[trim=0bp 0bp 0bp 0bp,clip,width=\troisbis\columnwidth]{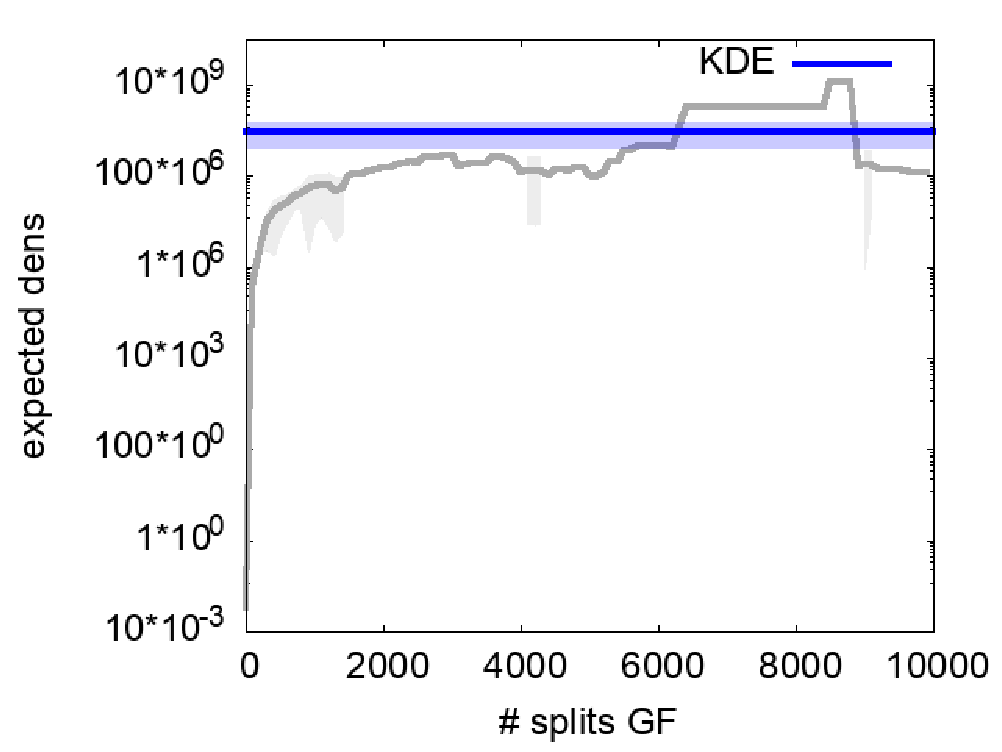}\\
  $T=500, J=2000$  & $T=15, J=10000$ \\ \Xhline{2pt}
  \end{tabular}
  \caption{\textsc{density}: zoom on the \domainname{abalone} domain: with $T=500$ trees, total of $J=2000$ as in Tables \ref{tab:gen-full-kde-1} and \ref{tab:gen-full-kde-2}, we are clearly beaten by KDE (left). A much smaller of bigger trees ($T=15,J=10000$) is enough to become competitive and beat KDE, albeit not statistically significantly (right). The right picture also displays the interest into getting pruning or early stopping algorithms to prevent eventual overfitting.}
    \label{tab:mdi-abalone}
\end{table*}

\begin{table}
  \centering
  \begin{tabular}{c|c|c}\Xhline{2pt}
     \hspace{\deltaad} \includegraphics[trim=0bp 0bp 0bp 0bp,clip,width=\troisbis\columnwidth]{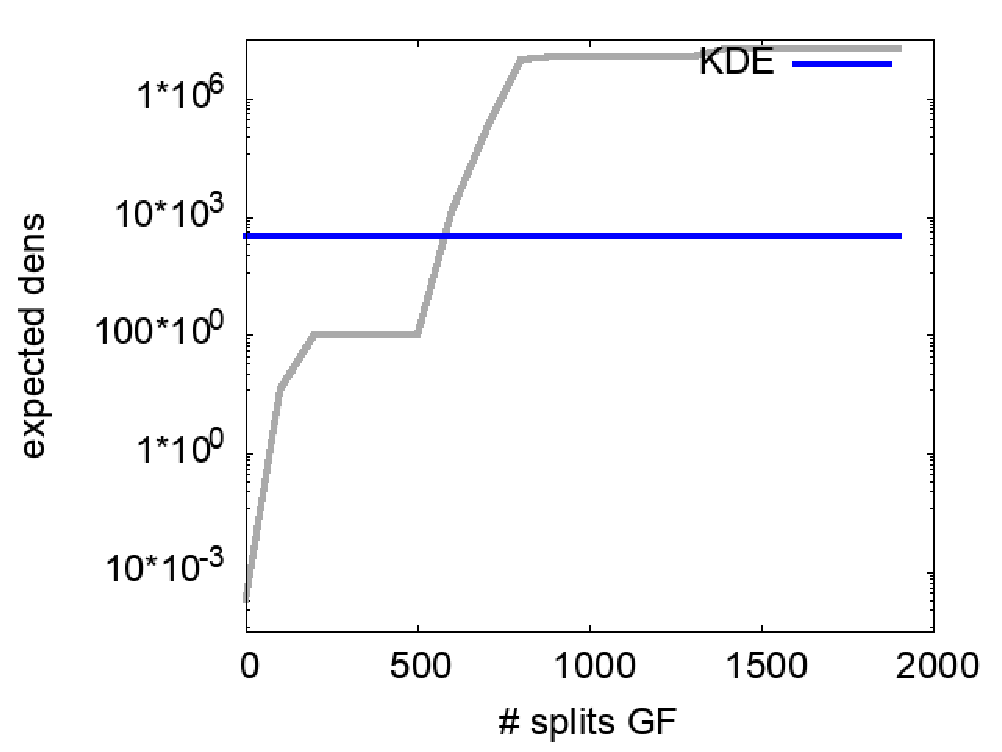} \hspace{\deltaad} & \hspace{\deltaad} \includegraphics[trim=0bp 0bp 0bp 0bp,clip,width=\troisbis\columnwidth]{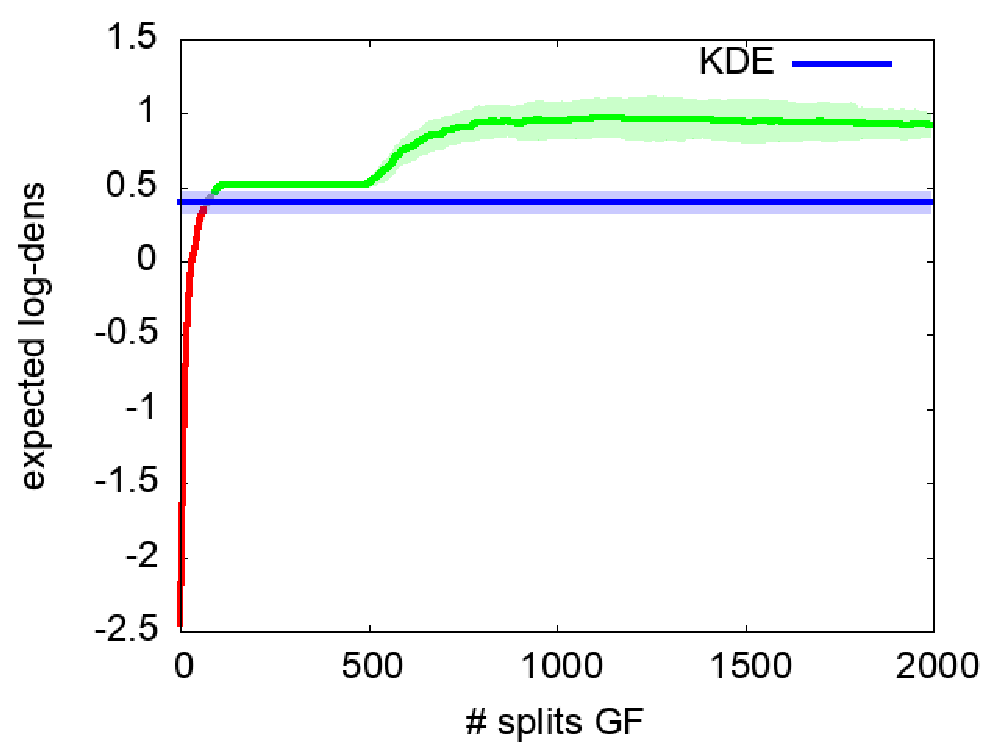} \hspace{\deltaad}  & \hspace{\deltaad}  \includegraphics[trim=0bp 0bp 0bp 0bp,clip,width=\troisbis\columnwidth]{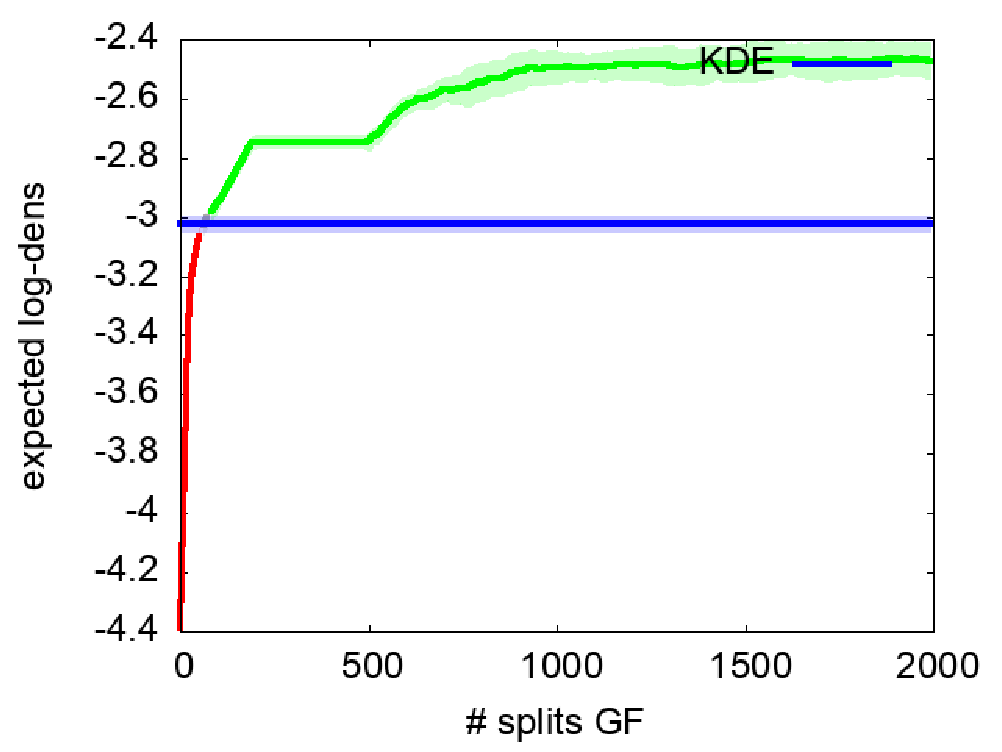}  \hspace{\deltaad} \\
  \hspace{\deltaad}  \domainname{iris} \hspace{\deltaad} & \hspace{\deltaad} \domainname{ringGauss} \hspace{\deltaad} & \hspace{\deltaad} \domainname{circGauss} \hspace{\deltaad} \\  \Xhline{2pt}
     \hspace{\deltaad} \includegraphics[trim=0bp 0bp 0bp 0bp,clip,width=\troisbis\columnwidth]{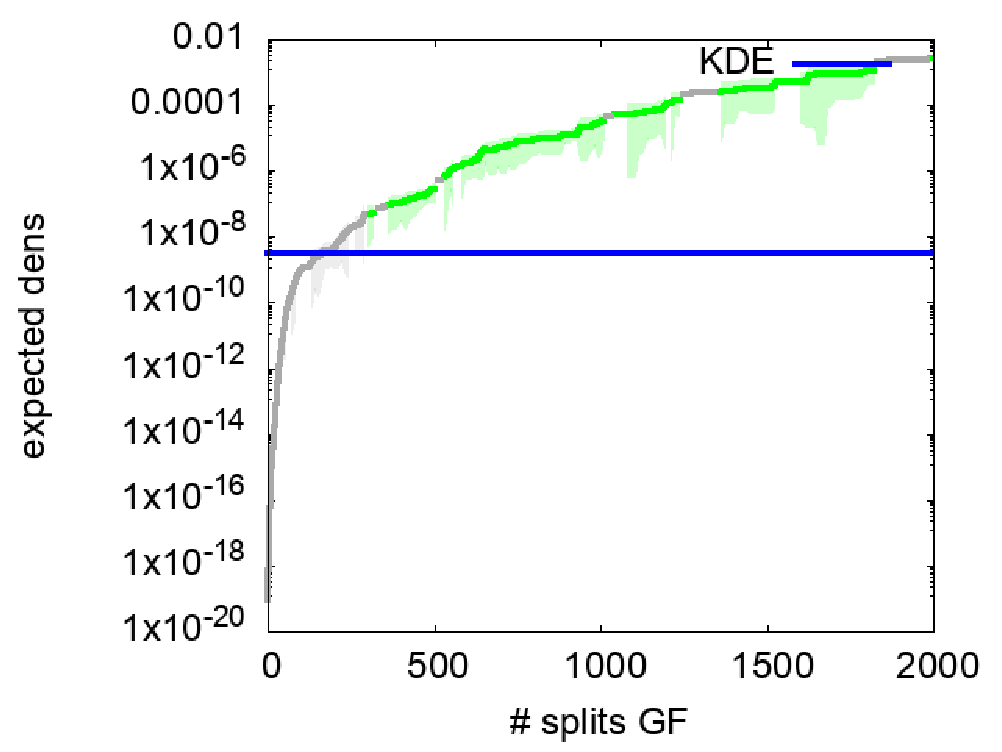} \hspace{\deltaad}  & \hspace{\deltaad}  \includegraphics[trim=0bp 0bp 0bp 0bp,clip,width=\troisbis\columnwidth]{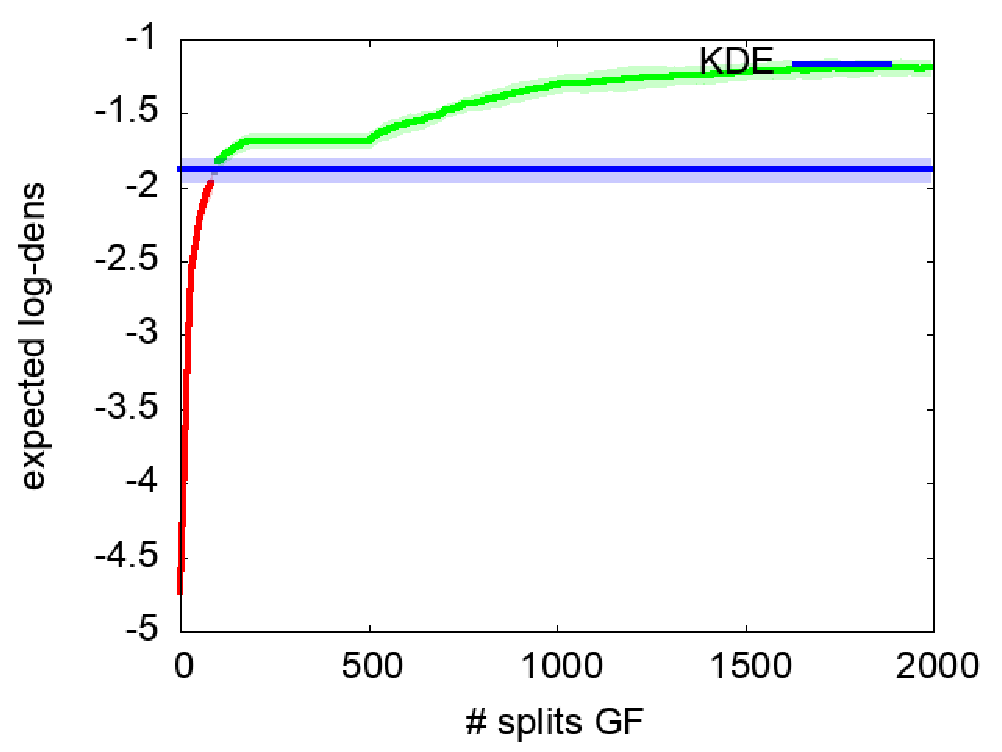}  \hspace{\deltaad} &  \hspace{\deltaad} \includegraphics[trim=0bp 0bp 0bp 0bp,clip,width=\troisbis\columnwidth]{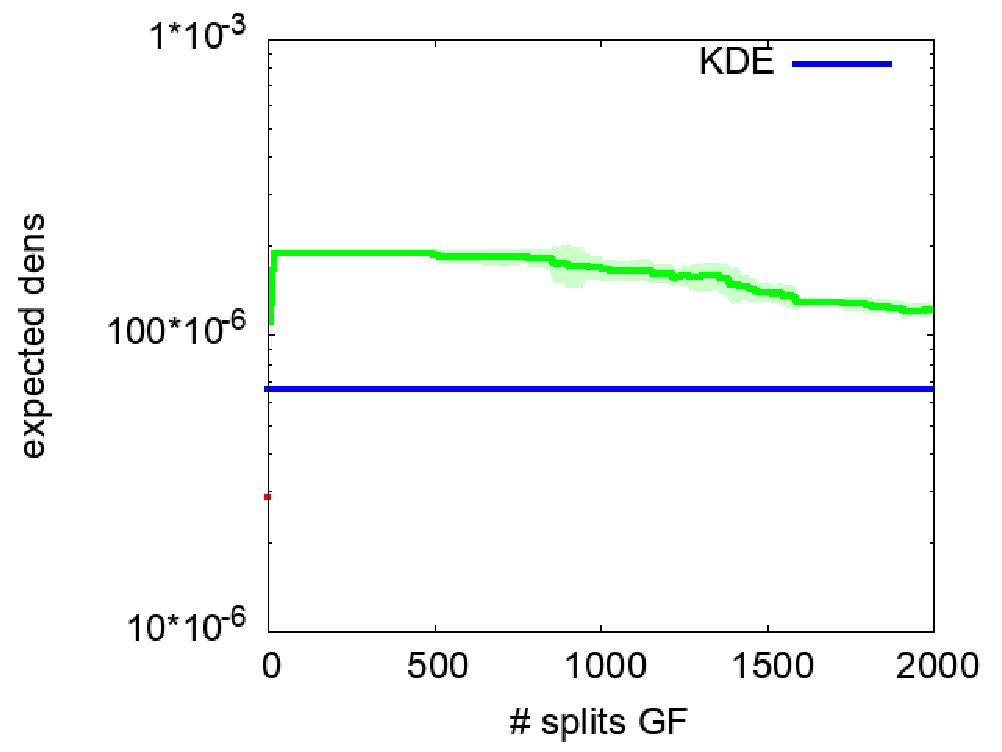} \hspace{\deltaad}\\
  \hspace{\deltaad}  \domainname{forestfires} \hspace{\deltaad} & \hspace{\deltaad} \domainname{randGauss} \hspace{\deltaad} & \hspace{\deltaad} \domainname{tictactoe} \hspace{\deltaad} \\  \Xhline{2pt}
     \hspace{\deltaad} \includegraphics[trim=0bp 0bp 0bp 0bp,clip,width=\troisbis\columnwidth]{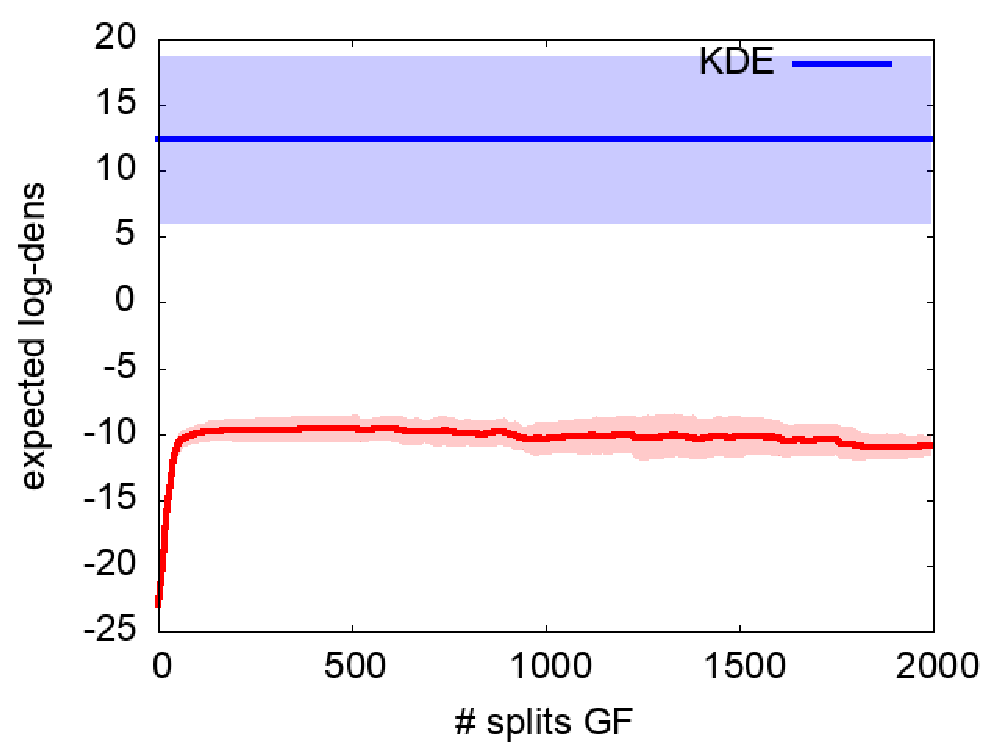} \hspace{\deltaad}  & \hspace{\deltaad}  \includegraphics[trim=0bp 0bp 0bp 0bp,clip,width=\troisbis\columnwidth]{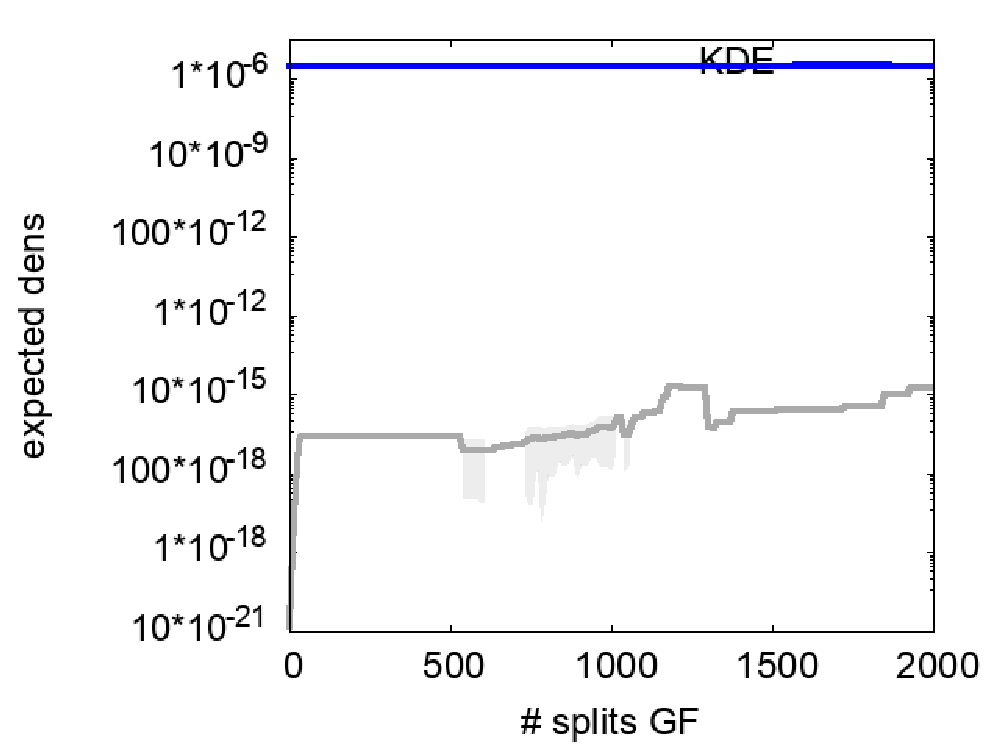}  \hspace{\deltaad} &  \hspace{\deltaad} \includegraphics[trim=0bp 0bp 0bp 0bp,clip,width=\troisbis\columnwidth]{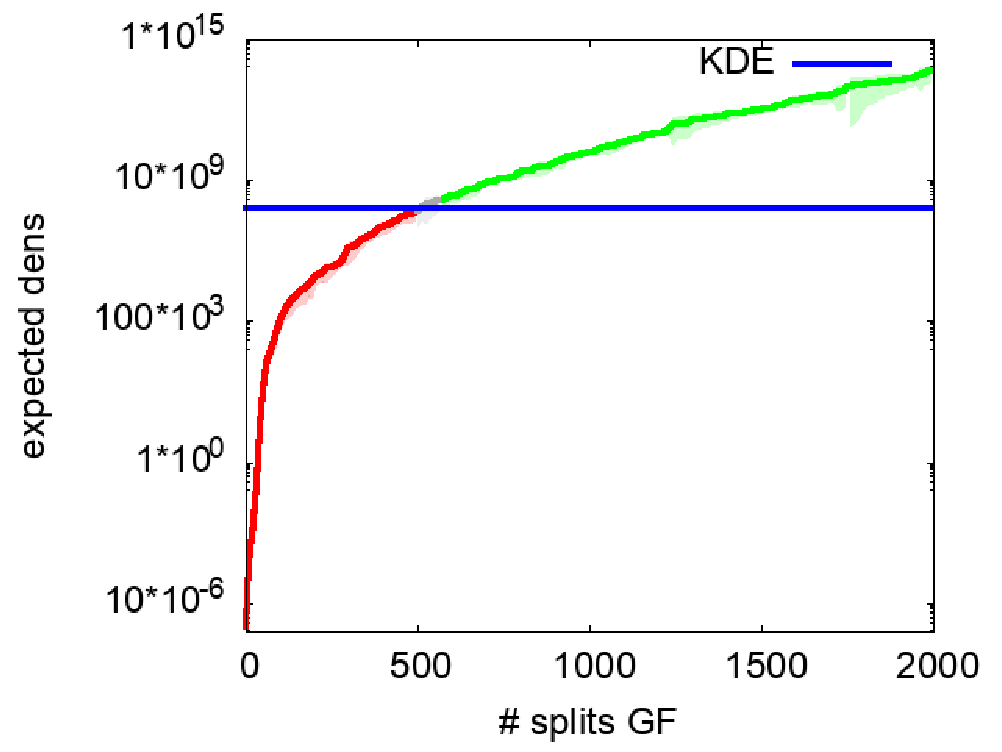} \hspace{\deltaad}  \\
  \hspace{\deltaad} \domainname{ionosphere} \hspace{\deltaad} & \hspace{\deltaad} \domainname{student$\_$performance$\_$mat} & \hspace{\deltaad}  \domainname{winered} \hspace{\deltaad}  \\  \Xhline{2pt}
  \end{tabular}
  \caption{Comparison of results for Kernel Density Estimation (KDE) and Generative Forests (us, \geot) with parameters $T=500$ trees, total of $J=2000$ splits in trees, on a first batch of domains, put in increasing size according to Table \ref{t-s-uci}. For each domain, we compute for \geot~at regular intervals the average $\pm$ standard deviation of either (i) density values or (ii) negative log density values for the test fold (sometimes, KDE gives 0 estimated density which prevents the computation of (ii)). Warning: some $y$ scales are logscale. We then compute the baseline of KDE and display its average $\pm$ standard deviation in blue on each picture. We perform, at each applicable iteration of \topdownGT~(\textit{i.e.} each iteration for many domains, which can be seen from the plots), a Student paired $t$-test over the five folds to compare the results with KDE. If we are statistically better ($p=0.1$), our plot is displayed in {\color{darkgreen} green}; if KDE is better, our plot is displayed in {\color{red} red}; otherwise it is displayed in {\color{Gray} gray}. When standard deviations are not displayed, it means average $\pm$ standard deviation exceeds the plot's $y$ scale.}
    \label{tab:gen-full-kde-1}
  \end{table}

  \begin{table}
  \centering
  \begin{tabular}{c|c|c}\Xhline{2pt}
     \hspace{\deltaad} \includegraphics[trim=0bp 0bp 0bp 0bp,clip,width=\troisbis\columnwidth]{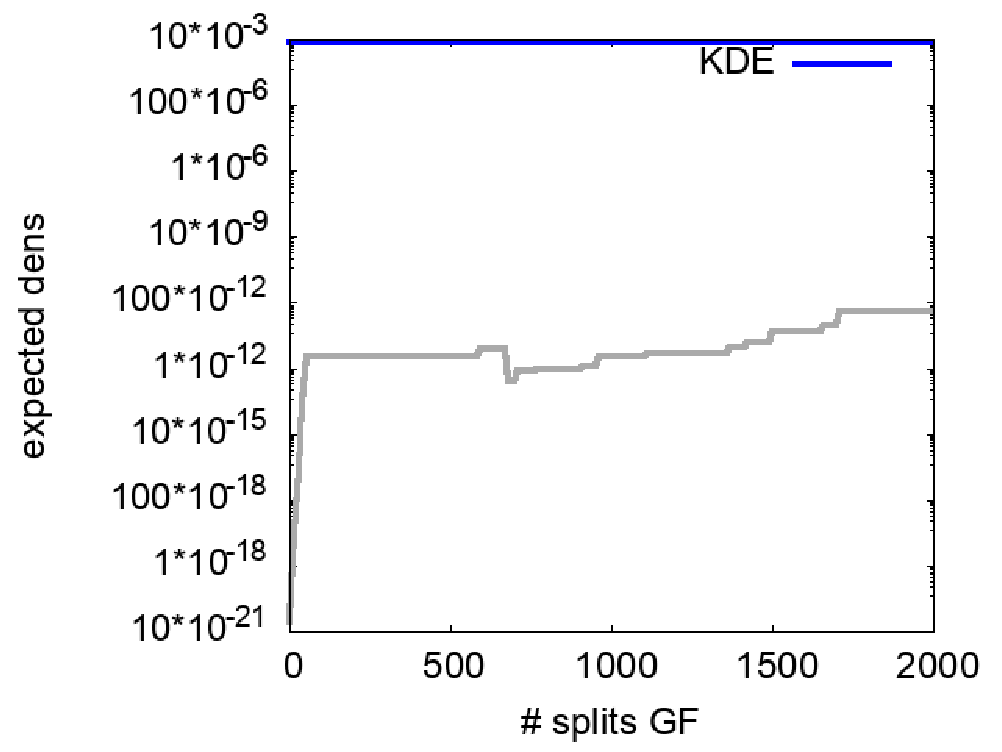} \hspace{\deltaad}  & \hspace{\deltaad}  \includegraphics[trim=0bp 0bp 0bp 0bp,clip,width=\troisbis\columnwidth]{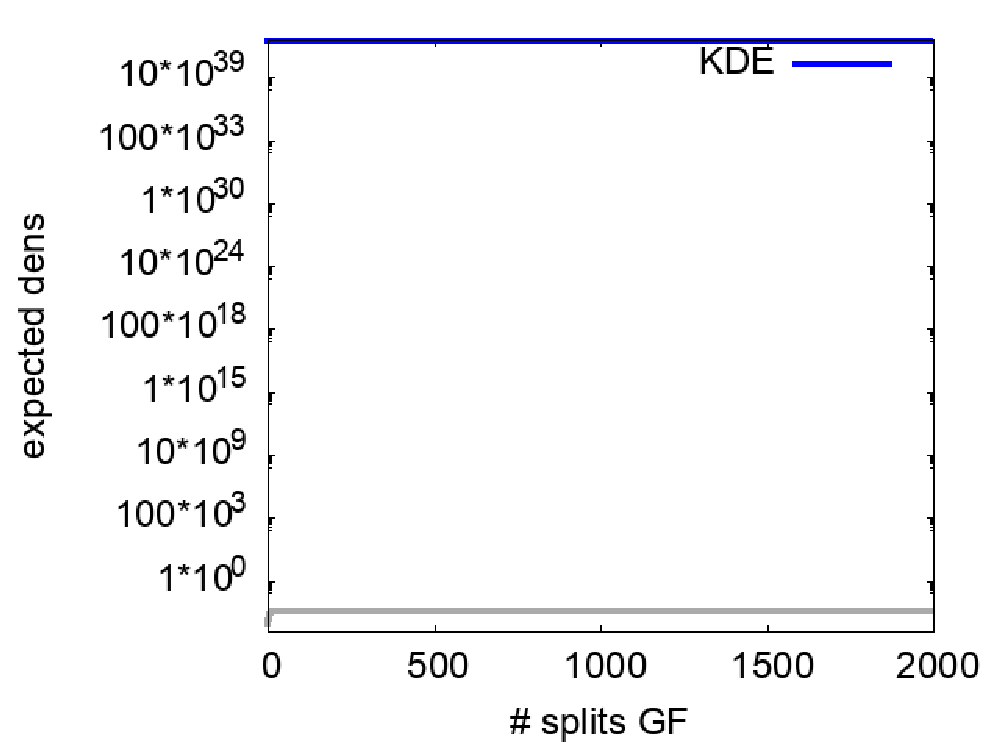}  \hspace{\deltaad} & \hspace{\deltaad} \includegraphics[trim=0bp 0bp 0bp 0bp,clip,width=\troisbis\columnwidth]{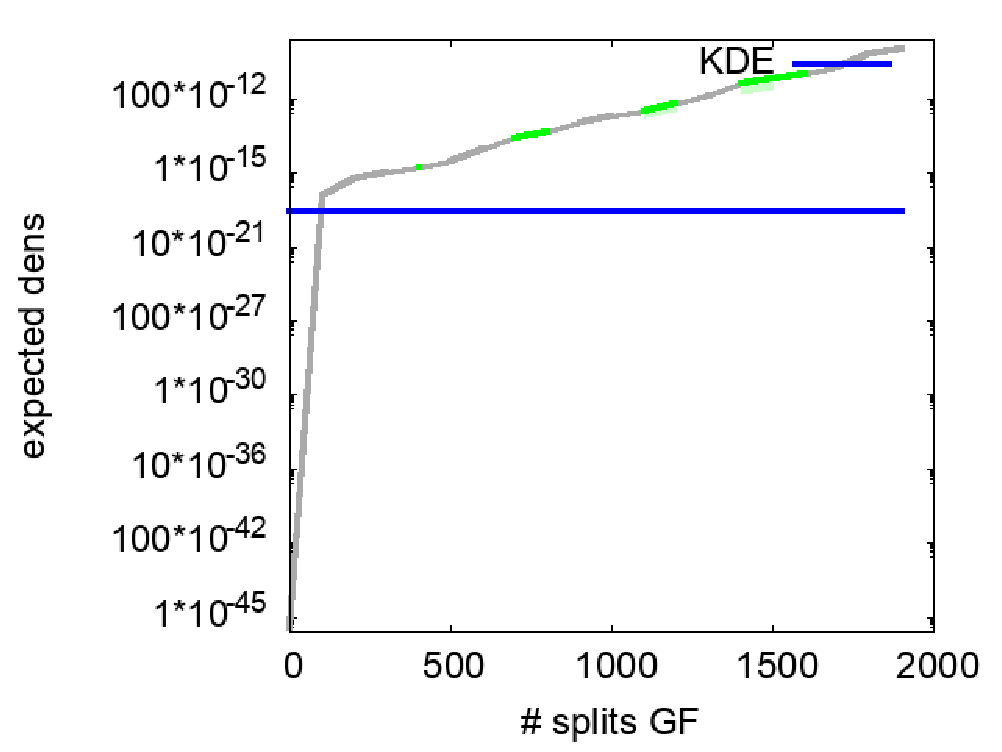} \hspace{\deltaad}  \\
  \hspace{\deltaad} \domainname{student$\_$performance$\_$por} \hspace{\deltaad} & \hspace{\deltaad} \domainname{analcatdata$\_$supreme} \hspace{\deltaad} & \hspace{\deltaad} \domainname{kc1} \hspace{\deltaad}  \\  \Xhline{2pt}
     \hspace{\deltaad}  \includegraphics[trim=0bp 0bp 0bp 0bp,clip,width=\troisbis\columnwidth]{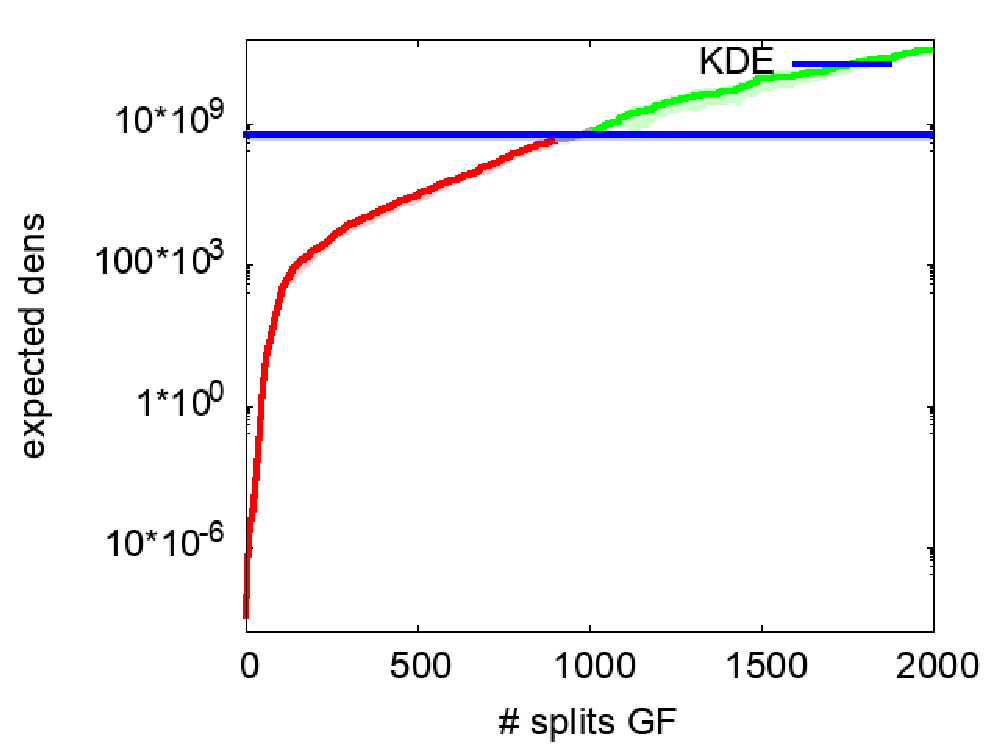}  \hspace{\deltaad} & \hspace{\deltaad} \includegraphics[trim=0bp 0bp 0bp 0bp,clip,width=\troisbis\columnwidth]{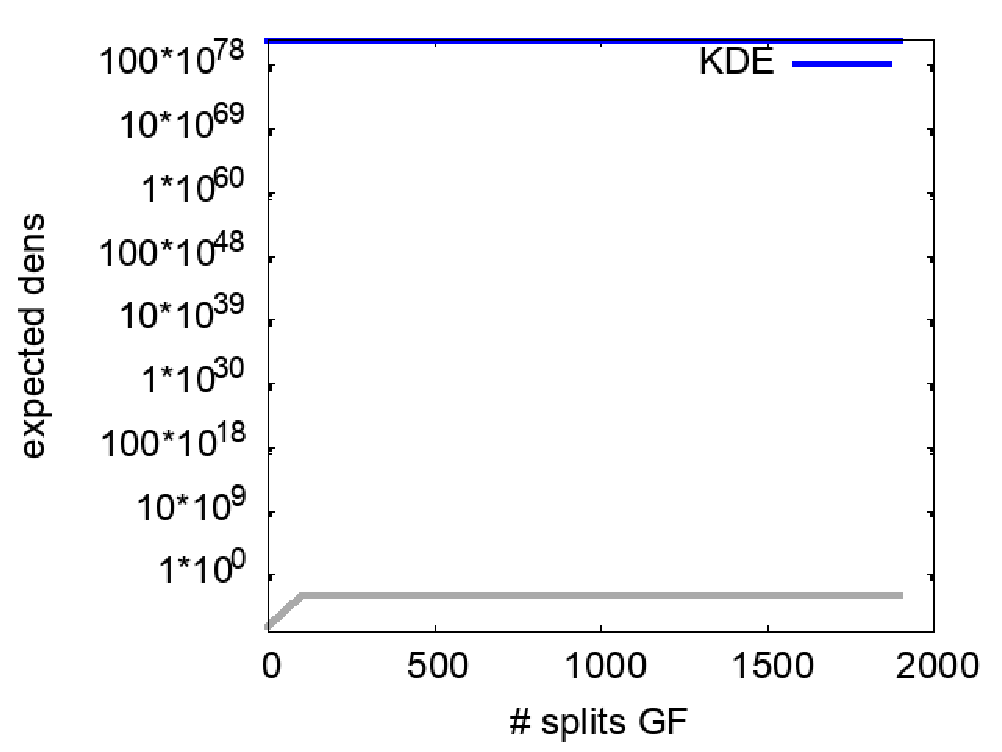} \hspace{\deltaad} & \hspace{\deltaad} \includegraphics[trim=0bp 0bp 0bp 0bp,clip,width=\troisbis\columnwidth]{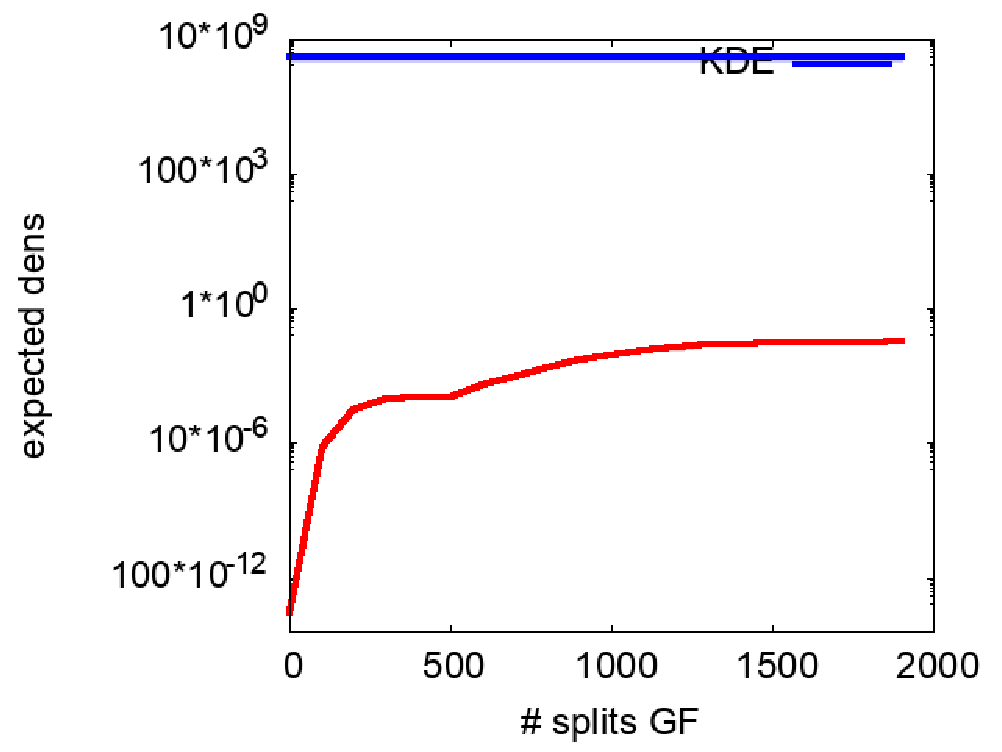} \hspace{\deltaad}  \\
  \hspace{\deltaad} \domainname{winewhite} \hspace{\deltaad} &  \hspace{\deltaad}  \domainname{compas} \hspace{\deltaad} & \hspace{\deltaad} \domainname{artificial$\_$characters} \hspace{\deltaad} \\  \Xhline{2pt}
  \end{tabular}
  \caption{Comparison of results for Kernel Density Estimation and Generative Forests (us, \geot) where the number of trees and number of iterations are ($T=500$ trees, total of $J=2000$ splits in trees), on a second batch of domains. Conventions follow Table \ref{tab:gen-full-kde-1}.}
    \label{tab:gen-full-kde-2}
\end{table}

\end{document}